\documentclass[twoside,11pt,preprint]{article}

\usepackage{jmlr2e}
\usepackage{editing}
\usepackage{url}
\usepackage{xcolor} 
\usepackage{amsmath}
\usepackage{pifont}
\usepackage{mathtools}
\usepackage{bbold}
\usepackage{float}
\usepackage{lipsum}
\usepackage{lineno}
\usepackage{bbm}
\usepackage{cases}
\usepackage{pgf,tikz}
\usepackage{pgfplots}
\usepackage{comment}
\usepackage{lipsum}

\usetikzlibrary{positioning,shadows,arrows,shapes,shapes.arrows,shapes.geometric,arrows.meta,trees,shapes.misc}
\usepackage[normalem]{ulem}
\usepackage{setspace}
\usepackage{multirow}
\usepackage{empheq}
\usepackage{subcaption}
\usepackage{array,ulem}
\usepackage{enumerate}
\usepackage{enumitem}
\usepackage{centernot}
\usepackage{wrapfig}
\usepackage{algorithmicx, algorithm, algpseudocode}
\usepackage{editing}
\usepackage{lipsum}

\tikzset{
	semi/.style={
		semicircle,
		draw,
		minimum size=2em
	}
}

\tikzstyle{FCM}=[
  every node/.style={draw=none, align=center, fill=none, text centered, anchor=center,font=\it},
  every path/.style={double,->},
  cx/.append style={anchor=center,sloped,-,text=red,font=\bf,pos=0.35},
  f1/.style={draw=,fill=gray!15,thick,inner sep=2pt,minimum width=2.5em, align=center, text centered}
]

\tikzstyle{fairmap}=[
	every node/.style={draw=none, align=center, fill=none, text centered, anchor=center,font=\it},
	every label/.style={font=\small, text = red},
	cx/.append style={anchor=center,sloped,rotate=90,text=red,font=\bf,pos=0.65},
	f1/.style={draw=,fill=gray!15,thick,inner sep=3pt,minimum width=2.5em, align=center, text centered},
	imp/.style={double, ->},
	adm/.style={double,->},
]

\tikzstyle{fairmapfull}=[
	every node/.style={draw=none, align=center, fill=none, text centered, anchor=center,font=\it},
	every label/.style={font=\small, text = red},
	cx/.append style={anchor=center,sloped,rotate=90,text=red,font=\bf,pos=0.65},
	f1/.style={draw=,fill=gray!15,thick,inner sep=3pt,minimum width=2.5em, align=center, text centered},
	adm/.style={double,->},
	pow/.style={o-Latex},
	dec/.style={dashed,-Latex},
	sepline/.style={dotted},
]

\tikzstyle{preproc}=[
	every node/.style={draw=none, align=center, fill=none, text centered, anchor=center,font=\it},
	cx/.append style={anchor=center,sloped,rotate=90,text=red,font=\bf,pos=0.65},
	f1/.style={draw=,fill=gray!15,thick,inner sep=3pt,minimum width=6em, minimum height=1.5cm, align=center, text centered},
	f2/.style={draw=,fill=white!15,thick,inner sep=3pt,minimum width=10em, minimum height=2em, align=center, text centered},
	imp/.style={double,->},
]

\definecolor{colone}{RGB}{99,172,229}
\definecolor{coltwo}{RGB}{231, 239,246}
\definecolor{colthree}{RGB}{173,203,227}

\newcolumntype{M}[1]{>{\centering\arraybackslash}m{#1}}
\newcommand{\pa}{\mathrm{pa}} 
\newcommand{\an}{\mathrm{an}} 
\newcommand{\de}{\mathrm{de}} 

\pgfmathdeclarefunction{gauss}{2}{\pgfmathparse{1/(#2*sqrt(2*pi))*exp(-((x-#1)^2)/(2*#2^2))}}

\newcommand{\bidir}{\dashleftarrow\dashrightarrow}
\newcommand{\strm}{\text{Str-}\{\text{DE,IE,SE}\}}
\newcommand{\fpcfa}{\text{FPCFA}(\strm, \text{TV}_{x_0, x_1}(y))}
\newcommand{\contrast}{\mathcal{C}(C_0, C_1, E_0, E_1)}
\newcommand{\powarrow}{\;\;\circ\!\!\longrightarrow\;}
\newcommand{\admarrow}{\implies}
\newcommand{\decomparrow}{\dashrightarrow\;\;}

\def\ci{{\perp\!\!\!\perp}}

\newcommand{\pr}{\mathbbm{P}}
\newcommand{\ex}{\mathbbm{E}}
\DeclareMathOperator*{\argmin}{arg\,min}

\usepackage{lastpage}
\jmlrheading{-}{2022}{1-\pageref{LastPage}}{-; Revised
-}{-}{-}{Ple{\v c}ko and Bareinboim}
\ShortHeadings{Causal Fairness Analysis}{Ple{\v c}ko and Bareinboim}
\firstpageno{1}

\title{Causal Fairness Analysis}

\author{\name Drago Ple{\v c}ko \email drago.plecko@stat.math.ethz.ch \\
       \addr Seminar f{\"u}r Statistik\\
       ETH Z{\"u}rich\\
      Z{\"u}rich, 8092, Switzerland \\
			\name Elias Bareinboim \email eb@cs.columbia.edu \\
			\addr Department of Computer Science \\
			Columbia University \\
			New York, 10027, United States}

\editor{-}

\begin{document}
\maketitle
\begin{abstract}
Decision-making systems based on AI and machine learning have been used throughout a wide range of real-world scenarios, including healthcare, law enforcement, education, and finance. It is no longer far-fetched to envision a future where autonomous systems will be driving entire business decisions and, more broadly, supporting large-scale decision-making infrastructure to solve society's most challenging problems. Issues of unfairness and discrimination are pervasive when decisions are being made by humans, and remain (or are potentially amplified) when decisions are made using  machines with little transparency, accountability, and fairness. 
In this paper, we introduce a framework for \textit{causal fairness analysis} with the intent of filling in this gap, i.e., understanding, modeling, and possibly solving issues of fairness in decision-making settings. The main insight of our approach will be to link the quantification of the disparities present on the observed data with the underlying, and often unobserved, collection of  causal mechanisms that generate the disparity in the first place, challenge we call the Fundamental Problem of Causal Fairness Analysis (FPCFA). In order to solve the FPCFA, we study the problem of decomposing variations and empirical measures of fairness that attribute such variations to structural mechanisms and different units of the population. Our effort culminates in the Fairness Map, which is the first systematic attempt to organize and explain the relationship between different criteria found in the literature. Finally, we study which causal assumptions are minimally needed for performing causal fairness analysis and propose a Fairness Cookbook, which allows data scientists to assess the existence of disparate impact and disparate treatment. 
\end{abstract}


\section{Introduction}

As society transitions to an AI-based economy, an increasing number of decisions that were once made by humans are now delegated to automated systems, and this trend will likely accelerate in the coming years. Automated systems may exhibit discrimination based on gender, race, religion, or other sensitive attributes, and so considerations about fairness in AI are an emergent discussion across the globe. 
Even though it might seem that the issue of unfairness in AI is a recent development, the origins of the problem can be traced back to long before the advent of AI and the prominence these systems have reached in the last years. Among  others, one prominent example is Martin Luther King Jr., who 
spoke of having a dream that his children ``will one day live in a nation where they will not be judged by the color of their skin, but by the content of their character". So little could he have anticipated that machine algorithms would one day use race for making decisions, and that the issues of unfairness in AI would be legislated under Title VII of the Civil Rights Act of 1964 \citep{act1964civil}, which he advocated and fought for \citep{oppenheimer1994kennedy, kotz2005judgment}.

The critical challenge underlying fairness in  AI systems lies in the fact that biases in decision-making exist in the real world from which various datasets are collected. 
Perhaps not surprisingly, a dataset collected from a biased reality will contain aspects of this biases as an imprint. 
In this context, algorithms are tools that may replicate or potentially even amplify the biases that exist in reality in the first place. 
As automated systems are a priori oblivious to ethical considerations, using them blindly could lead to the perpetuation of unfairness in the future.
More pessimistic analysts take this observation as a prelude to doomsday, which, in their opinion, suggests that we should be extremely wary and defensive against any AI. We believe a degree of caution is necessary, of course, but take a more positive perspective, and consider this transition to a more AI-based society as a unique opportunity to improve the current state of affairs. 

While many human decision-makers are hard to change, even when aware of their own biases, AI systems may be less brittle and more flexible. 
Still, one of the requirements to realize AI potential is a new mathematical framework that allows the description and assessment of legal notions of discrimination in a formal way.
Based on this framework, some of the tasks of fair ML will be to detect and quantify undesired discrimination based on society's current ethical standards, and to then design learning methods capable of removing such unfairness from future predictions and decisions. 
 This situation is somewhat unique in the context of AI because a new definition of ``ground truth" is required. The decision-making system cannot rely purely on learning from the data, which is contaminated with unwanted bias. 
It is currently unclear how to formulate the ideal inferential target\footnote{We believe this explains the vast number of fairness criteria described in the literature, which we will detail later on the paper.}, which would help bring about a fair world when deployed. This degree of flexibility in deciding the new ground truth also emphasizes the importance of normative work in this context. 
\footnote{One way of seeing this point a bit more formally goes as follows. We first consider the current version of the world, say $\pi$, and note that it generates a probability distribution $\mathcal{P}$. Training the machine learning algorithm with data from this distribution  $(\mathcal{D} \sim \mathcal{P})$ is replicating patterns from this reality, $\pi$. What we would want is to have an alternative, counterfactual reality $\pi'$, which induces a different distribution $\mathcal{P}'$ without the past biases. The challenge here is that thinking about and defining $\mathcal{P}'$ relies on going beyond $\mathcal{P}$, or the corresponding dataset, which is non-trivial, and yet one of our main goals.}  

In this paper, we build on two legal systems applied to  large bodies of cases throughout the US and the EU that are known as \textit{disparate treatment} and \textit{disparate impact} \citep{barocas2016big}. 
One of our key goals will be to develop a framework for causal fairness analysis grounded in these systems and translate them into exact mathematical language. The disparate treatment doctrine enforces the equality of treatment of different groups, prohibiting the use of the protected attribute (e.g., race) in the decision process. One of the legal formulations for proving disparate treatment is that ``a similarly situated person who is not a member of the protected class would not have suffered the same fate" \citep{barocas2016big}\footnote{This formulation is related to a condition known as \textit{ceteris paribus}, which represents the effect of the protected attribute on the outcome of interest while keeping everything else constant. From a causal perspective, this suggests that the disparate treatment doctrine is concerned with direct discrimination, a connection we draw formally later on in the manuscript.}. 
On the other hand, the disparate impact doctrine focuses on \textit{outcome fairness}, namely, the equality of outcomes among protected groups. Disparate impact discrimination occurs if a facially neutral practice has an adverse impact on members of the protected group. Under this doctrine most commonly fall the cases in which discrimination is unintended or implicit. 
The analysis can become somewhat intricate when variables are correlated with the protected attribute and may act as a proxy, while the law may not necessarily prohibit their usage due to their relevance to the business itself; this is known as ``business necessity'' or ``job-relatedness". Taking business necessity into account is the essence of disparate impact \citep{barocas2016big}. 
Through causal machinery, our framework will allow the data scientist to explain how much of the observed disparity can be attributed to each underlying causal mechanism. This, in turn, allows the data scientist to quantify the disparity explained by mechanisms that do not fall under business necessity and are considered discriminatory,  thereby providing a formal way of assessing disparate impact and accomodate for business necessity requirements.


\subsection*{Current state of affairs \& challenges}
The behavior of AI/ML-based decision-making systems is an emergent property following a complex combination of past (possibly biased) data and interactions with the environment. Predicting or explaining this behavior and its impact on the real-world can be a difficult task, even for the system designer who has the knowledge of how the system is built. 
Ensuring fairness of such decision-making systems, therefore, critically relies on contributions from two groups, namely:
\begin{enumerate}[label=\alph*.]
\item the AI and ML engineers who develop methods to detect bias and ensure adherence of ML systems to fairness measures, and
\item the domain experts, social scientists, economists, policymakers, and legal experts, who study the origins of these biases and can provide the societal interpretations of fairness measures and their expectations in terms of norms and standards. 
\end{enumerate}
Currently, these groups do not share a common starting point. 
It's indeed extremely difficult for them to understand each other and work together towards developing a fair specification of such complex systems, aligned with the many stakeholders involved. In this work, we argue that the language of structural causality can provide this common starting point and facilitate the discussion and exchange of ideas, goals, and expectations between these groups. In some sense, the connection with causal inference might be seen as natural in this context as the legal frameworks of anti-discrimination laws (for example, Title VII in the US) often require that to establish a \textit{prima facie} case of discrimination, the plaintiff must demonstrate ``a strong causal connection" between the alleged discriminatory practice and the observed statistical disparity (Texas Dept. of Housing and Community Affairs v. Inclusive Communities Project, Inc., 576 U.S. 519 (2015)).
Therefore, as discussed in subsequent sections, one of the requirements of our framework will be the ability to represent causal mechanisms underlying a given decision-making setting as well as to distinguish between notions of discrimination that would otherwise be statistically indistinguishable. 

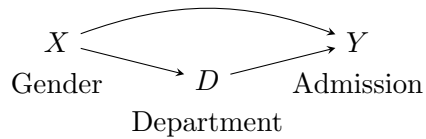
\begin{wrapfigure}{r}{0.30\textwidth}
\centering
\begin{tikzpicture}[>=stealth, rv/.style={thick}, rvc/.style={triangle, draw, thick, minimum size=7mm}, node distance=18mm]
		\pgfsetarrows{latex-latex};
	  \begin{scope}
	  \node[rv, label={below:{Gender}}] (1) at (-2,-0.5) {$X$};
	  \node[rv, label={below:{Department}}] (2) at (0,-1) {$D$};
	  \node[rv, label={below:{Admission}}] (3) at (2,-0.5) {$Y$};
	  \node (4) at (2,-2) {};
	  \draw[->] (1) -- (2);
	  \path[->] (1) edge[bend left = 20] (3);
	  \draw[->] (2) -- (3);
	  \end{scope}
	\end{tikzpicture}
\vspace{-0.40in}
\caption{A partial causal model for the Berkeley Admission example.}\label{fig:firstexample} 
\end{wrapfigure} 

Consider for instance the Berkeley Admission example, in which admission results of students applying to UC Berkeley were collected and analyzed \citep{bickel1975sex}. The analysis showed that male students are 14\% more likely to be admitted than their female counterparts, which raised concerns about the possibility of gender discrimination. The discussion of this example is often less focused on accuracy and appropriateness of the used statistical measures, and more on the plausible justification of disparity based on the mechanism underlying this disparity.
 A visual representation of the dynamics in this settings is shown in Fig.~\ref{fig:firstexample}. 
 In words, each student chooses a department of application. The department choice and student's gender might, in turn, influence the admission decision. In this example, there is a clear need for determining how much of the observed statistical disparity can be attributed to the direct causal path from gender to admission decision vs.~the indirect mechanism\footnote{As discussed later on, even among indirect paths, one may need to distinguish between mediated causal paths and  confounded non-causal paths, or, more generally, among a specific subset of these paths.} going through the department choice variable. Looking directly at gender for determining university admission would certainly be disallowed, whereas using department choice, which may be influenced by gender, might be deemed acceptable. The need to explain an observed statistical disparity, say in this case the 14\% difference in admission rates, through the underlying causal mechanisms -- direct and indirect -- is a recurring theme when assessing discrimination, even though it is sometimes considered only implicitly.

In fact, when AI tools are deployed in the real-world, a similar pattern of questions emerges. Examples include (but are not limited to) the debate over the origins and interpretation of discrimination in criminal justice \citep[COMPAS,][]{ProPublica}, the contribution of data vs. algorithms in the observed bias in face detection~\citep[e.g.,][]{facedetectionarticle, pmlr-v81-buolamwini18a}, and the business necessity vs.~risk of digital redlining in targeted advertising \citep{facebook2019redlining}. 
Intuitively, through these types of questions, society wants to draw a line between what is seen as discriminatory on the one hand, and what is seen as acceptable or justified by economic principles on the other.

\tikzstyle{roadmap}=[
	every node/.style={draw=none, align=center, fill=none, text centered, anchor=center,font=\it},
	every label/.style={circle, draw, fill = yellow},
	cx/.append style={anchor=center,sloped,rotate=90,text=red,font=\bf,pos=0.65},
	f1/.style={draw=,fill=gray!15,thick,inner sep=3pt,minimum width=6em, minimum height=4em, align=center, text centered},
	f2/.style={draw=,fill=white!15,thick,inner sep=3pt,minimum width=2em, minimum height=2em, align=center, text centered},
	f3/.style={draw=,fill=white!15,thick,inner sep=3pt,minimum width=1em, minimum height=2em, align=center, text centered},
	imp/.style={double,->},
]
\newcommand{\inx}{4}
\newcommand{\iny}{2.5}

Considering the above, a practitioner interested in implementing a fair decision-making system based on AI will face two challenges. The first stems from the fact that the current literature is abundant with different fairness measures, some of which are mutually incompatible \citep{corbett2018measure}, and choosing among these measures, even for the system designer, is usually a  non-trivial task. 
This challenge is compounded with the second challenge, which arises from the statistical nature of such fairness measures. As we will show both formally and empirically later on in the text, statistical measures alone cannot distinguish between different causal mechanisms that transmit change and generate disparity in the real world, even if an unlimited amount of data is available.
Despite this apparent shortcoming of purely statistical measures, much of the literature focuses on casting fair prediction as an optimization problem subject to fairness constraints based on such measures \citep{PedreschiKDD08,PedreschiSDM09,LuongKDD11,RuggieriKDD11,HajianICDMws12,Kamiran09,CaldersDMJ10,KamiranICDM10,ZliobaiteICDM11,kamiran2012data, kamiran2012decision,zemel2013learning,KorayKDDws14,romei2014multidisciplinary,dwork2012fairness,Friedler16,chouldechova2017fair,Pleiss17}, to cite a few. In fact, these methods may be insufficient for removing bias and perhaps even lead to unintended consequences and bias amplification, as it will become clear later on. 

The above observations highlight the importance of considering causal aspects when designing fair systems. Obtaining rich enough causal models of unobserved or partially observed reality is not always trivial in practice, yet it is crucial in the context of fair ML. Causal models must be built using inputs from domain experts, social scientists, and policy-makers, and a formal language is needed to express and scrutinize them. In this work, we lay down the foundations of interpreting legal doctrines of discrimination through causal reasoning, which we view as an essential step  towards the development of a new generation of more ethical and transparent AI systems.

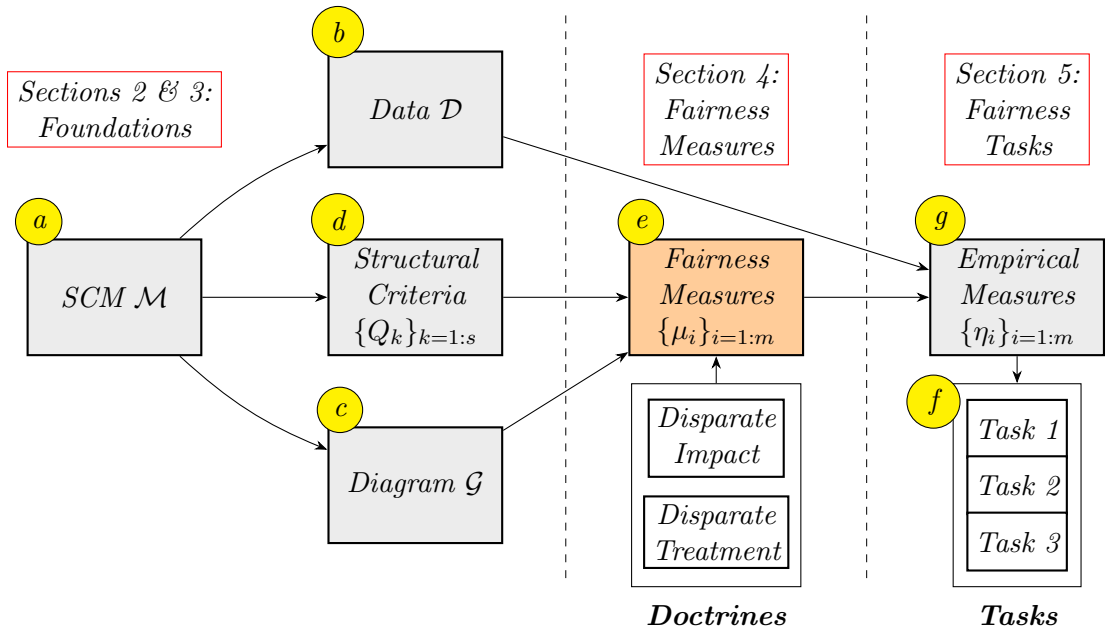
\begin{figure}
    \centering
    \begin{tikzpicture}[roadmap]
		\node [f1, label={above left:{a}}] (scm) at (0*\inx, 0*\iny) {SCM $\mathcal{M}$};
		
		\node [f1, label={above left:{d}}] (smf) at (1*\inx, 0*\iny) {Structural\\ Criteria\\ $\lbrace Q_k \rbrace_{k=1:s}$};
		\node [f1, label={above left:{b}}] (data) at (1*\inx, 1*\iny) {Data $\mathcal{D}$};
		\node [f1, label={above left:{c}}] (graph) at (1*\inx, -1*\iny) {Diagram $\mathcal{G}$};
		

		\draw [very thin, dashed] (2.5*\inx, 1.5*\iny) -- (2.5*\inx, -1.5*\iny);
        
        \node [f1, label={above left:{e}}, fill=orange!40] (measures) at (2*\inx, 0*\iny) {Fairness\\ Measures\\ $\lbrace \mu_i \rbrace_{i=1:m}$};
        
		\node[f2] (DI) at (2*\inx, -0.75*\iny) {Disparate\\ Impact};
		\node[f2] (DT) at (2*\inx, -1.25*\iny) {Disparate\\ Treatment};
		\node (norms) [draw,rectangle,minimum width=2.25cm,minimum height=2.7cm] at (2*\inx, -1*\iny) {};
		\node [below=0.1cm of norms] (normlab) {\textbf{Doctrines}};
		
		\draw [very thin, dashed] (1.5*\inx, 1.5*\iny) -- (1.5*\inx, -1.5*\iny);
        
        \node [f1, label={above left:{g}}] (empirical) at (3*\inx, 0*\iny) {Empirical\\ Measures\\ $\lbrace \eta_i \rbrace_{i=1:m}$};
		\node [f3] (task1) at (3*\inx, -0.7*\iny) {Task 1};
		\node [f3] (task2) at (3*\inx, -1*\iny) {Task 2};
		\node [f3] (task3) at (3*\inx, -1.3*\iny) {Task 3};
		\node (tasks) [draw,rectangle,minimum width=1.7cm,minimum height=2.7cm, label={above left:{f}}] at (3*\inx,-1*\iny) {};
		\node [below=0.1cm of tasks] (tasklab) {\textbf{Tasks}};
		
		\draw [-Stealth] (scm) to [bend left = 10] node[right, rotate=0]{} (data);
		\draw [-Stealth] (scm) to [bend right = 10] node[left, rotate=0]{} (graph);
		\draw [-Stealth] (scm) to [bend left = 0] node[right, rotate=0]{} (smf);
		\draw [-Stealth] (measures) to [bend right = 0] node[right, rotate=30]{} (empirical);
		\draw [-Stealth] (smf) to [bend right = 0] node[above, rotate=30]{} (measures);
		\draw [-Stealth] (data) to [bend left = 0] node[above left, rotate=0]{} (empirical);
		\draw [-Stealth] (graph) to [bend left = 0] (measures);
		\draw [-Stealth] (empirical) to [bend right = 0] (tasks);

		\draw [-Stealth] (norms) to [bend right = 0] (measures);

		\node [draw=red, rotate=0] at (0*\inx, 1*\iny)  {Sections 2 \& 3:\\ Foundations};
		\node [draw=red, rotate=0] at (2*\inx, 1*\iny)  {Section 4:\\Fairness\\ Measures};
  		\node [draw=red, rotate=0] at (3*\inx, 1*\iny)  {Section 5:\\Fairness\\Tasks};

	\end{tikzpicture}
    \caption{A mental map of the Causal Fairness Analysis pipeline.}
    \label{fig:roadmap}
\end{figure}

\subsection{Contributions}
To overcome the challenges described above, we will study fairness analysis through causal lenses and develop a framework for understanding, modeling, and potentially controlling for the biases present in the data. 
Fig.~\ref{fig:roadmap} contains the key elements involved in Causal Fairness Analysis as well as a roadmap of how this paper is organized. Specifically, in Sec.~\ref{priorart}, we cover the basic notions of causal inference, including structural causal models, causal diagrams, and data collection. In Sec.~\ref{foundations}, we introduce the essential elements of our theoretical framework. In particular, we define the notions of structural fairness that will serve as a baseline, ground truth for determining the presence or absence of discrimination under the disparate impact and disparate treatment doctrines.  In Sec.~\ref{Measures}, we introduce causal measures of fairness that can be computed from data in practice. We further draw the connection between such measures and the aforementioned legal doctrines. In Sec.~\ref{Tasks}, we introduce the tasks of Causal Fairness Analysis -- bias detection and quantification, fair prediction, and fair decision-making -- and show how they can be solved by building on the tools developed earlier. 
More specifically, our contributions are as follows: \begin{enumerate}
	\item We develop a general and coherent framework of Causal Fairness Analysis (Fig.~\ref{fig:roadmap}). This framework provides a common language to connect computer scientists and statisticians on the one hand, and legal and ethical experts on the other to tackle challenges of fairness in automated decision-making. 
	Further, this new framework grounds the legal doctrines of disparate impact and disparate treatment through the semantics of structural causal models.
	\item We formulate the Fundamental Problem of Causal Fairness Analysis (FPCFA), which outlines some critical properties that empirical measures of fairness should exhibit. In particular, we discuss which properties allow us to relate fairness measures with the specific causal mechanisms that generate the disparity observed in the data, thereby providing empirical basis for reasoning about structural causality. 
	\item We formalize the problem of decomposing variations between a protected attribute $X$ and an outcome variable $Y$. In particular, we show how the total variation (TV) can be decomposed based on different causal mechanisms and across different groups of units. These developments lead to the construction of the \textit{explainability plane} (Fig.~\ref{fig:Erefined}).  
	\item We introduce the TV family of measures (Table~\ref{table:tv-family}) and construct the first version of the \textit{Fairness Map} (Thm.~\ref{thm:map} and Fig.~\ref{fig:map}). The Map brings well-known fairness measures under the same theoretical umbrella and uncovers the structure that connects them. 
	\item We elicit the assumptions under which different causal fairness criteria can be evaluated. Specifically, we introduce the \textit{Standard Fairness Model} (SFM), which is a generic and simplified way of encoding causal assumptions and constructing the causal diagram. One desirable feature of the SFM is that it strikes a balance between simplicity of construction and informativeness for causal analysis (Def.~\ref{def:sfm} and Thm.~\ref{thm:sfm}).
	\item We develop the Fairness Cookbook that represents a practical solution that allows data scientists to assess the presence of disparate treatment and disparate impact. Furthermore, we provide an R-package for performing this task called \texttt{faircause}. 
	\item We study the implications of Causal Fairness Analysis on the fair prediction problem. In particular, we prove the Fair Prediction Theorem (Thm.~\ref{thm:fpt}) which shows that making TV being equal to zero during the training stage is almost never sufficient to ensure that causal measures of fairness are well-behaved.
\end{enumerate}

Readers familiar with causal inference may want to move straight to Sec.~\ref{foundations}, even though the next section's examples are used to motivate the problem of fairness.

\section{Foundations of Causal Inference} \label{priorart}
In this section, we introduce three fundamental building blocks that will allow us to formalize the challenges of fairness described above through a causal lense. 
First, we will define in Sec.~\ref{sec:scm} a general class of data-generating models known as \textit{structural causal models} (shown in Fig.~\ref{fig:roadmap}a). The key observation here is that the collection of mechanisms underpinning any decision-making scenario are causal, and therefore should be modeled through proper and formal causal semantics. 
Second, we will discuss in Sec.~\ref{sec:data} qualitatively different probability distributions that are induced by the causal generative process, and which will lead to the observed data and counterfactuals (Fig.~\ref{fig:roadmap}b). 
Third, we will introduce  in Sec.~\ref{sec:diagram} an object known as a \textit{causal diagram} (Fig.~\ref{fig:roadmap}c), which will allow the data scientist to articulate non-parametric assumptions over the space of generative models. These assumptions can be shown as necessary for the analysis, in a broader sense. Finally, we will define the \textit{standard fairness model} (SFM), which is a special class of diagrams that act as a template, allowing one to generically express entire classes of structural models. The SFM class, in particular, requires fewer modelling assumptions than the more commonly used causal diagrams.


\subsection{Structural Causal Models} \label{sec:scm}
The basic semantical framework of our analysis rests on the notion of structural causal model (SCM, for short), which is one of the most flexible class of generative models known to date \citep{pearl:2k}. The section will follow the presentation in \citep{bareinboim2020on}, which contains more detailed discussions and  proofs. First, we introduce and exemplify SCMs through the following definition:

\begin{definition}[Structural Causal Model (SCM) \citep{pearl:2k}] \label{def:SCM}
	A structural causal model (SCM) is a 4-tuple $\langle V, U, \mathcal{F}, P(u)\rangle$, where
  \begin{enumerate}
    \item $U$ is a set of exogeneous variables, also called background variables, that are determined by factors outside the model;
    \item $V = \lbrace V_1, ..., V_n \rbrace$ is a set of endogeneous (observed) variables, that are determined by variables in the model (i.e. by the variables in $U \cup V$);
    \item $\mathcal{F} = \lbrace f_1, ..., f_n \rbrace$ is the set of structural functions determining $V$, $v_i \gets f_i(\pa(v_i), u_i)$, where $\pa(V_i) \subseteq V \setminus V_i$ and $U_i \subseteq U$ are the functional arguments of $f_i$;
    \item $P(u)$ is a distribution over the exogeneous variables $U$.
  \end{enumerate}
\end{definition}
In words, each structural causal model can be seen as partitioning the variables involved in the
phenomenon into sets of exogenous (unobserved) and endogenous (observed) variables,
respectively, $U$ and $V$. The exogenous variables are determined ``outside'' of the model and their
associated probability distribution, $P(U)$, represents a summary of the world
external to the phenomenon that is under investigation. In our setting, these variables will represent the units
involved in the phenomenon, which correspond to elements of the population under study,
for instance, patients, students, customers. Naturally, their randomness (encoded in $P(U)$)
induces variations in the endogenous set $V$.

Inside the model, the value of each endogenous variable $V_i$ is determined by a causal process, $V_i \gets f_i(\pa(v_i), u_i)$, that maps the exogenous factors $U_i$ and a set of endogenous variables
$Pa_i$ (so called parents) to $V_i$.  These causal processes – or mechanisms – are assumed to be invariant unless explicitly intervened on (as defined later in the section). Together with the background factors, they represent the data-generating process according to which the values of the endogenous variables are determined. For concreteness and grounding of the definition, we revisit the Berkeley admission example through the lens of SCMs.

\begin{example}[Berkeley Admission \citep{bickel1975sex}]
    \label{ex:berkeley}
	During the application process for admissions to UC Berkeley, potential students choose a department to which they apply, which is labelled as $D$ (binary with $D = 0$ for arts \& humanities, $D=1$ for sciences). The admission decision is labelled as $Y$ ($y_1$ accepted, $y_0$ rejected) and the student's gender is labelled as $X$ ($x_0$ female, $x_1$ male)\footnote{In the manuscript, gender is discussed as a binary variable, which is a simplification of reality, used to keep the presentation of the concepts simple. In general, one might be interested in analyses of gender discrimination with gender taking non-binary values.}. 
	
	The SCM $\mathcal{M}$ is the 4-tuple $\langle V=\lbrace X, D, Y\rbrace, U=\lbrace U_X, U_D, U_Y\rbrace, \mathcal{F}, P(U) \rangle$, where $U_X, U_Y, U_D$ represent the exogenous variables, outside of the model, that affect $X, Y, D$, respectively. Also, the causal mechanisms $\mathcal{F}$ are given as follows \footnote{The given SCM can also be written as
	\begin{align} \label{eq:berkeley-samp-1}
		X &\gets \text{Bernoulli}(0.5) \\ 
		\label{eq:berkeley-samp-2}
		D &\gets \text{Bernoulli}(0.5 + \lambda X) \\
		\label{eq:berkeley-samp-3}
		Y &\gets \text{Bernoulli}(0.1 + \alpha X + \beta D).
	\end{align} }: 
	\begin{align} \label{eq:berkeley-cts-1}
		X &\gets \mathbb{1}(U_X < 0.5) \\ 
		\label{eq:berkeley-cts-2}
		D &\gets \mathbb{1}(U_D < 0.5 + \lambda X)  \\
		\label{eq:berkeley-cts-3}
		Y &\gets \mathbb{1}(U_Y < 0.1 + \alpha X + \beta D),
 	\end{align} and $P(U_X, U_D, U_Y)$ is such that $U_X, U_D, U_Y$ are independent $\text{Unif}(0, 1)$ random variables.
	
	In words, the population is partitioned into males and females, with equal probability (the exogenous  $U_X$ represents the population's biological randomness). Each applicant chooses a department $D$, and this decision depends on $U_D$ and gender $X$. The exogenous variable $U_D$ represents the individual's natural inclination towards studying science. Whenever $\lambda > 0$ in Eq.~\ref{eq:berkeley-cts-2}, the threshold for applying to a science department is higher for female individuals, which is a result of various societal pressures. Finally, the admission decision $Y$ possibly depends on gender (if $\alpha \neq 0$ in Eq.~\ref{eq:berkeley-cts-3}) and/or department of choice (if $\beta \neq 0$ in Eq.~\ref{eq:berkeley-cts-3}). The exogenous variable $U_Y$ in this case represents the impression the applicant left during an admission interview. Notice that female students and arts \& humanities students may need to leave a better interview impression in order to be admitted (depending on Eq.~\ref{eq:berkeley-cts-3}). $\hfill \square$
\end{example}
Another important notion for our discussion is that of a submodel, which is defined next:
\begin{definition}[Submodel \citep{pearl:2k}] \label{def:submodel}
    Let $\mathcal{M}$ be a structural causal model, $X$ a set of variables in $V$, and $x$ a particular value of $X$. A submodel $\mathcal{M}_{x}$ (of $\mathcal{M}$) is a 4-tuple:
    \begin{equation}
        \mathcal{M}_{x} = \langle V, U, \mathcal{F}_{x}, P(u)\rangle
    \end{equation}
    where 
    \begin{equation}
        \mathcal{F}_{x} = \lbrace f_i : V_i \notin X \rbrace \cup \lbrace X \gets x\rbrace,
    \end{equation}
    and all other components are preserved from $\mathcal{M}$. 
\end{definition}
In words, the SCM $\mathcal{M}_{x}$ is obtained from $\mathcal{M}$ by replacing all equations in $\mathcal{F}$ related to variables $X$ by equations that set $X$ to a specific value $x$. In the context of Causal Fairness Analysis, we might be interested in submodels in which the protected attribute $X$ is set to a fixed value $x$. Building on submodels, we introduce next the notion of potential response:
\begin{definition}[Potential Response \citep{pearl:2k}]\label{def:potentialresponse}
    Let $X$ and $Y$ be two sets of variables in $Y$ and $u \in \mathcal{U}$ be a unit. The potential response $Y_x(u)$ is defined as the solution for $Y$ of the set of equations $\mathcal{F}_x$ with respect to SCM $\mathcal{M}$. That is, $Y_x(u)$ denotes the solution of $Y$ in the submodel $\mathcal{M}_x$ of $\mathcal{M}$.
\end{definition}
In words, $Y_x(u)$ is the value variable $Y$ would take if (possibly contrary to observed facts) $X$ is set to $x$, for a specific unit $u$. In the Admission example, $Y_x(u)$ would denote the admission outcome for the specific unit $u$, had their gender $X$ been set to value $x$ by intervention (e.g., possibly contrary to their actual gender).  Potential responses are also called potential outcomes in the literature. 

\subsection{Observational \& Counterfactual Distributions}\label{sec:data}

Each SCM $\mathcal{M}$ induces different types of probability distributions, which represent different data collection modes and will play a key role in fairness analysis. 
We start with the observational distribution that represents a state of the underlying decision-making system in which the fairness analysts just collect data, without interfering in the decision-making process, as defined next. 
\begin{definition}[Observational Distribution \citep{bareinboim2020on}] \label{def:obs-dist}
An SCM $\mathcal{M} = \langle V, U, \mathcal{F}, P(u) \rangle$ induces a joint probability distribution $P(V)$ such that for each $Y \subseteq V$,
\begin{align}
    P^{\mathcal{M}}(y) = \sum_{u} \mathbb{1}\Big(Y(u) = y \Big) P(u),
\end{align} 
where $Y(u)$ is the solution for $Y$ after evaluating $\mathcal{F}$ with $U = u$.
\end{definition}
In words, the procedure can be described as follows:
\begin{enumerate}
    \item for each unit $U = u$, the structural functions $\mathcal{F}$ are evaluated following a valid topological order, and
    \item the probability mass P(U = u) is accumulated for each instantiation $U = u$ consistent with the event $Y = y$.
 \end{enumerate}
 
Throughout this manuscript, all the sums should be replaced by the corresponding integrals whenever suitable. 
To ground the discussion about this definition, we continue with the example above and see how the corresponding observational distribution is induced. 
 
\begin{example}[College Admission's Observational Distribution] 
Consider the SCM $\mathcal{M}$ in Eq.~\ref{eq:berkeley-cts-1}-\ref{eq:berkeley-cts-3}. The total variation (TV for short; also called demographic parity) generated by $\mathcal{M}$ depends on the structural mechanisms $\mathcal{F}$ and the distribution of exogenous variables $P(U_X, U_D, U_Y)$. The total variation can be written as:
\begin{align} \label{eq:btvbayes}
    P(y \mid x_1) - P(y \mid x_0) = \frac{P(y, x_1)}{P(x_1)} - \frac{P(y, x_0)}{P(x_0)}. 
\end{align}
Therefore, we compute the terms $P(y, x_1), P(x_1), P(y, x_0), P(x_0)$ based on the true, underlying SCM. Using Def.~\ref{def:obs-dist} and Eq.~\ref{eq:berkeley-cts-1}, we can see that:
\begin{align}
    P(x_1) = P(U_X < 0.5) = \frac{1}{2} = P(U_X > 0.5) = P(x_0).
\end{align}
Using the fact that $U_X$, $U_D$, and $U_Y$ are independent in the SCM, $P(y, x_1)$ can be computed in the following way (Def.~\ref{def:obs-dist}):
\begin{align}
    P(y, x_1) &= \sum_u \mathbb{1}(Y(u) = 1, X(u) = 1) P(u)  \\
              &= P(U_X < 0.5)\big[P(U_D > 0.5 + \lambda)P(U_Y < 0.1 + \alpha) + \\
              &\qquad\qquad\qquad\;\;\;\; P(U_D < 0.5 + \lambda)P(U_Y < 0.1 + \alpha + \beta)\big] \nonumber\\
              &= \frac{1}{2}[(\frac{1}{2}-\lambda)(0.1+\alpha) + (\frac{1}{2}+\lambda)(1+\alpha + \beta)] = \frac{1}{2}(0.1 + \alpha + (\frac{1}{2}+\lambda)\beta).
\end{align}
The computation above can be described as follows. Firstly, $X(u) = 1$ is equivalent with $U_X < 0.5$ (Eq.~\ref{eq:berkeley-cts-1}). Secondly, when $X(u) = 1$, there are two possibilities for the variable $D$ based on $U_D$ (see Eq.~\ref{eq:berkeley-cts-2}). Whenever $U_D > 0.5 + \lambda$, then $D(u) = 0$, and to have $Y(u) = 1$, we need $U_Y < 0.1 + \alpha$ (see Eq.~\ref{eq:berkeley-cts-3}). If $U_D < 0.5 + \lambda$, then $D(u) = 1$, and to have $Y(u) = 1$, we need $U_Y < 0.1 + \alpha + \beta$ (see Eq.~\ref{eq:berkeley-cts-3}). An analogous computation yields that:
\begin{align}
    P(y, x_0) &= \sum_u \mathbb{1}(Y(u) = 1, X(u) = 0) P(u)  \\
              &= \frac{1}{2}\big[\frac{1}{2}*0.1 + \frac{1}{2}*(0.1 + \beta)\big] = \frac{1}{2}(0.1 + \frac{\beta}{2}).
\end{align}
Putting the results together in Eq.~\ref{eq:btvbayes}, the TV equals
\begin{align}
    P(y\mid x_1) - P(y\mid x_0) &= \frac{\frac{1}{2}(0.1 + \alpha + (\frac{1}{2}+\lambda)\beta)}{\frac{1}{2}} - \frac{\frac{1}{2}(0.1 + \frac{\beta}{2})}{\frac{1}{2}}\\ &= \alpha + \lambda\beta.
\end{align}
In fact, after analyzing the admission dataset from UC Berkeley, a data scientist computes the observed disparity to be\footnote{The number below was actually evaluated from the actual real dataset, which is compatible with structural coefficients $\alpha = 0, \beta = \frac{7}{10}$, and $\lambda = \frac{2}{10}$.}
\begin{equation}
	   P(y \mid x_1) - P(y \mid x_0) = 14\%.
\end{equation}
In words, male candidates are 14\% more likely to be admitted than female candidates. The data scientist (who does not have access to the SCM $\mathcal{M}$ described above) might wonder if this disparity (14\%) means that female applicants are discriminated against. Also, she/he might wonder how the observed disparity relates to the SCM $\mathcal{M}$ given in Eq.~\ref{eq:berkeley-cts-1}-\ref{eq:berkeley-cts-3}. Our goal in this manuscript is to address these questions from first principles. $\hfill \square$ 
\end{example}
Next, we define another important family of distributions over possible counterfactual outcomes, which will be used throughout this manuscript:
\begin{definition}[Counterfactual Distributions \citep{bareinboim2020on}] \label{def:ctf-dist}
    An SCM $\mathcal{M} = \langle V, U, \mathcal{F}, P(u) \rangle$ induces a family of joint distributions over counterfactual events $Y_{x}, \dots, Z_{w}$ for any $Y, Z, \dots, X, W \subseteq V$:
    \begin{align} \label{eq:L3def}
        P^{\mathcal{M}}(y_x, \dots, z_w) = \sum_{u} \mathbb{1}\Big(Y_x(u) = y, \dots, Z_w(u) = z\Big) P(u).
    \end{align}
\end{definition}
The LHS in Eq.~\ref{eq:L3def} contains variables with different subscripts, which syntactically represent different potential responses (Def.~\ref{def:potentialresponse}), or counterfactual worlds. In words, the equation can be interpreted as follows:
\begin{enumerate}
    \item For each set of subscripts and variables ($X, \dots, W$ and $Y, \dots, Z$), replace the corresponding mechanism with appropriate constants to generate $\mathcal{F}_x, \dots, \mathcal{F}_w$ and create submodels $\mathcal{M}_{x}, \dots, \mathcal{M}_w$,
    \item For each unit $U = u$, evaluate the modified mechanisms $\mathcal{F}_x, ..., \mathcal{F}_w$ to obtain the potential response of the observables,
    \item The probability mass $P(U = u)$ is accumulated for each instance $U = u$ that is consistent with the events over the counterfactual variables, that is $Y_x = y, \dots, Z_w = z$, that is, $Y = y$ in $\mathcal{M}_x$, \dots, $Z = z$ in $\mathcal{M}_w$.
\end{enumerate}
\begin{example}[College Admission Counterfactual Distribution] 
Consider the SCM in Eq.~\ref{eq:berkeley-cts-1}-\ref{eq:berkeley-cts-3} and the following joint counterfactual distribution:
\begin{align} \label{eq:pnsex}
    P(y_{x_1}, y_{x_0}).
\end{align}
In the submodel $\mathcal{M}_{x_0}$ (where $X = 0$ is set by intervention), we have that $D_{x_0}(u) = 1$ is equivalent with $U_D < 0.5$. When $D_{x_0}(u) = 1$, $Y_{x_0}(u) = 1$ if and only if $U_Y < 0.1 + \beta$. Similarly, when $D_{x_0}(u) = 0$, $Y_{x_0}(u) = 1$ if and only if $U_Y < 0.1$. Therefore, we have that
\begin{align}
    Y_{x_0}(u) = 1 \iff ((U_D < 0.5) \wedge (U_Y < 0.1 + \beta)) \vee ((U_D > 0.5) \wedge (U_Y < 0.1)).
\end{align}
In the submodel $\mathcal{M}_{x_1}$, we have
\begin{align}
    Y_{x_1}(u) = 1 \iff &((U_D < 0.5+\lambda) \wedge (U_Y < 0.1 + \alpha + \beta)) \vee \\
    &((U_D > 0.5+\lambda) \wedge (U_Y < 0.1 + \alpha)) \nonumber.
\end{align}
Based on this, the expression in Eq.~\ref{eq:pnsex} can be evaluated using Def.~\ref{def:ctf-dist}, which leads to
\begin{align}
    P(y_{x_1}, y_{x_0}) =& \sum_{u}  \mathbb{1}(Y_{x_1}(u) = 1, Y_{x_0}(u) = 1) P(u)\\
                        =&P(U_D < 0.5)P(U_Y < 0.1  + \beta) + P(U_D > 0.5)P(U_Y < 0.1) \nonumber\\
                        =& 0.1 + \frac{\beta}{2}.
\end{align}
Interestingly, this distribution is never obtainable from observational data, since it involves both potential responses $Y_{x_0}, Y_{x_1}$, which can never be observed simultaneously.
$\hfill \square$ 
\end{example}
In most fairness analysis settings,  the data scientist will only have data $\mathcal{D}$ in the form of samples collected from the observational distribution. One significant result in this context is known as the \textit{causal hierarchy theorem}  (CHT, for short), which 
says that it is almost never possible (in an information theoretic sense) to recover the counterfactual distribution from the observational distribution alone \citep[Thm.~1]{bareinboim2020on}. Given this impossibility result and the unavailability of the SCM in most settings, the data scientist needs to resort to some sort of assumptions in order to possibly make claims about these underlying mechanisms, which is discussed in the next section.  


\subsection{Encoding Structural assumptions through Causal Diagrams}\label{sec:diagram}
Despite the fact that SCMs are well defined and provide the semantics to different families of probability distributions, and are important for fairness analysis, one critical observation is that they are usually not observable by the data scientist. A common way of encoding assumptions about the underlying SCM is through an object called a causal diagram. We describe below the constructive procedure that allows one to articulate a diagram from a coarse understanding of the SCM. 
\begin{definition}[Causal Diagram \citep{pearl:2k, bareinboim2020on}] \label{def:diagram}
	Let $\mathcal{M} = \langle V, U, \mathcal{F}, P(u)\rangle$ be an SCM. A graph $\mathcal{G}$ is said to be a \textit{causal diagram} (of $\mathcal{M}$) if:
	\begin{enumerate}
		\item there is a vertex for every endogenous variable $V_i \in V$,
		\item there is an edge $V_i \to V_j$ if $V_i$ appears as an argument of $f_j \in \mathcal{F}$,
		\item there is a bidirected edge $V_i\dashleftarrow\dasharrow V_j$ if the corresponding $U_i, U_j \subset U$ are correlated or the corresponding functions $f_i, f_j$ share some $U_{ij} \in U$ as an argument.
	\end{enumerate}
\end{definition}
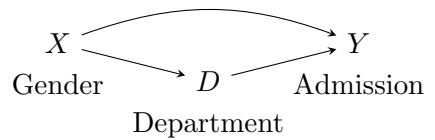
\begin{wrapfigure}{r}{0.30\textwidth}
\centering
\begin{tikzpicture}[>=stealth, rv/.style={thick}, rvc/.style={triangle, draw, thick, minimum size=7mm}, node distance=18mm]
		\pgfsetarrows{latex-latex};
	  \begin{scope}
	  \node[rv, label={below:{Gender}}] (1) at (-2,-0.5) {$X$};
	  \node[rv, label={below:{Department}}] (2) at (0,-1) {$D$};
	  \node[rv, label={below:{Admission}}] (3) at (2,-0.5) {$Y$};
	  \node (4) at (2,-2) {};
	  \draw[->] (1) -- (2);
	  \path[->] (1) edge[bend left = 20] (3);
	  \draw[->] (2) -- (3);
	  \end{scope}
	\end{tikzpicture}
\vspace{-0.40in}
\caption{A partial causal model for the Admissions' example.}\label{fig:firstexample2} 
\end{wrapfigure} 
In words, there is an edge from endogenous variables $V_i$ to $V_j$ whenever $V_j$ ``listens to" $V_i$ for determining its value\footnote{This construction lies at the heart of the type of knowledge causal models represent, as suggested in \citep[pp.~129]{pearl:why19}: ``This listening metaphor encapsulates the entire knowledge that a causal network conveys; the rest can be derived, sometimes by leveraging data.''}. Similarly, the existence of a bidirected edge between $V_i$ and $V_j$ indicates there is some shared, unobserved information affecting how both $V_i$ and $V_j$ obtain their values. Note that while the SCM contains explicit information about all structural mechanisms ($\mathcal{F}$) and exogenous variables ($P(u)$), the causal diagram, on the other hand, encodes information only about which functional arguments were possibly used as inputs to the functions in $\mathcal{F}$. That is, the diagram abstracts out the specifics of the functions $\mathcal{F}$ and retains information about their possible arguments. 

Furthermore, the existence of a directed arrow, e.g., $V_i \rightarrow V_j$, encodes the \textit{possibility} of the mechanism of $V_j$ to listen to variable $V_i$, but not the necessity. In words, the edges are in this sense non-committal; for instance, $f_j$ may decide not to take the value of $V_i$ into account. On the other hand, the assumptions are not really encoded in the arrows present in the diagram, but in the missing arrows; each missing arrow ascertains that one variable is \textit{certainly} not the argument of the other. The data scientist, in general, should try to specify as much knowledge as possible of this type. For concreteness, consider the following example.

\begin{example}[Admission's Causal Diagram] 
Consider again the SCM $\mathcal{M}$ in Ex.~\ref{ex:berkeley}, which is unknown by the data scientist trying to analyze the existence of discrimination in the admission process. To apply the graphical construction dictated by  Def.~\ref{def:diagram}, the data scientist starts the modeling process by examining each of the endogenous variables and the potential arguments of their corresponding mechanisms. For example, the mechanism
\begin{equation}
D \gets f_D(X, U_D)
\end{equation}
suggests that each applicant department's choice ($D$) is, possibly, a function of their gender $X$, 
regardless of the specific form about how this happens in reality. 
If that is the case, so the causal diagram $\mathcal{G}$ will contain the arrow $X \rightarrow D$. 
Again, an arrow in $\mathcal{G}$ does not commit to how the variables $X$ and $D$ interact, which is significantly less informative than the true mechanism given by Eq.~\ref{eq:berkeley-cts-2}. 
Continuing the causal modelling process, the data scientists may think about the admission's process, and consider that 
\begin{equation}
Y \gets f_Y(X, D, U_Y),
\end{equation}
which represents that how admission decisions come about may be influenced by gender and department choice.
If that is the case, the causal diagram $\mathcal{G}$ will also contain the arrows $X \rightarrow Y$ and $D \rightarrow Y$, respectively. Again, this stands in sharp contrast with how detailed the knowledge is presented in the true SCM $\mathcal{M}$, and, for instance, as delineated in Eq.~\ref{eq:berkeley-cts-3}. Interestingly enough, an entirely different functional form than that in Eq.~\ref{eq:berkeley-cts-3}, say 
\begin{equation}
Y \gets \mathbb{1}\big(U_Y < 0.1 + \beta XD\big),
\end{equation}
is also compatible with the causal diagram in Fig.~\ref{fig:firstexample}.

Lastly,  if the coefficient $\alpha$ is equal to $0$ 
in the mechanism described by Eq.~\ref{eq:berkeley-cts-3} (i.e., $Y \gets \mathbb{1}(U_Y < 0.1 + \alpha X + \beta D)$), this would still be compatible with the causal diagram $\mathcal{G}$. Again, the arrow allows for the possibility of functional dependence, but does not necessitate it. 
$\hfill \square$
\end{example}

\subsubsection{Standard Fairness Model} \label{sec:SFM}
Specifying the relationship among all pairs of variables, as required by the definition of a causal diagram, is possibly non-trivial in many practical settings. In this section, we will introduce the \textit{Standard Fairness Model}, which is a template-like model that represents a collection of causal diagrams and aims to alleviate the modeling requirements. 
\begin{definition}[Standard Fairness Model (SFM)] \label{def:sfm}
	The standard fairness model (SFM) is the causal diagram $\mathcal{G}_{\text{SFM}}$ over endogenous variables $\{X, Z, W, Y\}$  and given by
	\begin{center}
		\begin{tikzpicture}
	 [>=stealth, rv/.style={thick}, rvc/.style={triangle, draw, thick, minimum size=7mm}, node distance=18mm]
	 \pgfsetarrows{latex-latex};
	 \begin{scope}
		\node[rv] (0) at (0,1) {$Z$};
	 	\node[rv] (1) at (-1.5,0) {$X$};
	 	\node[rv] (2) at (0,-1) {$W$};
	 	\node[rv] (3) at (1.5,0) {$Y$};
	 	\draw[->] (1) -- (2);
		\draw[->] (0) -- (3);
	 	\path[->] (1) edge[bend left = 0] (3);
		\path[<->] (1) edge[bend left = 30, dashed] (0);
	 	\draw[->] (2) -- (3);
		\draw[->] (0) -- (2);
	 \end{scope}
	 \end{tikzpicture}
	\vspace{-0.2in}
	\end{center}
where the nodes represent:
\begin{itemize}
	\item the \textit{protected attribute}, labelled $X$ (e.g., gender, race, religion),
	\item the set of \textit{confounding} variables $Z$, which are not causally influenced by the attribute $X$ (e.g., demographic information, zip code),
	\item the set of \textit{mediator} variables $W$ that are possibly causally influenced by the attribute (e.g., educational level, or other job related information),
	\item the \textit{outcome} variable $Y$ (e.g., admissions, hiring, salary).
\end{itemize}
Nodes $Z$ and $W$ are possibly multi-dimensional or empty. Furthermore, for a causal diagram $\mathcal{G}$, the projection of $\mathcal{G}$ onto the SFM is defined as the mapping of the endogenous variables $V$ appearing in $\mathcal{G}$ into four groups $X, Z, W, Y$, as described above. The projection is denoted by $\Pi_{\text{SFM}}(\mathcal{G})$ and is constructed by choosing the protected attribute, the outcome of interest, and grouping the confounders $Z$ and mediators $W$.
\end{definition}
For simplicity, we assume $X$ to be binary (whereas $Z, W$, and $Y$ could be either discrete or continuous). 
For instance, by setting $Z = \emptyset$ and $W = \lbrace D \rbrace$, the causal diagram of the Admissions example can be represented by $\mathcal{G}_{\text{SFM}}$. To ground the definition further, consider the following well-known example. 
\begin{example}[COMPAS \citep{larson2016how}] \label{ex:compas}
The courts at Broward County, Florida, use machine learning to predict whether individuals released on parole are at high risk of re-offending within 2 years ($Y$). The algorithm is based on the demographic information $Z$ ($Z_1$ for gender, $Z_2$ for age), race $X$ ($x_0$ denoting White, $x_1$ Non-White), juvenile offense counts $J$, prior offense count $P$, and degree of charge $D$. The causal diagram for this setting is shown in Fig.~\ref{fig:compasdag}. The bidirected arrows between $X$ and $Z_1, Z_2$ indicate that the exogenous variable $U_X$ possibly shares information with exogenous variables $U_{Z_1}, U_{Z_2}$.
    \begin{figure}
        \centering
        \begin{subfigure}[b]{0.45\textwidth}
        \begin{tikzpicture}
	 	[>=stealth, rv/.style={thick}, rvc/.style={triangle, draw, thick, minimum size=8mm}, node distance=7mm]
	 	\pgfsetarrows{latex-latex};
	 	\begin{scope}
	 	\node[rv] (c) at (2.25,1.5) {${Z_1}$};
	 	\node[rv] (z2) at (3.75,1.5) {${Z_2}$};
	 	\node[rv] (a) at (0,0) {$X$};
	 	\node[rv] (m) at (1.5,-1.5) {$J$};
	 	\node[rv] (l) at (3,-1.5) {$P$};
	 	\node[rv] (r) at (4.5,-1.5) {${D}$};
	 	\node[rv] (y) at (6,0) {$Y$};
	 	\draw[->] (c) -- (m);
	 	\draw[->] (c) -- (l);
	 	\draw[->] (c) -- (r);
	 	\draw[->] (c) -- (y);
	 	\draw[->] (a) -- (m);
	 	\draw[->] (m) -- (l);
	 	\draw[->] (l) -- (r);
	 	
	 	\path[->] (a) edge[bend left = 0] (l);
	 	\path[->] (a) edge[bend left = 10] (r);
	 	\path[->] (a) edge[bend left = 0] (y);
	 	\path[->] (m) edge[bend right = 25] (r);
	 	\path[->] (m) edge[bend right = -10] (y);
	 	\path[->] (r) edge[bend right = 0] (y);
	 	
	 	\path[->] (z2) edge[bend right = 0] (y);
	 	\path[->] (z2) edge[bend right = 0] (l);
	 	\path[->] (z2) edge[bend right = 0] (r);
	 	\path[->] (z2) edge[bend right = 0] (m);
	 	
		\path[<->,dashed] (a) edge[bend left = 0](c);
		\path[<->,dashed] (a) edge[bend left = 0](z2);
	 	\end{scope}
	\end{tikzpicture}
	    \caption{Causal diagram of COMPAS dataset.}
	    \label{fig:compasdag}
        \end{subfigure}
        \hfill
        \begin{subfigure}[b]{0.45\textwidth}
        \begin{tikzpicture}
	 	[>=stealth, rv/.style={thick}, rvc/.style={triangle, draw, thick, minimum size=8mm}, node distance=7mm]
	 	\pgfsetarrows{latex-latex};
	 	\begin{scope}
	 	\node[rv] (c) at (2.25,1.5) {${Z_1}$};
	 	\node[rv] (z2) at (3.75,1.5) {${Z_2}$};
	 	\node[rv] (a) at (0,0) {$X$};
	 	\node[rv] (m) at (1.5,-1.5) {$J$};
	 	\node[rv] (l) at (3,-1.5) {$P$};
	 	\node[rv] (r) at (4.5,-1.5) {${D}$};
	 	\node[rv] (y) at (6,0) {$Y$};
	 	
	 	\node (Zset) [draw,rectangle,minimum width=2.8cm,minimum height=1cm,label={[above right]$Z$-set}] at (3,1.5) {};
	 	\node (Wset) [draw,rectangle,minimum width=3.6cm,minimum height=1cm, label={[above right]$W$-set}] at (3,-1.5) {};
	 	
	 	\path[->] (a) edge[bend left = 0] (y);
	 	\path[->] (a) edge[bend left = -20] (Wset);
	 	\path[->] (Zset) edge[bend left = 0] (Wset);
	 	\path[->] (Wset) edge[bend left = -20] (y);
	 	\path[->] (Zset) edge[bend left = 20] (y);
	 	
	 	\path[<->,dashed] (a) edge[bend left = 20](Zset);
	 	\end{scope}
	\end{tikzpicture}.
	\caption{Causal diagram projected onto the SFM.}
	\label{fig:compassfm}
        \end{subfigure}
        \caption{The causal diagram of COMPAS dataset and its projection onto the SFM.}
    \end{figure}
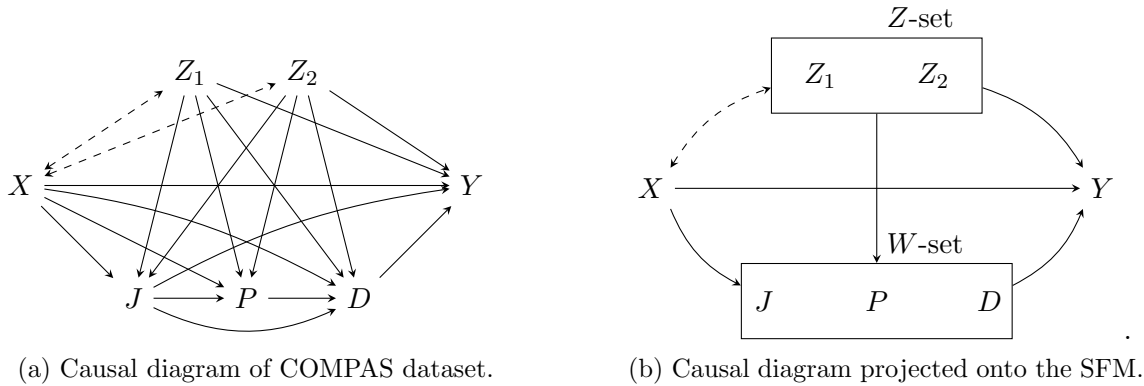
	This diagram can be standardized (projected on the SFM) by grouping the mediators $W = \lbrace J, P, D \rbrace$ and confounders $Z = \lbrace Z_1, Z_2\rbrace$. Formally, the SFM projection can be written as
	\begin{equation}
	    \Pi_{\text{SFM}}(\mathcal{G}) = \langle X =\lbrace X\rbrace, Z =\lbrace Z_1, Z_2\rbrace, W = \lbrace J, P,D\rbrace, Y = \lbrace Y\rbrace\rangle.
	\end{equation}
The projection is shown in Fig.~\ref{fig:compassfm}. Notice that the full diagram $\mathcal{G}$ is not needed for determining the SFM projection. The data scientist only needs to group the confounders and mediators, and determine whether there is latent confounding between any of the groups. 
	
Going back to Florida, after a period of using the algorithm, it is observed that Non-White individuals are 9\% more likely to be classified as high-risk, i.e.,
	\begin{align}
		P(y \mid x_1) -  P(y \mid x_0) = 9\%.
	\end{align}
	The reader might wonder if the disparity of 9\% means that racial minorities are discriminated by the legal justice system in Broward County. An important consideration here is how much of the disparity can be explained by the spurious association of race with age or gender (which potentially influence the recidivism prediction), the effect of race on the prediction mediated by juvenile and prior offense counts, or the direct effect of race on the prediction. $\hfill \square$
\end{example}
\noindent As noted in the example, the SFM does not explicitly assume the causal structure within the possibly multi-dimensional sets $Z$, $W$. In causal language, the SFM can be seen as an equivalence class of causal diagrams\footnote{A more detailed study on the properties of clustered diagrams can be found in \citep{anand:etal21}.}. For instance, under the SFM, if $Z = \lbrace Z_1, Z_2 \rbrace$, the relationship between $Z_1$ and $Z_2$ is not fully specified, and it may be the case that $Z_1 \rightarrow Z_2$, $Z_2 \rightarrow Z_1$, or of another type. Secondly, the SFM encodes assumptions about lack of hidden confounding, which is reflected through the absence of bidirected arrows between variable groups. We discuss  in Appendix \ref{appendix:IDEst}  how the lack of confounding assumptions can be relaxed.

\section{Foundations of Causal Fairness Analysis} \label{foundations}

In this section, we will introduce two main results that will allow us to understand and possibly solve the problem of fairness using causal tools. 
First, we will introduce in Sec.~\ref{sec:fairnesscriteria} a structural definition of fairness, which  leads to a natural way of expressing legal requirements based on the  doctrines of disparate treatment and impact. In particular, we will define the notion of \textit{fairness measure} and two key properties called \textit{admissibility} and \textit{decomposability}. Armed with these new notions, we will then be able to formally state the fundamental problem of causal fairness analysis. In words, these results suggest that reasoning about fairness requires an understanding of how to explain variations, in particular, how the outcome variable $Y$ can be explained in terms of the structural measures following variations of the protected attribute $X$. In Sec.~\ref{sec:explain-variations}, we  formalize the notion of a contrast, which allows us to understand the aforementioned variations from a factual-counterfactual perspective. 
We then prove how to decompose contrasts and re-express them in terms of the structural basis, which lead to the explainability plane and the decomposition of arbitrary types of contrast. The discussion is somewhat theoretical and we will provide examples to ground and make the main points more concrete.

\begin{example}[College's admissions, inspired by \citep{bickel1975sex}] \label{ex:berkeleyfpcfa}
During the process of application to undergraduate studies, prospective students choose a department to which they want to join ($D$), report their gender $X$ ($x_0$ female, $x_1$ male), and after a certain period they receive the admission decisions $Y$ ($y_1$ accepted, $y_0$ rejected). 
	
In reality, how applicants pick their department ($f_D$) and how the university decides on who to admit ($f_Y$) is represented by the SCM $\mathcal{M^*} = \langle V=\{X, D, Y\}, U=\{U_X, U_D, U_Y\}, $ $\mathcal{F^*}, P^*(U) \rangle$, where the pair $\langle \mathcal{F^*} , P^*(U) \rangle$ is such that
\begin{numcases}{\mathcal{F}^*, P^*(U) :}
		        X & $\gets$ \text{Bernoulli} (0.5)\label{eq:fpf-mstar}\\
		        D & $\gets$ \text{Bernoulli} (0.5 + $\frac{2}{10}$ X)  \\ 
		        Y & $\gets$ \text{Bernoulli} (0.1 + 0 * X + $\frac{7}{10}$ D). \label{eq:mstar-y}
\end{numcases}
Based on data that it made available from the previous admissions' cycle, the school is sued by a group of applicants who allege gender discrimination. 
In particular, they share with the court the following statistics: 
\begin{equation}
	    P(y \mid x_1) - P(y \mid x_0) = 14\%,
\end{equation}
which seems a devastating piece of evidence against the university. In words, it seems that male candidates are 14\% more likely to be admitted than their female counterparts. The natural question that arises is what could explain  such a disparity in the observed data? Would this be a textbook case of direct, gender-discrimination? 

Despite the fact that the court does not have access to the true $\mathcal{M}^*$, in reality, there is no direct discrimination at all since $f_Y$ (Eq.~\ref{eq:mstar-y}) does not take gender into account (note the zero coefficient multiplying $X$). In fact, female applicants are more likely to apply to arts \& humanities departments, which have lower admission rates, in turn causing a disparity in the overall admission rates.

The plaintiffs hire a team of (evil) data scientists that conduct their own study. After some time, the team comes back and claims to have understood the university decision-making process after a series of interviews and research, which is given by SCM  $\mathcal{M'} = \langle V=\{X, D, Y\}, U=\{U_X, U_D, U_Y\}, \mathcal{F'}, P'(U) \rangle$, where $\langle \mathcal{F'} , P'(U) \rangle$ are such that 
	
\begin{numcases}{\mathcal{F}', P'(U) :}
X & $\gets$ \text{Bernoulli} (0.5) \label{eq:fpf-mprime}\\
D & $\gets$ \text{Bernoulli}$(0.5 +\frac{2}{10} X)$  \\ 
Y & $\gets$ \text{Bernoulli}$(0.1 + \frac{14}{100} * X + 0 * D)$. \label{eq:mprime-y}
\end{numcases}
The only difference between $\mathcal{M^*}$ (the true set of mechanisms) and  $\mathcal{M}'$ (the hypothesized one) is $f_Y$. Interestingly enough, the hypothesized $f_Y$ (Eq.~\ref{eq:mprime-y}) takes gender ($X$) into account while discarding any information about applicants' department choices ($D$). Clearly, if this was indeed the true decision-making process by which the university selects students, the jury should condemn the university, since that would be a blatant case of direct discrimination. $\hfill \square$
\end{example} 

Interestingly, both SCMs $M^*$ and $M'$ generate the same total variation of 14\%.  Still, $\mathcal{M}^*$, which is the true generating model, doesn't suggest any type of gender discrimination, while $\mathcal{M}'$, which is false, suggests that the university's admissions decisions are purely based on gender. In summary, SCMs $M^*$ and $M'$ are qualitatively different (in the sense that the disparity is transmitted along different causal mechanisms), but they are indistinguishable based on TV. We next formalize this issue in more generality.

\subsection{Structural Fairness Criteria}\label{sec:fairnesscriteria}
To understand the issue discussed in the previous section, we start by noting that qualitative distinctions -- such as differentiating direct and indirect discrimination -- lie at the heart of some of the most important legal doctrines on discrimination. In particular, the doctrine of \textit{disparate treatment} asks the question on whether a different decision would have been reached for an individual, had she/he been of a different race or gender, while keeping all other attributes the same \citep{barocas2016big}. In causal terminology, the question is about disparities transmitted along the \textit{direct causal mechanism} between the attribute $X$ and the outcome $Y$. On the other hand, the doctrine of \textit{disparate impact} considers situations in which a facially neutral policy (that does not use race or gender explicitly) results in very different outcomes for racial or gender groups \citep{rutherglen1987disparate}. In this case, the concern is also with disparities transmitted along \textit{indirect and spurious} causal mechanisms. Motivated by these legal doctrines, we can mathematically define qualitative assessments about discrimination based on an SCM:

\begin{definition}[Structural Fairness Criterion]
Let $\Omega$ be a space of SCMs. A structural criterion $Q$ is a binary operator on the space $\Omega$, that is a map $Q: \Omega \to \{ 0, 1 \}$ that determines whether a set of causal mechanisms between $X$ and $Y$ exist or not, in a given SCM $\mathcal{M} \in \Omega$.
\end{definition}
For most of the manuscript, we wish to focus on structural criteria that capture direct, indirect, and spurious discrimination. We consider these criteria as elementary. A more refined and detailed structural notions are discussed in Sec.~\ref{DIBN}. We now formally define the three elementary structural fairness criteria, based on the functional relationships between $X$ and $Y$ encoded in an SCM:
\begin{definition}[Elementary Structural Fairness Criteria] \label{def:str-fair}
Let $\pa(V_i)$ and $\an$ be the parents and ancestors of $V_i$ in the causal diagram $\mathcal{G}$, respectively. For an SCM $\mathcal{M}$, define the following three structural criteria:
\begin{enumerate}[label=(\roman*)]
    \item Structural direct criterion: $$\text{Str-DE}_X(Y) = \mathbb{1}(X \in \pa(Y)).$$
	\item Structural indirect criterion: $$\text{Str-IE}_X(Y) = \mathbb{1}(X \in \an(\pa(Y))).$$
	\item Structural spurious criterion: $$\text{Str-SE}_X(Y) = \mathbb{1}\Big( (U_X \cap \an(Y) \neq \emptyset) \vee (\an(X) \cap \an(Y) \neq \emptyset) \Big).$$
\end{enumerate}
For $\text{Str-DE}_X(Y) = 0$, $\text{Str-IE}_X(Y) = 0$, and $\text{Str-SE}_X(Y) = 0$, we write DE-fair$_X(Y)$, IE-fair$_X(Y)$, and SE-fair$_X(Y)$, respectively.
\end{definition}

In words, the structural direct criterion verifies whether the attribute $X$ is a function of the mechanism $f_Y$, that is, if $Y$ is a function of $X$. The structural indirect criterion verifies whether there exist mediating variables, which are affected by $X$, that in turn influence $Y$. These two criteria are defined in terms of the functional relationships within $\mathcal{M}$, or $\mathcal{F}$. This means that they convey causal information about the relationship among endogenous variables.  Finally, the structural spurious criterion verifies whether there exist variables that both causally affect the attribute $X$ and the outcome $Y$. Different than the previous ones, this criterion relies on the relationships among the exogenous variables $U$, which relates to the confounding relation among the observables. 

We revisit the Admissions example to ground such notions: 

\begin{example}[Admissions continued]
In the SCM $\mathcal{M}$ defined in Eq.~\ref{eq:berkeley-samp-1}-\ref{eq:berkeley-samp-3}, the structural direct and indirect effects can be analyzed as follows:
\begin{enumerate}[label = (\roman*)]
    \item $Y$ is fair w.r.t. $X$ in terms of direct effect if and only if:
    \begin{align}
        \alpha = 0 \text{ in } \lbrace Y \gets \text{Bernoulli}(0.1 +\alpha X + \beta D) \rbrace.
    \end{align}
    \item $Y$ is fair w.r.t. $X$ in terms of indirect effect if and only if:
    \begin{equation}
        \begin{aligned}
        \lambda = 0 &\text{ in } \lbrace D \gets \text{Bernoulli}(0.5 + \lambda X) \rbrace \text{, or }\\
        \beta = 0 &\text{ in } \lbrace Y \gets \text{Bernoulli}(0.1 + \alpha X + \beta D) \big)\rbrace.
    \end{aligned}
    \end{equation}
\end{enumerate}
For the SCM $\mathcal{M}^*$ in Eq.~\ref{eq:fpf-mstar}-\ref{eq:mstar-y}, we can see that direct discrimination does not exist, since $\alpha = 0$, and therefore $X \notin \pa(Y)$ (see Def.~\ref{def:str-fair}(i)). However, indirect discrimination is present, since $\lambda = \frac{2}{10}$ and $\beta = \frac{7}{10}$, and therefore $X \in \an(\pa(Y))$ (see Def.~\ref{def:str-fair}(ii)). In contrast to this, for the SCM $\mathcal{M}'$ in Eq.~\ref{eq:fpf-mprime}-\ref{eq:mprime-y}, direct discrimination is present, since $\alpha = \frac{1}{7}$ and thus $X \in \pa(Y)$, but indirect discrimination is not, since $\beta = 0$ and thus $X \notin \an(\pa(Y))$. $\hfill \square$
\end{example}
Other meaningful structural fairness criteria could be defined using different logical combinations of these three elementary criteria. For instance, $Y$ can be called \textit{totally fair} with respect to $X$ ($\textit{Fair}_X(Y)$) if and only if direct, indirect, and spurious fairness are simultaneously true (i.e., $\textit{Fair}_X(Y) = \textit{DE-fair}_X(Y) \land \textit{IE-fair}_X(Y) \land \textit{SE-fair}_X(Y)$). Alternatively, causal fairness could be defined as $\textit{Causal-fair}_X(Y) = \textit{DE-fair}_X(Y) \land \textit{IE-fair}_X(Y)$, which encodes the non-existence of active causal influence from $X$ to $Y$ (neither direct nor mediated). 

These definitions of structural fairness represent idealized and intuitive criteria that can be evaluated whenever the true underlying mechanisms are known, i.e., the fully specified SCM $\mathcal{M}$. 
The importance of these measures, encoded through the structural mechanisms (Def.~\ref{def:str-fair}), stems from the fact that they underpin existing legal and societal notions of fairness. Therefore, they will be used as a benchmark to understand under what conditions, and how close other measures, which might be estimable from data, approximate these idealized and intuitive notions.

One central question is whether there exist quantitative measures of discrimination that can help us assess whether a structural criterion is satisfied or not. Firstly, we define a general fairness measure that can be computed from the SCM:
\begin{definition}[Fairness Measure]
Let $\Omega$ be a space of SCMs. A fairness measure $\mu$ is a functional on the space $\Omega$, that is a map $\mu: \Omega \to \mathbb{R}$, which quantifies the association of $X$ and $Y$ through any subset of causal mechanisms, in a given SCM $\mathcal{M} \in \Omega$.    
\end{definition}
Here, the definition of a fairness measure $\mu$ is kept as quite general. In Sec.~\ref{sec:explain-variations}, we will restrict our attention to a specific class of measures $\mu$ and explain their importance in the context of Causal Fairness Analysis. In the sequel, we introduce a notion that represents when a fairness measure $\mu$ is suitable for assessing a structural criterion $Q$:

\begin{definition}[Admissibility]
Let $\Omega$ be a class of SCMs on which a structural criterion $Q$ and a measure $\mu$ are defined. A measure $\mu$ is said to be admissible w.r.t. the structural criterion $Q$ within the class of models $\Omega$, or $(Q, \Omega)$-admissible, if:
\begin{equation} \label{eq:adm}
    \forall \mathcal{M} \in \Omega: Q(\mathcal{M}) = 0 \implies \mu(\mathcal{M}) = 0.
\end{equation}
\end{definition}
For simplicity, we will use admissibility instead of $(Q, \Omega)$-admissibility whenever the context is clear. 
The importance of having an admissible measure $\mu$ stems from the contrapositive of Eq.~\ref{eq:adm}, namely, if $\mu(\mathcal{M})$ can be measured or evaluated and $\mu(\mathcal{M}) \neq 0$, this means that the structural measure must be true, i.e., $Q(\mathcal{M}) = 1$. 
In other words, the measure $\mu$ will act as a link between the well-defined but unobservable structural measure and the observable and estimable world. 
For concreteness, consider the following result that formalizes the issue found in Example \ref{ex:berkeleyfpcfa}:
\begin{lemma}[TV is not admissible w.r.t. Str-{DE, IE, SE}] \label{lem:tvnotadmissible}
Let $\Omega$ be the space of Semi-Markovian SCMs which contain variables $X$ and $Y$. Let $\mu$ be the total variation measure TV$_{x_0, x_1}(y)$. Then $\mu$ is not admissible with respect to structural direct, indirect, or spurious criteria. That is,
\begin{align}
        (\text{Str-DE}(\mathcal{M}) = 0) &\centernot\implies (\text{TV}_{x_0, x_1}(y) = 0), \\
        (\text{Str-IE}(\mathcal{M}) = 0) &\centernot\implies (\text{TV}_{x_0, x_1}(y) = 0), \\
        (\text{Str-SE}(\mathcal{M}) = 0) &\centernot\implies (\text{TV}_{x_0, x_1}(y) = 0).
\end{align}
\end{lemma}
In fact, the reason why the TV measure is not admissible with respect to structural direct, indirect, and spurious criteria is because it captures the three types of variations together. 

To formalize this idea, we introduce the notion of \textit{decomposability} of a measure $\mu$, i.e.: 
\begin{definition}[Decomposability] \label{def:decompose}
Let $\Omega$ be a class of SCMs and $\mu$ be a measure defined over it. $\mu$ is said to be $\Omega$-decomposable if there exist measures 
\begin{align}
  \mu_1, \dots, \mu_k \text{ such that } 
  \mu = f(\mu_1, \dots, \mu_k),
\end{align} and where $f$ is a non-trivial function vanishing at the origin, i.e., $f(0, \dots, 0) = 0$.
\end{definition}
In words, decomposability states that a measure $\mu$ can be written as a function of measures $(\mu_i)_{i=1}^k$, and that if all measures $(\mu_i)_{i=1}^k$ are equal to $0$ for an SCM $\mathcal{M}$, then the measure $\mu$ must be $0$ as well. For concreteness, consider the following examples. 

\begin{example}[Covariance decomposition, after \citep{zhang2018non}]
\hphantom{Let} 
Let $\mu$ be the covariance measure between random variables $X$ and $Y$,  
\begin{align}
\text{Cov}(X, Y) = \ex[XY] - \ex[X]\ex[Y], 
\end{align}
which plays a role somewhat analogous to TV (and, more broadly, the observational distribution) whenever the system $F$ and $P(U)$ are linear and Gaussian. 
Further, let the causal covariance be defined as
\begin{align}
    \text{Cov}^c_{x}(X, Y) =  \text{Cov}(X, Y - Y_x).
\end{align}
Furthermore, let the spurious covariance be defined as
\begin{align}
    \text{Cov}^s_{x}(X, Y) =  \text{Cov}(X, Y_x).
\end{align}
Then, we can write
\begin{align}
    \text{Cov}(X, Y) = f\big(\text{Cov}^c_{x}(X, Y), \text{Cov}^s_{x}(X, Y)\big),
\end{align}
with the function $f(a, b) = a + b$, which satisfies $f(0, 0) = 0$. $\hfill \square$
\end{example}
Armed with the definitions of admissibility and decomposability, we are ready to formally define the first version of the problem studied here. 

\begin{definition}[Fundamental Problem of Causal Fairness Analysis (preliminary)] \label{def:fpcfa}
Consider a class of SCMs $\Omega$, and let 
\begin{itemize}
	\item  $Q_1, Q_2, ..., Q_k$ be a collection of structural fairness criteria, and 
	\item $\mu$ be a measure, 
\end{itemize}
both defined over $\Omega$. 
The Fundamental Problem of Causal Fairness Analysis is to find a collection of measures $\mu_1, \dots, \mu_k$
such that the following properties are satisfied: 
\begin{enumerate}[label=(\arabic*)]
	\item $\mu$ is decomposable w.r.t. $\mu_1, \dots, \mu_k$; 
	\item $\mu_1, \dots, \mu_k$ are admissible w.r.t. the structural fairness criteria $Q_1, Q_2, ..., Q_k$. 
\end{enumerate} 
In other words, find measures 
\begin{align}
  \mu_1, \dots, \mu_k \text{ that are admissible w.r.t. } Q_1, \dots, Q_k, 
\end{align}
respectively, and such that 
\begin{align}
  \mu = f(\mu_1, \dots, \mu_k),
\end{align} where $f$ is a non-trivial function vanishing at the origin, i.e., $f(0, \dots, 0) = 0$. $\square$
\end{definition}

\newcommand{\myx}{1.2}
\newcommand{\myy}{-1.6}
\begin{wrapfigure}{r}{0.38\textwidth}
	\centering
	\begin{tikzpicture}[fairmap]
		
		\node [f1] (tv) at (0*\myx, -0.25*\myy) {TV};
       
        \draw [draw=black, fill=gray!8] (-3,-2.2) rectangle (3,-1);
 
        \node [f1] (kse) at (-2*\myx, \myy) {$\mu_{\text{SE}}$};
        \node [f1] (se) at (-2*\myx, 2*\myy) {Str-SE};
        
        \node [f1] (kde) at (0*\myx, \myy) {$\mu_{\text{DE}}$};
        \node [f1] (de) at (0*\myx, 2*\myy) {Str-DE};
        
        \node [f1] (kie) at (2*\myx, \myy) {$\mu_{\text{IE}}$};
        \node [f1] (ie) at (2*\myx, 2*\myy) {Str-IE};

        \node (scm) at (0*\myx, 3*\myy) {SCM $\mathcal{M}^*$};
        
        \node (adm) at (1*\myx, 1.55*\myy) {\small admissible};
        \node (decomp) at (1*\myx, 0.3*\myy) {\small decomposable};
     

        \draw[-Latex] (scm) to [bend right = -10] (se);
        \draw[-Latex] (scm) to (de);
        \draw[-Latex] (scm) to [bend right = 10] (ie);
        \draw[-Latex,dashed] (tv) to [bend right = 10] (kse);
        \draw[-Latex,dashed] (tv) to (kde);
        \draw[-Latex,dashed] (tv) to [bend right = -10] (kie);
        \draw[adm] (se) to [bend right = 0] (kse);
        \draw[adm] (de) to (kde);
        \draw[adm] (ie) to [bend right = 0] (kie);
        
	\end{tikzpicture}
	\caption{Fundamental Problem of Fairness Analysis (TV version).}
	\label{fig:APDmotivation}
	\vspace{-0.2in}
\end{wrapfigure}
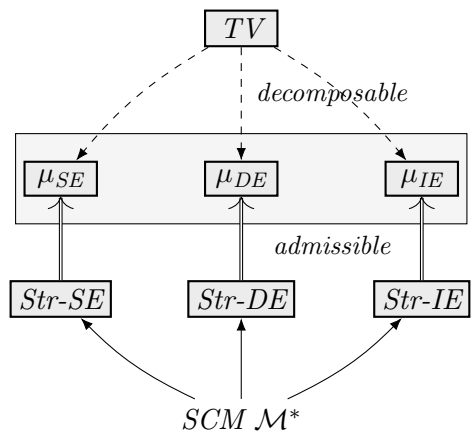

For grounding this discussion, we will consider that the measure $\mu$ is given by the TV\footnote{Naturally, other types of contrasts can be used as measures instead of TV, such as the covariance \citep{zhang2018non} or equality of odds \citep{hardt2016equality,zhang2018equality}.} and the structural measures will be $Str$-\{$DE, IE, SE$\}. We refer to this instance of the problem by $\fpcfa$.
Fig.~\ref{fig:APDmotivation} provides a visual summary of the FPCFA where TV is shown on the top and the structural measures $Str$-\{$DE, IE, SE$\} on the bottom. As you have just seen  in Lem.~\ref{lem:tvnotadmissible}, TV is not admissible relative to each of these structural measures.

The FPCFA asks for the existence of a set of measures $(\mu_{DE}, \mu_{IE}, \mu_{SE})$ that could act as a bridge between $TV$ and the more meaningful, albeit unobservable structural measures $Str$-\{$DE, IE, SE$\}. In fact, the FPCFA is solved whenever TV can be expressed in terms of $(\mu_{DE}, \mu_{IE}, \mu_{SE})$, and each of these measures is admissible w.r.t. to the corresponding structural measures. 
If that is the case, the measures $(\mu_{DE}, \mu_{IE}, \mu_{SE})$ could be seen as explaining  the variations of TV in terms of the most elementary, structural components. Interestingly, this is both a quantitive and a qualitative exercise. From TV's perspective, $(\mu_i)_{i=1}^k$ should account for all its variations, which is naturally a quantitive exercise. From the structural measures perspective, we would like to enforce soundness, namely, discrimination is indeed readable from the corresponding $(\mu_i)_{i=1}^k$, which is a  qualitative exercise. 

\subsection{Explaining Factual \& Counterfactual Variations} \label{sec:explain-variations}
In this section, the main task is studying how the variations in outcome $Y$ can be explained by changes of the protected attribute $X$. The result of this study is what we call the \textit{population-mechanism} plane, which we also refer to as the \textit{explainability plane} (Fig. \ref{fig:Erefined}). The methodology introduced by the plane will allow us to re-express different measures of fairness in an unified manner, which will facilitate their comparison in terms of  admissibility,  decomposability, and possibly other desirable properties. 

We start by introducing a quite general type of measure encoding the idea of contrast.
\begin{definition}[Contrast] \label{def:contrast}
Given a SCM $\mathcal{M}$,  a contrast $\mathcal{C}$ is any quantity of the form 
\begin{equation} \label{eq:contrast}
    \contrast = \ex[y_{C_1} \mid E_1] - \ex[y_{C_0}\mid E_0],
\end{equation}
where $E_0, E_1$ are observed (factual) clauses and $C_0, C_1$ are counterfactual clauses to which the outcome $Y$ responds. Furthermore, whenever 
\begin{enumerate}[label=(\alph*)]
    \item $E_0 = E_1$, the contrast $\mathcal{C}$ is said to be counterfactual; 
    \item $C_0 = C_1$, the contrast  $\mathcal{C}$ is said to be factual.
\end{enumerate}
\end{definition}
For simplicity\footnote{The results in this section hold for any real-valued random variable $Y$.}, we will focus on the binary case, in which a contrast can be written as
\begin{equation}
    P(y_{C_1} \mid E_1) - P(y_{C_0}\mid E_0).
\end{equation}
The purpose of a contrast is to compare the outcome of individuals who coincide with the observed event $E_1$ in the factual world and whose values were intervened on (possibly counterfactually) as defined by $C_1$, against individuals who coincide with the observed event $E_0$ in the factual world and whose values were intervened on (possibly counterfactually) as defined by $C_0$. The definition also distinguishes two special cases of contrasts. A counterfactual contrast captures only the difference in outcome induced by the difference in interventions $C_0, C_1$ (since $E_0 = E_1$). Complementary to this, a factual contrast captures only the difference induced by the observed events $E_0, E_1$ (since $C_0= C_1$). We now show why contrasts are useful for explaining variations:
\begin{theorem}[Contrast's Decomposition \& Structural Basis Expansion] \label{thm:contrasts}
Given a SCM $\mathcal{M}$ and let $\mathcal{C}$ be a contrast $    P(y_{C_1} \mid E_1) - P(y_{C_0}\mid E_0)$. $\mathcal{C}$ can be decomposed into its counterfactual and factual variations, namely:
\begin{align} \label{eq:contrastdecomposition}
    \underbrace{P(y_{C_1} \mid E_1) - P(y_{C_0}\mid E_1)}_{\text{counterfactual contrast}} + \underbrace{P(y_{C_0} \mid E_1) - P(y_{C_0}\mid E_0)}_{\text{factual contrast}}.
\end{align}
Furthermore, the corresponding counterfactual and factual contrasts admit the following structural basis expansions, respectively:
\begin{enumerate}[label=(\alph*)]
    \item Counterfactual contrast ($\mathcal{C}_{\text{ctf}}$), where $E_0 = E_1 = E$, can be expanded as
        \begin{align} \label{eq:down}
        P(y_{C_1} \mid E) - P(y_{C_0} \mid E) = \sum_u \, \big( \underbrace{y_{C_1}(u) - y_{C_0}(u)}_{\text{unit-level difference}} \big) \underbrace{P(u \mid E)}_{\text{posterior}},
        \end{align}
    \item Factual contrast ($\mathcal{C}_{\text{factual}}$), where $C_0 = C_1 = C$, can be expanded as
        \begin{align} \label{eq:up}
             P(y_{C} \mid E_1) - P(y_{C} \mid E_0) = \sum_u \underbrace{y_{C}(u)}_{\text{unit outcome}} \big(\underbrace{P(u \mid E_1) - P(u \mid E_0)}_{\text{posterior difference}}\big).
        \end{align}
\end{enumerate} 
\end{theorem}
The decomposition and structural basis expansion of contrasts presented in this theorem entail a fundamental connection of causal fairness measures with structural causal models. In particular, the decomposition given in Eq.~\ref{eq:contrastdecomposition} allows us to disentangle factual and counterfactual variations within any contrast.

We note that Eqs.~\ref{eq:down} and \ref{eq:up}  re-expresses the variations within the target quantity in terms of the underlying units and activated mechanisms, as references by the SCM. We would like to understand these qualitatively different types of variations separately.

\begin{wrapfigure}{r}{0.40\textwidth}
\centering
\includegraphics[keepaspectratio,width=0.40\columnwidth]{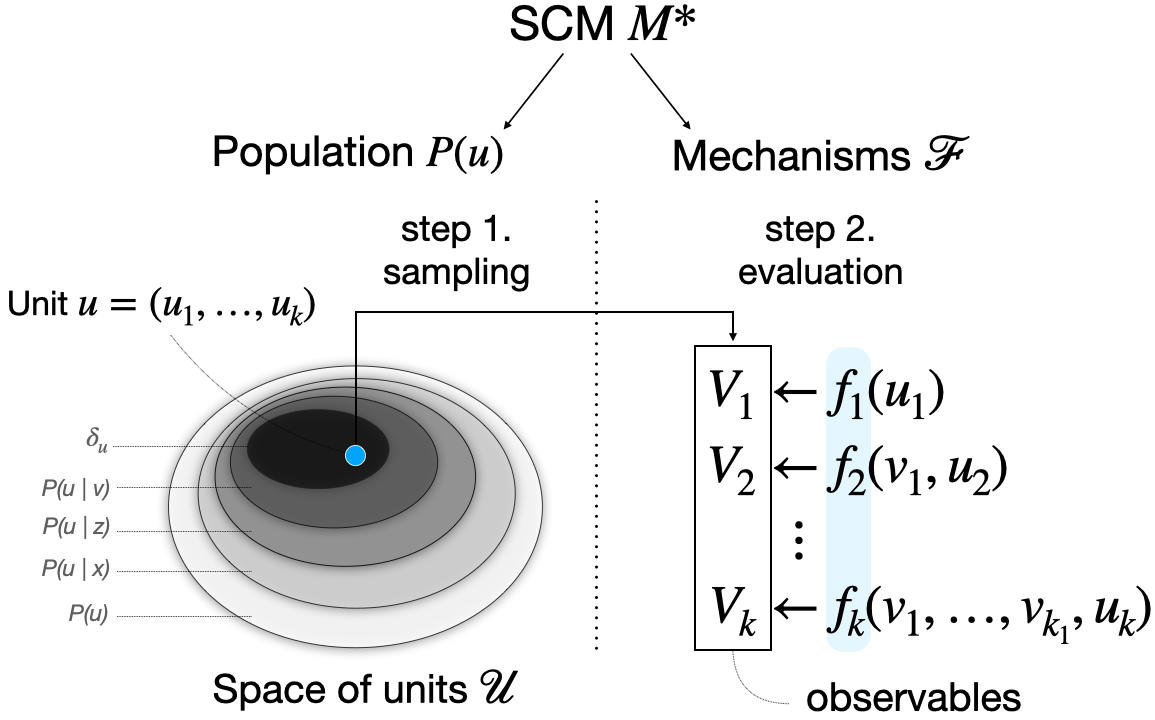}
\caption{Two-step generative process that includes sampling of an unit from the population (left), and evaluating it against the corresponding structural mechanisms (right). }
\label{fig:generative:perspective}
\vspace{-0.3in}
\end{wrapfigure}

First, we will take a generative interpretation over how the targeted variations are realized in terms of the SCM $\mathcal{M} = \langle V, U, \mathcal{F}, P(U)$. Fig.~\ref{fig:generative:perspective} illustrates the two-step generative process that goes as follow: 
\begin{itemize}[leftmargin=*]
    \item[] (1) \textbf{Sampling:} A unit $U=u$ is sampled from the population distributed according to $P(U)$; 
    \item[] (2) \textbf{Evaluation:}  This unit $u$ passes through the sequence of mechanisms $\mathcal{F}$, in causal order, until the values of the endogenous variables $V$ are realized. 
\end{itemize}
The l.h.s. of the figure shows the sampling process while the r.h.s. represents the evaluation process. As discussed in Sec.~\ref{sec:data}, if the system is not submitted to an intervention, this leads to the observational distribution. On the other hand, if the values of certain variables are fixed through intervention, this leads to the corresponding counterfactual distribution. 

Considering this two-step generative process, we re-examine the variations encoded in the structural basis expansion of Thm.~\ref{thm:contrasts}.  For convenience, we reproduce the  equation relative to the counterfactual variations in the sequel (Eq.~\ref{eq:down}):
\begin{align*}
P(y_{C_1} \mid E) - P(y_{C_0} \mid E) = \;
         \sum_u \, \big(\underbrace{y_{C_1}(u) - y_{C_0}(u)}_{\text{unit-level difference}} \big) \underbrace{P(u \mid E)}_{\text{posterior}}
\end{align*}
First, we consider the second factor in the r.h.s. of the expression. Note that  $P(u \mid E=e)$ represents the first step in the generative process in which units who naturally arise to value $E = e$ are drawn from the population. In fact, depending on the granularity of the evidence $E$, a different fraction of the population (or types of individuals) will be selected. 
For instance, if $E = \{\}$, the (posterior) distribution $P(u)$ is somewhat uninformative, and represents an average when units are drawn at random from the underlying population, regardless of their predispositions and characteristics. On the other hand, if $E = \lbrace X = x \rbrace$, the posterior distribution $P(u \mid x)$ would be more informative since it now includes units that naturally would have $X=x$. Naturally, this is less informative compared to more specific events such as $E = \lbrace X = x, Z = z\rbrace$ or $E = \lbrace X = x, Z = z, W = w, Y = y \rbrace$. In fact, the l.h.s. of the figure illustrates this increasingly more refined and informative set of events $E$, i.e., starting from picking individuals at random from the general population, $P(u)$, to a single individual $\delta_u$, where $\delta_u$ is the Dirac delta function. 
Second, we note that once the unit $U=u$ is selected, all randomness is vanished, and the unit will go through the set of mechanisms $\mathcal{F}$.  The first factor of the expression, $y_{C_1}(u) - y_{C_0}(u)$, describes the difference in response $y$ between conditions $C_1$ and $C_0$ for a fixed realization of exogenous variables $u$. As realizations of exogenous variables $U$ are indices for the different identities of units in the population, the quantity $y_{C_1}(u) - y_{C_0}(u)$ will be an unit-level quantity. 

In the context of fairness discussed here, consider the case when $C_1 = x_1$ and $C_0 = x_0$,  which could represent the protected attribute, for instance, males and females, or White and African-American. The quantity $y_{x_1}(u) - y_{x_0}(u)$ measures what the change in outcome $Y$ would be when changing the attribute $X$ from $x_0$ to $x_1$, for a specific unit $u$. For this particular choice of $C_0, C_1$, the quantity captures what is known as the \text{total causal effect} of $X$ on $Y$, that is it includes all the variations from $X$ to $Y$ translated across causal pathways.

In summary,  any counterfactual contrast $\mathcal{C}_{\text{ctf}}$ can be decomposed into two parts:
\begin{enumerate}
    \item A unit-level difference comparing the counterfactual worlds $C_1$ vs. $C_0$ to a specifc unit $U=u$. This quantity is determined by the causal mechanisms $\mathcal{F}$ of the SCM, and does not depend on the distribution $P(u)$. 
    \item A posterior distribution $P(u \mid E=e)$ that indicates the probability mass assigned to unit $u$ whenever the event $E=e$. By changing the granularity of the event $E$, the space of included units is restricted, making the measure more specific to a subpopulation (see Fig.~\ref{fig:generative:perspective} (l.h.s.)).
\end{enumerate}
Given that the selection of units is fixed (second factor), and the only thing that varies is the selection of the mechanisms (first factor) through the choices of the counterfactual conditions $C_1$ and $C_0$, this will generate variations downstream, so they will be inherently ``causal". In fact, the specific instantiation of $C_1$ and $C_0$ and $E=\{\}$ (i.e., $P(U)$) matches to the very definition of average causal effect, $P(y | do(x_1)) - P(y | do(x_0))$. 

We now re-examine the factual variations encoded in the structural basis expansion of Thm.~\ref{thm:contrasts}.  For convenience, we reproduce the  corresponding equation (Eq.~\ref{eq:up}):
\begin{align*}
P(y_{C} \mid E_1) - P(y_{C} \mid E_0) = \sum_u \underbrace{y_{C}(u)}_{\text{unit outcome}} \big(\underbrace{P(u \mid E_1) - P(u \mid E_0)}_{\text{posterior difference}}\big)
\end{align*}
In words, a factual contrast can be expanded as a sum of differences in the posteria $P(u\mid E_1) - P(u\mid E_0)$, weighted by unit-level outcomes $y_C(u)$.
We note that the difference in posteria represents the first step in the generative process in which two sets of units who naturally arise to values $E_1$ and $E_0$ are drawn from the population, respectively. Similarly to the previous discussion, different sub-populations will be selected depending on the granularity of the evidence $E_1$, $E_0$. The scope of these events is the same but their instantiations are different. 

This can be seen as complementary when compared to the counterfactual contrasts. Given that the mechanisms  are fixed (first factor), the component that generates variations is relative to the choice of units based on the factual conditions $E_0$ and $E_1$. We suggest this will generate upstream variations, which will be somewhat ``non-causal'' (also called spurious), which will be described in more details later on in the manuscript.  Still, for instance, we are mostly interested in setting $C_1 = C_0 = x$, so that $X = x$ along all causal pathways. The contrast then will capture the difference in probability mass assigned to $u$ in events $E_1$ and $E_0$. By definition, spurious effects are generated by variations that \textit{causally precede} $X$, so these cannot be captured by intervening on $X$. For this reason, we need to compare events $E_1$ and $E_0$, which have resulted in a different instantiation of the value of $X$. This factorization also suggest mathematically how causal and spurious effects are inherently different from each other. 


\begin{wrapfigure}{r}{0.50\textwidth}
\centering
\includegraphics[keepaspectratio,width=0.50\columnwidth]{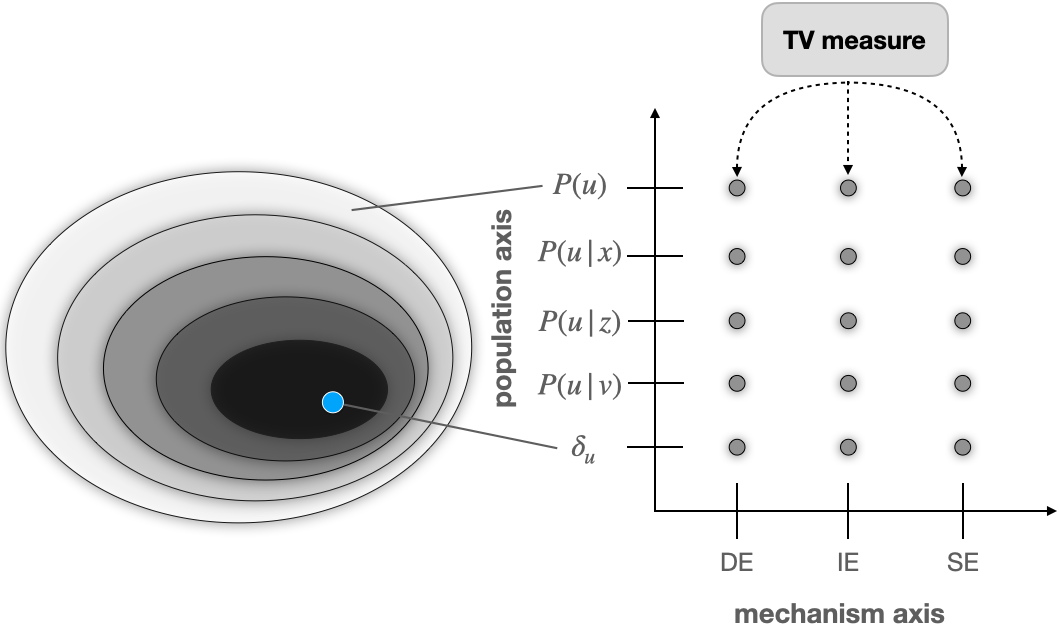}
\caption{In the population axis, contrasts are restricted to smaller subsets of units $u$ in the domain $\mathcal{U}$. At the same time, along the mechanism axis, we distinguish between direct, indirect, and spurious variations.}
\label{fig:Erefined}
\vspace{-0.15in}
\end{wrapfigure}

\paragraph{Explainability plane.} By decomposing variations via factual and counterfactual contrasts, and expanding them using the structural basis, we can give the essential structure of the measures used in Causal Fairness Analysis. The approach used for decomposing the total variation is shown in Fig.~\ref{fig:Erefined}, which we call the explainability plane.
As the figure illustrates, there are two separate axes of the decomposition. On the mechanism axis, we are decomposing the TV into its direct, indirect, and spurious variations. On the population axis, we are considering increasingly precise subsets of the space of units $\mathcal{U}$, which correspond to different posterior distributions. As we will see later, moving along the population axis will correspond to constructing increasingly more powerful fairness measures.

\section{TV family} \label{Measures}
In this section, we introduce a family of measures that populate the explainability plane in Fig.~\ref{fig:Erefined}. Since all the measures describe variations included within the TV measure, we refer to them as the \textit{TV family} (part $e$ of Fig.~\ref{fig:roadmap}). In particular, this section aims to explicitly solve the $\fpcfa$ discussed in Sec.~\ref{foundations}.

\subsection{Solving the Fundamental Problem of Causal Fairness Analysis} \label{MeasuresGap}
The measures in the TV-family are introduced in order. We start with measures that quantify discrimination in the entire population of units $u$ (corresponding to posterior $P(u)$), and reach measures that quantify discrimination for a single unit $u$ (corresponding to the posterior $\delta_u$, where $\delta$ is the Dirac delta function).

\subsubsection{Population level Contrasts - \(P(u)\)}
We first recall that the TV measure itself is not admissible with respect to structural criteria $\strm$, as shown in Lem.~\ref{lem:tvnotadmissible}. Specifically, the reason for this is that the TV captures variations between groups generated by any mechanism of association, both causal and non-causal, and does not distinguish them. Our first step is therefore to disentangle these variations -- the causal and non-causal (or spurious) -- within the TV. 
\begin{definition}[Total and spurious effects] \label{def:teexpse}
    Let the total effect and experimental spurious effect be defined as follows:
\begin{align}
    \text{TE}_{x_0, x_1}(y) &= P(y_{x_1}) - P(y_{x_0}) \\
    \text{Exp-SE}_{x}(y) &= P(y_{x}) - P(y \mid x)
\end{align}
Further, we write TE-fair$_X(Y)$ whenever $\text{TE}_{x_0, x_1}(y) = 0$, or simply TE-fair when $X$ and $Y$ are clear from the context. Exp-SE-fair is defined analogously.
\end{definition}
In words, TE measures the difference in the outcome $Y$ when setting $X = x_1$, compared to setting $X = x_0$. The measure can be visualized graphically as shown in Fig.~\ref{fig:graphTE}. In this case, $Y$ responds to the change in $X$ from $x_0$ to $x_1$ through two mechanisms. In fact, $Y$ variations in response to change in $X$ are realized through (i) the direct link, $X \rightarrow Y$, and (ii) through the indirect link via $W$, $X \rightarrow W \rightarrow Y$.  In the context of the COMPAS dataset in Ex.~\ref{ex:compas}, the total effect would be the average difference in recidivism prediction had an individual's race been White compared to had it been Non-White. Since the covariates $Z$ vary naturally in both counterfactual worlds (both sides of the expression), those are cancelled out and $Y$ variations can be explained in terms of the downstream variations in response to the change purely on $X$\footnote{The TE measure is also called causal effect and sometimes written in do-notation, $P(y \mid do(x_1)) - P(y \mid do(x_0))$. Obviously, this quantity has well-defined semantics given a SCM, despite the fact that no one intends or believes to set any of the protected attributes literally by intervention. Still, through the formal language of causality, one can contemplate these distinct counterfactual realities. In particular, one can disentangle and explain the sources of $Y$ variations in response to changes in $X$, including the ones through the causal pathways versus the non-causal ones, along the spurious paths.}.

In a complementary manner, the experimental spurious effect measures the average difference in outcome $Y$ when $X = x$ by intervention, counterfactually speaking, compared to simply observing that $X = x$. As shown graphically in Fig.~\ref{fig:graphExpSE}, note that since from $Y$'s perspective $X$ has the same value $x$ in both factors, the $Y$ variations can be explained in terms of the upstream effect in response to how $X$ naturally affected $Z$ versus how $Z$ varies free from the influence of $X$. In the COMPAS dataset, this would mean the average difference in recidivism prediction for individuals for whom the race is set to White by intervention, compared to simply observing the race to be White.  

Syntactically, following the discussion in Sec.~\ref{sec:explain-variations}, we can write these quantities in terms of contrasts (Def.~\ref{def:contrast}), namely: 
\begin{align}
\label{eq:te-basis} 
    \text{TE}_{x_0, x_1}(y) &= \mathcal{C}(x_0, x_1, \emptyset, \emptyset) \\
\label{eq:expse-basis} 
    \text{Exp-SE}_{x}(y) &= \mathcal{C}(\emptyset, x, x, \emptyset)
\end{align}

Based on these two notions, the TV can be decomposed into two distinct sources of variation, which correspond precisely to its causal and non-causal mechanisms:
\begin{figure}
     \centering
     \begin{subfigure}[b]{0.45\textwidth}
         \centering
             \begin{tikzpicture}
	 [>=stealth, rv/.style={thick}, rvc/.style={triangle, draw, thick, minimum size=7mm}, node distance=18mm]
	 \pgfsetarrows{latex-latex};
		\node[rv] (z1) at (0,1) {$Z$};
	 	
	 	\node[rv] (x1x) at (-1,0) {$X=x_1$};
	 	
	 	\node[rv] (w1) at (0,-1) {$W$};
	 	\node[rv] (y1) at (1.5,0) {$Y$};
	 	
	 	\draw[->, blue] (x1x) -- (w1);
		\draw[->] (z1) -- (y1);
	 	\path[->, blue] (x1x) edge[bend left = 0] (y1);
	 	\draw[->] (w1) -- (y1);
		\draw[->] (z1) -- (w1);
		
		\node (mns) at (2.25, 0) {\Large $-$};
		\node (a) at (0, -1.5) {$P(y_{x_1})$};
		\node (b) at (4.5, -1.5) {$P(y_{x_0})$};
		
		\node[rv] (z2) at (4.5,1) {$Z$};
	 	\node[rv] (x2x) at (3.5,0) {$X=x_0$};
	 	\node[rv] (w2) at (4.5,-1) {$W$};
	 	\node[rv] (y2) at (6,0) {$Y$};
	 	
	 	\draw[->, red] (x2x) -- (w2);
		\draw[->] (z2) -- (y2);
	 	\path[->, red] (x2x) edge[bend left = 0] (y2);
	 	\draw[->] (w2) -- (y2);
		\draw[->] (z2) -- (w2);
	 \end{tikzpicture}
\caption{Total effect TE$_{x_0, x_1}(y)$.}
\label{fig:graphTE}
     \end{subfigure}
     \hfill
     \begin{subfigure}[b]{0.45\textwidth}
         \centering
             \begin{tikzpicture}
	 [>=stealth, rv/.style={thick}, rvc/.style={triangle, draw, thick, minimum size=7mm}, node distance=18mm]
	 \pgfsetarrows{latex-latex};
		\node[rv] (z1) at (0,1) {$Z$};
	 	
	 	\node[rv] (x1x) at (-1,0) {$X=x$};
	 	
	 	\node[rv] (w1) at (0,-1) {$W$};
	 	\node[rv] (y1) at (1.5,0) {$Y$};
	 	
	 	\draw[->] (x1x) -- (w1);
		\draw[->] (z1) -- (y1);
	 	\path[->] (x1x) edge[bend left = 0] (y1);
	 	\draw[->] (w1) -- (y1);
		\draw[->] (z1) -- (w1);
		
		\node (mns) at (2.25, 0) {\Large $-$};
		\node (a) at (0, -1.5) {$P(y_{x})$};
		\node (b) at (4.5, -1.5) {$P(y \mid x)$};
		
		\node[rv] (z2) at (4.5,1) {$Z$};
	 	\node[rv] (x2) at (3.5,0) {$X=x$};
	 	\node[rv] (w2) at (4.5,-1) {$W$};
	 	\node[rv] (y2) at (6,0) {$Y$};
	 	
	 	\draw[->] (x2) -- (w2);
		\draw[->] (z2) -- (y2);
	 	\path[->] (x2) edge[bend left = 0] (y2);
		\path[<->, red] (x2) edge[bend left = 30, dashed] (z2);
	 	\draw[->] (w2) -- (y2);
		\draw[->] (z2) -- (w2);
	 \end{tikzpicture}
\caption{Experimental spurious effect Exp-SE$_{x}(y)$.}
\label{fig:graphExpSE}
     \end{subfigure}
     \hfill
     \begin{subfigure}[b]{0.45\textwidth}
         \centering
         \begin{tikzpicture}
	 [>=stealth, rv/.style={thick}, rvc/.style={triangle, draw, thick, minimum size=7mm}, node distance=18mm]
	 \pgfsetarrows{latex-latex};
		\node[rv] (z1) at (0,1) {$Z$};
	 	
	 	\node[rv] (x11) at (-1.1,0.3) {$X=x_1$};
	 	\node[rv] (x10) at (-1.1,-0.3) {$X=x_0$};
	 	
	 	\node[rv] (w1) at (0,-1) {$W$};
	 	\node[rv] (y1) at (1.5,0) {$Y$};
	 	
	 	\draw[->] (x10) -- (w1);
		\draw[->] (z1) -- (y1);
	 	\path[->, blue] (x11) edge[bend left = 0] (y1);
	 	\draw[->] (w1) -- (y1);
		\draw[->] (z1) -- (w1);
		
		\node (mns) at (2.25, 0) {\Large $-$};
		\node (a) at (0, -1.5) {$P(y_{x_1, W_{x_0}})$};
		\node (b) at (4.5, -1.5) {$P(y_{x_0})$};
		
		\node[rv] (z2) at (4.5,1) {$Z$};
	 	\node[rv] (x20) at (3.5,0) {$X=x_0$};
	 	\node[rv] (w2) at (4.5,-1) {$W$};
	 	\node[rv] (y2) at (6,0) {$Y$};
	 	
	 	\draw[->] (x20) -- (w2);
		\draw[->] (z2) -- (y2);
	 	\path[->, red] (x20) edge[bend left = 0] (y2);
	 	\draw[->] (w2) -- (y2);
		\draw[->] (z2) -- (w2);
	 \end{tikzpicture}
         \caption{Natural direct effect NDE$_{x_0, x_1}(y)$.}
         \label{fig:graphNDE}
     \end{subfigure}
     \hfill
     \begin{subfigure}[b]{0.45\textwidth}
         \centering
         \begin{tikzpicture}
	 [>=stealth, rv/.style={thick}, rvc/.style={triangle, draw, thick, minimum size=7mm}, node distance=18mm]
	 \pgfsetarrows{latex-latex};
		\node[rv] (z1) at (0,1) {$Z$};
	 	
	 	\node[rv] (x11) at (-1.1,0.3) {$X=x_1$};
	 	\node[rv] (x10) at (-1.1,-0.3) {$X=x_0$};
	 	
	 	\node[rv] (w1) at (0,-1) {$W$};
	 	\node[rv] (y1) at (1.5,0) {$Y$};
	 	
	 	\draw[->, blue] (x10) -- (w1);
		\draw[->] (z1) -- (y1);
	 	\path[->] (x11) edge[bend left = 0] (y1);
	 	\draw[->, blue] (w1) -- (y1);
		\draw[->] (z1) -- (w1);
		
		\node (mns) at (2.25, 0) {\Large $-$};
		\node (a) at (0, -1.5) {$P(y_{x_1, W_{x_0}})$};
		\node (b) at (4.5, -1.5) {$P(y_{x_1})$};
		
		\node[rv] (z2) at (4.5,1) {$Z$};
	 	\node[rv] (x20) at (3.5,0) {$X=x_1$};
	 	\node[rv] (w2) at (4.5,-1) {$W$};
	 	\node[rv] (y2) at (6,0) {$Y$};
	 	
	 	\draw[->, red] (x20) -- (w2);
		\draw[->] (z2) -- (y2);
	 	\path[->] (x20) edge[bend left = 0] (y2);
	 	\draw[->, red] (w2) -- (y2);
		\draw[->] (z2) -- (w2);
	 \end{tikzpicture}
        \caption{Natural indirect effect NIE$_{x_1, x_0}(y)$.}
        \label{fig:graphNIE}
     \end{subfigure}
     \caption{Graphical representations of measures used in TV$_{x_0, x_1}(y)$ decomposition.}
     \label{fig:graphMeasures}
\end{figure}
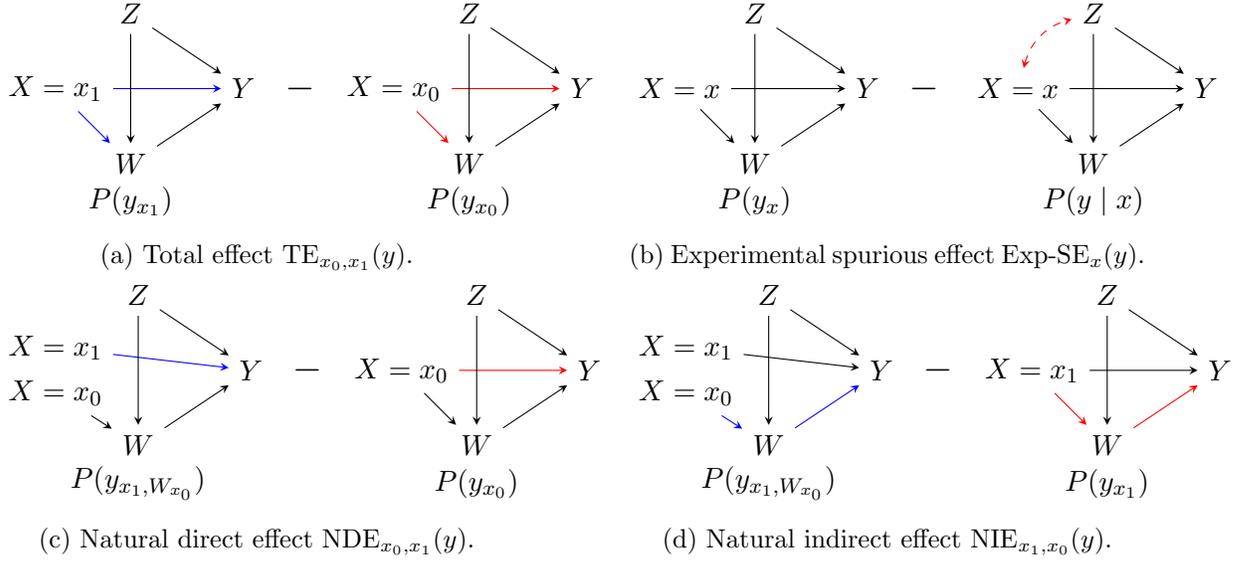
\begin{lemma}[TV decomposition I] \label{lemma:te+expse}
The total variation measure can be decomposed as
\begin{align}
    \text{TV}_{x_0, x_1}(y) =  \text{TE}_{x_0, x_1}(y) + (\text{Exp-SE}_{x_0}(y) - \text{Exp-SE}_{x_1}(y)).
\end{align}
\end{lemma}
Lem. \ref{lemma:te+expse} shows that the TV measure equals to the total effect on $Y$ when $X$ transitions from $x_0$ to $x_1$ plus the difference between the experimental spurious effect of $X = x_0$ and $X = x_1$ \footnote{An alternative way of interpreting this relation is by flipping TV and TE in the equation, namely: 
\begin{align}
\text{TE}_{x_0, x_1}(y) = \text{TV}_{x_0, x_1}(y) - (\text{Exp-SE}_{x_0}(y) - \text{Exp-SE}_{x_1}(y)).
\end{align}
This means that the total effect of transitioning $X$ from $x_0$ to $x_1$ on $Y$ is equal to the corresponding total variation of $Y$ minus the the difference in spurious effects of the baseline $X=x_0$ versus $X=x_1$. 
}. 
In other words, TV accounts for the sum of the directed (causal) and confounding paths from $X$ to $Y$. More formally, the lemma shows that the TV satisfies decomposability with respect to TE and Exp-SE. 

Interestingly, the TE itself is still not admissible w.r.t. Str-\{DE,IE\}, as it captures all causal influences of $X$ on $Y$, including the direct (through the direct link $X \rightarrow Y$) and indirect ones (i.e., paths via $W$). 
\begin{lemma}[TE inadmissibility]
The total effect measure TE$_{x_0, x_1}(y)$ is not admissible with respect to structural criteria Str-DE and Str-IE.
\end{lemma}
To solve $\fpcfa$, therefore, we will further need to disentangle the relationships within TE. In particular, we will need to determine the $Y$ variations that are a direct consequence of the protected attribute, and the ones that are mediated by other variables. In the literature, the total effect was shown to be decomposable into the measures known as the \textit{natural direct and indirect effects} \citep{pearl:01}. 
\begin{definition}[Natural direct and indirect effects] \label{def:ndenie}
    The natural direct and indirect effects are defined, respectively, as follows:
    \begin{align}
        \text{NDE}_{x_0, x_1}(y) &= P(y_{x_1, W_{x_0}}) - P(y_{x_0}) \label{eq:nde}  \\
        \text{NIE}_{x_1, x_0}(y) &= P(y_{x_1, W_{x_0}}) - P(y_{x_1}). \label{eq:nie}  
    \end{align}
Further, we write NDE-fair$_X(Y)$ for NDE$_{x_0, x_1}(y) = 0$, or simply NDE-fair when the attribute/outcome are clear from the context. The condition NIE-fair is defined analogously.
\end{definition}
Several observations are important making about these definitions. First in terms of semantics, the NDE captures the difference in Eq.~\ref{eq:nde}, namely, how the outcome $Y$ changes when setting $X = x_1$, but keeping the mediators $W$ at whatever value it would have taken had $X$ been $x_0$, compared to setting $X = x_0$ by intervention. 
This counterfactual statement is shown graphically in Fig.~\ref{fig:graphNDE}. Note that $Y$ ``perceives'' $X$ through the direct link (marked in blue) as if it is equal to $x_1$, written in counterfactual language as $y_{x_1}$, while $W$ perceives $X$ as if it is $x_0$, formally, $W_{x_0}$.  Putting these two together leads to the first factor in Eq.~\ref{eq:nde}, i.e., $y_{x_1, W_{x_0}}$.
The second factor in the contrast is $y_{x_0}$, which can be written equivalently as $y_{x_0, W_{x_0}}$, due to the consistency axiom. It represents the fact that both $Y$ and $W$ perceives $X$ at the same level, $x_0$ \footnote{
For further discussion on counterfactuals, see \citep[Sec.~7.2]{pearl:2k} and \citep{bareinboim2020on}.}. Whenever we subtract one from the other, in some sense, the variations coming from $X$ to $Y$ through $W$ are the same (since it perceives $X$ at the baseline level $x_0$), and what remains are the variations transmitted through the direct arrows, so the name direct effect. The qualification natural is because $W$ attains its value naturally, depending on the value of $X$, but not by interventions. 

Second, in the context of our COMPAS example, the NDE would measure how much the predicted probability of recidivism would have changed for an individual whose race was set by intervention to White, had their race been set to Non-White, but their juvenile and prior offense counts took a value they would have attained naturally (that is, a value naturally attained by White subjects). The contrast represented by the NDE  (in Eq.~\ref{eq:nde}) is known as a \textit{nested counterfactual}, since $X$ has distinct values when considering different variables. Albeit not realizable in the real world, it encodes significant types of variations that can be evaluated from a collection of mechanisms and fully specified SCM, and which is sometimes computable from data, as discussed in more details in Sec.~\ref{Identification}. 

Third, the definition of NIE follows a similar logic while flipping the sources of variations, as illustrated in Eq.~\ref{eq:nie}  and Fig.~\ref{fig:graphNIE}. More specifically, the outcome $Y$ responds to $X$ as being $x_1$ through the direct link in both factors of the contrast ($y_{x_1}$), which means that no direct influence from $X$ to $Y$ is ``active". On the other hand, $W$ responds to $X$ when varying from levels $X=x_1$ to $x_0$, formally written as $W_{x_1}$ versus $W_{x_0}$; this, in turn, affects $Y$, which formally is written as counterfactuals $y_{x_1, W_{x_1}}$ versus $y_{x_1, W_{x_0}}$. \footnote{The first term $y_{x_1, W_{x_1}}$ is equivalently written as $y_{x_1}$, which follows from the consistency axiom \citep[Sec.~7.2]{pearl:2k}. } The NIE is also a nested counterfactual. For the COMPAS example, the NIE would measure how much the predicted probability of recidivism would have changed for an individual whose race was  White, had their race been Non-White along the indirect causal pathway influencing the values of juvenile and prior offense counts. 

Syntactically, and following the discussion in Sec.~\ref{sec:explain-variations}, we can put these observations together and write the NDE and NIE as counterfactual contrasts (Eq.~\ref{eq:down}), namely: \footnote{Following prior discussion and reversing the usual simplification back, based on the application of the consistency axiom, these contrasts can more explicitly be written as: 

\begin{align}
        \text{NDE}_{x_0, x_1}(y) &= \mathcal{C}(\{x_0, W_{x_0}\}, \{x_1, W_{x_0}\}, \emptyset, \emptyset) \\
        \text{NIE}_{x_1, x_0}(y) &= \mathcal{C}(\{x_1, W_{x_1}\}, \{x_1, W_{x_0}\}, \emptyset, \emptyset).
\end{align}
It's evident when considering the NDE that the variations through the mediator $W$, $W_{x_0}$, coincide in both sides of the contrast and end up cancelling out, which means that all remaining variations are due to the direct change of $X$ from $x_0$ to $x_1$ in the first component of the pair. On the other hand, the direct variations in the NIE are both equal to $X=x_1$, which cancel out, and $Y$ changes are in response to the change in $W$, which varies differently depending on whether $X=x_1$ and $X=x_0$, or $W_{x_1}$ versus $W_{x_0}$. 
} 

\begin{align}
        \text{NDE}_{x_0, x_1}(y) &= \mathcal{C}(x_0, \{x_1, W_{x_0}\}, \emptyset, \emptyset) \\
        \text{NIE}_{x_1, x_0}(y) &= \mathcal{C}(x_1, \{x_1, W_{x_0}\}, \emptyset, \emptyset).
\end{align}

The notions of NDE and NIE, together with Exp-SE, in fact provide the first solution to the $\fpcfa$, as shown in the next result. 
\begin{theorem}[$\fpcfa$ solution (preliminary)] \label{thm:fpcfa-1st-sol}
The total variation measure can be decomposed as
    \begin{align}
        \text{TV}_{x_0, x_1}(y) &= \text{NDE}_{x_0, x_1}(y) - \text{NIE}_{x_1, x_0}(y) + (\text{Exp-SE}_{x_0}(y) - \text{Exp-SE}_{x_1}(y)).
    \end{align}
Furthermore, the measures NDE, NIE, and Exp-SE are admissible with respect to Str-DE, Str-IE, and Str-SE, respectively. More formally, we  write
\begin{align}
        \text{Str-DE-fair} &\implies \text{NDE-fair} \label{eq:admnde}\\
        \text{Str-IE-fair} &\implies \text{NIE-fair} \label{eq:admnie}\\
        \text{Str-SE-fair} &\implies \text{Exp-SE-fair}\label{eq:admexpse}.
\end{align}
Therefore, the measures $(\mu_{DE}, \mu_{IE}, \mu_{SE}) = (\text{NDE}_{x_0, x_1}(y), \text{NIE}_{x_1, x_0}(y), \text{Exp-SE}_{x}(y))$ solve the $\fpcfa$.
\end{theorem}
After showing a solution to $\fpcfa$, we make two important remarks. Firstly, the measures discussed so far admit a structural basis expansion (Thm.~\ref{thm:contrasts}) and can be  expanded as follows:
\begin{align} 
    \label{eq:tvdown}
    \text{TV}_{x_0, x_1}(y) &= \sum_u y(u) \big[P(u | x_1) - P(u\mid x_0)\big] \\ 
    \label{eq:tedown}
    \text{TE}_{x_0, x_1}(y) &= \sum_u \big[y_{x_1}(u) - y_{x_0}(u)\big]P(u) \\
    \label{eq:expseup}
    \text{Exp-SE}_{x}(y) &= \sum_u y_{x}(u) \big[P(u) - P(u\mid x)\big] \\
    \label{eq:ndedown}
    \text{NDE}_{x_0, x_1}(y) &= \sum_u \big[y_{x_1, W_{x_0}}(u) - y_{x_0}(u)\big] P(u) \\
    \label{eq:niedown}
    \text{NIE}_{x_1, x_0}(y) &= \sum_u \big[y_{x_1, W_{x_0}}(u) - y_{x_1}(u)\big] P(u).
\end{align}
The factorization in the display above connects the measures to the sampling-evaluation process discussed in Sec.\ref{sec:explain-variations}, explaining the observed contrasts in terms of unit-level quantities. We revisit this point shortly.
Secondly, one of the significant and practical implications of Thm.~\ref{thm:fpcfa-1st-sol} appears through the Eq.~\ref{eq:admnde}'s contrapositive (and Eqs.~\ref{eq:admnie}, \ref{eq:admexpse}), i.e.:
\begin{equation}
    (\text{NDE}_{x_0, x_1}(y) \neq 0) \implies  \neg \text{Str-DE-fair}. 
\end{equation}
Based on this, we have now a principled way of testing the following hypothesis:
\begin{align}
    H_0: \text{NDE}_{x_0, x_1}(y) = 0.
\end{align}
If the $H_0$ hypothesis is rejected, the fairness analyst can conclude that direct discrimination is present in the dataset. In contrast, any statistics or hypothesis test based on the TV are insufficient to test for the existence of a direct effect. 
\tikzstyle{fairmap}=[
	every node/.style={draw=none, align=center, fill=none, text centered, anchor=center,font=\it},
	every label/.style={font=\small, text = red},
	cx/.append style={anchor=center,sloped,rotate=90,text=red,font=\bf,pos=0.65},
	f1/.style={draw=,fill=gray!15,thick,inner sep=3pt,minimum width=2.5em, align=center, text centered},
	imp/.style={double,->},
]
\renewcommand{\myx}{1.6}
\renewcommand{\myy}{-1.5}
\definecolor{colone}{RGB}{99,172,229}
\definecolor{coltwo}{RGB}{231, 239,246}
\definecolor{colthree}{RGB}{173,203,227}
\begin{figure}
	\centering
	\begin{tikzpicture}[fairmap]
		
		\draw [very thin] (\myx, 0.5 * \myy) -- (9 * \myx, 0.5 * \myy);
		\draw [very thin, -{Latex[length=3mm]}] (0, 0) -- node[above, rotate=90]{population axis} (0, 2.2 * \myy);
		\draw [very thin, -{Latex[length=3mm]}] (0, 0) -- node[above]{mechanism axis} (9.25*\myx, 0);
		
		\draw [very thin] (\myx, 1.5 * \myy) -- (9 * \myx, 1.5*\myy);
		
		\draw [very thin] (\myx,0) -- (\myx, 1.5*\myy);
		\draw [very thin] (3*\myx,0) -- (3*\myx, 1.5*\myy);
		\draw [very thin] (5*\myx,0) -- (5*\myx, 1.5*\myy);
		\draw [very thin] (7*\myx,0) -- (7*\myx, 1.5*\myy);
		\draw [very thin] (9*\myx,0) -- (9*\myx, 1.5*\myy);
		
		\node [draw=red, rotate=90] at (0.5*\myx, \myy)  {$P(u)$};
		
		\node [f1] (tv) at (0.5*\myx, 0.25*\myy) {TV};
        
        \node (spurious) at (4 * \myx, 0.25*\myy) {Spurious};
        \node [f1] (exp-se) at (4*\myx, \myy) {Exp-SE};
        \node [f1] (se) at (4*\myx, 2*\myy) {S-SE};
        
        \node (causal) at (2 * \myx, 0.25*\myy) {Causal};
        \node [f1] (te) at (2*\myx, \myy) {TE};
        \node [f1] (ce) at (2*\myx, 2*\myy) {Str-TE};
        
        \node (direct) at (6 * \myx, 0.25*\myy) {Direct};
        \node [f1] (nde) at (6*\myx, \myy) {NDE};
        \node [f1] (de) at (6*\myx, 2*\myy) {Str-DE};
        
        \node (indirect) at (8 * \myx, 0.25*\myy) {Indirect};
        \node [f1] (nie) at (8*\myx, \myy) {NIE};
        \node [f1] (ie) at (8*\myx, 2*\myy) {Str-IE};

 		\node (te-se) at (1.25*\myx, 0.65*\myy) {$\wedge$};
 		\node (nde-nie) at (5.25 * \myx, 0.65*\myy) {$\wedge$};
        
 		\draw [-] (te) to [bend left = 0] (te-se);
 		\draw [-] (exp-se) to [bend right = 10] (te-se);
 		
 		\draw [-] (nde) to [bend left = 20] (nde-nie);
 		\draw [-] (nie) to [bend right = 10] (nde-nie);
 		
 		\draw [imp] (te-se) to [bend right = 10] (tv);
 		\draw [imp] (nde-nie) to [bend right = 10] (te);
 		
 		\draw [imp] (se) to [bend left = 0] (exp-se);
		\draw [imp] (ie) to (nie);
		\draw [imp] (de) to (nde);
		\draw [imp] (ce) to (te);

	\end{tikzpicture}
	\caption{Placing the total, experimental spurious, natural direct, and natural indirect effects along the population and mechanism axes that were first introduced in Fig.~\ref{fig:Erefined}.}
	\label{fig:babymap}
\end{figure}
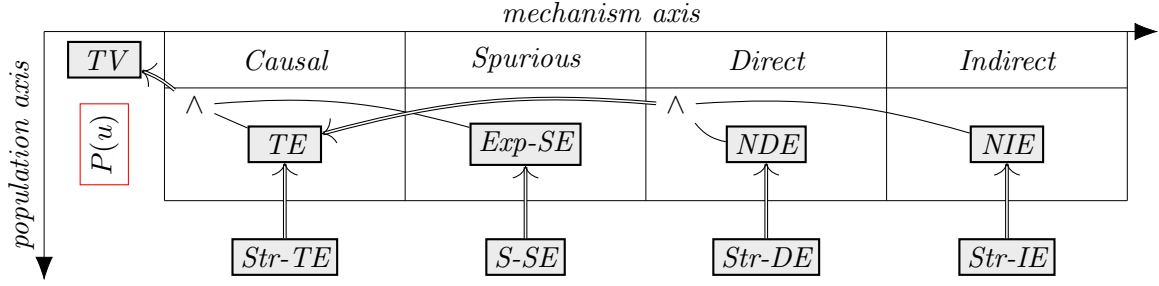

We display in Fig.~\ref{fig:babymap} the measures TE, NDE, NIE, and Exp-SE along the population and mechanism axes of the explainability plane (Fig.~\ref{fig:Erefined}). One may be tempted to surmise that the FPCFA is fully solved based on the results discussed so far. This is unfortunately not always the case, as illustrated next.  
\begin{example}[Limitation of the NDE] \label{ex:NDEfail}
A startup company is currently in hiring season. The hiring decision ($Y \in \lbrace 0,1\rbrace$ indicates whether the candidate is hired) is based on gender ($X \in \lbrace 0, 1\rbrace$ represents females and males, respectively), age ($Z \in \lbrace 0, 1\rbrace$, indicating younger and older applicants, respectively), and education level ($W \in \lbrace 0, 1\rbrace$ indicating whether the applicant has a PhD). The true SCM $\mathcal{M}$, unknown to the fairness analyst, is given by:
\begin{align}
    \label{eq:ndefail1}
    U &\gets N(0, 1)\\
    X &\gets \text{Bernoulli}(expit(U))\\
    Z &\gets \text{Bernoulli}(expit(U)) \\
    W &\gets \text{Bernoulli}(0.3) \\
    \label{eq:ndefail2}
    Y &\gets \text{Bernoulli}(\frac{1}{5}(X + Z - 2XZ) + \frac{1}{6}W),
\end{align}
where expit$(x) = \frac{e^x}{1 + e^x}$. In this case, the NDE can be computed as:
\begin{align}
        \text{NDE}_{x_0, x_1}(y) &= P(y_{x_1, W_{x_0}}) - P(y_{x_0}) \\
                                &= P(\text{Bernoulli}(\frac{1}{5}(1-Z) + \frac{1}{6}W) = 1)
                                 - P(\text{Bernoulli}(\frac{1}{5}(Z) + \frac{1}{6}W) = 1) \\ 
                                &= \sum_{z\in \lbrace 0,1\rbrace}\sum_{w \in \lbrace 0,1\rbrace} P(z, w)[\frac{1}{5}(1-z) + \frac{1}{6}w -\frac{1}{5}z - \frac{1}{6}w]\\
                                &= \sum_{z\in \lbrace 0,1\rbrace}\sum_{w \in \lbrace 0,1\rbrace} P(z)P(w) [\frac{1}{5}(1-2z)] \quad \text{since } P(z,w)=P(z)P(w) \\
                                &= \sum_{z\in \lbrace 0,1\rbrace} P(z)[\frac{1}{5}(1-2z)]
                                 = \frac{1}{2}\times\frac{1}{5} + \frac{1}{2}\times\frac{-1}{5} = 0. \label{eq:NDEfailmix}
\end{align}
In other words, the $\text{NDE}_{x_0, x_1}(y)$ is equal to zero. Still, perhaps surprisingly, the structural direct effect is present in this case, that is Str-DE-fair does not hold, since the outcome $Y$ is a function of gender $X$, as evident from the structural Eq.~\ref{eq:ndefail2}. $\hfill \square$
\end{example}
This example illustrates that even though the NDE is admissible with respect to structural direct effect, it may still be equal to 0 while structural direct effect exists. One can see through Eq.~\ref{eq:NDEfailmix} that the NDE is an aggregate measure over two distinct sub-populations. Specifically, when considering junior applicants, females are 20\% less likely to be hired (units with ($Z=0,X=0$)), whereas for senior applicants, males are 20\% less likely to be hired  (units with ($Z=1,X=1$)). Mixing these two groups together results in the cancellation of the two effects and the NDE equating to $0$, in turn, making it impossible for the analyst to detect discrimination using only the NDE. \footnote{This observation is structural, and despite of the number of samples available. In practice, depending on the sample size, some level of tolerance regarding the difference between these two groups may be present and still be undetectable through any statistical hypothesis testing. }

Another interesting way of understanding this phenomenon is through the structural basis expansion of the NDE. In Eq.~\ref{eq:ndedown}, the posterior weighting term is $P(u)$, which means that both younger and older applicants are included in the contrast. The fact that this contrast mixes somewhat heterogeneous units of the population, regarding the decision-making procedure to decide $Y$ ($f_y$),  motivates another important notion in fairness analysis:
\begin{definition}[Power] \label{def:power}
Let $\Omega$ be a space of SCMs. Let $Q$ be a structural criterion and $\mu_1$, $\mu_2$ fairness measures defined on $\Omega$. Suppose that $\mu_1$, $\mu_2$ are $(Q, \Omega)$-admissible. We say that $\mu_2$ is more powerful than $\mu_1$ if
\begin{align}
    \forall \mathcal{M} \in \Omega: \mu_2(\mathcal{M}) = 0 \implies \mu_1(\mathcal{M}) = 0. 
\end{align}
\end{definition}
The notion of power can be useful in the following context. Suppose there is an SCM $\mathcal{M}$ in the space $\Omega$ for which discrimination is present, $Q(\mathcal{M}) = 1$, while the measure $\mu_1$ is admissible but unable to capture it, i.e., $\mu_1(\mathcal{M}) = 0$. Still, another measure may exist such 
\newcommand{\dpax}{1.2}
\newcommand{\dpay}{-1.5}
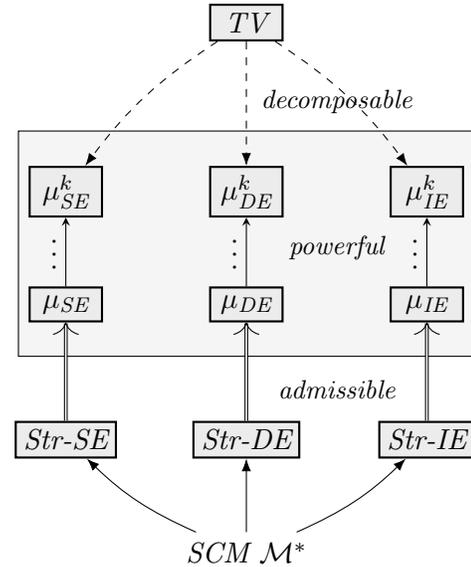
\begin{wrapfigure}{r}{0.4\textwidth}
	\centering
	\begin{tikzpicture}[fairmap]
		
		\node [f1] (tv) at (0*\dpax, -1.5*\dpay) {TV};
       
        \draw [draw=black, fill=gray!8] (-2.525*\dpax,1.46*\dpay) rectangle (2.525*\dpax,-0.538*\dpay);
 
        \node [f1] (1se) at (-2*\dpax, 0*\dpay) {$\mu^k_{\text{SE}}$};
        \node [f1] (kse) at (-2*\dpax, \dpay) {$\mu_{\text{SE}}$};
        \node [f1] (se) at (-2*\dpax, 2.2*\dpay) {Str-SE};
        
        \node [f1] (1de) at (0*\dpax,0*\dpay) {$\mu^k_{\text{DE}}$};
        \node [f1] (kde) at (0*\dpax, \dpay) {$\mu_{\text{DE}}$};
        \node [f1] (de) at (0*\dpax, 2.2*\dpay) {Str-DE};
        
        \node [f1] (1ie) at (2*\dpax, 0*\dpay) {$\mu^k_{\text{IE}}$};
        \node [f1] (kie) at (2*\dpax, \dpay) {$\mu_{\text{IE}}$};
        \node [f1] (ie) at (2*\dpax, 2.2*\dpay) {Str-IE};

        \node  (scm) at (0*\dpax, 3.2*\dpay) {SCM $\mathcal{M}^*$};
        
        \node (adm) at (1*\dpax, 1.75*\dpay) {\small admissible};
        \node (pow) at (1*\dpax, 0.5*\dpay) {\small powerful};
        \node (decomp) at (1*\dpax, -0.8*\dpay) {\small decomposable};
     

        \draw[-Latex] (scm) to [bend right = -10] (se);
        \draw[-Latex] (scm) to (de);
        \draw[-Latex] (scm) to [bend right = 10] (ie);
        \draw[-Latex,dashed] (tv) to [bend right = 10] (1se);
        \draw[-Latex,dashed] (tv) to (1de);
        \draw[-Latex,dashed] (tv) to [bend right = -10] (1ie);
        \draw[imp] (se) to [bend right = 0] (kse);
        \draw[imp] (de) to (kde);
        \draw[imp] (ie) to [bend right = 0] (kie);
        \draw[-stealth] (kse) to node[above, rotate=90]{$\dots$} (1se);
        \draw[-stealth] (kde) to node[above, rotate=90]{$\dots$} (1de);
        \draw[-stealth] (kie) to node[above, rotate=90]{$\dots$} (1ie);
        
	\end{tikzpicture}
	\caption{FPCFA with power relations. }
	\label{fig:APDmotivationPower}
	\vspace{-0.2in}
\end{wrapfigure}
 that $\mu_2(\mathcal{M}) \neq 0$. If this is the case, we would say that discrimination qualitatively described by criterion $Q$ can be detected using measure $\mu_2$, but not using $\mu_1$. We would then say that $\mu_2$ is \textit{more powerful} than $\mu_1$. Putting it differently, what Ex.~\ref{ex:NDEfail} showed was that the measure
\begin{align}
    \text{NDE}_{x_0,x_1}(y) = \mathcal{C}(x_0, \{x_1, W_{x_0}\}, \emptyset, \emptyset)
\end{align}
was not powerful enough. The reason in this case is that for the NDE, the conditioning events are $E_0 = E_1 = \emptyset$, which is not refined enough to capture the discrimination in the aforementioned example. Next, we re-write the definition of FPCFA to account for the measures' power:

\begin{definition}[FPCFA continued with power] \label{def:fpcfa-with-pow}
The Fundamental Problem of Causal Fairness Analysis is to find a collection of measures $\mu_1, \dots, \mu_k$
such that the following properties are satisfied: 
\begin{enumerate}[label=(\arabic*)]
	\item $\mu$ is decomposable w.r.t. $\mu_1, \dots, \mu_k$; 
	\item $\mu_1, \dots, \mu_k$ are admissible w.r.t. the structural fairness criteria $Q_1, Q_2, \dots, Q_k$. 
	\item $\mu_1, \dots, \mu_k$ are as powerful as possible.
\end{enumerate} 
\end{definition}

We provide in Fig.~\ref{fig:APDmotivationPower} an updated, visual representation of the FPCFA that accounts for the power relation across measures. In some sense, picking $(\text{NDE}_{x_0, x_1}(y), \text{NIE}_{x_1, x_0}(y),$ $\text{Exp-SE}_{x}(y))$ as the measures $(\mu_{DE}^k, \mu_{IE}^k, \mu_{SE}^k)$ helped to solve the original problem, but the gap between TV and the structural measures is so substantive that certain critical instances were left undetected. In the updated definition, the requirement is to find measures that are as powerful as possible, or in other words, the closest possible to the corresponding structural ones, $\strm$. 
In the sequel, we discuss how to construct increasingly more powerful measures by using more specific events $E$. 

\subsubsection{$X$-specific Contrasts - $P(u \mid x)$}
We will quantify the level of discrimination for a specific subgroup of the population for which $X(u)=x$ (for example, females) by considering contrasts with the conditioning event $E = \lbrace X = x\rbrace$. In fact, we are moving inwards in the population axis in Fig.~\ref{fig:Erefined}, following the discussion in Sec.~\ref{sec:explain-variations}, and the  
sub-population we are focusing on is more specific. More formally, this can be seen through the structural basis expansion (Eq.~\ref{eq:down}) and the fact that the 
 posterior after using the new $E$ becomes $P(u \mid X=x)$, which generates a family of $x$-specific measures:
 \begin{definition}[$x$-specific TE, DE, IE, and SE]
\label{def:xspecific}
The $x$-\{total, -direct, -indirect,$ $ -spurious\} effects are defined as follow 
    \begin{align}
        x\text{-TE}_{x_0, x_1}(y\mid x) &= P(y_{x_1} \mid x) - P(y_{x_0}\mid x)  \\
        x\text{-DE}_{x_0, x_1}(y\mid x) &= P(y_{x_1, W_{x_0}} \mid x) - P(y_{x_0}\mid x)  \\
        x\text{-IE}_{x_1, x_0}(y\mid x) &= P(y_{x_1, W_{x_0}}\mid x) - P(y_{x_1} \mid x)\\
        x\text{-SE}_{x_0, x_1}(y) &= P(y_{x_0} \mid x_1) - P(y_{x_0} \mid x_0).
    \end{align}
\end{definition}
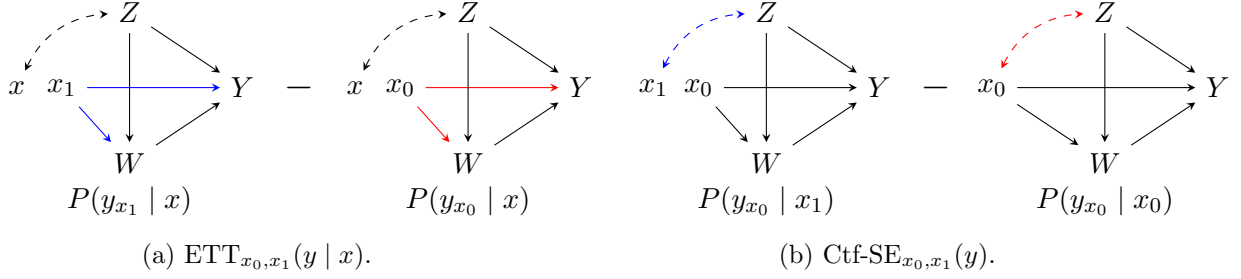
\begin{figure}
     \centering
     \begin{subfigure}[b]{0.45\textwidth}
         \centering
         \begin{tikzpicture}
	 [>=stealth, rv/.style={thick}, rvc/.style={triangle, draw, thick, minimum size=7mm}, node distance=18mm]
	 \pgfsetarrows{latex-latex};
		\node[rv] (z1) at (0,1) {$Z$};
	 	\node[rv] (x1) at (-1.5,0) {$x$};
	 	
	 	\node[rv] (x11) at (-0.9,0) {$x_1$};
	 	
	 	\node[rv] (w1) at (0,-1) {$W$};
	 	\node[rv] (y1) at (1.5,0) {$Y$};
	 	
	 	\draw[->, blue] (x11) -- (w1);
		\draw[->] (z1) -- (y1);
	 	\path[->, blue] (x11) edge[bend left = 0] (y1);
		\path[<->] (x1) edge[bend left = 30, dashed] (z1);
	 	\draw[->] (w1) -- (y1);
		\draw[->] (z1) -- (w1);
		
		\node (mns) at (2.25, 0) {\Large $-$};
		\node (a) at (0, -1.5) {$P(y_{x_1} \mid x)$};
		\node (b) at (4.5, -1.5) {$P(y_{x_0} \mid x)$};
		
		\node[rv] (z2) at (4.5,1) {$Z$};
	 	\node[rv] (x2) at (3,0) {$x$};
	 	\node[rv] (x20) at (3.6,0) {$x_0$};
	 	\node[rv] (w2) at (4.5,-1) {$W$};
	 	\node[rv] (y2) at (6,0) {$Y$};
	 	
	 	\draw[->, red] (x20) -- (w2);
		\draw[->] (z2) -- (y2);
	 	\path[->, red] (x20) edge[bend left = 0] (y2);
		\path[<->] (x2) edge[bend left = 30, dashed] (z2);
	 	\draw[->] (w2) -- (y2);
		\draw[->] (z2) -- (w2);
	 \end{tikzpicture}
         \caption{ETT$_{x_0, x_1}(y \mid x)$.}
         \label{fig:graphETT}
     \end{subfigure}
     \hfill
          \begin{subfigure}[b]{0.45\textwidth}
         \centering
             \begin{tikzpicture}
	 [>=stealth, rv/.style={thick}, rvc/.style={triangle, draw, thick, minimum size=7mm}, node distance=18mm]
	 \pgfsetarrows{latex-latex};
		\node[rv] (z1) at (0,1) {$Z$};
	 	\node[rv] (x1) at (-1.5,0) {$x_1$};
	 	
	 	\node[rv] (x1x) at (-0.9,0) {$x_0$};
	 	
	 	\node[rv] (w1) at (0,-1) {$W$};
	 	\node[rv] (y1) at (1.5,0) {$Y$};
	 	
	 	\draw[->] (x1x) -- (w1);
		\draw[->] (z1) -- (y1);
	 	\path[->] (x1x) edge[bend left = 0] (y1);
		\path[<->, blue] (x1) edge[bend left = 30, dashed] (z1);
	 	\draw[->] (w1) -- (y1);
		\draw[->] (z1) -- (w1);
		
		\node (mns) at (2.25, 0) {\Large $-$};
		\node (a) at (0, -1.5) {$P(y_{x_0} \mid x_1)$};
		\node (b) at (4.5, -1.5) {$P(y_{x_0} \mid x_0)$};
		
		\node[rv] (z2) at (4.5,1) {$Z$};
	 	\node[rv] (x2) at (3,0) {$x_0$};
	 	\node[rv] (w2) at (4.5,-1) {$W$};
	 	\node[rv] (y2) at (6,0) {$Y$};
	 	
	 	\draw[->] (x2) -- (w2);
		\draw[->] (z2) -- (y2);
	 	\path[->] (x2) edge[bend left = 0] (y2);
		\path[<->, red] (x2) edge[bend left = 30, dashed] (z2);
	 	\draw[->] (w2) -- (y2);
		\draw[->] (z2) -- (w2);
	 \end{tikzpicture}
     \caption{Ctf-SE$_{x_0, x_1}(y)$.}
     \label{fig:graphCtfSE}
     \end{subfigure}
     \caption{Graphical representations of some $x$-specific causal fairness measures. The blue and red color highlight where the contrast between the quantities lies.}
     \label{fig:graphMeasures}
\end{figure}
The $x$-TE is a well-known quantity and usually called the effect of treatment on the treated (ETT, for short) in literature, and appeared in \citep{heckman1998matching}, while the $x$-specific DE, IE, and SE are more recent quantities were introduced in \citep{zhang2018fairness}. 
\footnote{In fact, \cite{zhang2018fairness} originally named these quantities the \textit{counterfactual} DE, IE, and SE, but we highlight here that they are the $x$-specific counterparts of their marginal effects. This is for clarify of the discussion here since from now on in this paper, all quantities will be ``counterfactual'', in the sense of layer 3 in Pearl Causal Hierarchy \citep{bareinboim2020on}. }
Some observations ensue from these definitions. Firstly, these measures can be written as their structural basis and unit-level factorization (Eqs.~\ref{eq:down} and \ref{eq:up}), that is
\begin{align} 
    \label{eq:ettdown}
    x\text{-TE}_{x_0, x_1}(y\mid x) &= \sum_u [y_{x_1}(u) - y_{x_0}(u)]P(u\mid x) \\
    \label{eq:ctfdedown}
    x\text{-DE}_{x_0, x_1}(y\mid x) &= \sum_u [y_{x_1, W_{x_0}}(u) - y_{x_0}(u)]P(u\mid x) \\
    \label{eq:ctfiedown}
    x\text{-IE}_{x_1, x_0}(y\mid x) &= \sum_u [y_{x_1, W_{x_0}}(u) - y_{x_1}(u)]P(u\mid x) \\ 
        \label{eq:ctfsedown}
     x\text{-SE}_{x_0, x_1}(y) &= \sum_u y_{x_0}(u)[P(u\mid x_1) - P(u \mid x_0)].
\end{align}
To simplify the notation and the comparison with the measures discussed earlier, 
we re-write them as factual and counterfactual contrasts, namely:
    \begin{align}
    \label{eq:ettdown-basis}
        x\text{-TE}_{x_0, x_1}(y\mid x) &= \mathcal{C}(x_0, x_1, x, x) \\
        x\text{-DE}_{x_0, x_1}(y\mid x) &=  \mathcal{C}(x_0, \{x_1, W_{x_0}\}, x, x) \\
        x\text{-IE}_{x_1, x_0}(y\mid x) &=  \mathcal{C}(x_1, \{x_1, W_{x_0}\}, x, x)\\
        x\text{-SE}_{x_0, x_1}(y) &=  \mathcal{C}(x_0, x_0, x_1, x_0).
    \end{align}
Secondly, we will consider each of the measures individually. Starting with the $x$-TE, we note that it is simply a conditional version of the total effect (TE) for the subset of units $\mathcal{U}$ in which $X(u) = x$. This can be easily seen by comparing the contrast representation of the TE (Eq.~\ref{eq:te-basis}) versus the $x$-TE (Eq.~\ref{eq:ettdown-basis}), namely: 
\begin{eqnarray*}
	x\text{-TE}_{x_0, x_1}(y\mid x) =& \mathcal{C}(x_0, x_1, \underline{x, x})   \\ 
 \text{TE}_{x_0, x_1}(y) =& \mathcal{C}(x_0, x_1, \underline{\emptyset, \emptyset}), 
\end{eqnarray*}
which make it obvious that the former has $E_0 = E_1 = \emptyset$, whereas the latter  has $E_0 = E_1 = x$. Both measures, however, use the same counterfactual clauses $C_0 = x_0$ and $C_1 = x_1$. In terms of the sampling-evaluation process discussed earlier, even though these measures evaluate each unit in the same way (due to the same counterfactual clauses), the TE draws units at random from the population, while the $x$-TE filters them out based on $X$'s particular instantiation. The graphical visualization of the ETT is shown in Fig.~\ref{fig:graphETT} and can be compared with that of TE in Fig.~\ref{fig:graphTE}, for grounding the intuition. In words, note that the downstream effect of $X$ on $Y$ is the same, but now $Z$ is no longer disconnected from $X$, but varies in accordance to the event $X=x$. As we will show later on, in the startup hiring example (Ex.~\ref{ex:NDEfail}), the gender will lead to an additional source of information about age, which can be use in the measure. 

Thirdly, the counterfactual measures of direct and indirect effects, $x$-DE and $x$-IE, are conditional versions of the NDE and NIE, respectively. These observations are also reflected in Eqs.~\ref{eq:ctfdedown}-\ref{eq:ctfiedown}, in which the only difference compared to the general population measures is in the posterior weighting term $P(u\mid x)$, while for the NDE and NIE the weighting term is simply $P(u)$ (Eqs.~\ref{eq:ndedown}-\ref{eq:niedown}).  One difference relative to the natural DE and IE is that here a reference value, $X=x$, needs to be picked such that the baseline population can be selected. For instance, in the context of comparing the direct effect on $Y$ from transitioning $X$ from $x_0$ to $x_1$, one could more naturally set the baseline population to $X=x_0$. 

Fourthly, we consider the $x$-SE and its graphical representation, as shown in Fig.~\ref{fig:graphCtfSE}. This quantity also generalizes that of Exp-SE$_x(y)$ shown in Fig.~\ref{fig:graphExpSE}. The difference between these two quantities is in the weighting term, where $P(u) - P(u \mid x)$ in Exp-SE$_x(y)$ is replaced by $P(u\mid x_1) - P(u\mid x_0)$ in $x$-SE$_{x_0,x_1}(y)$. Despite its innocent appearance, this a substantive difference since the Exp-SE entails a comparison between the observational and interventional distributions, while $x$-SE is a purely counterfactual measure. \footnote{In terms of the Pearl Causal Hierarchy (PCH, for short), the quantity Exp-SE entails assumptions only relative to associational  and experimental quantities (PCH's layers 1 and 2), while the $x$-SE requires substantively stronger assumptions regarding counterfactuals (layer 3). For a more detailed discussion on that matter, refer to \citep{bareinboim2020on}.} 

After all, we can finally state the main result of this section, namely, that the quantities $x$-\{DE, IE, SE\} solve the $\fpcfa$. 
\begin{theorem}[$x$-specific $\fpcfa$ solution] \label{thm:fpcfa-1st-sol2}
The total variation measure can be decomposed as
    \begin{align}
        \text{TV}_{x_0, x_1}(y) &= {x\text{-DE}_{x_0, x_1}(y\mid x_0)} - {x\text{-IE}_{x_1, x_0}(y\mid x_0)} - {x\text{-SE}_{x_1, x_0}(y)}.
    \end{align}
Further, the measures x-\{DE, IE, SE\} are admissible w.r.t. Str-DE, Str-IE, Str-SE, respectively. Moreover, the counterfactual family is more powerful than NDE, NIE, and Exp-SE, respectively.
More formally, the admissibility relations can be written as:
\begin{align}
        \text{Str-DE-fair} &\admarrow x\text{-DE-fair} \label{eq:admctfde}\\
        \text{Str-IE-fair} &\admarrow x\text{-IE-fair} \label{eq:admctfie}\\
        \text{Str-SE-fair} &\admarrow x\text{-SE-fair}, \label{eq:admctfse}
\end{align}
and the power relations as:
\begin{align}
          x\text{-DE-fair} &\powarrow \text{NDE-fair}, \\
          x\text{-IE-fair} &\powarrow \text{NIE-fair}, \\
          x\text{-SE-fair} &\powarrow \text{Exp-SE-fair}.
\end{align}
Therefore, the measures $(\mu_{DE}, \mu_{IE}, \mu_{SE}) = (x\text{-DE}_{x_0, x_1}(y), x\text{-IE}_{x_1, x_0}(y), x\text{-SE}_{x_0, x_1}(y))$ solve the $\fpcfa$.
\end{theorem}

Similarly to the discussion in the general-population measures (i.e., $P(u)$), the significance, and practical implications of Thm.~\ref{thm:fpcfa-1st-sol2} appears through the Eq.~\ref{eq:admctfde}'s contrapositive (and Eqs.~\ref{eq:admctfie}, \ref{eq:admctfse}), i.e.:
\begin{equation}
    (x\text{-DE}_{x_0, x_1}(y) \neq 0) \implies  \neg \text{Str-DE-fair}. 
\end{equation}
Based on this, we have now a principled way of testing the following hypothesis:
\begin{align}
    H_0: x\text{-DE}_{x_0, x_1}(y) = 0.
\end{align}
If the $H_0$ hypothesis is rejected, the fairness analyst can conclude that direct discrimination is present in the dataset. Naturally, similar tests can be performed regarding the indirect and spurious structural measures. 

\begin{example}[Revisiting the Startup's hiring \& NDE lack of power] \label{ex:ndefail-fix-x}
Consider the SCM $\mathcal{M}$ given in Eq.~\ref{eq:ndefail1}-\ref{eq:ndefail2}. For $X = x_0$ we compute the $x$-specific direct effects as:
    \begin{align}
     x\text{-DE}_{x_0, x_1}(y \mid x_0) &= P(y_{x_1, W_{x_0}} \mid x_0) - P(y_{x_0}\mid x_0) \\
                                        &= P(\text{Bernoulli}(\frac{1}{5}(1-Z) + \frac{1}{6}W) = 1 \mid x_0)\\ 
                                        & - P(\text{Bernoulli}(\frac{1}{5}(Z) + \frac{1}{6}W) = 1 \mid x_0) \\
                                        &= \sum_{z \in \lbrace 0,1\rbrace} \sum_{w \in \lbrace 0,1\rbrace} P(w)P(z\mid x_0)[\frac{1}{5}(1-2z) + \frac{1}{6}w - \frac{1}{6}w] \\
                                        &=  \sum_{z \in \lbrace 0,1\rbrace} \frac{1}{5}(1-2z) P(z\mid x_0) = 0.036.
    \end{align}
In words, when considering female applicants ($X = x_0$), they are 3.6\% less likely of being hired than they would be, had they been male. In other words, direct discrimination is certainly present in the hiring process of the startup company. $\hfill \square$
\end{example}




\subsubsection{$Z$-specific Contrasts - $P(u\mid z)$}
One might also be interested in capturing discrimination for a specific subset of $\mathcal{U}$ for which $Z(u) = z$
similarly as for the $x$-specific measures. Here, we will consider two possibilities in terms of sub-population selection, first when event $Z(u) = z$ and then when $Z(u) = z, X(u) = x$. Before introducing the corresponding $z$- and $(x,z)$-specific quantities, we clarify one major difference compared to the general and $x$-specific case, namely in the spurious effects. As noted in Sec.~\ref{foundations}, spurious effects are captured by factual contrasts of the form
\begin{align}
    P(y_{x} \mid E_1) - P(y_{x} \mid E_0) = \sum_u y_x(u)[P(u \mid E_1) - P(u\mid E_0)],
\end{align}
which rely on comparing different units corresponding to events $E_0, E_1$. These spurious effects represent variations that causally precede $X$ and $Y$. Interestingly enough, under the assumptions of the SFM (Sec.~\ref{sec:SFM}), conditioning on $Z(u) = z$ closes all backdoor paths between $X$ and $Y$. In other words, fixing $Z$ also fixes the possible spurious variations, and therefore on a $z$- or $(x,z)$-specific level spurious effects are always equal to zero\footnote{Experienced readers might notice, in the presence of unobserved confounders (UCs), we could have more explicitly defined the corresponding, $z$-, $(x, z)$-specific notions
\begin{align}
    z\text{-SE}_x(y) &= P(y \mid x, z) - P(y_x \mid z),\\
    (x,z)\text{-SE}_{x_0,x_1}(y) &= P(y_x \mid x_1, z) - P(y_x \mid x_0, z).
\end{align}
Naturally, this would account for the spurious variations brought about by the UCs. For a more comprehensive treatment of these issues, we refer readers to Sec.~\ref{DIBN}.}. 
Therefore, we can consider the following measures:
\begin{definition}[$z$- and $(x, z)$-specific TE, DE, and IE]
    The $z$-specific total, direct and indirect effects are defined as
    \begin{align}
        z\text{-TE}_{x_0, x_1}(y\mid z) &= P(y_{x_1} \mid z) - P(y_{x_0}\mid z) \\
        z\text{-DE}_{x_0, x_1}(y\mid z) &= P(y_{x_1, W_{x_0}} \mid z) - P(y_{x_0}\mid z) \\
        z\text{-IE}_{x_1, x_0}(y\mid z) &= P(y_{x_1, W_{x_0}}\mid z) - P(y_{x_1} \mid z) \\
        (x,z)\text{-TE}_{x_0, x_1}(y\mid z) &= P(y_{x_1} \mid x, z) - P(y_{x_0}\mid x, z) \\
        (x,z)\text{-DE}_{x_0, x_1}(y\mid z) &= P(y_{x_1, W_{x_0}} \mid x, z) - P(y_{x_0}\mid x, z) \\
        (x,z)\text{-IE}_{x_1, x_0}(y\mid z) &= P(y_{x_1, W_{x_0}}\mid x, z) - P(y_{x_1} \mid x, z).
    \end{align}
\end{definition}
As before, the measures can be factorized using the corresponding unit-level outcomes:
\begin{align}
    z\text{-TE}_{x_0, x_1}(y\mid z) &= \sum_u [y_{x_1}(u) - y_{x_0}(u)]P(u\mid z) \label{eq:zte}\\
    z\text{-DE}_{x_0, x_1}(y\mid z) &= \sum_u [y_{x_1, W_{x_0}}(u) - y_{x_0}(u)]P(u\mid z) \label{eq:zde}\\
    z\text{-IE}_{x_1, x_0}(y\mid z) &= \sum_u  [y_{x_1, W_{x_0}}(u) - y_{x_1}(u)]P(u\mid z) \label{eq:zie}
\end{align}
\begin{align}
	(x,z)\text{-TE}_{x_0, x_1}(y\mid z) &= \sum_u [y_{x_1}(u) - y_{x_0}(u)]P(u\mid x, z) \label{eq:xzte}\\
    (x,z)\text{-DE}_{x_0, x_1}(y\mid z) &= \sum_u [y_{x_1, W_{x_0}}(u) - y_{x_0}(u)]P(u\mid x,z) \label{eq:xzde} \\
    (x,z)\text{-IE}_{x_1, x_0}(y\mid z) &= \sum_u  [y_{x_1, W_{x_0}}(u) - y_{x_1}(u)]P(u\mid x,z). \label{eq:xzie}
\end{align}
These quantities can also be represented more explicitly as contrasts:
\begin{align}
    z\text{-TE}_{x_0, x_1}(y\mid z) &=  \mathcal{C}(x_0, x_1, z, z)\\ 
    z\text{-DE}_{x_0, x_1}(y\mid z) &=  \mathcal{C}(x_0, \{x_1, W_{x_0}\}, z, z)\\ 
    z\text{-IE}_{x_1, x_0}(y\mid z) &=  \mathcal{C}(x_1, \{x_1, W_{x_0}\}, z, z) 
\end{align}
\begin{align}
	(x,z)\text{-TE}_{x_0, x_1}(y\mid z) &=  \mathcal{C}(x_0, x_1, \{x,z\}, \{x,z\})\\ 
    (x,z)\text{-DE}_{x_0, x_1}(y\mid x, z) &=  \mathcal{C}(x_0, \{x_1, W_{x_0}\}, \{x,z\}, \{x,z\}) \\
    (x,z)\text{-IE}_{x_1, x_0}(y\mid x, z) &=  \mathcal{C}(x_1, \{x_1, W_{x_0}\}, \{x,z\}, \{x,z\}).
\end{align}
The $z$-TE, $z$-DE, and $z$-IE (and similarly the $(x,z)$- counterparts) are simply conditional versions of TE, NDE, and NIE respectively, restricted to the subpopulation of $\mathcal{U}$ such that $Z(u) = z$ (or $Z(u) = z, X(u) = x$), which is reflected in the posterior weighting term which becomes $P(u\mid z)$ (or $P(u\mid x,z)$). 

Several important remarks are due. Using the sampling of units analogy from before, we notice that $z$-specific effects filter on units which have $Z(u) = z$, which means they provide us with a more refined lens for detecting discrimination than the general population measures. Similarly, the $(x, z)$-specific measures can be seen as additionally filtering the units on $Z(u) = z$, after they were filtered based on $X(u) = x$, which is precisely what $x$-specific measures have done. Therefore, $(x,z)$-specific measures can be seen as more refined than $x$- and $z$- specific ones. The only uncertainty left in terms of power is about comparing $x$-specific and $z$-specific measures.

Interestingly, under the SFM, the $(x,z)$-specific measures are equal to the $z$-specific measures. This result cannot be deduced from the structural basis expansions above (Eq.~\ref{eq:zde}-\ref{eq:xzie}), but requires the assumptions encoded in the SFM (namely the absence of backdoor paths from $X$ to $Y$ conditional on $Z$). This equivalence of $z$- and $(x,z)$-specific measures under the SFM shows that $z$-specific measures are in fact more powerful than the $x$-specific ones, although this need not be the case in general. Following this discussion, we are ready to present the main result regarding the measures introduced above (while, as discussed earlier, for the spurious effects we rely on the general and $x$-specific notions):
\begin{theorem}[$z$-specific $\fpcfa$ solution] \label{thm:fpcfa-1st-sol3}
The total variation measure can be decomposed as
    \begin{align}
        \text{TV}_{x_0, x_1}(y) &= {\sum_{z} z\text{-DE}_{x_0, x_1}(y\mid z) P(z)} - {\sum_z z\text{-IE}_{x_1, x_0}(y\mid z) P(z)} \nonumber \\
        &\;\;\;\; - (\text{Exp-SE}_{x_0}(y) - \text{Exp-SE}_{x_1}(y))\\
        &= {\sum_{z} (x,z)\text{-DE}_{x_0, x_1}(y\mid x, z) P(z \mid x)} - {\sum_z (x,z)\text{-IE}_{x_1, x_0}(y\mid x, z) P(z\mid x)} \nonumber \\ & \;\;\;\; - {x\text{-SE}_{x_1, x_0}(y)}.
    \end{align}
Further, the measures $z$-DE and $(x,z)$-DE are admissible w.r.t. Str-DE, whereas $z$-IE and $(x,z)$-IE are admissible w.r.t. Str-IE. Moreover, the following power relations hold:
\begin{align}
    (x,z)\text{-DE-fair} &\powarrow z\text{-DE-fair} \powarrow \text{NDE-fair}, \\
    (x,z)\text{-IE-fair} &\powarrow z\text{-IE-fair} \powarrow \text{NIE-fair},
\end{align}
and also
\begin{align}
    (x,z)\text{-DE-fair} &\powarrow x\text{-DE-fair},\\
    (x,z)\text{-IE-fair} &\powarrow x\text{-IE-fair}.
\end{align}
Additionally, under the SFM, we can say that:
\begin{align}
    z\text{-DE-fair} &\powarrow x\text{-DE-fair}, \\
    z\text{-IE-fair} &\powarrow x\text{-IE-fair}.
\end{align}
Therefore, under the SFM, the measures $(\mu_{DE}, \mu_{IE}, \mu_{SE}) = (z\text{-DE}_{x_0, x_1}(y), z\text{-IE}_{x_1, x_0}(y),$ $ x\text{-SE}_{x_0, x_1}(y))$ give a more powerful solution to $\fpcfa$ than the $x$-specific ones.
\end{theorem}
With $z$-specific measures in hand, we revisit Ex.~\ref{ex:NDEfail}, which showed that the NDE can equal $0$ even though direct discrimination exists:
\begin{example}[Revisiting the Startup's hiring \& NDE lack of power]
    Consider the SCM $\mathcal{M}$ given in Eq.~\ref{eq:ndefail1}-\ref{eq:ndefail2}. For $Z = 0$ we compute the $z$-specific direct effects as:
    \begin{align}
        z\text{-DE}(y \mid Z=0) &= P(y_{x_1, W_{x_0}} \mid Z = 0) - P(y_{x_0}\mid Z = 0) \\
                                &= P(\text{Bernoulli}(\frac{1}{5}(1-Z) + \frac{1}{6}W) = 1 \mid Z = 0)\\ & - P(\text{Bernoulli}(\frac{1}{5}(Z) + \frac{1}{6}W) = 1 \mid Z = 0)\nonumber\\
                                &= \sum_{w \in \lbrace 0,1\rbrace} P(w)[\frac{1}{5} + \frac{1}{6}w - \frac{1}{6}w] = \frac{1}{5}.
    \end{align}
In words, when considering younger applicants ($Z = 0$), females are 20\% less likely to be hired than their male counterparts. $\hfill \square$
\end{example}
Interesting enough, note that the $z$-specific DE is able to detect discrimination in the above example, and finds an even larger disparity transmitted through the direct mechanism compared to the $x$-specific DE measure in Ex.~\ref{ex:ndefail-fix-x}.

\subsubsection{More informative contrasts ($V' \subseteq V$-specific).} 
In case even more detailed measures of fairness are needed, we can consider specific subsets of the observed variables, $V' \subseteq V$. For example, we might be interested in quantifying discrimination for specific units $u$ that correspond to $Z(u) = z, W(u) = w$ (for example quantifying discrimination for a specific age group with a specific level of education). Other choices of $V'$ than $\lbrace Z, W\rbrace$ are possible, but due to a large number of possibilities, we do not cover all of them here. Instead, we define generic $v'$-specific measures for an arbitrary choice of $v'$:
\begin{definition}[$V' \subseteq V$-specific TE, DE and IE]
Let $V' \subseteq V$ be a subset of the observables $V$. For any fixed value of $V' = v'$, we define the $v'$-specific total, direct, and indirect effects as:
    \begin{align}
        v'\text{-TE}_{x_0, x_1}(y\mid v') &= P(y_{x_1} \mid v') - P(y_{x_0}\mid v')\\ 
        v'\text{-DE}_{x_0, x_1}(y\mid v') &= P(y_{x_1, W_{x_0}} \mid v') - P(y_{x_0}\mid v') \\ 
        v'\text{-IE}_{x_1, x_0}(y\mid v') &= P(y_{x_1, W_{x_0}}\mid v') - P(y_{x_1} \mid v').
    \end{align}
\end{definition}
Once more, these measures admit a structural basis expansion and which are written with the corresponding contrasts:
\begin{align}
    v'\text{-DE}_{x_0, x_1}(y\mid v') &= \sum_u [y_{x_1, W_{x_0}}(u) - y_{x_0}(u)]P(u\mid v') = \mathcal{C}(x_0, \{x_1, W_{x_0}\}, v', v')\\
    v'\text{-IE}_{x_1, x_0}(y\mid v') &= \sum_u  [y_{x_1, W_{x_0}}(u) - y_{x_1}(u)]P(u\mid v') = \mathcal{C}(x_1, \{x_1, W_{x_0}\}, v', v').
\end{align}
Similarly as in the $z$-specific case, the notion of a spurious effect is lacking whenever $Z \subseteq V'$, so once again we rely on previously developed notions of spurious effects. Importantly, the $v'$-specific measures give an even stronger solution to FPCFA than the $z$- or $(x,z)$-specific measures:
\begin{theorem}[$v'$-specific $\fpcfa$ solution] \label{thm:fpcfa-1st-sol4}
Suppose $V' \subseteq V$ is a subset of the observables that contains both $X$ and $Z$. The total variation measure can be decomposed as
    \begin{align}
        \text{TV}_{x_0, x_1}(y) &= {\sum_{v'} v'\text{-DE}_{x_0, x_1}(y\mid v') P(v' \mid x)} - {\sum_{v'} v'\text{-IE}_{x_1, x_0}(y\mid v') P(v'\mid x)} - {x\text{-SE}_{x_1, x_0}(y)}.
    \end{align}
Further, the measures $v'$-\{DE, IE\} are admissible w.r.t. Str-DE, Str-IE, respectively. Moreover, the $v'$-specific family is more powerful than the $(x,z)$-specific, namely:
\begin{align}
    v'\text{-DE-fair} &\powarrow (x,z)\text{-DE-fair},\\ 
    v'\text{-IE-fair} &\powarrow (x,z)\text{-IE-fair}. 
\end{align}
Therefore, the measures $(\mu_{DE}, \mu_{IE}, \mu_{SE}) = (v'\text{-DE}_{x_0, x_1}(y), v'\text{-IE}_{x_1, x_0}(y), x\text{-SE}_{x_0, x_1}(y))$ give a more powerful solution to $\fpcfa$ than the $z$- or $(x,z)$-specific ones.
\end{theorem}
The next example illustrates why having more flexible, $v'$-specific measures can be informative, and therefore useful in some practical   settings. 
\begin{example}[Startup hiring - Version II]
    A startup company is hiring employees. Let $X \in \{ x_0,x_1 \}$ denote female and male applicants respectively. The employment decision $Y \in \{ 0, 1\}$ is based on gender and education level $W$. The SCM $\mathcal{M}$ is given by:
    \begin{align}
    \label{eq:vspecfirst}
    X &\gets \text{Bernoulli}(0.5)\\
    W &\gets \mathcal{N}(14, 4)  \\
    \label{eq:vspeclast}
    Y &\gets \text{Bernoulli}\big(0.1 + \frac{W}{50} + 0.1 * X* \mathbb{1}(W < 20)\big).
\end{align}
Since there are no confounders ($Z = \emptyset$), general, $x$-specific and $z$-specific effects are all equal:
\begin{align}
    \text{NDE}_{x_0, x_1}(y) = x\text{-DE}_{x_0, x_1}(y \mid x) = z\text{-DE}_{x_0, x_1}(y \mid z) = 9.2\%.
\end{align}
Therefore, there is clearly direct discrimination against female employees by the company. 

The company argues in the legal proceedings that in the high-tech industry, they are mostly concerned with highly educated individuals. In words, they should be asked whether they discriminate highly educated female applicants, which is represented through the quantity $    w\text{-DE}_{x_0, x_1}(y \mid w > 20)$. This number can be computed as follows: 
\begin{align}
    w\text{-DE}_{x_0, x_1}(y \mid w > 20) = 0\%,
\end{align}
In words, the company claim was accurate since highly educated individuals were not discriminated against. $\hfill \square$
\end{example}
What the example shows is that $v'$-specific measures can sometimes capture aspects of discrimination that otherwise cannot be quantified using general, $x$-specific, or $z$-specific measures. 

\paragraph{Probabilities of causation.} Remarkably, the $v'$-specific measures carry a fundamental connection to what is known in the literature as \textit{probabilities of causation} \cite[Ch.~9]{pearl:2k}. For example, by picking event $v' = \lbrace x_0, y_0\rbrace$, the measure $v'$-TE becomes
\begin{equation} \label{eq:xyte-deriv}
    (x, y)\text{-TE}_{x_0, x_1}(y \mid x_0, y_0) = P(y_{x_1} \mid x_0, y_0) - P(y_{x_0} \mid x_0, y_0), 
\end{equation}
where $Y=y$ is a shortcut to $Y=1$. First, note that  $P(y_{x_0} \mid x_0, y_0) = P(y \mid x_0, y_0)$, since by the consistency axiom $Y = Y_{x_0}$ whenever $X=x_0$. Obviously,  $P(y \mid x_0, y_0) =0$ since $y_0 \neq 1$. Putting these together, the r.h.s. of Eq.~\ref{eq:xyte-deriv} can be re-written as
\begin{equation}
    (x, y)\text{-TE}_{x_0, x_1}(y \mid x_0, y_0) = P(y_{x_1} \mid x_0, y_0),
\end{equation}
which is known as the \textit{probability of sufficiency} \citep[Def.~9.2.2]{pearl:2k}. The measure computes the probability that a change in attribute from $X = x_0$ to $X = x_1$ produces a change in outcome from $Y = y_0$ to $Y = y_1$, or, in words, how much $X$'s value is ``sufficient" to produce $y_1$. Along similar lines, $v'$-TE for the event $v'=\lbrace x_1, y_1\rbrace$ can be written as
\begin{align}
    (x, y)\text{-TE}_{x_0, x_1}(y \mid x_1, y_1) &= P(y_{x_1} \mid x_1, y_1) - P(y_{x_0} \mid x_1, y_1)\\
                                                 &= 1 - P(y_{x_0} \mid x_1, y_1)\\ 
                                                 &= P(y_{x_0} = 0 \mid x_1, y_1),
\end{align}
which is known as the \textit{probability of necessity} \citep[Def.~9.2.1]{pearl:2k}. The second line of the derivation followed since by the consistency axiom, $Y_x1 = Y$, and also the fact that $Y=1$ in the factual world. The measures computes the probability that a change in attribute from $X = x_1$ to $X = x_0$ produces a change in outcome from $Y = y_1$ to $Y = y_0$, or how $X$'s value is ``necessary" to produce $y_1$. These two types of variations usually appear together and may be modeled through what is known as the probability of necessity and sufficiency (PNS). We refer readers to \cite[Ch.~9]{pearl:2k} for further discussion. 

\subsubsection{Unit-level Contrasts - $\delta_u$} 
Finally, the most powerful measures to consider are unit-level measures, as defined next:
\begin{definition}[Unit-level TE, DE, and IE] Given a unit $U = u$, the unit-level total, direct, and indirect effects are given by
    \begin{align}
        u\text{-TE}_{x_0, x_1}(y(u)) &= y_{x_1}(u) - y_{x_0}(u) = \mathcal{C}(x_0, x_1, u, u)\\
        u\text{-DE}_{x_0, x_1}(y(u)) &= y_{x_1, W_{x_0}}(u) - y_{x_0}(u) = \mathcal{C}(x_0, \{x_1, W_{x_0}\}, u, u)\\
        u\text{-IE}_{x_1, x_0}(y(u)) &= y_{x_1, W_{x_0}}(u) - y_{x_1}(u) = \mathcal{C}(x_1, \{x_1, W_{x_0}\}, u, u).
    \end{align}
\end{definition}
For unit-level measures the posterior distribution that is used as a weighting term is $\delta_u$, where $\delta$ is the Dirac delta function. The unit-level measures can be seen as the canonical basis under which all other measures are expanded. They also give the strongest theoretical solution to the FPCFA, once again, with the help of $x$-specific spurious effect developed earlier:
\begin{theorem}[unit-level $\fpcfa$ solution] \label{thm:fpcfa-sol-u}
The total variation measure can be decomposed as
    \begin{align}
        \text{TV}_{x_0, x_1}(y) &= {\sum_{u} u\text{-DE}_{x_0, x_1}(y(u)) P(v' \mid x)} - {\sum_{u} u\text{-IE}_{x_1, x_0}(y(u)) P(u\mid x)} - {x\text{-SE}_{x_1, x_0}(y)}.
    \end{align}
Further, the measures $u$-\{DE, IE\} are admissible w.r.t. Str-DE, Str-IE, respectively. Moreover, the $u$-specific family is more powerful than the $v'$-specific, namely:
\begin{align}
    u\text{-DE-fair} &\implies v'\text{-DE-fair}, \\
    u\text{-IE-fair} &\implies v'\text{-IE-fair}.
\end{align}
Therefore, the measures $(\mu_{DE}, \mu_{IE}, \mu_{SE}) = (u\text{-DE}_{x_0, x_1}(y), u\text{-IE}_{x_1, x_0}(y), x\text{-SE}_{x_0, x_1}(y))$ give the most powerful solution to $\fpcfa$.
\end{theorem}

The unit-level measures represent the most refined level at which  discrimination can be described. In fact, introducing these measures also brings us to the final level of the population axis of the explainability plane (Fig.~\ref{fig:Erefined}). Recall, the population axis ranges from the general population measures (with a posterior $P(u)$), all the way to the deterministic measures which consider a single unit (with a posterior $\delta_u$), eliciting a range of measures which may be useful for fairness analysis. We next move onto giving a systematic overview of the TV-family of measures that was introduced in this section.

\subsection{Summary of the TV-family \& the Fairness Map} \label{sec:fairmap}

To facilitate comparison and understanding after introducing the measures of the TV-family, we show next how they can be more explicitly written as contrasts:

\begin{lemma}[TV family as contrasts] \label{lemma:TVfamily}
    The TV-family of causal fairness measures is a collection of contrasts $\mathcal{C}(C_0, C_1, E_0, E_1)$ (Def.~\ref{def:contrast}) that follow the specific instantiations of counterfactual and factual clauses, $C_0, C_1, E_0, E_1$, as described in Table~\ref{table:tv-family}.
\end{lemma}

\begin{wraptable}{r}{6.4cm}
		\begin{tabular}{|M{0.15cm}||M{1.75cm}|M{0.3cm}|M{1.2cm}|M{0.3cm}|M{0.3cm}|}
            \hline
			& Measure & $C_0$ & $C_1$ & $E_0$ & $E_1$   \\[1mm] 	\hline \hline
			\multirow{5}{*}{\rotatebox{90}{general}} & TV$_{x_0, x_1}$ & $\emptyset$ & $\emptyset$ & $x_0$ & $x_1$ \\\cline{2-6}			
			& Exp-SE$_{x}$ & $x$ & $x$ & $\emptyset$ & $x$ \\\cline{2-6}
			& TE$_{x_0, x_1}$ & $x_0$ & $x_1$ & $\emptyset$ & $\emptyset$ \\\cline{2-6}
			& NDE$_{x_0, x_1}$ & $x_0$ & $x_1, W_{x_0}$ & $\emptyset$ & $\emptyset$ \\\cline{2-6}
			& NIE$_{x_0, x_1}$ & $x_0$ & $x_0, W_{x_1}$ & $\emptyset$ & $\emptyset$ \\\cline{2-6}
            \hline
            
		     \multirow{4}{*}{\rotatebox{90}{$X = x$}} & $x$-TE$_{x_0, x_1}$ & $x_0$ & $x_1$ & $x$ & $x$ \\\cline{2-6}
            & $x$-SE$_{x_0, x_1}$ & $x_0$ & $x_0$ & $x_0$ & $x_1$ \\\cline{2-6}
            & $x$-TE$_{x_0, x_1}$ & $x_0$ & $x_1$ & $x$ & $x$ \\\cline{2-6}
			& $x$-DE$_{x_0, x_1}$ & $x_0$ & $x_1, W_{x_0}$ & $x$ & $x$ \\\cline{2-6}
			& $x$-IE$_{x_0, x_1}$ & $x_0$ & $x_0, W_{x_1}$ & $x$ & $x$ \\\cline{2-6}
			\hline
			
			\multirow{3}{*}{\rotatebox{90}{$Z = z$}} & $z$-TE$_{x_0, x_1}$ & $x_0$ & $x_1$ & $z$ & $z$ \\\cline{2-6}
			& $z$-DE$_{x_0, x_1}$ & $x_0$ & $x_1, W_{x_0}$ & $z$ & $z$ \\\cline{2-6}
			& $z$-IE$_{x_0, x_1}$ & $x_0$ & $x_0, W_{x_1}$ & $z$ & $z$ \\\cline{2-6}
			\hline
			
			\multirow{3}{*}{\rotatebox{90}{$V' \subseteq V$}} & $v'$-TE$_{x_0, x_1}$ & $x_0$ & $x_1$ & $v'$ & $v'$ \\\cline{2-6}
			\multirow{3}{*}{\rotatebox{90}{}} & $v'$-TE$_{x_0, x_1}$ & $x_0$ & $x_1$ & $v'$ & $v'$ \\\cline{2-6}
			& $v'$-DE$_{x_0, x_1}$ & $x_0$ & $x_1, W_{x_0}$ & $v'$ & $v'$ \\\cline{2-6}
			& $v'$-IE$_{x_0, x_1}$ & $x_0$ & $x_0, W_{x_1}$ & $v'$ & $v'$ \\\cline{2-6}
			\hline
			
			\multirow{3}{*}{\rotatebox{90}{unit}} & $u$-TE$_{x_0, x_1}$ & $x_0$ & $x_1$ & $u$ & $u$ \\\cline{2-6}
			\multirow{3}{*}{\rotatebox{90}{}} & $u$-TE$_{x_0, x_1}$ & $x_0$ & $x_1$ & $u$ & $u$ \\\cline{2-6}

			& $u$-DE$_{x_0, x_1}$ & $x_0$ & $x_1, W_{x_0}$ & $u$ & $u$ \\\cline{2-6}
			& $u$-IE$_{x_0, x_1}$ & $x_0$ & $x_0, W_{x_1}$ & $u$ & $u$ \\\cline{2-6}

			\hline
		\end{tabular}
			    \caption{Measures of fairness in the TV-family. TE stands for total effect, Exp experimental, SE spurious, N natural, DE direct, IE indirect, $v'$ for an event $V' = v'$, where $V' \subseteq V$.}
		\label{table:tv-family}
\vspace{-0.985in}
\end{wraptable}

A few things are worth noting relative to this taxonomy. First, The measures are grouped in five categories, based on the granularity of the events $E_0, E_1$. For each of the contrasts, we define a criterion based on the resulting measure. Namely, we say $Y$ is fair with respect to $X$ in the $x$-TE measure if $x$-TE$_{x_0, x_1}(y \mid x) = 0 \; \forall x$. We write $x$-TE-fair$_X(Y)$ for this condition, or $x$-TE-fair, for short. 

Further note that Table~\ref{table:tv-family} has a distinct structure. In fact, the contrasts corresponding to TE, DE, and IE measures have repeating (equal) counterfactual clauses $C_0$ and $C_1$, whereas the conditioning event $E$ changes. Mathematically, the measures in the table, but for the spurious effects, can be written more succinctly as
\begin{align}
    \begin{rcases}
    \begin{dcases}
    E\text{-TE}_{x_0, x_1}(y\mid E) &= \mathcal{C}(x_0, x_1, E, E)\\
    E\text{-DE}_{x_0, x_1}(y\mid E) &= \mathcal{C}(x_0, \{x_1, W_{x_0}\}, E, E)\\
    E\text{-IE}_{x_0, x_1}(y\mid E) &= \mathcal{C}(x_0, \{x_0, W_{x_1}\}, E, E)
    \end{dcases}
    \end{rcases} \nonumber \\ \text{ for } E \in \{ \emptyset, x, z, v', u\}.
\end{align}

\noindent Apart from the overarching structure underlying the measures, as described in Table~\ref{table:tv-family}, there is more structure across them as delineated in the next result, which comes under the rubric of the fairness map. 
\begin{theorem}[Fairness Map]\label{thm:map}
	The total variation (TV) family of causal measures of fairness admits a number of relations of decomposability, admissibility, and power, which are represented in what we call the Fairness Map, as shown in Fig.~\ref{fig:map}.
\end{theorem}

In words, the measures of the TV family satisfy an entire hierarchy of relations in terms of the properties discussed so far, namely,  admissibility, decomposability, and power. This hierarchy is one of the main results of this manuscript. There are several observations worth making at this point. First, each arrow in Fig.~\ref{fig:map} corresponds to an implication, and the full and more syntactic version of the map is provided in the Appendix \ref{appendix:map}, including the proofs. 
There are different ways of reading the map, and perhaps the most natural one is to navigate along the two axes, mechanisms and population, which match the dimensions discussed earlier in the explainability plane (Fig.~\ref{fig:Erefined}/Sec.~\ref{sec:explain-variations}). 

First, note that the mechanism's axis is partitioned into two. First, there are the elementary structural fairness criteria (Def.~\ref{def:str-fair}), in which each represent a different type of mechanism. Second, there are the composite measures, which are the ones that are usually readable from the data. More prominently, the causal effects are marked in gray, which is also known as the total effects, and the total variation is shown on the left-top corner. 

In a complementary way, the population axis can also be partitioned as well. First, there are the structural measures (below the blue-dotted line), which are computable from the true SCM $\mathcal{M}$ and almost always unobservable. On the other side, there are the ``empirical'' measures (above the blue line), which are possibly computable depending on the combination of data and assumptions about the underlying generative processes. 

Given this initial acknowledgment, we note this is a rough characterization, and then navigate through the axis in a more detailed manner, along each of the axis separately.

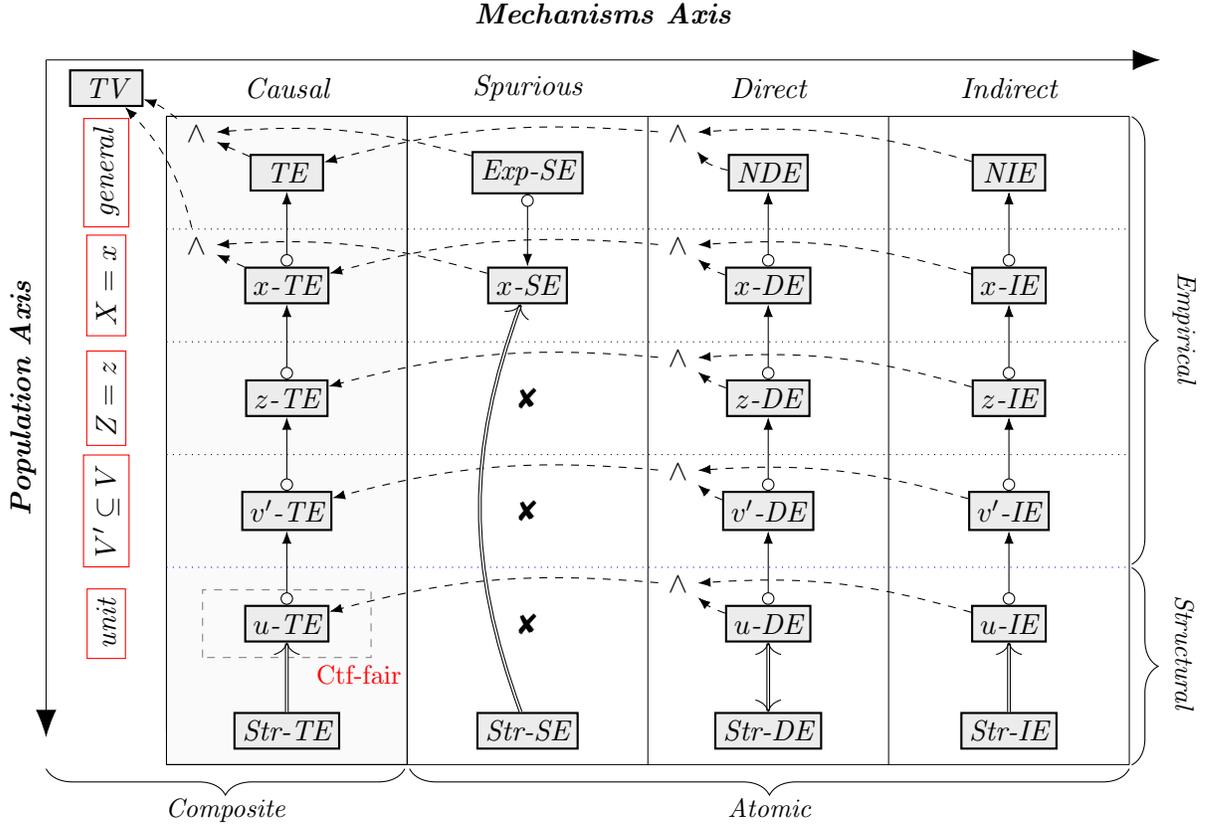
\begin{figure}[t]
	\centering
	\begin{tikzpicture}[fairmapfull]
	
        (0,0) -- node [black,above, yshift = 7.5] 
        {\small Composite} (2.98 * \myx, 0) ;
        (3.02 * \myx ,0) -- 
         node [black,above, yshift = 7.5] 
              {\small Atomic} 
        (9 * \myx, 0) ;

	    \draw [decorate,decoration={brace,mirror,amplitude=10pt}, yshift=2pt]
        (0,-9.5) -- 
         node [black,below, yshift = -7.5] 
              {\small Composite} 
        (2.98 * \myx, -9.5) ;

		\draw [decorate,decoration={brace,mirror,amplitude=10pt}, yshift=2pt]
        (3.02 * \myx ,-9.5) -- 
         node [black,below, yshift = -7.5] 
             {\small Atomic} 
        (9 * \myx, -9.5) ;

	    \draw [decorate,decoration={brace,amplitude=10pt}, yshift=2pt]
        (9.02 * \myx, 4.55 * \myy) -- 
         node [black, left, rotate=-90, xshift = 20, yshift = 20] {\small Structural} 
        (9.02 * \myx, 6.3 * \myy ) ;

	    \draw [decorate,decoration={brace,amplitude=10pt}, yshift=2pt]
        (9.02 * \myx,  0.55*\myy) -- 
         node [black, left, rotate=-90, xshift = 20, yshift = 20] {\small  Empirical} 
        (9.02 * \myx, 4.52 * \myy ) ;

	    

        \draw [draw=black, fill=gray!4] (\myx,\myy*0.5) rectangle (\myx * 3,\myy*6.25);

		\draw [very thin] (\myx, 0.5 * \myy) -- (9 * \myx, 0.5 * \myy);
		\draw [very thin, -{Latex[length=3.5mm]}] (0, 0) -- node[above, rotate=90]{\textbf{Population Axis}} (0, 6 * \myy);

		\draw [very thin,  -{Latex[length=3.5mm]}] (0, 0) -- node[above, yshift = 10]{\textbf{Mechanisms Axis}} (9.25*\myx, 0);	

		\draw [sepline] (\myx, 1.5 * \myy) -- (9 * \myx, 1.5*\myy);
		\draw [sepline] (\myx,2.5 * \myy) -- (9 * \myx, 2.5*\myy);
		\draw [sepline] (\myx,3.5 * \myy) -- (9 * \myx, 3.5*\myy);
		\draw [sepline, color=blue] (\myx,4.5 * \myy) -- (9 * \myx, 4.5*\myy);
		\draw [very thin] (\myx,6.25 * \myy) -- (9 * \myx, 6.25*\myy);

		\draw [very thin] (\myx,-0.75) -- (\myx, 5.95*\myy);		
		\draw [very thin] (3*\myx,-0.75) -- (3*\myx, 6.25*\myy);
		\draw [very thin] (5*\myx,-0.75) -- (5*\myx, 6.25*\myy);
		\draw [very thin] (7*\myx,-0.75) -- (7*\myx, 6.25*\myy);
		\draw [very thin] (9*\myx,-0.75) -- (9*\myx, 6.25*\myy);
		
		\node [draw=red, rotate=90] at (0.5*\myx, \myy)  {general};
  		\node [draw=red, rotate=90] at (0.5*\myx, 2*\myy)  {$X = x$};
  		\node [draw=red, rotate=90] at (0.5*\myx, 3*\myy)  {$Z = z$};
  		\node [draw=red, rotate=90] at (0.5*\myx, 4*\myy)  {$V' \subseteq V$};
  		\node [draw=red, rotate=90] at (0.5*\myx, 5*\myy)  {unit};
		
		\node [f1] (tv) at (0.5*\myx, 0.25*\myy) {TV};
        
        \node (spurious) at (4 * \myx, 0.25*\myy) {Spurious};
        \node [f1] (exp-se) at (4*\myx, \myy) {Exp-SE};
        \node [f1] (ctf-se) at (4*\myx, 2*\myy) {$x$-SE};
        \node (zse) at (4*\myx, 3*\myy) {\ding{56}};
        \node (v-se) at (4*\myx, 4*\myy) {\ding{56}};
        \node (unit-se) at (4*\myx, 5*\myy) {\ding{56}};
        \node [f1] (se) at (4*\myx, 5.95*\myy) {Str-SE};
        
        \node (causal) at (2 * \myx, 0.25*\myy) {Causal};
        \node [f1] (te) at (2*\myx, \myy) {TE};

        \node [f1] (ett) at (2*\myx, 2*\myy) {$x$-TE};
        \node [f1] (zte) at (2*\myx, 3*\myy) {$z$-TE};
        \node [f1] (vce) at (2*\myx, 4*\myy) {$v'$-TE};
        \node [f1] (uce) at (2*\myx, 5*\myy) {$u$-TE};
        \draw[draw=gray, dashed, label={above left:{Ctf-fair}}] (1.3*\myx, 5.3*\myy) rectangle ++(1.4*\myx, -0.6*\myy);
        \node [text=red, font=\small] at (2.6*\myx, 5.45*\myy) {Ctf-fair};
        \node [f1] (ce) at (2*\myx, 5.95*\myy) {Str-TE};
        
        \node (direct) at (6 * \myx, 0.25*\myy) {Direct};
        \node [f1] (nde) at (6*\myx, \myy) {NDE};
        \node [f1] (ctf-de) at (6*\myx, 2*\myy) {$x$-DE};
        \node [f1] (zde) at (6*\myx, 3*\myy) {$z$-DE};
        \node [f1] (vde) at (6*\myx, 4*\myy) {$v'$-DE};
        \node [f1] (ude) at (6*\myx, 5*\myy) {$u$-DE};
        \node [f1] (de) at (6*\myx, 5.95*\myy) {Str-DE};
        
        \node (indirect) at (8 * \myx, 0.25*\myy) {Indirect};
        \node [f1] (nie) at (8*\myx, \myy) {NIE};
        \node [f1] (ctf-ie) at (8*\myx, 2*\myy) {$x$-IE};
        \node [f1] (zie) at (8*\myx, 3*\myy) {$z$-IE};
        \node [f1] (vie) at (8*\myx, 4*\myy) {$v'$-IE};
        \node [f1] (uie) at (8*\myx, 5*\myy) {$u$-IE};
        \node [f1] (ie) at (8*\myx, 5.95*\myy) {Str-IE};

 		\node (te-se) at (1.25*\myx, 0.65*\myy) {$\wedge$};
 		\node (nde-nie) at (5.25 * \myx, 0.65*\myy) {$\wedge$};
 		\node (ctf-de-ie) at (5.25 * \myx, 1.65*\myy) {$\wedge$};
 		\node (z-de-ie) at (5.25 * \myx, 2.65*\myy) {$\wedge$};
 		\node (v-de-ie) at (5.25 * \myx, 3.65*\myy) {$\wedge$};
 		\node (u-de-ie) at (5.25 * \myx, 4.65*\myy) {$\wedge$};

 		\node (ctf-se-ett) at (1.25*\myx, 1.65*\myy) {$\wedge$};
        
 		\draw [dec] (te) to [bend left = 0] (te-se);
 		\draw [dec] (exp-se) to [bend right = 10] (te-se);
 		
 		\draw [dec] (ett) to [bend right = 0] (ctf-se-ett);
 		\draw [dec] (ctf-se) to [bend right = 10] (ctf-se-ett);

 		\draw [dec] (nde) to [bend left = 20] (nde-nie);
 		\draw [dec] (nie) to [bend right = 10] (nde-nie);
 		
 		\draw [dec] (zde) to [bend left = 10] (z-de-ie);
 		\draw [dec] (zie) to [bend right = 10] (z-de-ie);

 		\draw [dec] (ctf-de) to [bend left = 10] (ctf-de-ie);
 		\draw [dec] (ctf-ie) to [bend right = 10] (ctf-de-ie);
 		
 		\draw [dec] (vde) to [bend left = 10] (v-de-ie);
 		\draw [dec] (vie) to [bend right = 10] (v-de-ie);
 		
 		\draw [dec] (ude) to [bend left = 10] (u-de-ie);
 		\draw [dec] (uie) to [bend right = 10] (u-de-ie);
 		
 		\draw [dec] (te-se) to [bend right = 10] (tv);
 		\draw [dec] (nde-nie) to [bend right = 10] (te);
 		\draw [dec] (z-de-ie) to [bend right = 10] (zte);
 		\draw [dec] (ctf-de-ie) to [bend right = 10] (ett);
 		\draw [dec] (v-de-ie) to [bend right = 10] (vce);
 		\draw [dec] (u-de-ie) to [bend right = 10] (uce);
 		\draw [dec] (ctf-se-ett) to [bend right = 15] (tv);

 		\draw [pow] (exp-se) to [bend left = 0] (ctf-se);
 		\draw [pow] (ctf-de) to [bend right = 0] (nde);
 		\draw [pow] (zde) to [bend right = 0] (ctf-de);
 		\draw [pow] (vde) to [bend left = 0] (zde);
 		\draw [pow] (ude) to [bend left = 0] (vde);
 		\draw [pow] (zie) to [bend right = 0] (ctf-ie);
 		\draw [pow] (ctf-ie) to [bend right = 0] (nie);
 		\draw [pow] (vie) to [bend left = 0] (zie);
 		\draw [pow] (uie) to [bend left = 0] (vie);
 		
 		\draw [pow] (ett) to [bend right = 0] (te);
 		\draw [pow] (zte) to [bend right = 0] (ett);
 		\draw [pow] (vce) to (zte);
 		\draw [pow] (uce) to (vce);
 		
 		\draw [adm] (se) to [bend left = 20] (ctf-se);
		\draw [adm, <->] (de) to (ude);
		\draw [adm, ->] (ie) to (uie);
		\draw [adm] (ce) to (uce);

	\end{tikzpicture}
	\caption{Fairness Map for the TV family of measures. The $x$-axis represent the mechanisms (causal, spurious, direct, and indirect), and the $y$-axis the events that capture increasingly more granular sub-populations, from general $(P(u))$ to unit level, and structural. The arrow $\admarrow$ indicates relations of admissibility,$\powarrow$ of power, and $\decomparrow$ of decomposability.}
	\label{fig:map}
\end{figure}

\paragraph{Population axis (vertical) -- Admissibility \& Power relations.\\} 

\noindent When reading the map vertically, from bottom to top, one can find all power and admissibility relations from Thm.~\ref{thm:fpcfa-1st-sol} to Thm.~\ref{thm:fpcfa-sol-u}. For example, the last column of the map (``indirect'') shows that
\begin{align}
    \text{Str-IE} \admarrow u\text{-IE} \powarrow v'\text{-IE} \powarrow z\text{-IE} \powarrow x\text{-IE} \powarrow \text{NIE}.
\end{align}
In words, this says that
\begin{enumerate}
	\item[(i)] unit IE is \textit{admissible} w.r.t. structural IE;
	\item[(ii)] unit IE is \textit{more powerful than} $v'$-IE, which is \textit{more powerful than} $z$-IE, which is \textit{more powerful than} $x$-IE, which is \textit{more powerful} than NIE;
	\item[(iii)] by transitivity of the admissibility and power relations, it follows that every measure in the column is \textit{admissible} w.r.t. structural IE.
\end{enumerate} 
The other columns of the map can be interpreted in a similar fashion.


\paragraph{Mechanisms axis (horizontal) -- Decomposability relations.\\} 

\noindent When reading the map horizontally, from the right to the left, the  decomposability relations are encoded. For example, consider the first row of the map (``general''), it shows that 
\begin{align}
    \text{TE} \decomparrow& \text{NDE} \land \text{NIE} \\ 
    \text{TV} \decomparrow& \text{TE} \land \text{Exp-SE}, 
\end{align}
In words, this says that 
\begin{enumerate}
	\item[(i)] the total variation (TV) can be decomposed into the total  (TE) and experimental spurious effects (Exp-SE);
	\item[(ii)] the total effect (TE) can further be decomposed into natural direct effect (NDE)  and natural indirect effect (NIE).   
	\item[(iii)] More explicitly, these relations can be combined and written as: 
	\begin{align}
    \text{TV} \decomparrow  \text{NDE} \land \text{NIE} \land \text{Exp-{SE}}.
\end{align}
\end{enumerate}

More strongly, this can be stated for every level of the population axis (i.e., the TE is decomposed into DE and IE at every level), as shown next:

\begin{corollary}[Extended Mediation Formula] \label{cor:fm-extmed}
The total effect admits a decomposition into its direct and indirect parts, at every level of granularity of event $E$ in the Fairness Map in Fig.~\ref{fig:map}.  Formally, we can say that
\begin{align}
    \text{TE}_{x_0, x_1}(y) &= \text{NDE}_{x_0, x_1}(y) - \text{NIE}_{x_1, x_0}(y) \\
    x\text{-TE}_{x_0, x_1}(y \mid x) &= x\text{-DE}_{x_0, x_1}(y \mid x) - x\text{-IE}_{x_1, x_0}(y \mid x) \\
    z\text{-TE}_{x_0, x_1}(y \mid z) &= z\text{-DE}_{x_0, x_1}(y \mid z) - z\text{-IE}_{x_1, x_0}(y \mid z)\\
    v'\text{-TE}_{x_0, x_1}(y \mid v') &= v'\text{-DE}_{x_0, x_1}(y \mid v') - v'\text{-IE}_{x_1, x_0}(y \mid v')\\
    u\text{-TE}_{x_0, x_1}(y(u)) &= u\text{-DE}_{x_0, x_1}(y(u)) - u\text{-IE}_{x_1, x_0}(y(u)).
\end{align}
\end{corollary}

Furthermore, the TV measure admits different expansions into DE, IE, and SE measures (as shown in Thm.~\ref{thm:fpcfa-1st-sol}-\ref{thm:fpcfa-sol-u}). The importance of these decompositions was already stated earlier, as they played a crucial role in solving the decomposability part of the FPCFA. 

In summary, the Fairness Map represents a general, theoretical solution to the FPCFA, and shows how the gap between the observed (TV in the top left of the map) and the structural (bottom of the map) can be bridged from first principles. The map therefore, in principle, closes the problem pervasive throughout the literature, as formalized earlier in this  manuscript.

\subsection{The Identification Problem \& the FPCFA in practice} \label{sec:ID}

The Fairness Map introduced in Thm.~\ref{thm:map} contains various admissible measures w.r.t. to different structural mechanisms. All these measures are well-defined and computable from the underlying data-generating model, the true SCM $\mathcal{M}$. However, $\mathcal{M}$ is not available in practice, which was the very motivation for engaging in the discussions so far, and finding proxies for the structural measures. One key consideration that follows is which of these measures can be computed in practice, given (1) a set of assumptions $\mathcal{A}$ about the underlying $\mathcal{M}$ and (2) data from past decisions generated by $\mathcal{M}$. This question indeed can be seen as a problem of  \textit{identifiability} \cite[Sec.~3.2.4]{pearl:2k}. We formalize this notation considering the context of this discussion. 

\begin{definition}[Identifiability]
Let the true, generative SCM $\mathcal{M} = \langle V, U, P(U), F \rangle$, and a set of assumptions $\mathcal{A}$ and an observational distribution $P(v)$ generated by it. Let $\Omega_{\mathcal{A}}$ the space of all SCMs compatible with $\mathcal{A}$. 
Let $\phi$ be a query that can be computed from $\mathcal{M}$.  The quantity $\phi$ is said to be identifiable from $\Omega_{\mathcal{A}}$ and the observational distribution $P(V)$ if
\begin{align}
    \forall \mathcal{M}_1, \mathcal{M}_2 \in \Omega_{\mathcal{A}}:
 \mathcal{A}^{\mathcal{M}_1} =   \mathcal{A}^{\mathcal{M}_2} \text{ and }    \\P^{\mathcal{M}_1}(V) =  P^{\mathcal{M}_2}(V) \implies \phi(\mathcal{M}_1) = \phi(\mathcal{M}_2).
\end{align}
In words, if any two SCMs agree with the set of assumptions ($\mathcal{A}$) and also generate the same observational distribution ($P(v)$), then they should agree with the answer to the query $\phi$. 
\end{definition}

A query $\phi$ is identifiable if it can be uniquely computed from the combination of qualitative assumptions and empirical data. 
In fact, the lack of identifiability means that one cannot compute the value of $\phi$ from the observational data and set of assumptions, i.e., the gap between the true generative process, $\mathcal{M}$, and the feature that we are trying to obtain from it, $\phi$, is too large, and cannot be bridged through the pair $\langle \mathcal{A}, P(v) \rangle$. In practice, one common way of articulating assumptions about $\mathcal{M}$ is through the use of causal diagrams. Whenever the causal diagram is known, we can then write the following : 
\begin{align}
    \Omega^{\mathcal{G}} = \{ \mathcal{M}: \mathcal{M} \text{ compatible with } \mathcal{G} \}, 
\end{align}
where compatibility is related to sharing the same causal diagram, which encodes qualitative assumptions, following the construction in Def.~\ref{def:diagram}\footnote{For a more formal account of this notion, see discussion on CBNs in \cite[Sec.~1.3]{bareinboim2020on})}.

\begin{example}[(Non-)Identifiability of measures]
Let $\Omega^{\mathcal{G}}$ be the space of SCMs that are compatible with the causal diagram $\mathcal{G}$
\begin{center}
    \begin{tikzpicture}
		[>=stealth, rv/.style={thick}, rvc/.style={triangle, draw, thick, minimum size=7mm}, node distance=18mm]
		\pgfsetarrows{latex-latex};
		\begin{scope}
		\node[rv] (1) at (-2,-0.5) {$X$};
		\node[rv] (2) at (0,-1.5) {$W$};
		\node[rv] (3) at (2,-0.5) {$Y$};
		\draw[->] (1) -- (2);
		\path[->] (1) edge[bend left = 0] (3);
		\draw[->] (2) -- (3);
		\draw[<->, dashed] (2) edge[bend right = 30] (3);
		\end{scope}
\end{tikzpicture} .	
\end{center}
When considering the quantities TE$_{x_0, x_1}(y)$ and NIE$_{x_0, x_1}(y)$ in this context, we can say that:
\begin{enumerate}[label=(\roman*)]
    \item quantity TE$_{x_0, x_1}(y)$ is identifiable over $\Omega^{\mathcal{G}}$,
    \item quantity NIE$_{x_0, x_1}(y)$ is not identifiable over $\Omega^{\mathcal{G}}$.
\end{enumerate}
In fact, for any SCM in $\Omega^{\mathcal{G}}$, we have that     $\text{TE}_{x_0, x_1}(y)$  is equal to 
\begin{align}
P(y \mid x_1) - P(y \mid x_0).
\end{align}
To show that NIE$_{x_0, x_1}(y)$ is not identifiable, consider the following two SCMs:
\begin{numcases}{\mathcal{M}_1 := }
X & $\gets$ $U_X$ \label{eq:id-m1-1st}\\
W & $\gets$ $\mathbb{1}(U_D < 0.2 +0.4 X + 0.4 U_{WY})$  \\
Y & $\gets$ $\mathbb{1}(U_Y < 0.1X + \underline{0.7}W + \underline{0.1}U_{WY})$, \label{eq:id-m1-last}
\end{numcases}
\begin{numcases}{\mathcal{M}_2 := }
X & $\gets$ $U_X$ \label{eq:id-m2-1st}\\
W & $\gets$ $\mathbb{1}(U_D < 0.2 + 0.4 X + 0.4 U_{WY})$  \\
Y & $\gets$ $\mathbb{1}(U_Y < 0.2X + \underline{0.1}W + \underline{0.7}U_{WY})$,\label{eq:id-m2-last}
\end{numcases}
where $U_X, U_D, U_{WY}$ and $U_Y$ are independent, exogenous variables, with $U_X, U_{WY}$ binary with $P(U_X = 1) = P(U_{WY} = 1) = \frac{1}{2}$, and $U_D, U_Y$ distributed uniformly Unif$[0, 1]$.
 Both $\mathcal{M}_1, \mathcal{M}_2$ are compatible with $\mathcal{G}$ and hence are in $\Omega^{\mathcal{G}}$. The reader can verify that the two SCMs generate the same observational distribution. However, computing that
\begin{align}
    \text{NIE}^{\mathcal{M}_1}_{x_0, x_1}(y) = 28\%  \neq  \text{NIE}^{\mathcal{M}_2}_{x_0, x_1}(y) = 4\% 
\end{align}
shows lack of identifiability in the given context.  $\hfill \square$
\end{example}

Following the discussion in Sec.~\ref{sec:diagram}, we noted that 
one SCM $M$ induces a particular causal diagram $\mathcal{G}$. Still, 
specifying the precise $\mathcal{G}$ may be non-trivial in practice, and we hence introduced the standard fairness model (SFM). In this case, we will be particularly interested in the set of SCM defined by the SFM projection of the causal diagram, which is called $\Omega^{SFM}$. Reasoning within the  $\Omega^{SFM}$ space has two interesting  consequences. First, identification is in principle more challenging since this context is generally larger, containing more SCMs than the true $\Omega^{\mathcal{G}}$. Given that more SCMs implies the possibility of finding an alternative SCM that agrees with the assumptions and $P(v)$, and disagrees in the query, identifiability will in general be less frequent.  Still, second, since the SFM projection encodes fewer assumptions than the specific causal diagram $\mathcal{G}$, from the fairness analyst perspective, it will be in general easier to elicit such knowledge to construct a diagram. This situation is more visibly seen through Fig.~\ref{fig:spaceID}.

We now extend the FPCFA to account for the identifiability issues discussed above:
\begin{definition}[FPCFA continued with Identifiability] \label{def:fpcfa-with-id}
$[\Omega$, $Q$ as before$]$
Let the true, unobserved generative SCM $\mathcal{M} = \langle V, U, P(U), F \rangle$, and let $\mathcal{A}$ be a set of assumption and $P(v)$ be the observational distribution generated by it. 
Let $\Omega^{\mathcal{A}}$ the space of all SCMs compatible with $\mathcal{A}$. 
The Fundamental Problem of Causal Fairness Analysis is to find a collection of measures $\mu_1, \dots, \mu_k$
such that the following properties are satisfied: 
\begin{enumerate}[label=(\arabic*)]
	\item $\mu$ is decomposable w.r.t. $\mu_1, \dots, \mu_k$; 
	\item $\mu_1, \dots, \mu_k$ are admissible w.r.t. the structural fairness criteria $Q_1, Q_2, \dots, Q_k$. 
	\item $\mu_1, \dots, \mu_k$ are as powerful as possible.
	\item $\mu_1, \dots, \mu_k$ are identifiable from the observational distribution $P(v)$ and class $\Omega^{\mathcal{A}}$.
\end{enumerate} 
\end{definition}
\begin{wrapfigure}{r}{0.40\textwidth}
\centering
\includegraphics[keepaspectratio,width=0.40\columnwidth]{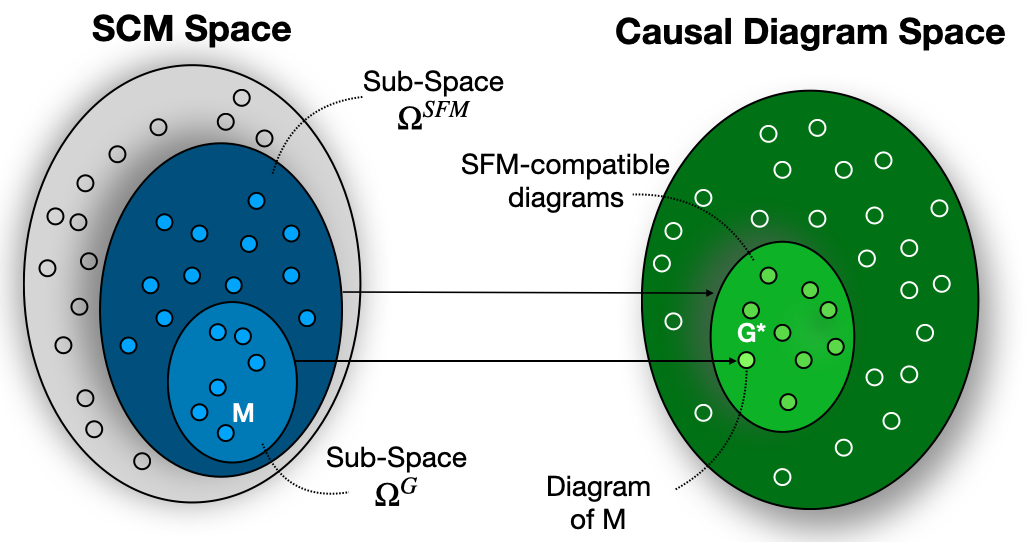}
\caption{Spaces of SCMs (left) and Causal Diagrams (right). (Left) Each point corresponds to a fully instantiated SCM. The SCMs compatible with the diagram $\mathcal{G}$ are shown in light blue, and the ones with the SFM in dark blue. (Right) Each point corresponds to a causal diagram. The lightest green dot corresponds to the true diagram $\mathcal{G}$, while the ones in the light green area correspond to different diagrams compatible with the SFM assumption. }
\label{fig:spaceID}
\vspace{-1.1in}
\end{wrapfigure}

The first question we ask is about solving the Step (4) of FPCFA when having the full causal graph $\mathcal{G}$. To this end, we state the following theorem:
\begin{theorem}[Identifiability over $\Omega^{\mathcal{G}}$] \label{thm:id-graph}
   Let $\mathcal{G}$ be a causal diagram compatible with the SFM and let $\Omega^{\mathcal{G}}$ be the context defined based on $\mathcal{G}$. Then,
   \begin{enumerate}[label=(\roman*)]
       \item TE, NDE, NIE, and Exp-SE are identifiable,
       \item $x$-TE, $x$-DE, $x$-IE, and $x$-SE are identifiable,
       \item $z$-TE, $z$-DE, and $z$-IE are identifiable,
       \item if $\{W, Y\} \cap V' \neq \emptyset$, then $v'$-TE , $v'$-DE , and $v'$-IE are not identifiable except in degenerate cases,
       \item $u$-TE, $u$-DE, and $u$-IE are not identifiable except in degenerate cases.
   \end{enumerate}
By degenerate cases we refer to instances in which a measure is equal to $0$ and identifiable from the absence of pathways. 
\end{theorem}
For example, $v'$-DE or $u$-DE could be identifiable (and equal to $0$) if the causal diagram $\mathcal{G}$ does not contain the arrow $X \rightarrow Y$ (this is a case we call degenerate in the above theorem). In summary, we can claim that general, $x$-specific, and $z$-specific measures are identifiabile over $\Omega^{\mathcal{G}}$ whenever $\mathcal{G}$ is compatible with the SFM. However, $v'$ or unit level measures are in general not identifiable, without additional assumptions. 

The important next question we ask is whether there is a gap in solving the FPCFA under the context $\Omega^{SFM}$ compared to $\Omega^{\mathcal{G}}$. In the first instance, as shown in the following theorem, the answer is negative, showing formally show why our definition of the SFM is indeed sensible in the context of $\fpcfa$:
\begin{theorem}[Identifiability over $\Omega^{SFM}$ \& Soundness of SFM] \label{thm:sfm}
Under the Standard Fairness Model (SFM) the orientation of edges within possibly multidimensional variable sets $Z$ and $W$ does not change any of general, $x$-specific or $z$-specific measures. That is, if two diagrams $G_1$ and $G_2$ have the same projection to the Standard Fairness Model, i.e.,
\begin{equation}
    \Pi_{\text{SFM}}(\mathcal{G}_1) = \Pi_{\text{SFM}}(\mathcal{G}_2)
\end{equation}
then any measure $\mu(P(v), G)$ will satisfy
\begin{equation}
   \mu(P(v), \mathcal{G}_1) = \mu(P(v), \mathcal{G}_2) = \mu(P(v), \mathcal{G}_{\text{SFM}}). 
\end{equation}
That is, if measures $\mu_1, \dots, \mu_k$ in Step (4) of FPCFA in Def.~\ref{def:fpcfa-with-id} are identifiable over the class of SCMs $\Omega^{\mathcal{G}}$ corresponding to a causal diagram $\mathcal{G}$, then they are also identifiable over the class of SCMs $\Omega^{SFM}$ corresponding to the diagram's SFM projection $\mathcal{G}_{\text{SFM}}$. The notation $\mu(P(v), \mathcal{G})$ indicates the measures are computed based on the observational distribution $P(v)$ and the causal diagram $\mathcal{G}$ (as opposed to being computed based on the SCM $\mathcal{M}$ as before).
\end{theorem}
The proofs of Thm.~\ref{thm:id-graph} and ~\ref{thm:sfm} are given in Appendix \ref{appendix:sfm}, together with a discussion on relaxing the assumptions of the SFM, and a discussion on the estimation of measures. The theorem shows that the SFM projection of a diagram $\mathcal{G}_{\text{SFM}}$ is equally useful as the fully specified diagram $\mathcal{G}$ for computing any of the general, $x$-specific or $z$-specific measures in Lem.~\ref{lemma:TVfamily}. That is, specifying more precisely the causal structure contained in multivariate nodes $Z$ and $W$ would not change the values the different measures. The SFM projection $\mathcal{G}_{\text{SFM}}$ can be understood as a coarsening of the equivalence class of SCMs compatible with the graph $\mathcal{G}$. Perhaps surprisingly, this coarsening does not hurt the identifiability of some of the most interesting measures. Moreover, for computing the $v'$-specific and unit-level measures, additional assumptions would be necessary, even if the full diagram $\mathcal{G}$ was available (see Appendix \ref{appendix:sfm} for more details). The key observation is that $v'$-specific measures require the identification of the joint counterfactual distribution $P(v'_{x_0}, v'_{x_1})$, and these two potential outcomes are never observed simultaneously. Therefore, unless we are interested in $v'$-specific or unit-level measures, we can simply focus on constructing the $\mathcal{G}_{\text{SFM}}$ and not worry about full details of the diagram $\mathcal{G}$. The formulation of FPCFA with identifiability uncovers an interesting interplay of power and identifiability, in which increasingly strong assumptions are needed to identify more powerful measures.

\subsection{Other relations with the literature}
Equipped with the Fairness Map, which was the culmination of understanding the relationship between a multitude of measures, we can now analyze the connection of Causal Fairness Analysis with some influential previous works that articulated other measures in the literature. In particular, we will discuss the criteria of counterfactual fairness in Sec.~\ref{sec:ctf-fair} and individual fairness in Sec.~\ref{sec:ctf-if}. 

\subsubsection{Criterion 1. Counterfactual fairness} \label{sec:ctf-fair}
One criterion that has received considerable attention in the literature is called ``\textit{counterfactual fairness}'' \citep{kusner2017counterfactual}. Noteworthy in terms of terminology, the name ``counterfactual fairness'' is a misnomer, and somewhat misleading, as there are various measures that are counterfactual in nature and could be employed to reason about fairness, following the previous discussion and the Fairness Map (Fig.~\ref{fig:map}). Regardless of the name, the criterion has important limitations that we elaborate on this section.

To begin with, the definition of the proposed criterion is somewhat ambiguous in regard to whether it represents a unit-level quantity or a probabilistic-type of counterfactual\footnote{For various reasons, probabilistic measures tend to be discussed in the literature.}. To understand the issue, we list in the sequel three possible definitions compatible with the original paper, and then discuss their interpretations:

\begin{enumerate}[label=(\roman*)]
\item Counterfactual Fairness -- Unit-level (Ctf$_{\text{fair}}^{\text{ (u)}}$): 
\begin{align} \label{eq:ctf-fair-org2}
y_{x}(u) - y_{x'}(u) = 0, \;\;\forall  x, x', u \in \mathcal{U}.
 \end{align} 
\item Counterfactual Fairness -- Unit-level/probabilistic version (Ctf$_{\text{fair}}^{ (up)}$): 
\begin{align} \label{eq:ctf-fair-org-mix}
P(y_{x}(u) \mid X=x, W=w) = P(y_{x'}(u) \mid X=x, W=w), \;\; \forall x, x', w.
\end{align}
\item Counterfactual Fairness -- Population-level  (Ctf$_{\text{fair}}^{\text{ (p)}}$): 
\begin{align} \label{eq:ctf-fair-org-pop}
P(y_{x} \mid X=x, W=w) = P(y_{x'} \mid X=x, W=w), \;\; \forall x, x', w.
\end{align}
\end{enumerate}
In fact, the paper use the unit-level probabilistic version (Ctf$_{\text{fair}}^{ (up)}$) as its core definition \citep[Def.~5]{kusner2017counterfactual}, which is a direct translation to our notation so as to make the context and comparisons more transparent. \footnote{In particular, the original paper uses $A$ for the protected attribute, where we use $X$, and it uses $X$ for the remaining attributes where we use $W$. } The authors ``emphasize that counterfactual fairness is an individual-level definition, which is substantially different from comparing different individuals that happen to share the same “treatment” $X = x$ and coincide on the values of $W = w$" \citep[Sec.~3]{kusner2017counterfactual}. 
Interestingly, this seems a deliberate choice and suggest a unit-level definition of fairness. Importantly, the probabilistic unit-level (Ctf$_{\text{fair}}^{\text{ (up)}}$) and the unit-level definition (Ctf$_{\text{fair}}^{\text{ (u)}}$) are equivalent, as shown next: 

\begin{proposition}[Ctf$_{\text{fair}}^{\text{(up)}} \iff$ Ctf$_{\text{fair}}^{\text{ (u)}}$]
The unit-level counterfactual fairness (Eq.~\ref{eq:ctf-fair-org2})  and 
the  unit-level/probabilistic counterfactual fairness  (Eq.~\ref{eq:ctf-fair-org-mix}) criteria 
 are equivalent. 
\end{proposition}
This proposition suggests that the notation used in the original definition of the counterfactual fairness criterion, Ctf$_{\text{fair}}^{(up)}$, entails some confusion. In words, once the unit $U=u$ is specified, as originally stated in the criterion, $Y_x(u)$ is fully determined. It is therefore redundant, and there is no need for considering or conditioning on event $X=x, W=w$, as this is implied by the choice of unit $u$. 

However, the authors also state that ``the distribution over possible predictions for an individual should remain unchanged in a world where an individual’s protected attributes had been different" \citep[Sec.~1]{kusner2017counterfactual}  As explained above, if the unit $U=u$ is known, there are no probabilities involved, and the statements are deterministic. Therefore, under the alternative description the authors provide, a different formulation of the criterion is needed. In fact, if the goal is to have a probabilistic counterpart of Eq.~\ref{eq:ctf-fair-org2}, as the above statement might lead one to think, then the unit $U=u$ should be removed altogether, which leads more explicitly to Ctf$_{\text{fair}}^{\text{(p)}}$ definition, as displayed in Eq.~\ref{eq:ctf-fair-org-pop}.  
Interestingly, using structural basis expansion from Thm.~\ref{thm:contrasts}, we can show the relation of the unit- and the probabilistic-level definitions:
\begin{proposition}[Ctf$_{\text{fair}}^{\text{(p)}}$ is a probabilistic average of Ctf$_{\text{fair}}^{\text{ (u)}}$] \label{prop:ctf-fprob-avg}
Consider the following measure:
\begin{align}
    (x, w)\text{-TE}_{x,x'}(y\mid x, w) = P(y_{x} \mid X=x, W=w) - P(y_{x'} \mid X=x, W=w).
\end{align}
Then, the Ctf$_{\text{fair}}^{\text{(p)}}$ criterion is equivalent to $(x, w)\text{-TE}_{x,x'}(y\mid x, w) = 0, \;\;\forall x, x', w$. Furthermore, the measure underlying the Ctf$_{\text{fair}}^{\text{(p)}}$ criterion can be written as
\begin{align}
    (x, w)\text{-TE}_{x,x'}(y\mid x, w) 
                                        &= \sum_u [y_{x}(u) - y_{x'}(u)] P(u \mid x, w).
\end{align}
\end{proposition}
In words, Prop.~\ref{prop:ctf-fprob-avg} shows that probabilistic counterfactual fairness criterion takes an average of the unit level differences $y_{x}(u) - y_{x'}(u)$, weighted by the posterior $P(u \mid x, w)$, and requires the average to be equal to $0$. Note the difference between this definition and the unit-level definition, which requires every unit-level difference $y_{x}(u) - y_{x'}(u)$ to be $0$.

After explaining the difference between the two possible and qualitatively different interpretations of counterfactual fairness, and clearing up the notational confusion with respect to fixing a unit $U = u$, we now discuss somewhat more serious issues regarding the criterion, including from a conceptual, technical, and practical viewpoints. In fact, the issues listed below apply to both the Ctf$_{\text{fair}}^{\text{(u)}}$ and Ctf$_{\text{fair}}^{\text{(p)}}$ interpretations of counterfactual fairness, with the three major points being: 
\begin{enumerate}[label= \arabic*.]
    \item inadmissibility of Ctf$_{\text{fair}}^{\text{(u)}}$ and Ctf$_{\text{fair}}^{\text{(p)}}$ with respect to $\strm$,
    \item lack of accounting for spurious effects, and 
    \item hardness/impossibility of identifiability. 
\end{enumerate}

\subsection*{Issue 1. Inadmissiblity w.r.t. structural direct, indirect, and spurious effects} 
In the context of the discussion that lead to the conclusions in Sec.~\ref{sec:fairmap}, it is somewhat natural to expect that the counterfactual fairness measure is inadmissible w.r.t. any of the structural criteria, as more formally shown in the sequel. 

\begin{proposition}[Unit-TE, $(x, w)$-TE not admissible]
     The unit-level total effect (unit-TE$_{x_0, x_1}(y)$) and the $(x, w)$-specific total effect ($(x,w)$-TE$_{x_0, x_1}(y \mid x, w)$) are both not admissible w.r.t. the structural direct, indirect, and spurious criteria. Formally, we write
        \begin{align}
            \text{Str-DE-fair} &\centernot\implies \text{unit-TE-fair},\;\; \text{Str-DE-fair} \centernot\implies (x, w)\text{-TE-fair} \\
            \text{Str-IE-fair} &\centernot\implies \text{unit-TE-fair},\;\;\; \text{Str-IE-fair} \centernot\implies (x, w)\text{-TE-fair}\\
            \text{Str-SE-fair} &\centernot\implies \text{unit-TE-fair},\;\; \text{Str-SE-fair} \centernot\implies (x, w)\text{-TE-fair}.
        \end{align}
\end{proposition}
   
\noindent The importance of this result stems from the fact that even if one is able to  ascertain 
\begin{align*}
y_{x_1}(u) -& y_{x_0}(u) = 0 \;\;\forall u, \text{ or }\\
P(y_{x_1} \mid X=x, W=w) -& P(y_{x_0} \mid X=x, W=w) = 0 \;\;\forall x, w,
\end{align*}
it could still be that case that neither the direct nor the indirect (nor the spurious) effects are equal to $0$. The broader discussion around the Fairness Map, and the idea of decomposability of measures into admissible ones was introduced precisely to avoid such  situations. The next example highlights this issue more vividly. 

\begin{example}[Startup Hiring Continued - Salaries] \label{ex:ctf-fair-fail}
The startup company from Ex.~\ref{ex:NDEfail} has closed the hiring season. In the hiring process, the company achieved demographic parity, which means in this context that 50\% of new hires were female. Now, the company needs to decide on each employee's salary. In order to be ``fair'',  each employee is evaluated on how well they perform their tasks. The salary $Y$ is then determined based on this information, but, due to a subconscious bias of the executive determining the salaries, gender also affects how salaries are determined. The SCM $\mathcal{M}^*$ corresponding to this process is:

\begin{empheq}[left ={\mathcal{F}^*, P^*(U): \empheqlbrace}]{align}
		        X  &\gets U_X \label{eq:ctf-fair-fail-1} \\
 		        W  &\gets -X + U_W \label{eq:ctf-fair-fail-2}\\
 		        Y  &\gets X + W + U_Y.  \label{eq:ctf-fair-fail-3} \\ \nonumber\\
 		            U_X &\in \{0,1\}, P(U_X = 1) = 0.5, \\
 		            U_W &, U_Y \sim N(0, 1).       
\end{empheq}

\noindent For any unit $u = (u_x, u_w, u_y)$, we can compute that 
\begin{align}
    y_{x_1}(u) - y_{x_0}(u) = 
    \underbrace{(1 + (-1 + u_w) + u_y)}_{y_{x_1}(u)} - \underbrace{(0 + (-0 + u_w) + u_y)}_{y_{x_0}(u)} = 0, \label{eq:ctf-fair-fail-unit-cancel}
\end{align}
showing that unit-level total effect is 0. Furthermore, for each choice of $X=x, W=w$, it is also true that 
\begin{align}
    P(y_{x_1} \mid X=x, W=w) - P(y_{x_0} \mid X=x, W=w) = 0.
\end{align}
Therefore, both interpretations of the counterfactual fairness criterion are satisfied. However, direct discrimination against female employees still exists since the $f_y$ mechanism in Eq.~\ref{eq:ctf-fair-fail-3} assigns a higher salary to male employees. On the other hand, the mechanism $f_w$ in Eq.~\ref{eq:ctf-fair-fail-2} shows that female employees are better at performing their tasks, and should therefore be paid more. Nevertheless, the superior performance of female employees in performing their tasks is cancelled out by the direct discrimination favoring males (as witnessed by Eq.~\ref{eq:ctf-fair-fail-unit-cancel}). In effect, they are paid the same as they would be had they been male.  $\hfill \square$
\end{example}
The inability of total effect to detect direct and indirect effects stems from the fact that the total effect is decomposable (see Corol.~\ref{cor:fm-extmed}). The example above illustrates the first critical shortcoming of the criterion proposed by \cite{kusner2017counterfactual}, as in any other composite measure, and any optimization procedure based on it, i.e., zeroing the Ctf$_{\text{fair}}$ measure, may lead to unintended side effects and discrimination if implemented in the real world.

\subsection*{Issue 2. Ancestral closure \& Spurious effects} 
The purported criterion rules out, by construction, the possibility of existence of any spurious types of variations. In particular, the argument relies on the notion introduced in the paper called \textit{ancestral closure} (AC, for short) w.r.t. the protected attribute set. The AC requires that all protected attributes and their parents, and all their ancestors, should be measured and included in the set of endogenous variables. 
This is obviously a very stringent requirement, which is hard to ascertain in practice. 
The paper then argues that ``the fault should be at the postulated set of protected attributes rather than with the definition of counterfactual fairness, and that typically we should expect set $X$ to be closed under ancestral relationships given by the causal graph. For instance, if Race is a protected attribute, and Mother’s race is a parent of Race, then it should also be in $X$". 

Conceptually speaking, we contrast this constraint over the space of models with the very existence of dashed-bidirected arrows in causal diagrams, as discussed earlier. These arrows in particular allow for the possibility that there are variations between $X$ and $Z$ that can be left unexplained in the model, or unmeasured confounders may exist. 
Practically speaking, assuming that no bidirected arrows exist is a strong assumption that do not hold in many settings. For instance, consider the widely recognized phenomenon in the fairness literature known as \textit{redlining} \citep{zenou2000racial, hernandez2009redlining}. In some practical settings, the location where loan applicants live may correlate with their race. Applications might be rejected based on the zip code, disproportionately affecting certain minority groups in the real world. 

It has been reported in the literature that correlation between gender and location, or religious and location may possibly exist, and therefore, should be acknowledged through modeling. For instance, the one-child policy affecting mainly urban areas in China had visible effects in terms of shifting the gender ratio towards males \citep{hesketh2005effect, ding2006family}. Beyond race or gender, religious segregation is also a recognized phenomenon in some urban areas \citep{brimicombe2007ethnicity}. Again, while we make no claim that location affects race (or religion), or vice-versa, the bidirected arrows give a degree of modeling flexibility that allows for the encoding of such co-variations. Still, this without making any commitment to whatever historical processes and other complex dynamics that took place  and  generated such imbalance in the first place. To corroborate this point, consider the following example:
\begin{example}[Spurious associations in COMPAS \& Adult datasets]
A data scientist is trying to understand the correlation between the features in the COMPAS dataset. The protected attribute $X$ is race, and the demographic variables $Z_1$, $Z_2$ are age and sex. The data scientist tests two hypotheses, namely:
\begin{align}
    H^{(1)}_0: X \ci Z_1,\\
    H^{(2)}_0: X \ci Z_2.
\end{align}
The association of $X$ and $Z_1$, $Z_2$ are shown graphically in the bottom row of Fig.~\ref{fig:indeptests0}. Both of the hypotheses are rejected ($p$-values $< 0.001$). However, possible confounders of this relationship are not measured in the corresponding dataset.

Similarly, the same data scientist is now trying to understand the correlation of the features in the Adult dataset. The protected attribute $X$ is gender, and the demographic variables $Z_1$, $Z_2$ are age and race. The data scientist tests the independence of sex and age ($X \ci Z_1$), and sex and race ($X \ci Z_2$), and both hypotheses are rejected (p-values $< 0.001$, see Fig.~\ref{fig:indeptests0} top row). Again, possible confounders of this relationship are not measured in the corresponding dataset, meaning that the attribute $X$ cannot be separated from the confounders $Z_1, Z_2$ using any of the observed variables. $\hfill \square$
\begin{figure}
	\centering
	\includegraphics[height=90mm,angle=0]{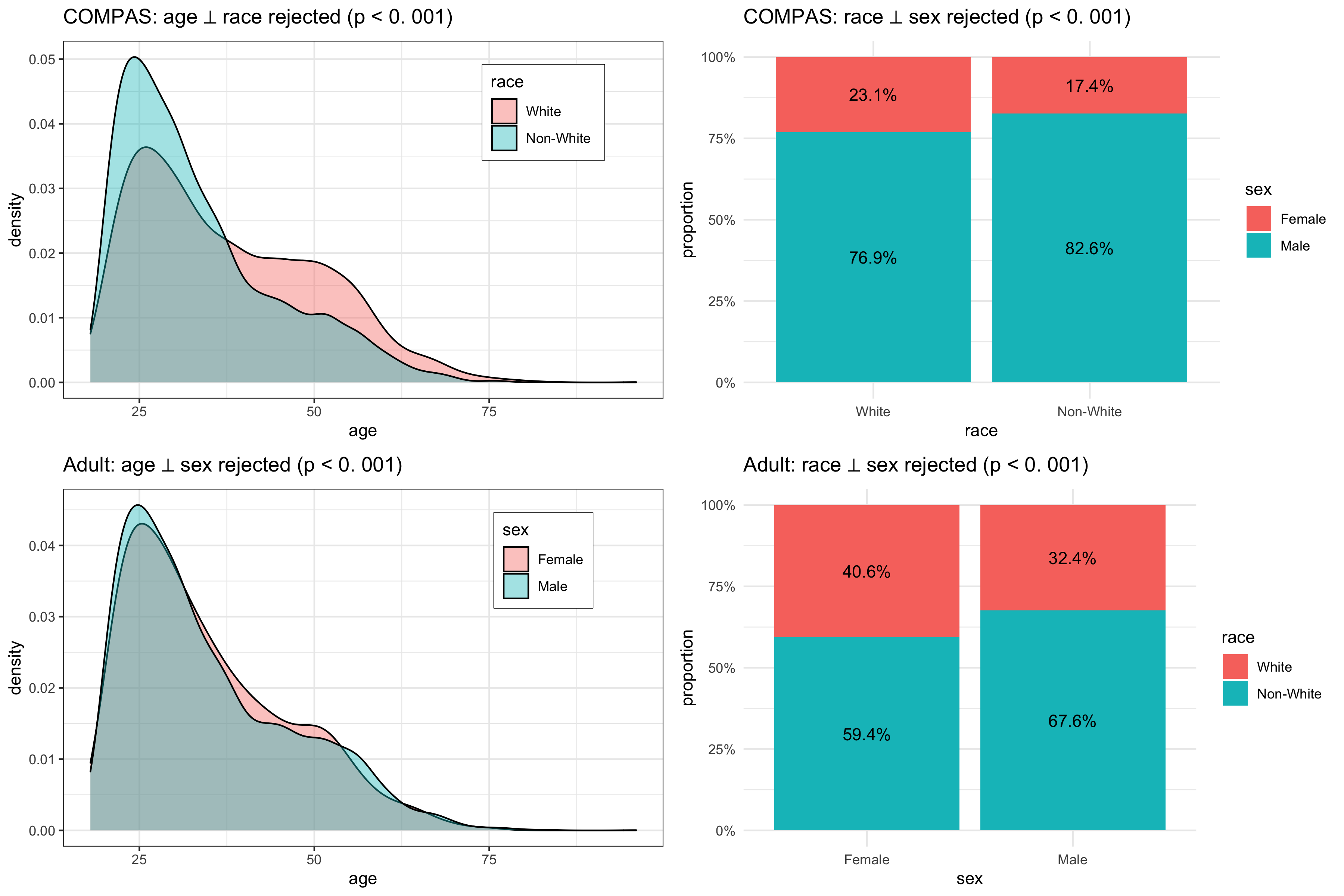}
	\caption{Testing for independence of the protected attribute ($X$) and the confounders ($Z$) on the Adult and COMPAS datasets.}
	\label{fig:indeptests0}
\end{figure}
\end{example}
As this example illustrates, from a both conceptual and practical standpoint, disallowing the possibility of non-causal relationships and confounding induced by some historical or societal context, and the associated spurious effects, can be an major limitation to any type of fairness analysis. 

\subsection*{Issue 3. Lack of identifiability}
An important practical property of any fairness measure is its identifiability under different sets of causal assumptions. We introduced the notion of identifiability in Sec.~\ref{sec:ID} to better understand when a fairness measure can be used in practice. We then discussed some necessary assumptions for measures in the Fairness Map to be identifiable. A significant implication of this prior discussion in the context of counterfactual fairness is highlighted by the following result:

\begin{proposition}[Unit-TE, $(x, w)$-TE not identifiable] \label{prop:ctf-fair-nonid}
    Suppose that $\mathcal{M}$ is a Markovian model and that $\mathcal{G}$ is the associated causal diagram. Assume that the set of mediators between $X$ and $Y$ is non-empty, $W \neq \emptyset$. Then, the measures unit-TE$_{x_0, x_1}(y)$ and $(x,w)$-TE$_{x_0, x_1}(y \mid x, w)$ are not identifiable from observational data, even if the fully specified diagram $\mathcal{G}$ is known.
\end{proposition}
The proposition shows that the measures on which counterfactual fairness is based are never computable from observational data and the causal diagram, even for models in which Markovianity is assumed to hold, a strong assumption. The main issue with these quantities is that they require knowledge of the joint distribution of counterfactual outcomes $Y_{x_1}, Y_{x_0}$, which are never observed at the same time
\footnote{Such quantities can be identified under additional, stronger assumptions, such as monotonicity \citep{tian2000probabilities,plevcko2020fair}.}

The issue discussed above obviously curtails the generality of the proposed method, since the underlying measures are not identifiable immediately, as illustrated next. 
\begin{example}[Non-ID of Ctf$_{\text{fair}}^{\text{(u)}}$, Ctf$_{\text{fair}}^{\text{(p)}}$ - Startup Salaries Continued]
Consider the SCM $\mathcal{M}^*$ of the Startup Salaries example (Ex.~\ref{ex:ctf-fair-fail}) given in Eq.~\ref{eq:ctf-fair-fail-1}-\ref{eq:ctf-fair-fail-3}. In $\mathcal{M}^*$ we showed that
    \begin{align}
         (x,w)\text{-TE}_{x_0, x_1}(y\mid x,w) = 0.
    \end{align}
Consider now an alternative SCM $\mathcal{M}'$ given by:
\begin{empheq}[left ={\mathcal{F}', P'(U): \empheqlbrace}]{align}
		        X  &\gets U_X \label{eq:ctf-fair-nonid-1} \\
 		        W  &\gets -X + (-1)^X U_W \label{eq:ctf-fair-nonid-2}\\
 		        Y  &\gets X + W + U_Y,  \label{eq:ctf-fair-nonid-3}
\end{empheq}
with the same distribution $P(u)$ over the units as for $\mathcal{M}^*$.
It's verifiable that that $\mathcal{M}'$ generates the same observational distribution as $\mathcal{M}^*$ and has the same causal diagram $\mathcal{G}$.
However, notice that for $u = (1, u_w, u_y)$, we have
    \begin{align}
        u\text{-TE}_{x_0, x_1}(y) = y_{x_1}(u) - y_{x_0}(u) = -2u_w \neq 0 \text{ whenever } u_w \neq 0. 
    \end{align}
    Furthermore, we have that 
    \begin{align}
        (x,w)\text{-TE}_{x_0, x_1}(y\mid x, w > 0) \neq 0.
    \end{align}
    Therefore, $\mathcal{M}^*$ and $\mathcal{M}'$ generate the same observational distribution and have the same causal diagram, but differ substantially with respect to counterfactual fairness. $\hfill \square$
\end{example}
The example constructed above is not atypical, but stems from the general non-identifiability result in Prop.~\ref{prop:ctf-fair-nonid}. These results raise the question as to whether counterfactual fairness criteria -- either  Ctf$_{\text{fair}}^{\text{(u)}}$ or Ctf$_{\text{fair}}^{\text{(p)}}$ -- can be used for the purpose of bias detection in any practical setting. In fact, to circumvent the identifiability issue discussed above, the proposal of the paper is that ``the model $\mathcal{M}^*$ must be provided" \citep[Sec.~4.2]{kusner2017counterfactual}. This means that the fully specified causal model $\mathcal{M}^*$ is needed to assess the existence of discrimination. The assumptions put forward in our manuscript are concerned with constructing the causal diagram $\mathcal{G}$, or the simplified version of the diagram in the form of an SFM. In stark contrast, the assumptions needed to provide the model $\mathcal{M}^*$ are orders of magnitude stronger than those needed for constructing the causal diagram or the SFM. This level of knowledge requires reading the intentions and minds of decision-makers, or having access to the internal systems and strategic secrets of companies, which are usually not accessible to outsiders.
On the more mathematical side, as alluded to earlier, inducing such a structural model from observational data alone is almost never possible \cite[Thm.~1]{bareinboim2020on}. 

\subsubsection{Criterion 2. Individual fairness} \label{sec:ctf-if}
In this section, we discuss a prominent measure introduced by \cite{dwork2012fairness} called \textit{individual fairness} (IF, for short). 
One of the most natural intuitions behind fairness is that if we constraint the population in a way that the units are the same but for the protected attribute, this would allow us to make claims about the impact of variations of this attribute. In fact, since nothing else remains to explain the observed disparities, the differences in outcome would be attributable to the change in the protected attribute. 

To ground this intuition, we introduced in Sec.~\ref{foundations} the explainability plane (Fig.~\ref{fig:Erefined}) that spawns the population and the mechanisms axes. In terms of the population axis, we noted that as the event $E=e$ is enlarged, the corresponding measure of fairness became more and more \textit{individualized}. Formally, the restriction on the observed information translates into a more precise subpopulation of the space of unobservable units $\mathcal{U}$. The analysis discussed earlier here relied on three observations that will be key to compare other causal measures with the IF measure, and try to understand its causal implications. 
First, the plane is contingent on the assumptions encoded in the SFM. As we will show formally, assumptions about the underlying causal structure are also relevant in the framework of IF. Secondly, the explainability plane considers admissibility and power of different measures, and we use these notions to place and understand the IF condition in the context of the Fairness Map. Thirdly, as highlighted by our analysis of the FPCFA, optimizing based on a specific composite criterion may in fact fail to remove bias that could be in principle detected when a more fine-grained analysis of the causal mechanisms generating the disparity is undertaken. We discuss conditions under which the IF framework is optimizing based on composite measures, with practical examples in which this may lead to unintended and potentially harmful side effects. 
We start with the definition of individual fairness:
\begin{definition}[Individual Fairness] \label{def:IF}
Let $d$ be a fairness metric on $\mathcal{X} \times \mathcal{Z} \times \mathcal{W}$. An outcome $Y$ is said to satisfy individual fairness if
\begin{align}\label{eq:IF}
    |P(y \mid x, z, w) - P(y \mid x', z', w')| \leq d((x, z, w), (x', z', w')),
\end{align}
$\forall \; x, x', w, w', z, z'$.
\end{definition} 
The framework of IF assumes the existence of a fairness metric $d$ that computes the distance between two individuals described by attributes $(x, z, w)$ and $(x', z', w')$, while the outcome $y$ is not taken into account.
In words, IF requires that individuals who are similar with respect to metric $d$ need to have a similar outcome. This requirement is represented by a Lipschitz property in Eq.~\ref{eq:IF}. If the distance between two values of the covariates, $d((x, z, w), (x', z', w'))$, is smaller than $\epsilon$, then the criterion in Eq.~\ref{eq:IF} implies that individuals who coincide with these covariate values must have a similar probability of a positive outcome, that is
\begin{align}
    |P(y \mid x, z, w) - P(y \mid x', z', w')| \leq \epsilon.
\end{align} 
We now look at the implications of the IF criterion, and observe some possible shortcomings that can result from ignoring the causal structure.

\begin{table}\centering
	\begin{tabular}{|M{0.35cm}|M{0.75cm}|M{6cm}|M{6cm}|}
	    \cline{1-4}
		\multicolumn{2}{|c|}{} & Example~\ref{ex:startup-hiring-iii}A & Example~\ref{ex:startup-hiring-iii}B \\[1mm]\hline
		\multirow{2}{1.5cm}{\rotatebox{90}{SCM $\mathcal{M}$\;\;\;\;}} & $\mathcal{F}$ &\begin{gather}
    X \gets U_{XY}\;\;\;\;\;\;\;\;\;\;\;\;\;\;\;\;\;\;\;\;\;\; \;\;\; \\
    Z \gets U_Z\;\;\;\;\;\;\;\;\;\;\;\;\;\;\;\;\;\;\;\;\;\; \;\;\;\;\; \\
    Y \gets X - U_{XY} + Z + U_Y \label{eq:startup-hiring-iiiA-3}
\end{gather} & \begin{gather}
    X \gets U_{XZ}\;\;\;\;\;\;\;\;\;\;\;\; \\
    Z \gets U_{XZ} + U_{ZY} \\
    Y \gets U_{ZY} + U_Y \;\;\label{eq:startup-hiring-iiiB-3}
\end{gather}\\[1mm]\cline{2-4}
		& $P(u)$ & \vspace{5pt} $U_{XY} \sim \text{Bernoulli}(0.5)$, $U_Z, U_Y \sim N(0, 1)$ \vspace{5pt} & \vspace{5pt} $U_{XZ} \sim \text{Bernoulli}(0.5)$, $U_{ZY}, U_{Y} \sim N(0, 1)$ \vspace{5pt} \\[1mm]\hline
		\rotatebox{90}{diagram} & $\mathcal{G}$ &\begin{tikzpicture}
		[>=stealth, rv/.style={thick}, rvc/.style={triangle, draw, thick, minimum size=7mm}, node distance=18mm]
		\pgfsetarrows{latex-latex};
		\begin{scope}
		\node[rv] (1) at (-2,-0.5) {$X$};
		\node[rv] (2) at (0,0.5) {$Z$};
		\node[rv] (3) at (2,-0.5) {$Y$};
		\path[->] (1) edge[bend left = 0] (3);
		\draw[->] (2) -- (3);
		\path[<->, dashed] (1) edge[bend left = -20] (3);
		\end{scope}
	\end{tikzpicture} & \begin{tikzpicture}
		[>=stealth, rv/.style={thick}, rvc/.style={triangle, draw, thick, minimum size=7mm}, node distance=18mm]
		\pgfsetarrows{latex-latex};
		\begin{scope}
		\node[rv] (1) at (-2,-0.5) {$X$};
		\node[rv] (2) at (0,0.5) {$Z$};
		\node[rv] (3) at (2,-0.5) {$Y$};
		\path[->] (1) edge[bend left = 0] (3);
		\path[<->, dashed] (1) edge[bend left = 20] (2);
		\path[<->, dashed] (2) edge[bend left = 20] (3);
		\end{scope}
	\end{tikzpicture} \\\hline
	\end{tabular}
		\caption{An example of two situations in which the IF criterion has different meanings. }\medskip\small
	\label{table:startup-hiring-iii}
\end{table}

\subsection*{Issue 1. IF is oblivious to causal structure}
The IF definition in Eq.~\ref{eq:IF} is agnostic with respect to the underlying causal structure that generated the data. We start with two examples of a hiring process that are on the surface similar, but differ with respect to the underlying causal structure. As we will see, this will show that the implications of the IF criterion can be quite different, which will highlight the fact that the causal structure cannot be dismissed when using this criterion. 

\begin{example}[Startup Hiring III] \label{ex:startup-hiring-iii}
Suppose that two startup companies, A and B, are hiring employees. Let $X$ (sex) represent the protected attribute, $Z$ the candidates performance on an aptitude test, and $Y$ the overall score for job hiring $Y$. The set of mediators $W$ is in this case empty. The hiring process is similar, yet there is a difference between the two companies.
In both instances, we assume age is a latent, unobserved factor, which has shared information with gender. In company A, age affects the salary directly, whereas in company B, age affects the aptitude test result. Additionally, in company B the aptitude test result has shared information with the salary, represented by the unobserved variable which measures how much the candidate prepared for the interview day. The respective SCMs and  causal diagrams are shown in Table~\ref{table:startup-hiring-iii}.
Suppose that the fairness metric $d$ is in both cases is
\begin{align}
    d((x, z), (x', z')) = |z - z'|.
\end{align}
Then, the IF criterion can be written as
\begin{align}
    \big|\ex[y \mid x, z] - \ex[y \mid x', z']\big| \leq d((x, z), (x', z')) = |z - z'|. \;\;\forall x, x', z, z'.
\end{align}
Notice that in company A, we can compute that
\begin{align}
    \ex^{\mathcal{M}_A}[y \mid x, z] &= \ex^{\mathcal{M}_A}[X - U_{XY} + Z + U_Y \mid x, z] \\
                                     &= \underbrace{\ex^{\mathcal{M}_A}[X - U_{XY} \mid x, z]}_{=0 \text{ as } X = U_{XY}} +\; \ex^{\mathcal{M}_A}[Z \mid x, z] + \underbrace{\ex^{\mathcal{M}_A}[U_Y \mid x, z]}_{\substack{= 0 \text{ as } U_Y \sim N(0,1),\\ U_Y \ci Z, X}} \\
                                     &= z.
\end{align}
Therefore, we can conclude that
\begin{align}
    \big|\ex^{\mathcal{M}_A}[y \mid x_1, z] - \ex^{\mathcal{M}_A}[y \mid x_0, z']\big| = |z-z'|.
\end{align}
In company B, however, we can compute:
\begin{align}
    \ex^{\mathcal{M}_B}[y \mid x, z] &= \ex^{\mathcal{M}_B}[U_{ZY} + U_Y \mid x, z] \\
                                     &= {\ex^{\mathcal{M}_B}[Z - U_{XZ} \mid x, z]} + \underbrace{\ex^{\mathcal{M}_B}[U_Y \mid x, z]}_{\substack{= 0 \text{ as } U_Y \sim N(0,1),\\ U_Y \ci Z, X}} \\
                                     &= \ex^{\mathcal{M}_B}[Z - X \mid x, z] = z - x.
\end{align}
Therefore, the IF criterion is not satisfied, which can be shown by computing:
\begin{align}
    \big|\ex^{\mathcal{M}_B}[y \mid x_1, z] - \ex^{\mathcal{M}_B}[y \mid x_0, z']\big| = |z-1-z'|.
\end{align}
When assessing direct discrimination on a structural level, in company A, the mechanism $f_y$ in Eq.~\ref{eq:startup-hiring-iiiA-3} shows the presence of direct discrimination. In company B, however, the mechanism $f_y$ in Eq.~\ref{eq:startup-hiring-iiiB-3} shows no direct discrimination. We could pick a more empirical measure of DE, such as the NDE (Def.~\ref{def:ndenie}). Evaluating the NDE using the generated data: 
\begin{align}
    \text{NDE}^{\mathcal{M}_A}_{x_0, x_1}(y) = 1,\\
    \text{NDE}^{\mathcal{M}_B}_{x_0, x_1}(y) = 0, 
\end{align}
which is consistent with the observed discrimination at the structural level.
$\hfill\square$
\end{example} 
Somewhat paradoxically, the example illustrates that in company A direct discrimination exists, yet the IF criterion is satisfied, whereas in company B the criterion is not fulfilled, but there is no direct discrimination. This example, even though perhaps surprising at first, is reflective of the fact that IF does not take the causal structure into account. Our conclusion is that without the causal diagram, the consequences of using IF might be unclear. Therefore, from this point forward, we assume the SFM structure, and look at the IF framework in this fixed context. 

\subsection*{Issue 2. IF captures the direct effect only under the SFM}
We next show that under the assumptions of the standard fairness model, the IF condition given in Eq.~\ref{eq:IF} has causal implications. In other words, we investigate where the IF condition can be placed in the Fairness Map in Fig.~\ref{fig:map}. An initial difficulty arises from the fact that the IF criterion is not written in the form of a contrastive measure (which were studied in Sec.~\ref{foundations}). Therefore, instead of using the exact IF criterion, we look at a criterion that is implied by the IF criterion, but is itself a contrastive measure. This criterion is based on the measure known as the observational direct effect:

\begin{definition}[Observational direct effect]
The observational direct effect (Obs-DE, for short) is defined as
\begin{align}
    \text{Obs-DE}_{x_0, x_1}(y \mid z, w) = P(y \mid x_1, z, w) - P(y \mid x_0, z, w).
\end{align}
Based on this measure, we define the Obs-DE-fair criterion as:
\begin{align} \label{eq:IF-DE}
    \text{Obs-DE-fair} \iff \text{Obs-DE}_{x_0, x_1}(y \mid z, w) = 0 \;\; \forall z, w.
\end{align}
\end{definition}

\noindent The Obs-DE-fair criterion is implied by IF whenever the fairness metric $d$ satisfies 
\begin{align}
    d((x_1, z,w), (x_0, z, w)) = 0 \;\; \forall z, w,
\end{align}
that is, when the metric $d$ does not depend on the protected attribute $X$.
The Obs-DE condition can then be obtained from Eq.~\ref{eq:IF} by setting $(x, z, w) = (x_1, z, w)$ and $(x', z', w') = (x_0, z, w)$. The Obs-DE criterion, which is implied by the IF condition under certain assumptions, is admissible with respect to structural direct criterion: 
\begin{proposition}[Admissibility of Obs-DE w.r.t. Str-DE and IF] \label{prop:IF-adm-str-de}
    Suppose that the metric $d$ does not depend on the $X$ variable, that is
    \begin{align}
        d((x, z,w), (x', z', w')) = d((z,w), (z', w')).
    \end{align}
    Then, the IF criterion in Eq.~\ref{eq:IF} implies the Obs-DE-fair criterion in Eq.~\ref{eq:IF-DE}. Furthermore, under the assumptions of the standard fairness model
    the Obs-DE measure is admissible with respect to Str-DE, that is
    \begin{align}
        \text{Str-DE-fair} \admarrow \text{Obs-DE-fair.}
    \end{align}
\end{proposition}
A further positive result shows that the Obs-DE criterion is in fact powerful in the context of detecting direct discrimination (again under suitable assumptions):
\begin{proposition}[Power of IF w.r.t. Str-DE] \label{prop:IF-pow-str-de}
    Suppose that the Obs-DE-fair criterion in Eq.~\ref{eq:IF-DE} holds. Under the assumptions of the standard fairness model, the Obs-DE measure is more powerful than $z$-DE, $x$-DE and NDE:
    \begin{align}
        \text{Obs-DE-fair} \powarrow z\text{-DE-fair} \powarrow x\text{-DE-fair} \powarrow \text{NDE-fair}.
    \end{align}
\end{proposition}
Under the SFM\footnote{The exact assumption needed here can be written as $Y_{x,z,w} \ci X, Z,W$. This assumption is encoded in the SFM.} $P(y \mid x_1, z, w) - P(y \mid x_0, z, w)$ equals what is known as the \textit{controlled direct effect}
\begin{align}
    \text{CDE}_{x_0, x_1} := P(y_{x_1,z, w}) - P(y_{x_0,z, w}).
\end{align}
Therefore, under certain assumptions, the constraint implied by IF in fact precludes the existence of a direct effect and has a valid causal interpretation. Importantly, the assumptions that are needed are of a causal nature, and ignoring the causal diagram of the data generating model can lead to undesired consequences when using the IF condition (see Ex.~\ref{ex:startup-hiring-iii}).

To continue the discussion, we consider two distinct cases when choosing the fairness metric $d$, on which much of the IF framework relies:
\begin{enumerate}[label=(\roman*)]
    \item metric $d$ is sparse, meaning that it does not depend on all variables in the sets $Z, W$,
    \item metric $d$ is complete, meaning that it depends on all variables in the sets $Z, W$.
\end{enumerate}
We now consider these two cases separately, and point out their possible drawbacks. We emphasize that our goal is not to pick a metric but to shed light on the fundamental interplay between the arguments/properties of the fairness metric $d$ and the underlying causal mechanisms, which describes where the decision-making process takes place in the real world and from where data is collected. 

\subsection*{Issue 3. Sparse metrics $d$ lead to lack of admissibility}
\paragraph{From individual to global.} 
Suppose that the IF condition in Eq.~\ref{eq:IF} holds. Under suitable causal assumptions, the condition precludes the existence of direct discrimination, as was shown above. However, even if the IF condition holds, the disparity between the groups corresponding to $X = x_0$ and $X = x_1$ (measured by the TV) could still be large, if the conditional distributions 
$$ Z, W \mid X = x_0 \text{ and } Z,W \mid X = x_1$$
differ. This observation leads to the second step of the framework of \cite{dwork2012fairness}.
The authors provide the following significant result:
\begin{proposition}[Optimal Transport bound on TV \citep{dwork2012fairness}]
    Let $d$ be a fairness metric, and suppose that the individual fairness condition in Eq.~\ref{eq:IF} holds. Let the optimal transport cost between $Z,W \mid X = x_1 \text{ and } Z, W \mid X = x_0$ be denoted by 
    \begin{align}
        \text{OTC}_{x_0, x_1}^d((Z, W)). 
    \end{align}
    Then, the TV measure between the groups is bounded by the optimal transport cost up to a constant $C_d$ dependent on the metric $d$ only, namely
    \begin{align}
        |\text{TV}_{x_0, x_1}(y)| \leq C_d * \text{OTC}_{x_0, x_1}^d((Z, W)).
    \end{align}
\end{proposition}
In words, if the optimal transport (OT) distance between distributions $$ Z,W \mid X = x_1 \text{ and } Z, W \mid X = x_0,$$
with the metric $d$ measuring the transport cost, is small, the TV measure is consequently small as well. Here, however, there is an important nuance, stemming from the decomposability of the TV measure, as shown in the following proposition:
\begin{proposition}[Inadmissibility of OTC]
    The optimal transport cost $\text{OTC}_{x_0, x_1}^d((Z, W))$ is not admissible with respect to structural indirect and structural spurious criteria. Formally, we write that:
    \begin{align}
        \text{Str-IE-fair} &\centernot\implies \big(\text{OTC}_{x_0, x_1}^d((Z, W)) = 0 \big),\\
        \text{Str-SE-fair} &\centernot\implies \big(\text{OTC}_{x_0, x_1}^d((Z, W)) = 0 \big).
    \end{align}
\end{proposition} 
\noindent To see the relevance of the proposition above, we proceed by means of an example, in which the above optimal transport distance is small and the TV is minimized, but in which indirect and spurious discrimination  still exist.
\begin{example}[Startup Hiring IV] \label{ex:IF-sparse-metric}
Suppose that a startup company is hiring accountants. Let $X$ (sex) be the protected attribute, $Z$ be the age of the candidate and $W$ their performance on an accountancy test, upon which the job decision $Y$ is based. The following SCM $\mathcal{M}^*$ describes the situation:

\begin{empheq}[left ={\mathcal{F}^*, P^*(U): \empheqlbrace}]{align}
                    \label{eq:startup-iv-fx}
		            X &\gets U_{XZ} \\
                    Z &\gets -U_{XZ} + U_Z \label{eq:startup-iv-fz}\\
                    W &\gets X + Z + U_W \label{eq:startup-iv-fw}\\
                    Y &\gets \mathbb{1}(U_Y < \text{expit}(W)),  \label{eq:startup-hiring-iv-y}
                    \\\nonumber\\
 		            U_{XZ} &\in \{0,1\}, P(U_{XZ} = 1) = 0.5, \\
 		            U_Z &, U_{W}, U_{Y} \sim \text{Unif}[0,1],
\end{empheq}

where $\text{expit}(x) = \frac{e^x}{1+e^x}$. The $f_w$ mechanism in Eq.~\ref{eq:startup-iv-fw} shows that older candidates perform better at the test, and that women perform better than men, given equal age. However, due to latent confounding, arising from a specific historical context, women tend to leave the profession at an earlier age (mechanisms $f_x, f_z$ in Eq.~\ref{eq:startup-iv-fx} and~\ref{eq:startup-iv-fz} show that lower age is correlated with being female, through the $U_{XZ}$ variable). 
The causal graph representing this situation is given by
\begin{center}
		\begin{tikzpicture}
	 [>=stealth, rv/.style={thick}, rvc/.style={triangle, draw, thick, minimum size=7mm}, node distance=18mm]
	 \pgfsetarrows{latex-latex};
	 \begin{scope}
		\node[rv] (0) at (-0.5,0.75) {$Z$};
	 	\node[rv] (1) at (-1.5,-1) {$X$};
	 	\node[rv] (2) at (0.5,-1) {$W$};
	 	\node[rv] (3) at (2.5,-1) {$Y$};
	 	\draw[->] (1) -- (2);
		\path[<->] (1) edge[bend left = 30, dashed] (0);
	 	\draw[->] (2) -- (3);
		\draw[->] (0) -- (2);
	 \end{scope}
	 \end{tikzpicture}.
	\end{center}
Importantly, the marginal distributions $W \mid X = x_0$ and $W \mid X = x_1$ are equal in $\mathcal{M}^*$. 
An outside authority, which certifies whether discrimination is present, decides that the metric $d$ is given by:
\begin{align}
d((x, z, w), (x', z', w')) = |w-w'|.
\end{align}
In this case, we have that
\begin{align}
    | P(y \mid x, z, w) - P(y \mid x', z', w') | =& | \text{expit}(w) - \text{expit}(w') | \\  \leq & \frac{1}{4} | w - w'|,
\end{align}
where the last inequality follows from an application of the mean value theorem.
Furthermore, the optimal transport cost is $0$, because the marginal distributions of $W$ are matching between the groups. There is no direct discrimination, since $Y$ is not a function of $X$ (Eq.~\ref{eq:startup-hiring-iv-y}). Therefore, the IF criterion is satisfied and the TV measure equals $0$. However, when applying the decomposition of TV found in the $x$-specific solution to $\fpcfa$ in Thm.~\ref{thm:fpcfa-1st-sol2}, we have that
\begin{align}
    \text{TV}_{x_0, x_1}(y) &= {x\text{-DE}_{x_0, x_1}(y\mid x_0)} - {x\text{-IE}_{x_1, x_0}(y\mid x_0)} - {x\text{-SE}_{x_1, x_0}(y)} \\
                            &= \underbrace{(0\%)}_{direct} - \underbrace{(14\%)}_{indirect} - \underbrace{(-14\%)}_{spurious} \label{eq:startup-iv-cancel},
\end{align}
which indicates that even though the TV equals $0$, the spurious and indirect effects exist. 
$\hfill \square$
\end{example}
Notice the following about the example. Women, who are naturally better at their jobs, are interviewed at a younger age. If the source of the confounding comes from the fact that women (willingly) advance to a different profession in later stages of their career, then the cancellation of spurious and indirect effects in Eq.~\ref{eq:startup-iv-cancel} might be acceptable. If, however, the spurious effect stems from a confounding mechanism in which women abandon their careers for certain adverse reasons, then the situation could reasonably be deemed unfair. Without causal considerations, these two cases are indistinguishable. This example is inspired by an example of the original IF paper, which says that ``the imposition of a metric already occurs in many classification processes, including credit scores for loan applications" \citep[Sec.~6.1.1]{dwork2012fairness}. Notice that such a metric is based on a single mediator $W$, similar to the metric in Ex.~\ref{ex:IF-sparse-metric}. 

A possible objection to Ex.~\ref{ex:IF-sparse-metric} is that the metric $d$ does not include all confounders and mediators $Z, W$, which introduces a different issues, as discussed next.

\subsection*{Issue 4. Complete metrics $d$ do not allow for business necessity}

We now suppose that the fairness metric $d$ includes all variables in $Z, W$. If this is the case, then the optimal transport condition implies the independence of $X$ and the $Z, W$ variables, as shown in the following proposition:
\begin{proposition}[OTC $\implies X \ci Z,W$]
    Suppose that the metric $d$ is of the following form
    \begin{align}
        d((x, z, w), (x', z', w')) = \|z - z'\| + \|w - w'\|,
    \end{align}
    where $\|\cdot\|$ is any norm on $\mathbb{R}^d$. Then, we have that the optimal transport condition implies the independence of $X$ and $\{Z, W\}$, namely:
    \begin{align}
        \text{OTC}_{x_0, x_1}^d((Z, W)) = 0 \implies X \ci Z, W.
    \end{align}
\end{proposition}
Furthermore, if the metric $d$ does not consider $X$ then the IF condition implies the independence of $X$ and $Y$ conditional of $Z, W$.
\begin{proposition}[IF $\implies X \ci Y \mid Z, W$]
    Suppose that $d$ is a fairness metric and suppose that the IF condition in Eq.~\ref{eq:IF} holds. Then, for a binary outcome $Y$, $X \ci Y \mid Z, W$.
\end{proposition}
Finally, putting the above two propositions together implies that the variable $X$ is independent from all other observables in $V$, as shown next:
\begin{proposition}[OTC $\wedge$ IF $\implies X \ci V \setminus \{X\}$]
    Suppose that the metric $d$ is of the form $d((x, z, w), (x', z', w')) = \|z - z'\| + \|w - w'\|,$ 
    where $\|\cdot\|$ is any norm on $\mathbb{R}^d$. Suppose also that $\text{OTC}_{x_0, x_1}^d((Z, W)) = 0$ and the IF condition in Eq.~\ref{eq:IF} holds. Then we have that
    \begin{align}
        X \ci Z,W, Y.
    \end{align}
\end{proposition}
The proposition shows that if (i) the metric $d$ includes all variables in $Z, W$; (ii) the IF condition holds; (iii) the optimal transport distance is small, then the protected attribute $X$ is independent from all other endogenous variables in the system. As we will discuss later in Sec.~\ref{Tasks}, this can be a very strong requirement in practice, which requires completely removing the influence of $X$, and is not compatible with considerations about business necessity under the disparate impact doctrine.

\section{Fairness tasks} \label{Tasks}
The main goal of this section is to equip the reader with the tools for solving fairness problems in practice, building on the foundations introduced in previous sections. We classify fairness problems into three tasks, in increasing order of difficulty:
\begin{enumerate}[label=Task \arabic*.]
	\item \textbf{Bias detection and quantification:} the first and most basic task of fair ML. We may consider operating with a dataset $\mathcal{D}$ of past decisions, or in infinite samples with an observed distribution $P(V)$ over variables $V$. The task is to define a mapping $$M: \mathcal{P} \to \mathbb{R},$$
	where $\mathcal{P}$ is the set of possible distributions $P(V)$. $M$ is viewed as a \textit{fairness measure} and it is often constructed so that $M(P(V)) = 0$ would suggest the absence of some form of discrimination.
	\item \textbf{Fair prediction:} The task of fair prediction, usually, relies on a certain measure of fairness. The task is to learn a distribution $P^*(V)$ while maximizing utility $U(P(V))$ and satisfying $$ |M(P^*(V))| \leq \epsilon,$$ where $M$ is a measure of fairness as discussed in Task 1. Fair classification and fair regression problems fall into this category\footnote{Different categories of fair prediction methods exist, namely pre-processing, in-processing, and post-processing methods. These will be discussed separately in Sec.~\ref{fairprediction}.}.
	\item \textbf{Fair decision-making:} In fair decision-making, the well-being of certain groups over time is considered. Notions of affirmative actions also fall into this category. We might be interested in designing a policy $\pi$, which at every time step affects the observed distribution $P_t(V)$ (which now changes over time steps) so that we have
	$$ P_{t+1}(V) = \pi(P_t(V)),$$
	and we are, perhaps, interested in controlling how $M(P_t(V))$ changes with $t$.
\end{enumerate}
Note that these three tasks form a certain hierarchy, and are introduced in order of difficulty. Fair prediction often relies on a specific fairness measure; fair decision-making often relies on both a fairness measure and fair predictions. The first two tasks are discussed in Sec.~\ref{biasdetection} and Sec.~\ref{fairprediction}, respectively, while the last task (fair decision-making) is left for future work.
\subsection{Task 1: Bias Detection \& Quantification} \label{biasdetection}
In the context of Task 1, we distinguish two different, but closely related subtasks. These subtasks are referred to as bias detection and bias quantification. In bias detection, we are interested in providing a binary decision rule $\psi$ which determines whether discrimination is present or not. In bias quantification, we are interested in how strong the discrimination is, and therefore provide a real-valued number, instead of a binary decision. In what follows, we give the mathematical formulation of the two subtasks, together with an approach for how to solve them.   

\begin{definition}[Bias Detection under SFM]
Let $\Omega$ be a space of SCMs. Let $Q$ be a structural fairness criterion, $Q: \Omega \to \{0, 1\}$, determining whether a causal mechanism within the SCM $\mathcal{M} \in \Omega$ is active ($Q(\mathcal{M}) = 0$ if mechanism not active, $Q(\mathcal{M}) = 1$ if active). The task of bias detection is to test the hypothesis
\begin{align}
    H_0: Q(\mathcal{M}) = 0,
\end{align}
that is, constructing a mapping $\psi(\mathcal{G}_{\text{SFM}}, \mathcal{D})$ into $\{0, 1\}$, which provides a decision rule for testing $H_0$, based on the standard fairness model $\mathcal{G}_{\text{SFM}}$ and the data $\mathcal{D}$.
\end{definition}
In words, we are interested whether direct, indirect, or spurious discrimination exists (corresponding to $Q \in \strm$, see Def.~\ref{def:str-fair}).
The null hypothesis $H_0$ assumes that discrimination is not present, and the decision rule $\psi$ determines whether $H_0$ should be rejected based on the SFM and the available data. Notice, crucially, that $\psi$ is a function of $\mathcal{G}_{\text{SFM}}$ and $\mathcal{D}$. This stems from the fact that the SCM $\mathcal{M}$ is never available to the data scientist. Therefore, we cannot directly reason about $Q(\mathcal{M})$, but instead need to find an \textit{admissible} measure $\mu$ that satisfies
\begin{align}
    Q(\mathcal{M}) = 0 \implies \mu(\mathcal{M}) = 0,
\end{align}
where $\mu(\mathcal{M})$ can be computed in practice. 
Recall the result from Prop.~\ref{lem:tvnotadmissible} which shows that the TV measure is not admissible with respect to $\strm$ and therefore should not be used for bias detection when one is interested in direct, indirect, and spurious effects. Moreover, we note that solving the bias detection task depends on solving the $\fpcfa$, which we now restate in the form more suitable for Task 1:
\begin{definition}[FPCFA continued for Task 1] \label{def:fpcfa-task-1}
$[\Omega$, $Q$ as before$]$
Let the true, unobserved generative SCM $\mathcal{M} = \langle V, U, P(U), F \rangle$, and let $\mathcal{A}$ be a set of assumption and $P(v)$ be the observational distribution generated by it. 
Let $\Omega^{\mathcal{A}}$ the space of all SCMs compatible with $\mathcal{A}$. 
The Fundamental Problem of Causal Fairness Analysis is to find a collection of measures $\mu_1, \dots, \mu_k$
such that the following properties are satisfied: 
\begin{enumerate}[label=(\arabic*)]
	\item $\mu$ is decomposable w.r.t. $\mu_1, \dots, \mu_k$; 
	\item $\mu_1, \dots, \mu_k$ are admissible w.r.t. the structural fairness criteria $Q_1, Q_2, \dots, Q_k$. 
	\item $\mu_1, \dots, \mu_k$ are as powerful as possible.
	\item $\mu_1, \dots, \mu_k$ are identifiable from the observational distribution $P(v)$ and class $\Omega^{\mathcal{A}}$.
\end{enumerate}
The final step of FPCFA for Task 1 is
\begin{enumerate}[resume, label=(\arabic*)]
    \item estimate $\mu_1, \dots, \mu_k$ and their $(1-\alpha)$ confidence intervals from the observational data and the SFM projection of the causal diagram.
\end{enumerate}
\end{definition}
Upon solving $\fpcfa$ for Task 1, we obtain measures $\mu_i$ based on which the decision rule $\psi$ can be constructed. In particular, the decision rule $\psi$ will be constructed by computing the $(1-\alpha)$ confidence interval for $\mu_i$ using bootstrap. If the interval excludes $0$, the $H_0$ hypothesis is rejected. 

The derived measures $\mu_i$ obtained from solving $\fpcfa$ for Task 1 can also be used for the related task of bias quantification:

\begin{definition}[Bias Quantification under SFM]
Let $\Omega$ be a space of SCMs and let $(Q_i)_{i=1:3} = \strm$. The task of bias quantification is concerned with finding a mapping $\phi: \Omega \to \mathbb{R}^3$ where the $i$-th component $\phi_i$ is admissible with respect to $Q_i$. 
\end{definition}
In words, the amount of discrimination is summarized using a 3-dimensional statistic. Each component of the statistic corresponds to one of the direct, indirect, or spurious effects. The measures $\mu_i$ obtained from $\fpcfa$ can be used to solve the task of bias quantification, by setting
\begin{align}
    \phi(\mathcal{M}) = \big(\mu_{\text{DE}}(\mathcal{M}), \mu_{\text{IE}}(\mathcal{M}), \mu_{\text{SE}}(\mathcal{M})\big).
\end{align}
We can now discuss a specific proposal for the measures $\mu_i$.

\paragraph{Measures $\mu_i$ for Task 1.}
Following the $x$-specific solution of $\fpcfa$ from Thm.~\ref{thm:fpcfa-1st-sol2}, we use the following measures:
\begin{align}
    \mu_{\text{DE}} \text{ is given by }& x\text{-DE}_{x_0, x_1}(y \mid x_0) = P(y_{x_1, W_{x_0}} \mid x_0) - P(y_{x_0} \mid x_0) \\
    \mu_{\text{IE}} \text{ is given by }& x\text{-IE}_{x_1, x_0}(y \mid x_0) = P(y_{x_1, W_{x_0}} \mid x_0) - P(y_{x_1} \mid x_0) \\
    \mu_{\text{SE}} \text{ is given by }& x\text{-SE}_{x_1, x_0}(y) = P(y_{x_1} \mid x_0) - P(y_{x_1} \mid x_1).
\end{align}
Moreover, the solution also showed that the TV can be decomposed as:
\begin{align}
    \text{TV}_{x_0, x_1}(y) = \underbrace{x\text{-DE}_{x_0, x_1}(y \mid x_0)}_{\mu_{\text{DE}}} - \underbrace{x\text{-IE}_{x_1, x_0}(y \mid x_0)}_{\mu_{\text{IE}}} - \underbrace{x\text{-SE}_{x_1, x_0}(y \mid x_0)}_{\mu_{\text{SE}}}. 
\end{align}
In words, the TV equals the $x$-specific direct effect with a transition $x_0 \to x_1$, minus the $x$-specific indirect effect with the opposite transition $x_1 \to x_0$ and minus the $x$-specific spurious effect with the transition $x_1 \to x_0$. One critical point to note is that such a decomposition is not unique, since the TV can also be decomposed as:
\begin{align}
    \text{TV}_{x_0, x_1}(y) = -x\text{-DE}_{x_1, x_0}(y \mid x_0) + x\text{-IE}_{x_1, x_0}(y \mid x_0) - x\text{-SE}_{x_1, x_0}(y \mid x_0). 
\end{align}
To achieve symmetry and avoid picking a specific order, we propose using the average of the two decompositions. In particular, define the symmetric $x$-specific direct and indirect effects as:
\begin{definition}[Symmetric $x$-specific direct and indirect effect]
The symmetric $x$-specific direct and indirect effects are defined as:
\begin{align}
    x\text{-DE}^{\text{sym}}_{x}(y\mid x) = \frac{1}{2}\big(x\text{-DE}_{x_0, x_1}(y \mid x) - x\text{-DE}_{x_1, x_0}(y \mid x)\big) \\
    x\text{-IE}^{\text{sym}}_{x}(y\mid x) = \frac{1}{2}\big(x\text{-IE}_{x_0, x_1}(y \mid x) - x\text{-IE}_{x_1, x_0}(y \mid x)\big).
\end{align}
\end{definition}
Therefore, we propose to use $x\text{-DE}^{\text{sym}}_{x}(y \mid x_0)$ and $x\text{-IE}^{\text{sym}}_{x}(y \mid x_0)$ instead of $x\text{-DE}_{x_1, x_0}(y \mid x_0)$ and $x\text{-IE}_{x_1, x_0}(y \mid x_0)$ for Task 1. The benefit of these alternative measures is that no single transition $x_0 \to x_1$ has to be chosen for computing the direct/indirect effect, but both $x_0 \to x_1$ and $x_1 \to x_0$ transitions are considered, by taking the average of the two. Such an approach offers measures of direct and indirect effect which are symmetric with respect to the change in the protected attribute, unlike their counterparts that consider a single transition. 

\subsubsection{Legal Doctrines - A Formal Approach}
Equipped with specific measures that can be used to perform bias detection and quantification, we offer a formal approach for assessing the legal doctrines of disparate impact and treatment. Our operational approach is described in Algorithm~\ref{algo:cookbook}, and is one of the highlights of the manuscript. The algorithm takes the dataset $\mathcal{D}$, the SFM projection $\Pi_{\text{SFM}}(\mathcal{G})$ of the causal diagram, and the Business Necessity Set (BN-set) as an input. When using the SFM, the allowed BN-sets are $\emptyset, \{Z\}, \{W\}$, and $\{Z, W\}$\footnote{Handling more involved BN-sets is discussed in detail in Sec.~\ref{DIBN}.}. We next apply the Fairness Cookbook in practice.

\begin{algorithm}
  \caption{Fairness Cookbook for Task 1}
  \begin{algorithmic}[1]
    \Statex \textbullet~\textbf{Inputs:} Dataset $\mathcal{D}$, SFM projection $\Pi_{\text{SFM}}(\mathcal{G})$, Business Necessity Set BN-set.
    \State \textbf{Obtain the dataset $\mathcal{D}$.}
    \State \textbf{Determine the Standard Fairness Model projection $\Pi_{\text{SFM}}(\mathcal{G})$} of the causal diagram $\mathcal{G}$ corresponding to the SCM $\mathcal{M}$. Note that the full diagram $\mathcal{G}$ need not be specified for this.\newline
    \textit{Additionally: are there known bidirected edges between $X, Z, W$, and $Y$ groups?} If yes, go to Appendix \ref{appendix:IDEst} and consider the estimation in presence of bidirected edges. Otherwise continue to next step.
    \State \textbf{Consider the existence of Disparate Treatment:}
    \begin{itemize}
        \item compute the measure $x\text{-DE}^{\text{sym}}_{x}(y \mid x_0)$ and its 95\% confidence interval (for bias quantification, return this result and skip to next step)
        \item test the hypothesis
        \begin{align}
            H_0^{(x\text{-DE})}&: x\text{-DE}^{\text{sym}}_{x}(y \mid x_0) = 0.
        \end{align}
        \begin{itemize}
            \item if $H_0^{(x\text{-DE})}$ not rejected $\implies$ no evidence of disparate treatment
            \item if $H_0^{(x\text{-DE})}$ rejected $\implies$ evidence of disparate treatment
        \end{itemize}
        \item \textit{Additionally:} if no evidence of disparate treatment in overall population, for $Z=z$ test the hypothesis $H_0^{\text{($z$-DE)}}: z\text{-DE}^{\text{sym}}_{x}(y \mid z) = 0.$
    \end{itemize}
  \State \textbf{Consider the existence of Disparate Impact:}
    \begin{itemize}
    \item compute the measures $x\text{-IE}^{\text{sym}}_{x}(y \mid x_0)$ and $x\text{-SE}_{x_1, x_0}(y)$ and their 95\% confidence interval (for bias quantification, return this result and terminate the algorithm)
        \item if $W \notin$ BN-set, test the hypothesis
        \begin{align}
            H_0^{(x\text{-IE})}&: x\text{-IE}^{\text{sym}}_{x}(y \mid x_0) = 0.
        \end{align}
        \begin{itemize}
            \item if $H_0^{(x\text{-IE})}$ not rejected $\implies$ no evidence of disparate impact
            \item if $H_0^{(x\text{-IE})}$ rejected $\implies$ evidence of disparate impact
            \item \textit{Additionally:} if no evidence of disparate impact in overall population, for $Z=z$ test the hypothesis $ H_0^{\text{($z$-IE)}}: z\text{-IE}^{\text{sym}}_{x}(y \mid z) = 0$.
        \end{itemize}
        \item if $Z \notin$ BN-set, test the hypothesis
        \begin{align}
            H_0^{(x\text{-SE})}&: x\text{-SE}_{x_1, x_0}(y) = 0.
        \end{align}
        \begin{itemize}
            \item if $H_0^{(x\text{-SE})}$ not rejected $\implies$ no evidence of disparate impact
            \item if $H_0^{(x\text{-SE})}$ rejected $\implies$ evidence of disparate impact
        \end{itemize}
    \end{itemize}
  \end{algorithmic}
  \label{algo:cookbook}
\end{algorithm}

\subsubsection{Empirical Evaluation}
The practical usefulness of the Fairness Cookbook for Task 1 is demonstrated on two examples. Firstly, we apply the cookbook for the task of bias detection to the US Census 2018 dataset. After that, we apply the cookbook for the task of temporal bias quantification on a College Admissions dataset.

\begin{example}[US Government Census 2018] \label{ex:census}
	The United States Census of 2018 collected broad information about the US Government employees, including demographic information $Z$ ($Z_1$ for age, $Z_2$ for race, $Z_3$ for nationality), gender $X$ ($x_0$ female, $x_1$ male), marital and family status $M$, education information $L$, and work-related information $R$. In an initial analysis, a data scientist observed that male employees on average earn \$14000/year more than female employees, that is
	\begin{align}
		\ex[y \mid x_1] -  \ex[y \mid x_0] = \$14000.
	\end{align}
	Following the Fairness Cookbook, the data scientist does the following:\newline
	\textbf{SFM projection:}
	the SFM projection of the causal diagram $\mathcal{G}$ of this dataset is given by
	\begin{equation}
	    \Pi_{\text{SFM}}(\mathcal{G}) = \langle X = \lbrace X \rbrace,  Z = \lbrace Z_1, Z_2, Z_3 \rbrace, W = \lbrace M, L, R\rbrace, Y = \lbrace Y \rbrace\rangle.
	\end{equation}
	\textbf{Disparate treatment:} when considering disparate treatment, she computes $x\text{-DE}^{\text{sym}}_{x}(y \mid x_0)$ and its 95\% confidence interval to be 
	\begin{equation}
	    x\text{-DE}^{\text{sym}}_{x}(y \mid x_0) = \$9980 \pm \$1049.
	\end{equation}
	The hypothesis $H_0^{(x\text{-DE})}$ is thus rejected, providing evidence of disparate treatment of females.\newline
	\textbf{Disparate impact:} when considering disparate impact, she computes Ctf-SE, Ctf-IE and their respective 95\% confidence intervals:
	\begin{align}
	    x\text{-DE}^{\text{sym}}_{x}(y \mid x_0) &= \$5126 \pm \$778,\\
	    x\text{-SE}_{x_1, x_0}(y) &= -\$1675 \pm  \$955.
	\end{align}
	The data scientist decides that the differences in salary explained by the spurious correlation of gender with age, race, and nationality are not considered discriminatory. Therefore, she tests the hypothesis
	$$H_0^{(x\text{-IE})}: x\text{-IE}^{\text{sym}}_{x}(y \mid x_0) = 0,$$ which is rejected, indicating evidence of disparate treatment of female employees of the government. Measures computed in the example are visualized in Fig.~\ref{fig:census}. $\hfill \square$
\end{example}

\begin{figure}
	\centering
	\includegraphics[height=60mm,angle=0]{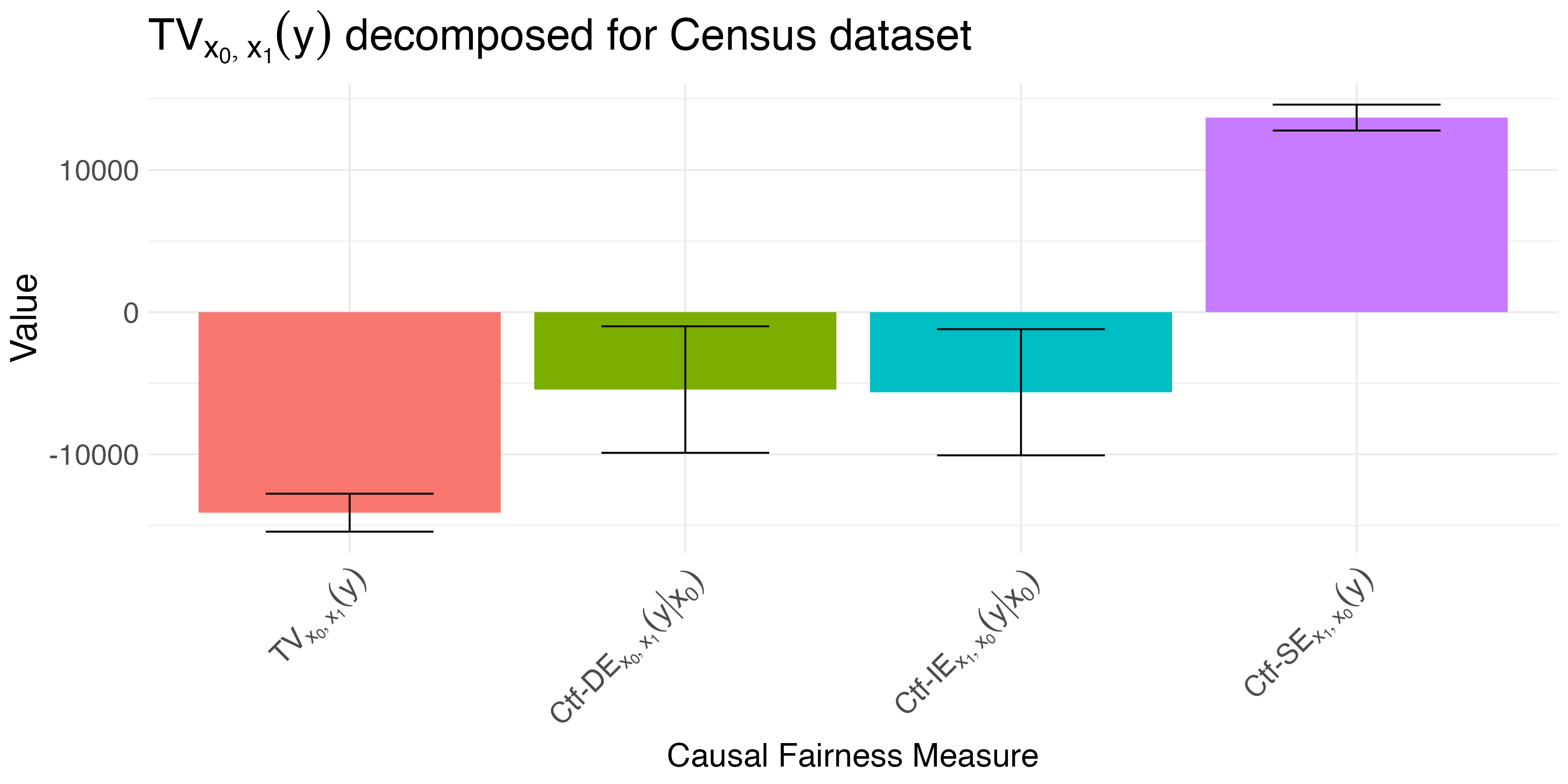}
	\caption{Measures obtained when applying the Fairness Cookbook for Task 1 on the Government Census 2018 dataset.}
	\label{fig:census}
\end{figure}

\begin{example}[Bias Quantification in College Admissions] \label{ex:col-adm-time}
A university in the United States admits applicants every year. The university is interested in quantifying discrimination in the admission process and track it over time, between 2010 and 2020. The data generating process changes over time, and can be described as follows. Let $X$ denote gender ($x_0$ female, $x_1$ male). Let $Z$ be the age at time of application ($Z = 0$ under 20 years, $Z = 1$ over 20 years) and let $W$ denote the department of application ($W = 0$ for arts\&humanities, $W = 1$ for sciences). Finally, let $Y$ denote the admission decision ($Y = 0$ rejection, $Y = 1$ acceptance). The application process changes over time and is given by

\begin{empheq}[left ={\mathcal{F}(t), P(U): \empheqlbrace}]{align}
		        X & \gets \mathbb{1} ( U_X < 0.5 + 0.1U_{XZ})\label{eq:fcb2-x}\\
		        Z & \gets \mathbb{1} (U_Z < 0.5 + \kappa(t) U_{XZ})  \\ 
		        W & \gets \mathbb{1} (U_W < 0.5 + \lambda(t) U_{XZ})  \\ 
		        Y & \gets \mathbb{1} (U_Y < 0.1 + \alpha(t)X + \beta(t)W + 0.1Z). \label{eq:fcb2-y} \\
		        \nonumber\\
 		            U_{XZ} &\in \{0,1\}, P(U_{XZ} = 1) = 0.5, \\
 		            U_X &, U_Z, U_W, U_Y \sim \text{Unif}[0, 1].       
\end{empheq}
The coefficients $\kappa(t), \lambda(t), \alpha(t), \beta(t)$ change every year, and obey the following dynamics:
\begin{align}
    \kappa(t+1) &= 0.9\kappa(t) \\
    \lambda(t+1) &= \lambda(t) (1 - \beta(t)) \\
    \beta(t+1) &= \beta(t) (1 - \lambda(t)) f(t), f(t) \sim \text{Unif}[0.8, 1.2] \\
    \alpha(t+1) &= 0.8\alpha(t).
\end{align}
The equations can be interpreted as follows. The coefficient $\kappa(t)$ decreases over time, meaning that the overall age gap between the groups decreases. The coefficient $\lambda(t)$ decreases compared to the previous year, by an amount dependent on $\beta(t)$. In words, the rate of application to arts\&humanities departments decreases if these departments have lower overall admission rates (i.e., students are less likely to apply to departments that are hard to get into). Further, $\alpha(t)$, which represents gender bias, decreases over time. Finally, $\beta(t)$ represent the increase in the probability of admission when applying to a science department. Its value depends on the value from the previous year, multiplied by $(1 - \lambda(t))$ and the random variable $f(t)$. Multiplication by the former factor describes the mechanism in which the benefit of applying to a science department decreases if a larger proportion of students apply for it. The latter factor describes a random variation over time which describes how well (in relative terms) the science departments are funded, and can be seen as depending on research and market dynamics in the sciences.

The head data scientist at the university decides to use the Fairness Cookbook for performing bias quantification. The SFM projection of the causal diagram $\mathcal{G}$ of the dataset is given by
	\begin{equation}
	    \Pi_{\text{SFM}}(\mathcal{G}) = \langle X = \lbrace X \rbrace,  Z = \lbrace Z \rbrace, W = \lbrace W\rbrace, Y = \lbrace Y \rbrace\rangle.
	\end{equation}
After that, the analyst estimates the quantities
\begin{align}
    x\text{-DE}^{\text{sym}}_{x}(y \mid x_0), x\text{-DE}^{\text{sym}}_{x}(y \mid x_0), \text{ and } x\text{-SE}_{x_1, x_0}(y) \;\;\; \forall t \in \{2010, \dots, 2020\}.
\end{align}
The temporal dynamics of the estimated measures of discrimination (together with the ground truth values obtained from the SCM $\mathcal{M}_t$) are shown graphically in Fig.~\ref{fig:col-adm-time}. $\hfill \square$    
\end{example}

\begin{figure}
	\centering
	\includegraphics[height=70mm,angle=0]{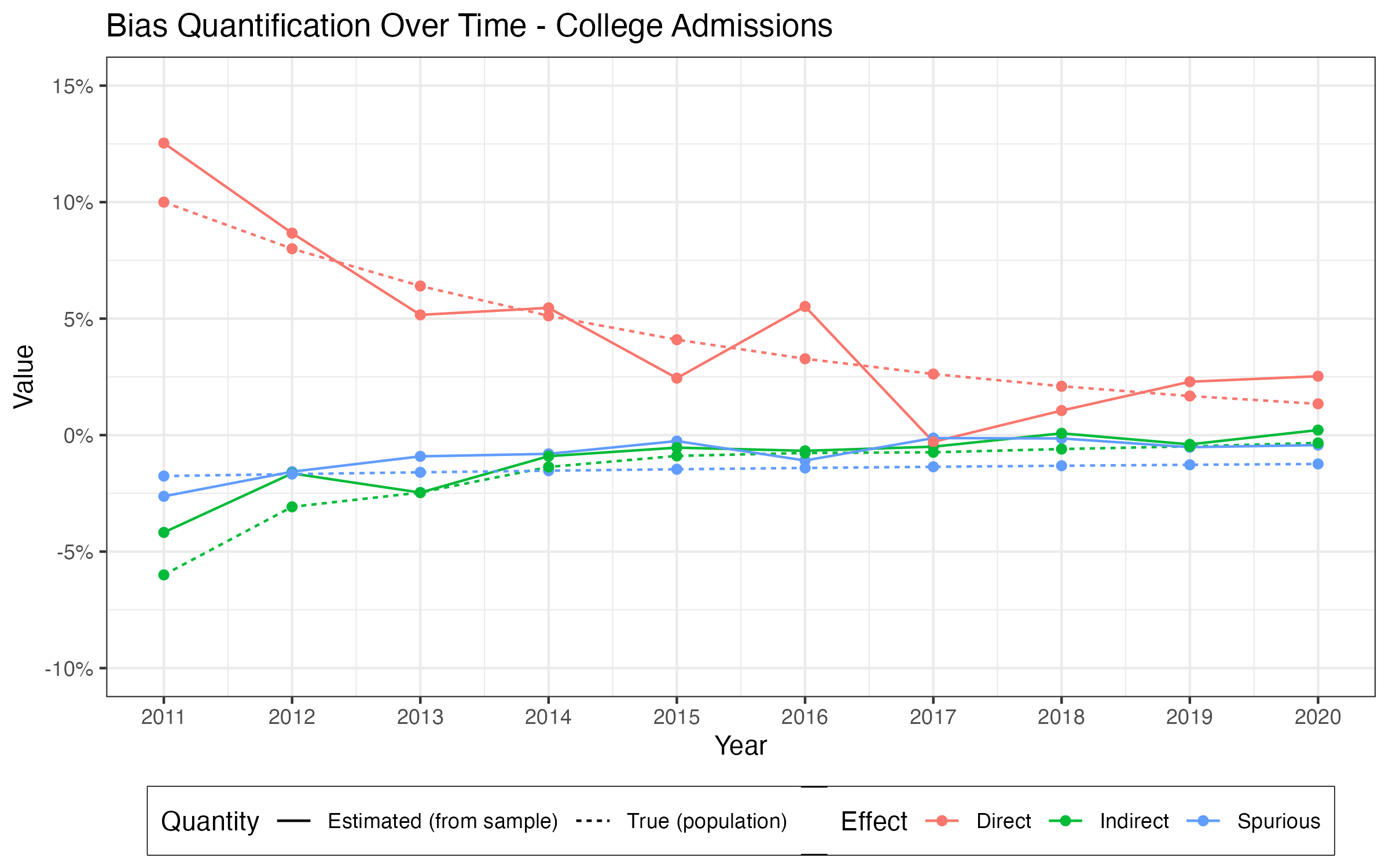}
	\caption{Tracking bias over time in the synthetic College Admissions dataset from Ex.~\ref{ex:col-adm-time}, between years 2010 and 2020. Both the estimated values from simulated samples (solid line) and the true population values (dashed lines) are shown, for direct (red), indirect (green), and spurious (blue) effects.}
	\label{fig:col-adm-time}
\end{figure}
\subsection{Task 2. Fair Prediction} \label{fairprediction}
We are now ready to discuss Task 2, which builds on similar foundations as the previous task. The section is organized as follows. 

\begin{enumerate}[label=(\roman*)]
    \item We first discuss previous literature on (fair) prediction; in particular, we discuss post-processing, in-processing, and pre-processing methods. 
    \item We formalize the $\fpcfa$ for Task 2, which is the problem that needs to be solved s.t. causally meaningful fair predictions can be obtained. 
    \item We introduce the Fair Prediction Theorem (Thm.~\ref{thm:fpt}) that explains  why  standard methods for fair prediction,  agnostic to the causal structure, fail in solving FPCFA.
    \item We develop two alternative formulations of the fair prediction optimization problem capable of remedying the shortcomings of methods found in the  literature.
\end{enumerate}

\subsubsection{Prediction}
In the context of prediction, one is generally interested in constructing a predictor $\widehat{Y}$ of $Y$, which is a function of $X, Z$ and $W$. More precisely, from a causal inference point of view, this process can be conceptualized as constructing an additional mechanism $\widehat{Y} \gets f_{\widehat{Y}}(x, z, w)$ in the SCM, which is under our control, as shown in Fig.~\ref{fig:sfmpred}. A typical choice of $f_{\widehat{Y}}$ in the context of regression is the estimate of $\ex[Y \mid X = x, Z=z, W = w]$, whereas for classification a rounded version of such an estimate is often considered.

\begin{figure}
	\begin{center}
		\begin{tikzpicture}
	 [>=stealth, rv/.style={thick}, rvc/.style={triangle, draw, thick, minimum size=7mm}, node distance=18mm]
	 \pgfsetarrows{latex-latex};
	 \begin{scope}
	    
	    \fill [gray!5, rotate = 45] (-1.8,-1.4) rectangle  (2.2,1.6);
	    \draw [ultra thin, rotate = 45] (-1.8,-1.4) rectangle (2.2,1.6) node[above right, yshift=-0.7cm, xshift=0.6cm] {Original diagram $\mathcal{G}$};
		\node[rv] (0) at (0,1.5) {$Z$};
		\node[rv] (1) at (-1.5,0) {$X$};
		\node[rv] (2) at (0,-1.5) {$W$};
		\node[rv] (3) at (1.6,0.5) {$Y$};
		
		\node[rv, blue] (4) at (1.6,-1.5) {$\widehat{Y}$};
		\draw[->] (1) -- (2);
		\draw[->] (0) -- (3);
		\path[->] (1) edge[bend left = 0] (3);
		\path[<->] (1) edge[bend left = 30, dashed] (0);
		\draw[->] (2) -- (3);
		\draw[->] (0) -- (2);
		\draw[->, blue] (0) -- (4);
		\draw[->, blue] (1) -- (4);
		\draw[->, blue] (2) -- (4);
	 \end{scope}
	 \end{tikzpicture}
	\end{center}
	\caption{Standard Fairness Model (SFM) extended with a blue node $\widehat{Y}$, for the task of (fair) prediction.}
	\label{fig:sfmpred}
\end{figure}
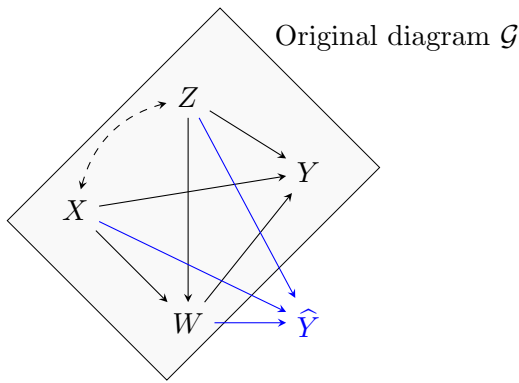

When constructing \textit{fair} predictions, one is additionally interested in ensuring that the constructed $\widehat{Y}$ also satisfies a fairness constraint. In the fairness literature, there are three broad categories for achieving this. These approaches are referred to as post-processing, in-processing, and pre-processing. We now cover them in order. Although there are many possible target measures of fairness which the predictor $\widehat{Y}$ could satisfy, in this manuscript we focus on methods that aim to achieve the condition TV$_{x_0, x_1}(\widehat{y}) = 0$.

\subsubsection{Post-processing} \label{Postproc}
Post-processing methods are the simplest and most easily described. First, one constructs a predictor $f_{\widehat{Y}}$ without applying fairness constraints. The output of $f_{\widehat{Y}}(x, z, w)$ is then taken and transformed using a transformation $T$, such that the constructed predictor 
\begin{align}
    \widehat{Y} \gets T(f_{\widehat{Y}}(x, z, w)),
\end{align}
satisfies the condition TV$_{x_0, x_1}(\widehat{y}) = 0$. We illustrate the post-processing methods with an example. The \textit{reject-option classification} of \cite{kamiran2012decision} starts by estimating the probabilities of belonging to the positive class, ${P}(y)$ (label the estimates with $f_{\widehat{Y}}(x, z, w)$). The classifier $\widehat{Y}$ is then constructed such that
\begin{align*}
	\widehat{Y}(x, z, w) = \mathbb{1}(f_{\widehat{Y}}(x, z, w) > \theta_x),
\end{align*}
where $\theta_{x_0}, \theta_{x_1}$ are group-specific thresholds chosen so that $\widehat{Y}$ satisfies TV$_{x_0, x_1}(\widehat{y}) = 0$, and also that $\theta_{x_0}, \theta_{x_1}$ are as close as possible to $0.5$ (to minimize the loss in accuracy). An important question we discuss shortly is whether the $\widehat{Y}$ constructed in such a way also behaves well from a causal perspective.

\subsubsection{In-processing} \label{Inproc}
In-processing methods take a different route. Instead of massaging unconstrained predictions, they attempt to incorporate a fairness constraint into the learning process. This in effect means that the mechanism $f_{\widehat{Y}}$ is no longer unconstrained, but is required to lie within a class of functions which satisfy the TV constraint. Broadly speaking, this is achieved by formulating an optimization problem of the form
\begin{alignat}{2}
&\text{arg min}_{f_{\widehat{Y}}}        &\qquad& L\big(Y, f_{\widehat{Y}}(x, w, z)\big) \label{eq:optimize}\\
&\text{subject to} &      & \text{TV}_{x_0, x_1}({\widehat{y}}) \leq \epsilon,\label{eq:group}\\
&                  &      & ||f_{\widehat{Y}}(x, w, z) - f_{\widehat{Y}}(x', w', z')|| \leq \tau((x, w, z), (x', w', z')).\label{eq:individual}
\end{alignat}
where $L$ is a suitable loss function\footnote{A common choice here is the loss $\ex\big[Y - f_{\widehat{Y}}(x, w, z)\big]^2$.} and $\tau$ is a metric on the covariates $V\setminus Y$. 
In the language of \cite{dwork2012fairness}, the TV minimization constraint in Line~\ref{eq:group} ensures \textit{group fairness}, whereas the constraint in Line~\ref{eq:individual} ensures \textit{covariate-specific fairness}\footnote{This notion corresponds to \textit{individual fairness} in the work of \cite{dwork2012fairness}. Causally speaking, this would be seen as a \textit{covariate-specific} fairness constraint, as the term individual is overloaded.}, meaning that predictions for individuals with similar covariates $x, z, w$ should be similar. 
Exactly formulating and efficiently solving problems as in Lines~\ref{eq:optimize}-\ref{eq:individual} constitutes an entire field of research. Due to space limitations, we do not go into full detail on how this can be achieved, but rather name a few well-known examples. \cite{zemel2013learning} use a clustering-based approach, whereas \cite{zhang2018mitigating} use an adversarial network approach. \cite{kamishima2012fairness} add a mutual information constraint to control the TV in parametric settings. \cite{agarwal2018reductions} formulate a saddle-point problem with moment-based constraints to achieve the desired minimization of the TV.
The mentioned methods differ in many practical details, but all attempt to satisfy the constraint $\text{TV}_{x_0, x_1}({\widehat{y}}) = 0$ by constraining the learner $f_{\widehat{Y}}$. Again, the question arises as to whether constructing the mechanism $f_{\widehat{Y}}$ so that TV equals 0 can provide guarantees about the causal behavior of the predictor.

\subsubsection{Pre-processing} \label{Preproc}
The last category of methods are the pre-processing methods. Here, the aim is to start from a distribution $P(x, w, z, y)$ and find its ``fair version", labeled $\widetilde{P}(x, w, z, y)$. Sometimes an exact mapping between $\tau: \mathcal{V} \to \mathcal{V}$ is constructed\footnote{$\mathcal{V}$ here denotes the domain in which the observables $V$ take values.}, and $\tau$ can even be stochastic. In that case, the transformed distribution $\widetilde{P}$ is defined as:
\begin{align}
    \widetilde{P}(v) = \ex_{\tau}\big[P\circ\tau(v)\big].
\end{align}
The fair pre-processing methods formulate an optimization problem that attempts to find the optimal $\widetilde{P}(V)$, where optimality is defined as minimizing some notion of distance to the original distribution $P(V)$. There are two different approaches here, that have different causal implications:
\begin{enumerate}[label=(\Alph*)]
    \item the protected attribute $X$ should be independent from the rest of observables $V \setminus X$ in the fair distribution $\widetilde{P}(V)$, written $X \ci V\setminus X$,
    \item the protected attribute $X$ should be independent from the the outcome $Y$ in the fair distribution $\widetilde{P}(V)$, written $X \ci Y$.
\end{enumerate}
The first approach requires that the effect of the attribute $X$ is entirely erased from the data. The second, less stringent option requires the independence $X \ci Y$ in $\widetilde{P}(V)$, which is equivalent with having TV$_{x_0, x_1}(\widehat{y}) = 0$. These two cases will be discussed separately in the remainder of the section. In Fig.~\ref{fig:fair-pred-summary} we give a schematic representation of the three categories of fair prediction methods, and in particular how they relate to a typical machine learning workflow. We next move onto formulating $\fpcfa$ for Task 2.

\begin{figure}
	\centering
	

		
		
		
		
		
		
		
		
		
		

		


    \includegraphics[width=\textwidth]{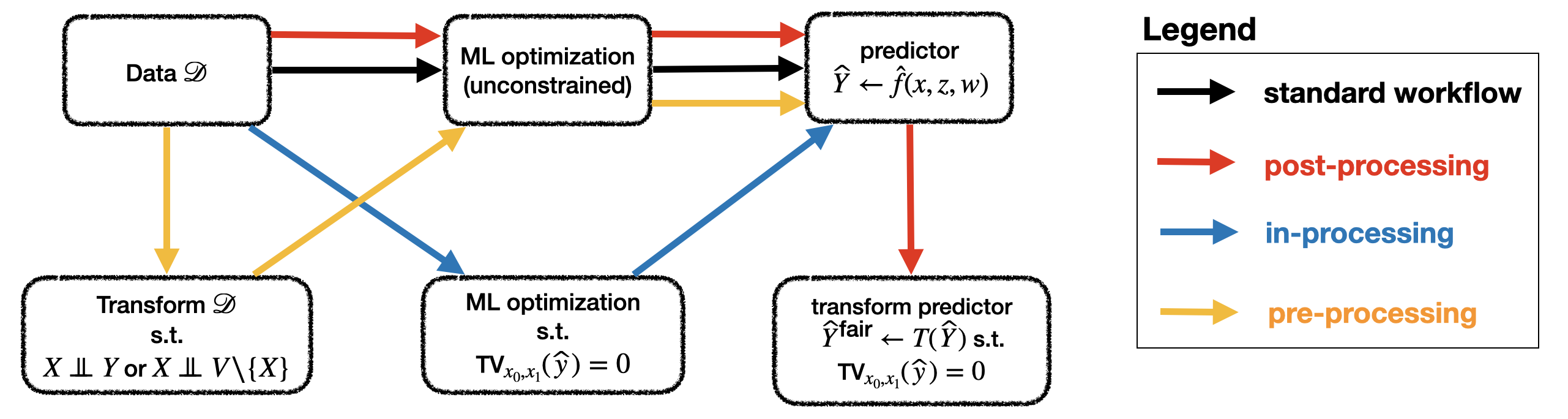}
	\caption{A schematic summary of the post-processing (red arrows), in-processing (blue), and pre-processing (yellow) fair prediction methods, compared to a typical ML workflow (black).}
	\label{fig:fair-pred-summary}
\end{figure}

\subsubsection{$\fpcfa$ for Task 2.}
Building on the previous definition of $\fpcfa$, we can now state its version in the context of fair prediction:
\begin{definition}[FPCFA continued for Task 2] \label{def:fpcfa-task-2}
$[\Omega$, $Q$ as before$]$
Let the true, unobserved generative SCM $\mathcal{M} = \langle V, U, P(U), F \rangle$, and let $\mathcal{A}$ be a set of assumption and $P(v)$ be the observational distribution generated by it. 
Let $\Omega^{\mathcal{A}}$ the space of all SCMs compatible with $\mathcal{A}$. 
The Fundamental Problem of Causal Fairness Analysis is to find a collection of measures $\mu_1, \dots, \mu_k$
such that the following properties are satisfied: 
\begin{enumerate}[label=(\arabic*)]
	\item $\mu$ is decomposable w.r.t. $\mu_1, \dots, \mu_k$; 
	\item $\mu_1, \dots, \mu_k$ are admissible w.r.t. the structural fairness criteria $Q_1, Q_2, \dots, Q_k$. 
	\item $\mu_1, \dots, \mu_k$ are as powerful as possible.
	\item $\mu_1, \dots, \mu_k$ are identifiable from the observational distribution $P(v)$ and class $\Omega^{\mathcal{A}}$.
\end{enumerate}
The final step of FPCFA for Task 2 is to construct an alternative SCM $\mathcal{M}'$ such that
\begin{enumerate}[resume, label=(\arabic*)]
    \item the measures $\mu_1, \dots, \mu_k$ satisfy that
    \begin{align}
        \mu_1(\mathcal{M}') = \dots = \mu_k(\mathcal{M}') = 0.
    \end{align}
\end{enumerate}
\end{definition}
To make matters explicit, in the formulation of $\fpcfa$ for Task 2, we want to ensure that the constructed predictor $\widehat{Y}$ satisfies
\begin{align} \label{eq:x-spec-cond}
    x\text{-DE}^{\text{sym}}_{x}(\widehat{y} \mid x_0) = x\text{-IE}^{\text{sym}}_{x}(\widehat{y} \mid x_0) = x\text{-SE}_{x_1, x_0}(\widehat{y}) = 0,
\end{align}
instead of just requiring that TV$_{x_0, x_1}(\widehat{y}) = 0$. The question we address formally next is whether the condition in Eq.~\ref{eq:x-spec-cond} can be achieved by methods that focus on minimizing TV. For this purpose, we prove the Fair Prediction Theorem that is formulated for in-processing methods in the linear case:

\begin{theorem}[Fair Prediction Theorem] \label{thm:fpt}
Let SFM$(n_Z, n_W)$ be the standard fairness model with $|Z| = n_Z$ and $|W| = n_W$. Let $E$ denote the set of edges of SFM$(n_Z, n_W)$. Further, let $\mathcal{S}^{\textit{linear}}_{n_Z, n_W}$ be the space of linear structural causal models (with the exception of $X$ variable which is Bernoulli) compatible with the SFM$(n_Z, n_W)$ and whose structural coefficients are drawn uniformly from $[-1, 1]^{|E|}$. An SCM $\mathcal{M} \in \mathcal{S}^{\textit{linear}}_{n_Z, n_W}$ is said to be $\epsilon$-TV-compliant if
\begin{alignat}{2}
    \label{eq:fpt-objective}
    \widehat{f}_{\text{fair}} = &\argmin_{f \text{linear}}        &\qquad& \ex[Y - f(X, Z, W)]^2\\
    &\text{subject to} &      & TV_{x_0, x_1}(f) = 0 \label{eq:fpt-constraint}
\end{alignat}
also satisfies
\begin{align}
     |x\text{-DE}_{x_0, x_1}(\widehat{f}_{\text{fair}} \mid x_0)| \leq \epsilon,  \\
     |x\text{-IE}_{x_0, x_1}(\widehat{f}_{\text{fair}} \mid x_0)| \leq \epsilon, \\
     |x\text{-SE}_{x_0, x_1}(\widehat{f}_{\text{fair}})| \leq \epsilon. 
\end{align}
Under the Lebesgue measure over $[-1, 1]^{|E|}$, the set of $0$-TV-compliant SCMs in SFM$(n_Z, n_W)$ has measure 0. Furthermore, for any $n_Z, n_W$, there exists an $\epsilon = \epsilon(n_Z, n_W)$ such that
\begin{align}
    \pr (\mathcal{M} \text{ is } \epsilon\text{-TV-compliant}) \leq \frac{1}{4}.
\end{align}
\end{theorem}
The proof is given in Appendix \ref{appendix:fpt}. The theorem states that, for a random linear SCM, the optimal fair predictor with TV measure equal to 0 will almost never have the $x$-specific fairness measures equal to 0. The remarkable implication of the theorem is that minimizing the TV measure provides no guarantees that the direct, indirect and spurious effects are also minimized. That is, the resulting fair classifier might not be causally meaningful. 

The Fair Prediction Theorem considers the linear case for in-processing methods, but we conjecture that it has implications for more complex settings too (see also empirical evidence on real data below). For example, note that in the optimization problem in Lines~\ref{eq:fpt-objective}-\ref{eq:fpt-constraint} we are searching over linear functions $f$ of $X, Z,$ and $W$. For pre-processing methods that achieve $X \ci \widehat{Y}$, the space of allowed functions $f$ would be even more flexible, but the underlying optimization problem would remain similar. Even though formal results are difficult to provide, our observations raise a serious concern about whether any of the fair prediction methods in the literature provide predictors that are well-behaved in a causal sense. We now exemplify this point empirically, by applying several well-known fair prediction methods on the COMPAS dataset.

\subsubsection{Empirical evaluation of the Fair Prediction Theorem}
Consider the following example based on the COMPAS dataset.

\begin{example}[COMPAS continued for Fair Prediction]
A team of data scientists from ProPublica have shown that the COMPAS dataset from Broward County contains a strong racial bias against minorities. They are now interested in producing fair predictions $\widehat{Y}$ on the dataset, to replace the biased predictions. To this end they implement:
\begin{enumerate}[label=(\roman*)]
	\item \label{class:rf} \textbf{baseline:} a random forest classifier trained without any fairness constraints,
	\item \label{class:reweigh} \textbf{pre-processing:} a logistic regression classifier trained with the \textit{reweighing} method \citep{kamiran2012data},
	\item \label{class:reductions} \textbf{in-processing:} \textit{fair reductions} approach of \cite{agarwal2018reductions} with a logistic regression base classifier,
	\item \label{class:rejectoption} \textbf{post-processing:} a random forest classifier trained without fairness constraints, with \textit{reject-option} post-processing applied \citep{kamiran2012decision}.
\end{enumerate}
The fair prediction algorithms \ref{class:reweigh}, \ref{class:reductions}, and \ref{class:rejectoption} are intended to set the TV measure to $0$. After constructing these predictors, the team make use of the Fairness Cookbook in Algorithm~\ref{algo:cookbook}. Following the steps of the Fairness Cookbook, the team computes the TV measure, together with the appropriate measures of direct, indirect, and spurious discrimination.

The obtained decompositions of the TV measures are shown in Figures \ref{fig:empirical-fp}\ref{class:reweigh}, \ref{fig:empirical-fp}\ref{class:reductions}, and \ref{fig:empirical-fp}\ref{class:rejectoption}. The ProPublica team notes that all methods substantially reduce the $\text{TV}_{x_0,x_1}(\widehat{y})$, however, the measures of direct, indirect, and, spurious effects are not necessarily reduced to $0$, consistent with the Fair Prediction Theorem. $\hfill \square$
\end{example}
\begin{figure}
	\centering
	\includegraphics[height=120mm,angle=0]{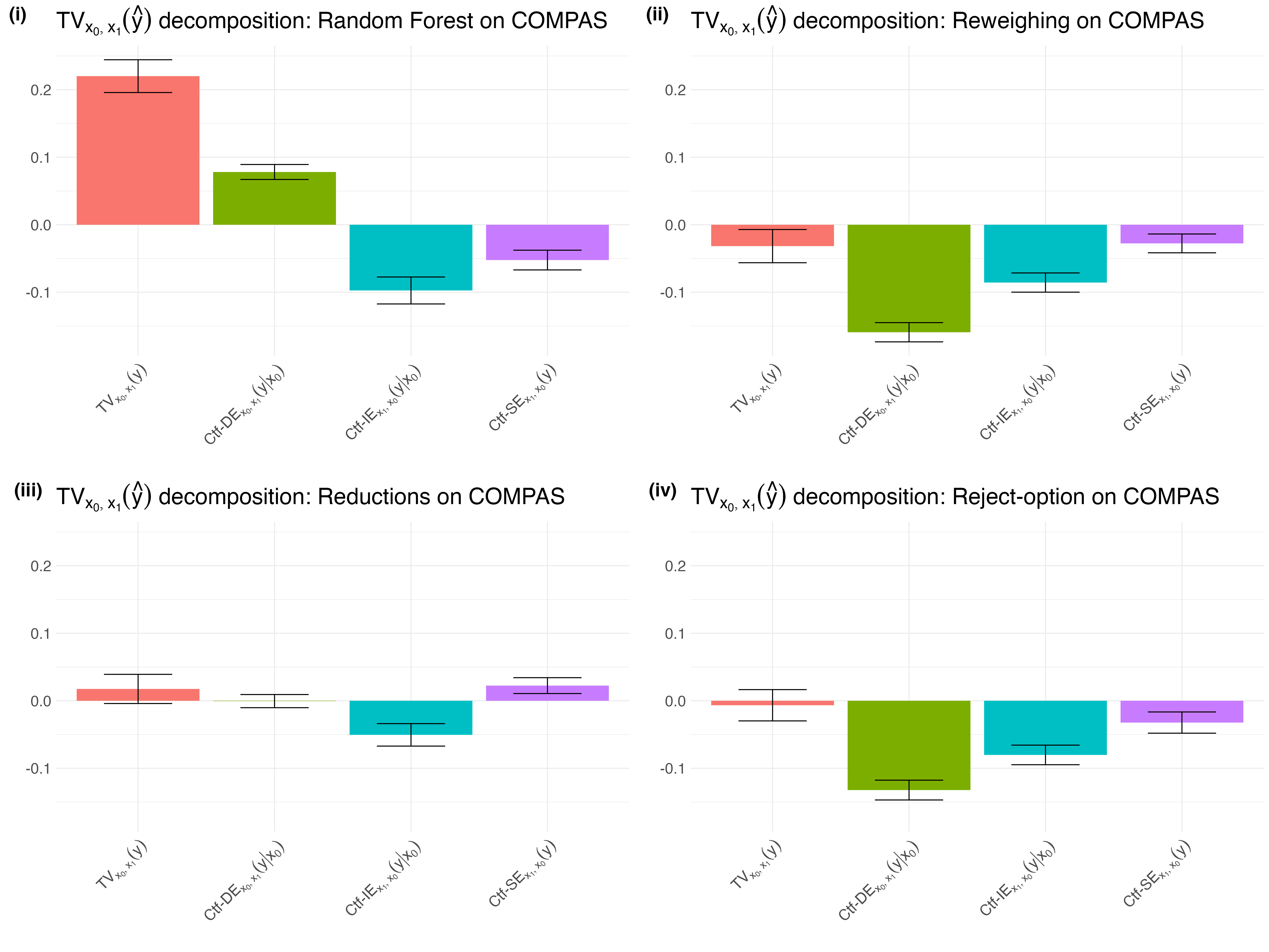}
	\caption{Causal Fairness Analysis applied to a standard prediction method (random forest, subfigure (i)) and three different fair prediction algorithms (reweighing \citep{kamiran2012data} in subfigure (ii), reductions \citep{agarwal2018reductions} in subfigure (iii), and reject-option \citep{kamiran2012decision} in subfigure (iv)). All of the fair predictions methods reduce the TV measure, but fail to nullify the causal measures of fairness. Confidence intervals of the measures, obtained using bootstrap, are shown as vertical bars.}
	\label{fig:empirical-fp}
\end{figure}

One class of fair prediction methods that are not addressed by the discussion above are the pre-processing methods that achieve the independence of the protected attribute with all the observables, namely $X \ci V \setminus \{X\}$. 

\subsubsection{Pre-processing methods that achieve $X \ci V \setminus \{X\}$}

A pre-processing method that achieves attribute independence ($X \ci V \setminus \{X\}$) is the proposal of \cite{dwork2012fairness}, in which in the pre-processing step the distribution 
$$ V \setminus \{X\} \mid X = x_0 \text{ is transported onto } V \setminus \{X\} \mid X = x_1. $$
However, as witnessed by the following example, such an approach does not guarantee that causal measures of fairness vanish:
\begin{example}[Failure of Optimal Transport methods] \label{ex:ot-fails}
A company is hiring prospective applicants for a new job position. Let $X$ denote gender ($x_0$ for male, $x_1$ for female), $W$ denotes a score on a test (taking to values, $\pm \epsilon$), $Y$ the outcome of the application ($Y = 0$ for no job offer, $Y = 1$ for job offer). The following SCM $\mathcal{M}$ describes the data generating process:
\begin{empheq}[left ={\mathcal{F}, P(U): \empheqlbrace}]{align}
                X &\gets U_X \\
                W &\gets \epsilon (2U_W-1) \\
                Y &\gets \begin{cases} U_Y \vee 1(W > 0) \text{ if } X = x_0 \\  U_Y \vee 1(W < 0) \text{ if } X = x_1  \end{cases} \\
		        \nonumber\\
 		            U_X &, U_Z, U_Y, U_Y \sim \text{Bernoulli}(0.5).       
\end{empheq}
After first part of the selection process, the company realized that they are achieving demographic parity
\begin{align}
    \text{TV}_{x_0, x_1}(y) = 0,
\end{align}
but they are uncertain whether they are causally fair, with respect to the direct and indirect effects. For this reason, they choose to optimally transport the conditional distributions, namely
\begin{align}
    W, Y \mid x_1 \overset{\tau}{\mapsto} W, Y \mid x_0,
\end{align}
where $\tau$ denotes the optimal transport map between the two distributions. By doing so, the company aims to make sure that both the direct and the indirect effect are equal to $0$.

The obtained optimal transport map $\tau$ can be described as follows:
\begin{align}
    \tau(w, y) = \begin{cases}
        (-\epsilon, 0) &\text{ if } (w,y) = (\epsilon, 0) \\
        (\epsilon, 1) &\text{ if } (w,y) = (\epsilon, 1) \\
        (\pm\epsilon, 1) \text{ w.p. } \frac{1}{2} &\text{ if } (w,y) = (-\epsilon, 1) 
    \end{cases}
\end{align}
The conditional distributions $W, Y \mid x_0$ and $W, Y \mid x_1$ are shown in Fig.~\ref{fig:ot-fails-map}, together with the optimal transport map. Denote by $\widetilde{W}, \widetilde{Y}$ the transformed values of $W, Y$. After the transformation, we compute the indirect effect, comparing the potential outcomes $\widetilde{Y}_{x_0}$ and $\widetilde{Y}_{x_0, \widetilde{W}_{x_1}}$, where the latter describes the potential outcome where $X = x_0$ along the direct pathway, and $W$ behaves like $\widetilde{W}$ under the intervention $X = x_1$. We compute the as follows:
\begin{align}
P(\widetilde{y}_{x_0, \widetilde{W}_{x_1}}) &= \sum_w P(\widetilde{y}_{x_0, w}, \widetilde{W}_{x_1} = w) \\
                                &= P(y_{x_0, \epsilon}, \widetilde{W}_{x_1} = \epsilon) + P(y_{x_0, -\epsilon}, \widetilde{W}_{x_1} = -\epsilon),
\end{align}
where the first term $P(y_{x_0, \epsilon}, \widetilde{W}_{x_1} = \epsilon)$ equals $\frac{1}{2}$, corresponding to $(U_W, U_Y) = (1, \{0, 1\})$ (weighted w.p. $1$) and $(U_W, U_Y) = (0, \{0, 1\})$ (weighted w.p. $\frac{1}{2}$). The second term, $P(y_{x_0, -\epsilon}, \widetilde{W}_{x_1} = -\epsilon)$ equals $\frac{1}{8}$, corresponding to $(U_W, U_Y) = (0, 1)$ (weighted w.p. $\frac{1}{2}$). Thus, we have that $P(y_{x_0, \widetilde{W}_{x_1}}) = \frac{5}{8}$. The term $P(\widetilde{y}_{x_0})= P(y_{x_0}) = P(y \mid x_0) = \frac{3}{4}$. Putting together, we have that
\begin{align}
    \text{NIE}_{x_0, x_1}(\widetilde{y}) = P(\widetilde{y}_{x_0, \widetilde{W}_{x_1}}) - P(\widetilde{y}_{x_0}) = \frac{5}{8} - \frac{3}{4} = -\frac{1}{8},
\end{align}
showing that the indirect effect after the optimal transport step is non-zero.
$\hfill \square$
\end{example}
\begin{figure}
     \centering
     \begin{subfigure}[b]{0.45\textwidth}
         \centering
         \begin{tikzpicture}
              [>=stealth, rv/.style={thick}, rvc/.style={triangle, draw, thick, minimum size=7mm}, node distance=18mm]
	 \pgfsetarrows{latex-latex};
	 \begin{scope}
		\node[rv] (1) at (-1.5,0) {$X$};
		\node[rv] (2) at (0,-1.5) {$W$};
		\node[rv] (3) at (1.5,0) {$Y$};
		\node[rv] (4) at (0, -2.5) {};
		\draw[->] (1) -- (3);
		\draw[->] (2) -- (3);
	 \end{scope}
         \end{tikzpicture}
         \caption{Causal graph corresponding to Example~\ref{ex:ot-fails}.}
         \label{fig:ot-fails-dag}
     \end{subfigure}
     \hfill
     \begin{subfigure}[b]{0.45\textwidth}
         \centering
         \includegraphics[width=\textwidth]{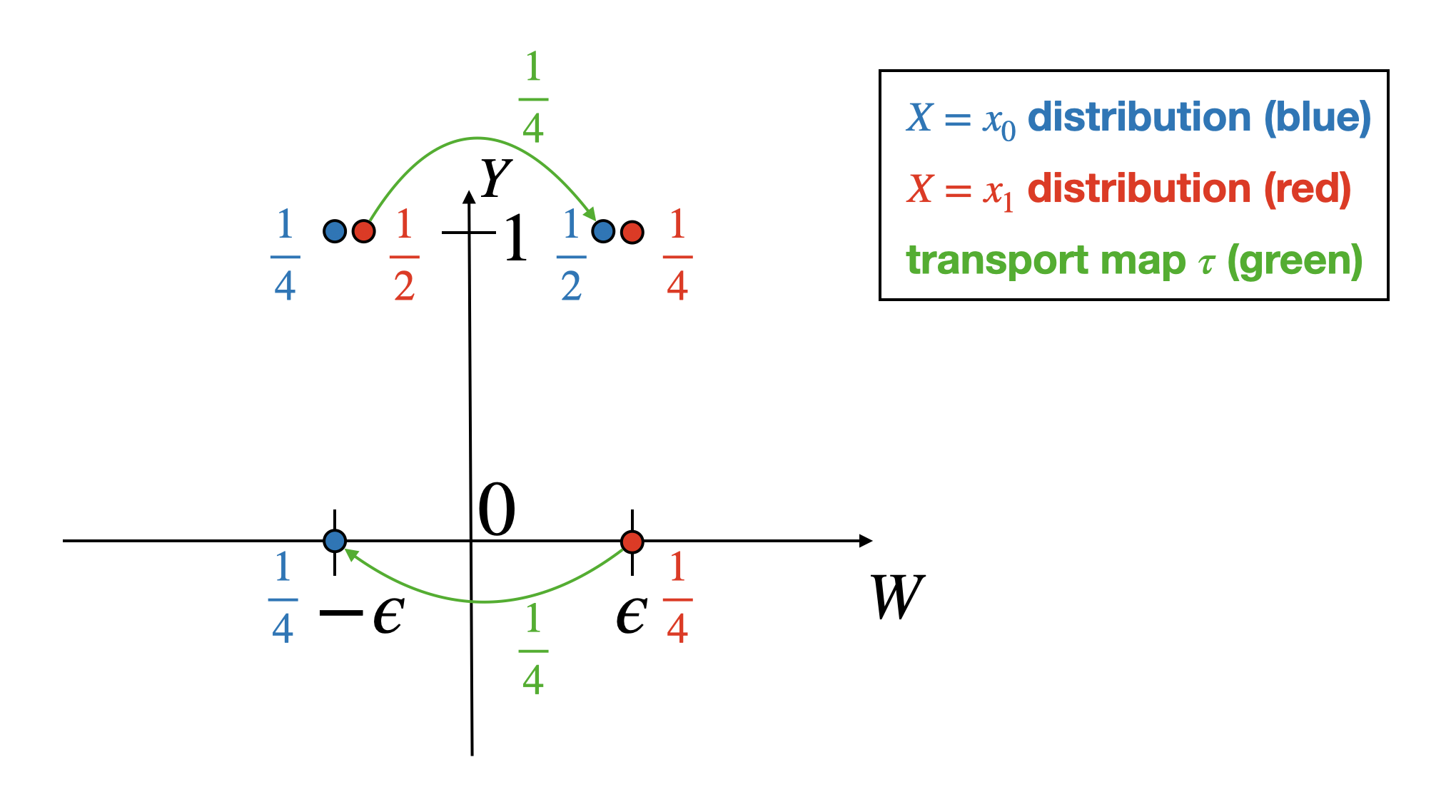}
         \caption{Conditional distributions $W, Y\mid x_0$ (blue) $W, Y\mid x_1$ (red), and the optimal transport map $\tau$ (green) from Example~\ref{ex:ot-fails}.}
         \label{fig:ot-fails-map}
     \end{subfigure}
\end{figure}
The reader might wonder about the underlying issue for why all of the discussed methods from previous literature fail. We next move onto explaining the shortcomings of these methods in more detail, and give two possible formulations that can help when constructing causally meaningful predictors.

\subsubsection{Towards the solution}
Our next goal is to remedy the pitfalls of the fair prediction methods discussed so far. In particular, we outline a strategy for ensuring that direct, indirect, and spurious effects vanish (or a subset of them, in case of business necessity). There are two conditions that are needed to guarantee causal behaviour of our predictor:
\begin{enumerate}[label = (\Roman*)]
    \item the causal structure of the SFM is preserved for the predictor $\widehat{Y}$,
    \item The identification expressions of $x$-DE, $x$-IE, and $x$-SE equal 0 in the new SCM $\mathcal{M}'$.
\end{enumerate}
We first show formally that the two conditions provide guarantees for the constructed classifier $\widehat{Y}$:

\begin{proposition}[Fair Predictor Causal Conditions] \label{prop:fair-pred-causal}
Let $\mathcal{M}$ be an SCM compatible with the SFM and let $\widehat{Y}$ be a predictor of the outcome $Y$ satisfying:
\begin{enumerate}[label=(\alph*)]
    \item $X, Z, W$ and $\widehat{Y}$ are compatible with the SFM,
    \item the identification expressions for $x\text{-DE}^{\text{sym}}_{x}(y \mid x_0), x\text{-DE}^{\text{sym}}_{x}(y \mid x_0), \text{ and } x\text{-SE}_{x_1, x_0}(y)$ equal $0$, namely
    \begin{align}
        \sum_{z, w} [P(y \mid x_1, z, w) - P(y\mid x_0, z, w)]P(w \mid x_1, z)P(z \mid x_0) &= 0 \\
        \sum_{z, w} [P(y \mid x_1, z, w) - P(y\mid x_0, z, w)]P(w \mid x_0, z)P(z \mid x_0) &= 0 \\ 
        \sum_{z, w} P(y\mid x_0, z, w)[P(w \mid x_1, z) - P(w \mid x_0, z)]P(z \mid x) &= 0 
                \end{align} 
        \begin{align}
        \sum_{z, w} P(y\mid x_1, z, w)[P(w \mid x_1, z) - P(w \mid x_0, z)]P(z \mid x) &= 0 \\
        \sum_z P(y \mid x_1, z)[P(z \mid x_1) - P(z \mid x_0)] &= 0.
    \end{align}
\end{enumerate}
Then, the predictor $\widehat{Y}$ satisfies:
\begin{align} \label{eq:x-spec-cond-rep}
    x\text{-DE}^{\text{sym}}_{x}(\widehat{y} \mid x_0) = x\text{-IE}^{\text{sym}}_{x}(\widehat{y} \mid x_0) = x\text{-SE}_{x_1, x_0}(\widehat{y}) = 0.
\end{align}
\end{proposition}
Based on the proposition, we can offer two solutions that give causally meaningful fair predictors, which are discussed next.

\subsubsection{Causally aware in-processing} The first option for constructing fair predictions that obey causal constraints is via in-processing. The simple idea is to replace the constraint TV$_{x_0,x_1}(\widehat{y}) = 0$ with a number of constraints that represent the identification expressions of the important causal quantities that we wish to minimize. After that, we can use the fact that the causal structure of the SFM is inherited for a predictor $\widehat{Y}$ constructed with in-processing. The formal statement of the in-processing approach is given in the following theorem:
\begin{theorem}[In-processing with causal constraints] \label{thm:fp-inproc}
Let $\mathcal{M}$ be an SCM compatible with the SFM. Let $\widehat{Y}$ be constructed as the optimal solution to
\begin{alignat}{2}
    \label{eq:inproc-causal-objective}
    \widehat{Y} = &\argmin_{f}        &\qquad& \ex[Y - f(X, Z, W)]^2\\
    &\text{subject to} &      & x\text{-DE}^{\text{ID}}_{x_0, x_1}(\widehat{y} \mid x_0) = 0 \\
    &&      & x\text{-DE}^{\text{ID}}_{x_1, x_0}(\widehat{y}\mid x_0) = 0 \\
    &&      & x\text{-IE}^{\text{ID}}_{x_0, x_1}(\widehat{y} \mid x_0) = 0 \\
    &&      & x\text{-IE}^{\text{ID}}_{x_1, x_0}(\widehat{y} \mid x_0) = 0 \\
    &&      & x\text{-SE}^{\text{ID}}_{x_1, x_0}(\widehat{y}) = 0
\end{alignat}
where $x\text{-DE}^{\text{ID}}$, $x\text{-IE}^{\text{ID}}$, and $x\text{-SE}^{\text{ID}}$ represent the identification expressions of the corresponding measures (as shown in Prop.~\ref{prop:fair-pred-causal}). Then the predictor $\widehat{Y}$ satisfies 
\begin{align} \label{eq:x-spec-cond-rep-2}
    x\text{-DE}^{\text{sym}}_{x}(\widehat{y} \mid x_0) = x\text{-IE}^{\text{sym}}_{x}(\widehat{y} \mid x_0) = x\text{-SE}_{x_1, x_0}(\widehat{y}) = 0.
\end{align}
\end{theorem}
The following remark shows that the result of the theorem holds even more broadly than just for the standard fairness model:
\begin{remark}[Robustness of in-processing with causal constraints]
Thm.~\ref{thm:fp-inproc} is stated for an SCM that is compatible with the SFM. However, such an assumption can be relaxed. In particular, the result of the theorem remains true even if the bidirected edges $X \bidir Y$, $Z \bidir Y$, and $W \bidir Y$ are present in the model.
\end{remark}

\subsubsection{Causally aware pre-processing}
After discussing a suitable in-processing approach, we can offer an approach based on pre-processing, inspired by the optimal transport approach of \cite{dwork2012fairness}:

\begin{definition}[Causal Individual Fairness (Causal IF)]
Let $\mathcal{M}$ be an SCM compatible with the SFM. Let the business necessity set be denoted as BN-set, taking values
\begin{align}
    \text{BN-set} \in \big\{ \emptyset, \{Z\}, \{W\}, \{Z, W\}  \big\}.
\end{align}
The Causal Individual Fairness algorithm performs sequential optimal transport of the distributions of $Z, W,$ and $Y$ (in this fixed topological ordering) conditional on the values of the parental set. The procedure is described formally in Algorithm~\ref{algo:causal-if}. 
\end{definition}
\begin{algorithm}[t]
    \caption{Causal Individual Fairness (Causal IF)} \label{algo:causal-if}
    \begin{algorithmic}
        \Statex \textbullet~\textbf{Inputs:} Dataset $\mathcal{D}$, SFM projection $\Pi_{\text{SFM}}(\mathcal{G})$, Business Necessity Set BN-set.
        \For{$V' \in \{ Z, W, Y \}$}
\If{$V' \notin $ BN-set}
    \State transport $V' \mid x_0, \pa(V')$ onto $V' \mid x_1, \tau^{\pa(V')}(\pa(V'))$
    \State let $\tau^{V'}$ denote the transport map 
\ElsIf{$V' \in $ BN-set}
    \State transport $V' \mid x, \pa(V')$ onto $V' \mid x, \tau^{\pa(V')}(\pa(V'))$ for $x \in \{x_0, x_1\}$
    \State let $\tau^{V'}$ denote the transport map
\EndIf
\EndFor
\end{algorithmic}
\end{algorithm}
To be even more precise, causal IF starts by optimally transporting
$$ Z \mid X = x_0 \text{ onto }  Z \mid X = x_1,$$
unless $Z$ is in the business necessity set. Let $\tau^Z$ denote the optimal transport map. Then, in the next step, the distribution of $W$ is transported, namely, $$ W \mid X = x_0, Z = z \text{ onto }  Z \mid X = x_1, Z = \tau^Z(z) \;\; \forall z.$$ In the final step the distribution of $Y$ is transported $$ Y \mid X = x_0, Z = z, W = w \text{ onto }  Z \mid X = x_1, Z = \tau^Z(z), W = \tau^W(w) \;\; \forall z, w.$$

\begin{theorem}[Soundness Causal Individual Fairness] \label{thm:causalIF}
Let $\mathcal{M}$ be an SCM compatible with the SFM. Let $\tau^Y$ be the optimal transport map obtained when applying Causal IF. Define an additional mechanism of the SCM $\mathcal{M}$ such that
\begin{align}
    \widetilde{Y} \gets \tau^Y(Y; X, Z, W).
\end{align}
For the transformed outcome $\widetilde{Y}$, we can then claim:
\begin{align} \label{eq:x-spec-cond-causal-if}
\text{if } Z \notin \text{BN-set} &\implies x\text{-SE}_{x_1, x_0}(\widetilde{y}) = 0. \\
\text{if } W \notin \text{BN-set} &\implies x\text{-IE}^{\text{sym}}_{x}(\widetilde{y} \mid x_0) =  0.
\end{align}
Furthermore, the transformed outcome $\widetilde{Y}$ also satisfies
\begin{align}
    x\text{-DE}^{\text{sym}}_{x}(\widetilde{y} \mid x_0) = 0.
\end{align}
\end{theorem}
The full proof of the theorem can be found in Appendix~\ref{appendix:causalIF}. After showing that the Causal IF procedure provides certain guarantees for the causal measures of fairness, we go back to Ex.~\ref{ex:ot-fails} to see exactly why the method of joint optimal transport fails to produce a causally meaningful predictor: 
\begin{remark}[Why joint Optimal Transport fails]
The first term in the indirect effect in Ex.~\ref{ex:ot-fails} after the joint transport map was applied was expanded as:
\begin{align}
    P(\widetilde{y}_{x_0, \widetilde{W}_{x_1}}) &= \sum_w P(\widetilde{y}_{x_0, w}, \widetilde{W}_{x_1} = w)
\end{align}
Typically, whenever such an effect is identifiable, one would expect the independence relation 
\begin{align}
    \widetilde{Y}_{x_0, w} \ci \widetilde{W}_{x_1}
\end{align}
to hold. However, the joint optimal transport map $\tau(w, y)$, which determines the value of $\widetilde{W}_{x_1}$, also depends on the exogenous variable $U_Y$. For this reason, $\widetilde{W}_{x_1}$ is also a function of $U_Y$, but so is $\widetilde{Y}_{x_0, w}$. In this example, the joint optimal transport step introduced spurious shared information between exogenous variables $U_W$ and $U_Y$, which resulted in
\begin{align}
    P(\widetilde{y}_{x_0, w}, \widetilde{W}_{x_1} = w) \neq P(\widetilde{y}_{x_0, w})P(\widetilde{W}_{x_1} = w),
\end{align}
which disables us from providing guarantees that the indirect effect vanishes. The Causal IF transport method, on the other hand, circumvents this problem and guarantees that $\widetilde{Y}_{x_0, w} \ci \widetilde{W}_{x_1}$.
\end{remark}

\section{Disparate Impact and Business Necessity} \label{DIBN}
In this section, we generalize the analysis introduced earlier, including the Fairness Cookbook (Alg.~\ref{algo:cookbook}), to consider more refined settings described by an arbitrary causal diagram. The main motivation for doing so comes from the observation that when analyzing disparate impact, quantities such as Ctf-DE$_{x_0,x_1}(y\mid x_0)$, Ctf-IE$_{x_0,x_1}(y\mid x_0)$, and Ctf-SE$_{x_0,x_1}(y)$ are insufficient to account for certain \textit{business necessity} requirements. For concreteness, consider the following example. 

\begin{example}[COMPAS continued]
The courts at Broward County, Florida, were using machine learning to predict whether individuals released on parole are at high risk of re-offending within 2 years. The algorithm is based on the demographic information $Z$ ($Z_1$ for gender, $Z_2$ for age), race $X$ ($x_0$ denoting White, $x_1$ Non-White), juvenile offense counts $J$, prior offense count $P$, and degree of charge $D$.

A causal analysis using the Fairness Cookbook revealed that:
\begin{align}
  \text{Ctf-IE}_{x_1, x_0}(y\mid x_1) &= -5.3\% \pm 0.4\%,\\
  \text{Ctf-SE}_{x_1, x_0}(y) &= -4.3\% \pm 0.9\%,
\end{align}
potentially indicating presence of disparate impact. Based on this information, a legal team of ProPublica filed a lawsuit to the district court, claiming discrimination w.r.t. the Non-White subpopulation based on the doctrine of disparate impact. After the court hearing, the judge rules that using the attributes age ($Z_1$), prior count ($P$), and charge degree ($D$) is not discriminatory, but using the attributes juvenile count ($J$) and gender ($Z_2$) is discriminatory. Data scientists at ProPublica need to consider how to proceed in the light of this new requirement for discounting the allowable attributes in the quantiative analysis. 
\end{example}
The difficulty in this example is that the quantity Ctf-SE$_{x_1, x_0}(y)$ measures the spurious discrimination between the attribute $X$ and outcome $Y$ as generated by \textit{both confounders} $Z_1$ and $Z_2$. Since using the confounder $Z_1$ is not considered discriminatory, but using the confounder $Z_2$ is,  the quantity Ctf-SE$_{x_1, x_0}(y)$ needs to be refined such that the spurious variations based on the different confounders are disentangled. A similar challenge is presented while computing the Ctf-IE$_{x_0, x_1}(y\mid x_0)$ measure. In fact, a more refined analysis is needed to disentangle the indirect and spurious variations that comes from the confounding set $Z$ or mediating set $W$ such that they can be explained separately. 

\subsection{Refining spurious discrimination}

\subsubsection{Markovian case}

\subsubsection{Semi-Markovian models}

\subsubsection{Identification of set-specific spurious effects}



\subsection{Refining indirect effects}
\subsubsection{Identification of set-specific indirect effects}

\subsection{Extended Fairness Cookbook}
\subsection{Extended Fairness Map}

\section{Conclusions} 
Modern automated decision-making systems based on AI are fueled by data, which encodes many complex historical processes and past, potentially discriminatory practices. Such data, imprinted with undesired biases, cannot by itself be used and expected to produce fair systems, regardless of the level of statistical sophistication of the methods used or the amount of data available. In light of this limitation in which more data or clever methods are not the solution, the AI designer is left to search for a new notion of what a fair reality should look like. By and large, the literature on fair machine learning attempts to address this question by formulating (and then optimizing) statistical notions about how fairness should be measured. Still, as many of the examples in this manuscript demonstrated, statistical notions fall short of providing a satisfactory answer for what a fair reality should entail. Using decision systems that arise when only considering statistical notions of fairness may be causally meaningless, or even have unintended and possibly catastrophic  consequences. 

We combined in this manuscript two ingredients to address this challenge, (i) the language of causality and (ii) legal doctrines of discrimination so as to provide a sound basis for imagining what a fair reality should look like and represent society’s norms and expectations. This formalization of the fairness problem will allow the communication between the key stakeholders involved in developing such systems in practice, including computer scientists, statisticians, and data scientists on the one hand and social, legal, and ethical experts on the other. A key observation is that mapping social, ethical, and legal norms onto statistical measures is a challenging task. A formulation we propose explicitly in the form of the Fundamental Problem of Causal Fairness Analysis is to map such social norms onto the underlying causal mechanisms and causal measures of fairness associated with these particular mechanisms. We believe such an approach can help data scientists to be more transparent when measuring discrimination and can also help social scientists to ground their principles and ideas in a formal mathematical language that is amenable to implementation. 

The final important distinction introduced in this manuscript is between the different fairness tasks, namely (i) bias detection and quantification, (ii) fair prediction, and (iii) fair decision-making. The first task helps us to understand how much (and if any) bias exists in our data. The task of fair prediction allows us to correct for (parts or entirety) of this bias and envisage a more fair world in which such bias is removed. We leave as future work the precise formulation of the third task based on the principles developed here. Achieving fairness in the real world requires interventions, such as affirmative action. However, such interventions can have various complex consequences and implications, and interfacing the principles introduced in this manuscript with key ideas in economics and econometrics is an essential next step in designing fair systems.

\section{Acknowledgements} 
We thank Kai-Zhan Lee for the feedback and help in improving this manuscript. 
This work was done in part while Drago Plecko was visiting the Causal AI lab at Columbia University. 
This research was supported in part by the NSF, ONR, AFOSR, DoE, Amazon, JP Morgan, and The Alfred P. Sloan Foundation. 

\newpage
\appendix

\section{Proofs}
In this section, we provide the proofs of the main theorems presented in the manuscript. In particular, we give the proof for the Fairness Map theorem (Thm.~\ref{thm:map}), soundness of the SFM theorem (Thm.~\ref{thm:sfm}), the Fair Prediction theorem (Thm.~\ref{thm:fpt}), and the soundness of the Causal Individual Fairness procedure (Thm.~\ref{thm:causalIF}).

\subsection{Proof of Thm.~\ref{thm:map}} \label{appendix:map}
The proof of Thm.~\ref{thm:map} is organized as follows. The full list of implications contained in the Fairness Map in Fig.~\ref{fig:map} is given in in Tab.~\ref{table:map-imp}. For each implication, we indicate the lemma in which the implication proof is given.

\begin{table}
    \renewcommand{\arraystretch}{1.3}
	\begin{tabular}{|M{0.5cm}|M{10cm}|M{1.5cm}|} 
        \hline
		& Implication & Proof \\[1mm] 	\hline
		\multirow{4}{*}{\rotatebox{90}{\small power\;\;}} & Unit-TE $\implies$ $v'$-TE $\implies$ $z$-TE $\implies$ ETT $\implies$ TE & Lem.~\ref{lem:causal-power} \\[1mm] 	\cline{2-3}
		&Unit-DE $\implies$ $v'$-DE $\implies$ $z$-DE $\implies$ Ctf-DE $\implies$ NDE & Lem.~\ref{lem:causal-power} \\[1mm] 	\cline{2-3}
		&Unit-IE $\implies$ $v'$-IE $\implies$ $z$-IE $\implies$ Ctf-IE $\implies$ NIE & Lem.~\ref{lem:causal-power} \\[1mm] 	\cline{2-3}
		&Exp-SE $\iff$ Ctf-SE & Lem.~\ref{lem:spurious-power}\\[1mm] \hline
		\multirow{3}{*}{\rotatebox{90}{\small admissibility\,}}& S-SE $\implies$ Ctf-SE & Lem.~\ref{lem:spurious-adm}\\[1mm] 	\cline{2-3}
		&S-DE $\implies$ unit-DE & Lem.~\ref{lem:direct-adm}\\[1mm] \cline{2-3}
		&S-IE $\implies$ unit-IE & Lem.~\ref{lem:indirect-adm}\\[1mm] \hline
		\multirow{7}{*}{\rotatebox{90}{\small decomposability\;\;\;\;}} &NDE $\wedge$ NIE $\implies$ TE & Lem.~\ref{lem:ext-med} \\[1mm] 	\cline{2-3}
		&Ctf-DE  $\wedge$ Ctf-IE $\implies$ ETT & Lem.~\ref{lem:ext-med} \\[1mm] 	\cline{2-3}
		&$z$-DE $\wedge$ $z$-IE $\implies$ $z$-TE & Lem.~\ref{lem:ext-med} \\[1mm] 	\cline{2-3}
		&$v'$-DE  $\wedge$ $v'$-IE $\implies$ $v'$-TE & Lem.~\ref{lem:ext-med} \\[1mm] 	\cline{2-3}
		&unit-DE $\wedge$ unit-IE $\implies$ unit-TE & Lem.~\ref{lem:ext-med} \\[1mm] 	\cline{2-3}
		&TE $\wedge$ Exp-SE $\implies$ TV & Lem.~\ref{lem:tv-decompositions} \\[1mm] 	\cline{2-3}
		&ETT $\wedge$ Ctf-SE $\implies$ TV & Lem.~\ref{lem:tv-decompositions} \\[1mm] 	\hline
	\end{tabular}
	\caption{List of implications in the Fairness Map in Fig.~\ref{fig:map}.}
	\label{table:map-imp}
\end{table}

\begin{lemma}[Power relations of causal effects] \label{lem:causal-power}
    The total, direct, and indirect effects admit the following relations of power:
    \begin{align}
       \text{unit-TE}_{x_0, x_1}(y(u)) = 0 \;\forall u &\implies v'\text{-TE}_{x_0, x_1}(y\mid v') = 0 \;\forall v' \implies z\text{-TE}_{x_0, x_1}(y \mid z) = 0 \;\forall z\\ &\implies \text{ETT}_{x_0, x_1}(y \mid x) = 0 \;\forall x \implies \text{TE}_{x_0, x_1}(y) = 0, \\
       \text{unit-DE}_{x_0, x_1}(y(u)) = 0 \;\forall u &\implies v'\text{-DE}_{x_0, x_1}(y\mid v') = 0 \;\forall v' \implies z\text{-DE}_{x_0, x_1}(y\mid z) = 0 \;\forall z\\ &\implies \text{Ctf-DE}_{x_0, x_1}(y \mid x) = 0 \;\forall x \implies \text{NDE}_{x_0, x_1}(y) = 0, \\
       \text{unit-IE}_{x_0, x_1}(y(u)) = 0 \;\forall u &\implies v'\text{-IE}_{x_0, x_1}(y\mid v') = 0 \;\forall v' \implies z\text{-IE}_{x_0, x_1}(y\mid z) = 0 \;\forall z\\ &\implies \text{Ctf-IE}_{x_0, x_1}(y \mid x) = 0 \;\forall x \implies \text{NIE}_{x_0, x_1}(y) = 0.
    \end{align}
\end{lemma}
\begin{proof}
    We prove the statement for total effects (the proof for direct and indirect is analogous). We start by showing that ETT is more powerful than TE.
	\begin{align*}
		\text{TE}_{x_0,x_1}(y) =& P(y_{x_1}) - P(y_{x_0}) \\
		=& \sum_{x} \big[P(y_{x_1} \mid x) - P(y_{x_0} \mid x)\big] P(x)\\
		=& \sum_{x} \text{ETT}_{x_0, x_1}(y \mid x) P(x).
	\end{align*}
	Therefore, if $\text{ETT}_{x_0, x_1}(y \mid x) = 0 \;\forall x$ then $\text{TE}_{x_0, x_1}(y) = 0$. Next, we can write
	\begin{align*}
	    \text{ETT}_{x_0, x_1}(y \mid x) &= P(y_{x_1} \mid x) - P(y_{x_0} \mid x) \\
	                                    &= \sum_z \big[P(y_{x_1} \mid x, z) - P(y_{x_0} \mid x, z)\big]P(z \mid x)\\
	                                    &= \sum_z \big[P(y_{x_1} \mid z) - P(y_{x_0} \mid z)\big]P(z \mid x) && Y_x \ci X \mid Z \text{ in SFM}\\
	                                    &= \sum_z z\text{-TE}_{x_0, x_1}(y\mid z)P(z \mid x).
	\end{align*}
	Therefore, if $z\text{-TE}_{x_0, x_1}(y\mid z) = 0 \;\forall z$ then $\text{ETT}_{x_0, x_1}(y\mid x) = 0\;\forall x$. Next, for a set $V' \subseteq V$ such that $Z \subseteq V'$, we can write
	\begin{align*}
	    z\text{-TE}_{x_0, x_1}(y) &= P(y_{x_1} \mid z) - P(y_{x_0} \mid z) \\
	                              &= \sum_{v'\setminus z} P(y_{x_1} \mid z, v'\setminus z) - P(y_{x_0} \mid z, v'\setminus z) P(v'\setminus z \mid z) \\
	                              &= \sum_{v'\setminus z} v'\text{-TE}_{x_0, x_1}(y\mid v')P(v'\setminus z \mid z).
	\end{align*}
	Therefore, if $v'\text{-TE}_{x_0, x_1}(y\mid v') = 0 \;\forall v'$ then $z\text{-TE}_{x_0, x_1}(y\mid z) = 0 \;\forall z$. Next, notice that
	\begin{align*}
	    v'\text{-TE}_{x_0, x_1}(y) &= P(y_{x_1} \mid v') - P(y_{x_0} \mid v') \\
	                               &= \sum_u  \big[ y_{x_1}(u) - y_{x_0}(u) \big]P(u \mid v') \\
	                               &= \sum_u  \text{unit-TE}_{x_0, x_1}(y(u))P(u \mid v').
	\end{align*}
	Therefore, if $\text{unit-TE}_{x_0, x_1}(y(u)) = 0 \;\forall u$ then $v'\text{-TE}_{x_0, x_1}(y\mid v') = 0 \;\forall v'$.
\end{proof}

\begin{lemma}[Power relations of spurious effects] \label{lem:spurious-power}
    The criteria based on Ctf-SE and Exp-SE are equivalent. Formally, 
    \begin{align}
        \text{Exp-SE}_{x}(y) = 0\; \forall x \iff \text{Ctf-SE}_{x, x'}(y) = 0\; \forall x \neq x'.
    \end{align}
\end{lemma}
\begin{proof}
    \begin{align*}
		\text{Exp-SE}_x(y) &= P(y \mid x) - P(y_x) \\
		                   &= P(y \mid x) - P(y_x\mid x)P(x) - P(y_x\mid x')P(x') \\
		                   &= P(y \mid x)[1 - P(x)] - P(y_x\mid x')P(x') \\
		                   &= P(y \mid x)P(x') - P(y_x\mid x')P(x') \\
		                   &= -P(x')\text{Ctf-SE}_{x', x}(y).
	\end{align*}
	Assuming $P(x') > 0$, the claim follows. 
\end{proof}

\begin{lemma}[Admissibility w.r.t. structural direct] \label{lem:direct-adm}
    The structural direct effect criterion ($X \notin \pa(Y)$) implies the absence of unit-level direct effect. Formally:
    \begin{align}
        \text{S-DE} \implies \text{unit-DE}_{x_0, x_1}(y(u)) = 0 \;\forall u.
    \end{align}
\end{lemma}
\begin{proof}
    Suppose that $X \notin \pa(Y)$. Note that:
    \begin{align*}
        \text{unit-DE}_{x_0, x_1}(y(u)) &= y_{x_1, W_{x_0}}(u) - y_{x_0}(u) \\
                                        &= f_Y(x_1, W_{x_0}(u), Z(u), u_Y) -  f_Y(x_0, W_{x_0}(u), Z(u), u_Y) \\
                                        &= f_Y(W_{x_0}(u), Z(u), u_Y) -  f_Y(W_{x_0}(u), Z(u), u_Y) && X \notin \pa(Y)\\
                                        &= 0.
    \end{align*}
\end{proof}

\begin{lemma}[Admissibility w.r.t. structural indirect] \label{lem:indirect-adm}
    The structural indirect effect criterion ($\de(X) \notin \pa(Y)$) implies the absence of unit-level indirect effect. Formally:
    \begin{align}
        \text{S-IE} \implies \text{unit-IE}_{x_1, x_0}(y(u)) = 0 \;\forall u.
    \end{align}
\end{lemma}
\begin{proof}
    Suppose that $\de(X) \notin \pa(Y)$. Let $W_{de} \subseteq W$ be the subset of mediators $W$ which are in $\de(X)$, and let $W^C_{de}$ be its complement in $W$. Then, by assumption, $W_C \notin \pa(Y)$. We can write:
    \begin{align*}
        \text{unit-IE}_{x_1, x_0}(y(u)) &= y_{x_1, W_{x_0}}(u) - y_{x_1}(u) \\
                                        &= f_Y(x_1, (W^C_{de})_{x_0}(u), Z(u), u_Y) -  f_Y(x_1, (W^C_{de})_{x_1}(u), Z(u), u_Y) \\
                                        &= f_Y(x_1, W^C_{de}(u), Z(u), u_Y) -  f_Y(x_1, W^C_{de}(u), Z(u), u_Y) && W^C_{de} \notin \de(X)\\
                                        &= 0.
    \end{align*}
\end{proof}

\begin{lemma}[Admissibility w.r.t. structural spurious] \label{lem:spurious-adm}
    The structural spurious effect criterion ($U_X \cap \an(Y) = \emptyset \wedge \an(X) \cap \an(Y) = \emptyset$) implies counterfactual spurious effect is $0$. Formally:
    \begin{align}
        \text{S-SE} \implies \text{Ctf-SE}_{x_0, x_1}(y) = 0 \;\forall u.
    \end{align}
\end{lemma}
\begin{proof}
    Note that S-SE implies there is no open backdoor path between $X$ and $Y$. As a consequence, we know that
    \begin{align*}
        Y_{x} \ci X.
    \end{align*}
    Furthermore, the absence of backdoor paths also implies we can use the 2nd rule of do-calculus (Action/Observation Exchange). Therefore, we can write:
    \begin{align*}
        \text{Ctf-SE}_{x_0, x_1}(y) &= P(y_{x_0} \mid x_1) - P(y \mid x_0) \\
                                    &= P(y_{x_0}) - P(y \mid {x_0}) && \text{since } Y_{x} \ci X \\
                                    &= P(y_{x_0}) - P(y_{x_0}) &&\text{Action/Observation Exchange}\\
                                    &= 0.
    \end{align*}
\end{proof}

\begin{lemma}[Extended Mediation Formula] \label{lem:ext-med}
    The total effect can be decomposed into its direct and indirect contributions on every level of the population axes in the explainability plane. Formally, we write:
    \begin{align}
    \text{TE}_{x_0, x_1}(y) &= \text{NDE}_{x_0, x_1}(y) - \text{NIE}_{x_1, x_0}(y) \\
    \text{ETT}_{x_0, x_1}(y \mid x) &= \text{Ctf-DE}_{x_0, x_1}(y \mid x) - \text{Ctf-IE}_{x_1, x_0}(y \mid x) \\
    z\text{-TE}_{x_0, x_1}(y \mid z) &= z\text{-DE}_{x_0, x_1}(y \mid z) - z\text{-IE}_{x_1, x_0}(y \mid z)\\
    v'\text{-TE}_{x_0, x_1}(y \mid v') &= v'\text{-DE}_{x_0, x_1}(y \mid v') - v'\text{-IE}_{x_1, x_0}(y \mid v')\\
    \text{unit-TE}_{x_0, x_1}(y(u)) &= \text{unit-DE}_{x_0, x_1}(y(u)) - \text{unit-IE}_{x_1, x_0}(y(u)).
\end{align}
\end{lemma}
\begin{proof}
    The proof follows from the structural basis expansion from Eq.~\eqref{eq:down}. In particular, note that
    \begin{align}
        E\text{-TE}_{x_1, x_0}(y \mid E) &= P(y_{x_1} \mid E) - P(y_{x_0} \mid E) \\
        &= P(y_{x_1} \mid E) - P(y_{x_1, W_{x_0}} \mid E) + P(y_{x_1, W_{x_0}} \mid E)  - P(y_{x_0} \mid E) \\
        &= -E\text{-IE}_{x_1, x_0}(y \mid E) + E\text{-DE}_{x_1, x_0}(y \mid E).
    \end{align}
    By using different events $E$ the claim follows.
\end{proof}

\begin{lemma}[TV Decompositions] \label{lem:tv-decompositions}
    The total variation (TV) measure admits the following two decompositions
    \begin{align}
        \text{TV}_{x_0, x_1}(y) &= \text{Exp-SE}_{x_1}(y) + \text{TE}_{x_0, x_1}(y) - \text{Exp-SE}_{x_0}(y)\\
		                        &= \text{ETT}_{x_0, x_1}(y \mid x_0) - \text{Ctf-SE}_{x_1, x_0}.
    \end{align}
\end{lemma}
\begin{proof}
    We write
	\begin{align*}
		\text{TV}_{x_0, x_1}(y) &= P(y \mid x_1) - P(y \mid x_0) \\
		&= P(y\mid x_1) - P(y_{x_1}) + P(y_{x_1}) - P(y_{x_0}) + P(y_{x_0}) - P(y \mid x_0) \\
		&= \text{Exp-SE}_{x_1}(y) + \text{TE}_{x_0, x_1}(y) - \text{Exp-SE}_{x_0}(y).
	\end{align*}
	Alternatively, we can write
	\begin{align*}
		\text{TV}_{x_0, x_1}(y) &= P(y \mid x_1) - P(y \mid x_0) \\
		&= P(y\mid x_1) - P(y_{x_1} \mid x_0) + P(y_{x_1} \mid x_0)  - P(y \mid x_0) \\
		&= \text{ETT}_{x_0, x_1}(y \mid x_0) - \text{Ctf-SE}_{x_1, x_0}(y),
	\end{align*}
	which completes the proof.
\end{proof}

\subsection{Soundness of the SFM: Proof of Theorem \ref{thm:sfm}} \label{appendix:sfm}
\begin{proof}
    The proof consists of two parts. In the first part, we show that the quantities where the event $E$ is either of $\emptyset, \lbrace x \rbrace, \lbrace z \rbrace$ (corresponding to the first three rows of the fairness map) are identifiable under the assumptions of the Standard Fairness Model. We in particular show that TE$_{x_0, x_1}(y)$, Exp-SE$_{x}(y)$, TE$_{x_0, x_1}(y \mid z)$, ETT$_{x_0, x_1}(y \mid x)$ and Ctf-DE$_{x_0, x_1}(y \mid x)$ are identifiable (it follows from very similar arguments that all other quantities are also identifiable. From this, it follows that for any graph $\mathcal{G}$ compatible with $\mathcal{G}_{\text{SFM}}$, the quantities of interest are (i) identifiable; (ii) their identification expression is the same. This in turn shows that using $\mathcal{G}_{\text{SFM}}$ instead of the full $\mathcal{G}$ does not hurt identifiability of these quantities.
    In the second part of the proof, we show that any contrast defined by an event $E$ which contains either $W = w$ or $Y = y$ is not identifiable under some very mild conditions (namely the existence of a path $X \rightarrow W_{i_1} \rightarrow ... \rightarrow W_{i_k} \rightarrow Y$). This part of the proof, complementary to the first part, shows that for contrasts with event $E$ containing post-treatment observations, even having the full graph $\mathcal{G}$ would not make the expression identifiable.
    All of the proofs here need to be derived from first principles, since the graph $\mathcal{G}_{\text{SFM}}$ contains ``groups" of variables $Z$ and $W$, making the standard identification machinery \citep{pearl:2k} not directly applicable.
    
    \noindent Part I: Note that for identifying TE$_{x_0, x_1}(y)$ we need to identify $P(y_x)$. We can write
    \begin{align*}
        P(y_x) &= P(y \mid do(x)) \\
               &= \sum_z P(y \mid do(x), z) P(z \mid do(x)) && \text{Law of Total Probability} \\
               &= \sum_z P(y \mid x, z) P(z) && (Y \ci X \mid Z)_{_{\mathcal{G}_{\underline{X}}}}, (X \ci Z)_{\mathcal{G}_{\overline{X}}}
    \end{align*}
    from which it follows that TE$_{x_0, x_1}(y) = \sum_z [P(y \mid x_1, z) - P(y \mid x_0, z)] P(z)$. Note that the identifiability of TE$_{x_0, x_1}(y \mid z)$ also follows from the above derivation, namely TE$_{x_0, x_1}(y \mid z) = \sum_z [P(y \mid x_1, z) - P(y \mid x_0, z)]$, and so does Exp-SE$_x(y) = \sum_z P(y \mid x, z)[P(z) - P(z \mid x)]$. We are now left with showing that ETT$_{x_0, x_1}(y \mid x)$ and Ctf-DE$_{x_0, x_1}(y \mid x)$ are also identifiable. These are Layer 3, counterfactual quantities and therefore rules of do-calculus will not suffice. To be able to use independence statements of counterfactual variables, we will make use of the \textbf{make-cg} algorithm of \cite{shpitser2012counterfactuals} for construction of counterfactual graphs, which extends the twin-network approach of \cite{balke1994counterfactual}. Therefore, when using for an expression of the form $Y_x = y, X = x'$, we obtain the following counterfactual graph
    \begin{center}
		\begin{tikzpicture}
	 [>=stealth, rv/.style={thick}, rvc/.style={triangle, draw, thick, minimum size=9mm}, node distance=18mm]
	 \pgfsetarrows{latex-latex};
	 \begin{scope}
		\node[rv] (0) at (0,1) {$Z$};
		\node[rv] (1) at (-1.5,0) {$X$};
		\node[rv] (2) at (0,-1) {$W_x$};
		\node[rv] (3) at (1.5,0) {$Y_x$};
		\path[<->, dashed] (0) edge[bend right = 10] (1);
		\draw[->] (0) -- (3);
		\draw[->] (0) -- (2);
		\draw[->] (2) -- (3);
	 \end{scope}
	 \end{tikzpicture}
	\end{center}
    from which we can see that $Y_x \ci X \mid Z$. Therefore, 
    \begin{align*}
        \text{ETT}_{x_0, x_1}(y) &= P(y_{x_1} \mid x) - P(y_{x_0} \mid x) \\
                                 &= \sum_z [P(y_{x_1} \mid x, z) - P(y_{x_0} \mid x, z)] P(z \mid x) && \text{Law of Total Probability}\\
                                 &= \sum_z [P(y \mid x_1, z) - P(y \mid {x_0}, z)] P(z \mid x) && Y_x \ci X \mid Z.
    \end{align*}
    Finally, for identifying Ctf-DE$_{x_0, x_1}(y \mid x)$ we use make-cg applied to $\mathcal{G}_{\text{SFM}}$ and $y_{x_1, w}, w_{x_0}, x, z$ to obtain
    \begin{center}
		\begin{tikzpicture}
	 [>=stealth, rv/.style={thick}, rvc/.style={triangle, draw, thick, minimum size=9mm}, node distance=18mm]
	 \pgfsetarrows{latex-latex};
	 \begin{scope}
		\node[rv] (0) at (0,1) {$Z$};
		\node[rv] (1) at (-1.5,0) {$X$};
		\node[rv] (2) at (0,-1) {$W_{x_0}$};
		\node[rv] (3) at (1.5,0) {$Y_{x_1, w}$};
		\path[<->, dashed] (0) edge[bend right = 10] (1);
		\draw[->] (0) -- (3);
		\draw[->] (0) -- (2);
	 \end{scope}
	 \end{tikzpicture}
	\end{center}
	from which we can say that $Y_{x_1, w} \ci (W_{x_0}, X) \mid Z$. Therefore, we can write
	\begin{align*}
	    \text{Ctf-DE}_{x_0, x_1}(y \mid x) &= P(y_{x_1, W_{x_0}} \mid x) - P(y_{x_0, W_{x_0}} \mid x) \\ 
	                                       &= \sum_z [P(y_{x_1, W_{x_0}} \mid x, z) - P(y_{x_0, W_{x_0}} \mid x, z)] P(z \mid x) && \text{Law of Total Probability} \\
	                                       &= \sum_{z, w} [P(y_{x_1, w}, w_{x_0} \mid x, z) - P(y_{x_0, w}, w_{x_0} \mid x, z)]P(z \mid x) && \text{Counterfactual unnesting} \\
	                                       &= \sum_{z, w} [P(y_{x_1, w} \mid x, z) - P(y_{x_0, w} \mid x, z)]P(w_{x_0} \mid z)P(z \mid x) && Y_{x_1, w} \ci W_{x_0} \mid Z \\
	                                       &= \sum_{z, w} [P(y \mid x_1, z, w) - P(y_{x_0} \mid x_0, z, w)]P(w \mid x_0, z)P(z \mid x) && Y_{x} \ci X \mid Z, W_{x_0} \ci X \mid Z.
	\end{align*}
	
	\noindent Part II: We next need to show that any contrast with either $W = w$ or $Y = y$ in the event $E$ is not identifiable, even if using the full graph $\mathcal{G}$. We show this for the quantity $P(y_{x_1} \mid x_0, w)$, since other similar quantities work analogously. Assume for simplicity that (i) variable $Z = \emptyset$; (ii) there are no bidirected edges between the $W$ variables. The latter assumption clearly makes the identifiability task easier, since adding bidirected edges can never help identification of quantities. Before we continue, we give an example of a graph $\mathcal{G}$ compatible with $\mathcal{G}_{\text{SFM}}$ for which $P(y_{x_1} \mid x_0, w)$ is identifiable. Consider the graph
	\begin{center}
		\begin{tikzpicture}
	 [>=stealth, rv/.style={thick}, rvc/.style={triangle, draw, thick, minimum size=9mm}, node distance=18mm]
	 \pgfsetarrows{latex-latex};
	 \begin{scope}
		\node[rv] (1) at (-1.5,0) {$X$};
		\node[rv] (2) at (-0.75,-1.5) {$W_{1}$};
		\node[rv] (3) at (0.75,-1.5) {$W_{2}$};
		\node[rv] (4) at (1.5,0) {$Y$};
		\draw[->] (1) -- (4);
		\draw[->] (1) -- (2);
		\draw[->] (3) -- (4);
	 \end{scope}
	 \end{tikzpicture}
	\end{center}
	and notice that
	\begin{align*}
	    P(y_{x_1} \mid x_0, w_1, w_2) &= P(y_{x_1} \mid x_0, w_2) && Y_{x_1} \ci W_1 \\
	                                  &= P(y_{x_1} \mid w_2) && Y_{x_1} \ci X \\
	                                  &= P(y \mid x_1, w_2) && (Y \ci X)_{\mathcal{G}_{\underline{X}}}.
	\end{align*}
	However, this example is somewhat pathological, since there is no indirect path between $X$ and $Y$ mediated by $W$. In this case, considering the set $W$ is arguably not relevant. Therefore, assume instead that a path $X \rightarrow W_{i_1} \rightarrow ... \rightarrow W_{i_k} \rightarrow Y$ exists. Then, when applying make-cg to $\mathcal{G}$ and $y_{x_1}, x_0, w$ the resulting counterfactual graph will contain
	\begin{center}
		\begin{tikzpicture}
	 [>=stealth, rv/.style={thick}, rvc/.style={triangle, draw, thick, minimum size=9mm}, node distance=18mm]
	 \pgfsetarrows{latex-latex};
	 \begin{scope}
		\node[rv] (x) at (-2,0) {$X$};
		\node[rv] (w1) at (0,0) {$W_{i_1}$};
		\node[rv] (w2) at (2,0) {$W_{i_2}$};
		\node (dots) at (4, 0) {\dots};
		\node[rv] (wk) at (6,0) {$W_{i_k}$};

		\node[rv] (w1c) at (0,-2) {${W_{i_1}}_{x_1}$};
		\node[rv] (w2c) at (2,-2) {${W_{i_2}}_{x_1}$};
		\node (dotsc) at (4, -2) {\dots};
		\node[rv] (wkc) at (6,-2) {${W_{i_k}}_{x_1}$};
		\node[rv] (yc) at (8,-2) {$Y_{x_1}$};
		
		\draw[->] (x) -- (w1);
		\draw[->] (w1) -- (w2);
		\draw[->] (w2) -- (dots);
		\draw[->] (dots) -- (wk);
		
		\draw[->] (w1c) -- (w2c);
		\draw[->] (w2c) -- (dotsc);
		\draw[->] (dotsc) -- (wkc);
		\draw[->] (wkc) -- (yc);
		
		\path[<->] (w1) edge[dashed] (w1c);
		\path[<->] (w2) edge[dashed] (w2c);
		\path[<->] (wk) edge[dashed] (wkc);
		
	 \end{scope}
	 \end{tikzpicture}
	\end{center}
	as a subgraph and therefore when applying the ID$^*$ algorithm of \cite{shpitser2012counterfactuals}, we will encounter a C-component $\lbrace W_i, {W_i}_{x_1} \rbrace$ which will result in non-identifiability of the overall expression. Therefore, even having access to the full $\mathcal{G}$ will not help us identify contrasts that include observations of post-treatment variables, completing the proof.
\end{proof}

\subsection{Proof of Theorem \ref{thm:fpt}} \label{appendix:fpt}

\begin{proof}
Considering the following SFM
\begin{center}
		\begin{tikzpicture}
	 [>=stealth, rv/.style={thick}, rvc/.style={triangle, draw, thick, minimum size=7mm}, node distance=18mm]
	 \pgfsetarrows{latex-latex};
	 \begin{scope}
	    \node[rv] (u) at (-1.5, 1.5) {$U$};
		\node[rv] (0) at (0,1.5) {$Z$};
		\node[rv] (1) at (-1.5,0) {$X$};
		\node[rv] (2) at (0,-1.5) {$W$};
		\node[rv] (3) at (1.5,0) {$Y$};
		\draw[->] (1) -- (2);
		\draw[->] (0) -- (3);
		\path[->] (1) edge[bend left = 0] (3);
		\path[->] (u) edge[bend left = 0, dashed] (0);
		\path[->] (u) edge[bend left = 0, dashed] (1);
		\draw[->] (2) -- (3);
		\draw[->] (0) -- (2);
	 \end{scope}
	 \end{tikzpicture}
	\end{center}
we can write the linear structural equation model as follows:
\begin{align}
    U &\gets N(0, 1) \\
    X &\gets \text{Bernoulli}(\text{expit}(U)) \\
    Z &\gets a_{UZ}U + a_{ZZ}Z \epsilon_Z \\
    W &\gets a_{XW}X + a_{ZW}Z + a_{WW}W + \epsilon_W \\
    Y &\gets a_{XY}X + a_{ZY}Z + a_{WY}W + \epsilon_Y
\end{align}
where matrices $a_{ZZ}, a_{WW}$ are upper diagonal, making the above SCM valid, in the sense that no variable is a functional argument of itself. For simplicity, we assume $\epsilon_Z \sim N(0, I_{n_Z})$, $\epsilon_W \sim N(0, I_{n_W})$ and $\epsilon_Y \sim N(0, 1)$. The coefficients $a$ of the above model are assumed to be drawn uniformly from $[-1, 1]^{|E|}$, where $|E|$ is the number of edges with a linear coefficient.

\noindent By expanding out, the outcome $Y$ can be written $$ Y = \sum_{V_i \in X, Z, W} a_{V_iY}V_i + \epsilon_Y,$$ and the linear predictor of $Y$, labeled $f$ can be written as $$ f(X, Z, W) = \sum_{V_i \in X, Z, W} \tilde{a}_{V_iY}V_i.$$ The objective of the optimization can then be written as
\begin{align*}
        \ex[Y - f(X, Z, W)]^2 &= \ex\big[\sum_{V_i \in X, Z, W} a_{V_iY}-\widetilde{a}_{V_iY})V_i + \epsilon_Y\big]^2\\
        &=\ex[\epsilon_Y^2] + \ex\big[\sum_{V_i, V_j \in X, Z, W} (a_{V_iY}-\widetilde{a}_{V_iY})(a_{V_jY}-\widetilde{a}_{V_jY})V_iV_j\big] \\
        &= 1 + (a_{VY} - \widetilde{a}_{VY})^T \ex[VV^T](a_{VY} - \widetilde{a}_{VY}),
\end{align*}
when written as a quadratic form with the characteristic matrix $\ex[VV^T]$. Here, (with slight abuse of notation) the set $V$ includes $X, Z, W$. Further, the constraint TV$_{x_0, x_1}(f) = 0$ is in fact a linear constraint on the coefficients $\widetilde{a}_{VY}$, since we have that 
\begin{align*}
    TV_{x_0, x_1}(f) = (\ex[V \mid x_1] - \ex[V \mid x_0])^T \widetilde{a}_{VY}.
\end{align*}
We write 
\begin{align}
    c &= \ex[V \mid x_1] - \ex[V \mid x_0], \label{eq:coef-c} \\
    \Sigma &= \ex[VV^T] \label{eq:coef-sig}
\end{align}
and note that our optimization problem can be written as
\begin{alignat}{2}
&\argmin_{\widetilde{a}_{VY}}        &\qquad& (a_{VY} - \widetilde{a}_{VY})^T \Sigma (a_{VY} - \widetilde{a}_{VY})\\
&\text{subject to} &      & c^T\widetilde{a}_{VY} = 0.
\end{alignat}
The objective is a quadratic form centered at $a_{VY}$. Geometrically, the solution to the optimization problem is the meeting point of an ellipsoid centered at $a_{VY}$ with the characteristic matrix $\Sigma$ and the hyperplane through the origin with the normal vector $c$. The solution is given explicitly as
\begin{align*}
    \widehat{a}_{VY} = a_{VY} - \frac{c^Ta_{VY}\Sigma^{-1}c}{c^T\Sigma^{-1}c}.
\end{align*}
We next analyze the constraints $$ \text{Ctf-DE}_{x_0, x_1}(\widehat{f}_{\text{fair}} \mid x_0) =  \text{Ctf-IE}_{x_0, x_1}(\widehat{f}_{\text{fair}} \mid x_0) = \text{Ctf-SE}_{x_0, x_1}(\widehat{f}_{\text{fair}}) = 0.$$
The first constraint $\text{Ctf-DE}_{x_0, x_1}(\widehat{f}_{\text{fair}} \mid x_0)$ can be simply written as $\widehat{a}_{XY}(x_1 - x_0) = 0$, and since $x_1 - x_0 = 0$, the constraint can be written as $c_1^T\widehat{a}_{VY} = 0$ where $c_1 = (1 \; 0 \; \dots \; 0)^T$. Similarly, but more involved, the Ctf-IE constraint can be written as $c_2^T\widehat{a}_{VY} = 0$ where entries of $c_2$ corresponding to $W_i$ variables are $$ \ex[W_i \mid x_1] - \ex[{W_i}_{x_0} \mid x_1], $$ and $0$ everywhere else. Finally, the Ctf-SE constraint can be written as $c_3^T\widehat{a}_{VY} = 0$ where entries of $c_3$ corresponding to $W_i$ variables are $$ \ex[{W_i}_{x_0} \mid x_1] - \ex[{W_i} \mid x_0], $$
and the entries corresponding to $Z_i$ variables $$ \ex[Z_i \mid x_1] - \ex[Z_i \mid x_0].$$ Notice also that $c_1 + c_2 + c_3 = c$. We note that 
\begin{align*}
    \ex[Z \mid x_1] - \ex[Z \mid x_0] = (I- a_{ZZ})^{-1}a_{UZ} \delta_u^{01}
\end{align*}
where $\delta_u^{01} = \ex[U \mid x_1] - \ex[U \mid x_0]$ is a constant. Similarly,
\begin{align*}
    \ex[{W}_{x_0} \mid x_1] - \ex[{W} \mid x_0] = (I-a_{WW})^{-1}a_{ZW}(I- a_{ZZ})^{-1}a_{UZ} \delta_u^{01}.
\end{align*}
Furthermore, for the indirect effect, we have that
\begin{align*}
    \ex[W_i \mid x_1] - \ex[{W_i}_{x_0} \mid x_1] = (x_1-x_0) \sum_{\text{paths } X \rightarrow W_i} \prod_{\text{edges } V_k\rightarrow V_l} a_{V_kV_l}.
\end{align*}
Therefore, we can now see how the three constraints can be expressed in terms of the structural coefficients in $a$. What remains is understanding the entries of the $\Sigma$ matrix. Note that $\ex[V_iV_j]$ can be computed by considering all \textit{treks} from $V_i$ to $V_j$. A trek is a path that first goes backwards from $V_i$ until a certain node, and the forwards to $V_j$. The slight complication comes from the treks with the turning point at $U$ that pass through $X$, as the SCM is not linear at $X$. Nonetheless, in this case the contribution to the covariance of $V_i$ and $V_j$ equals the product of the coefficients on the trek multiplied by $\ex[XU]$. Therefore, we note that
\begin{align*}
    \ex[V_iV_j] = \sum_{\substack{\text{treks }T_s\\ \text{from } V_i \text{ to } V_j}} \lambda (T_s) \prod_{\substack{\text{edges } V_k\rightarrow V_l\\ \in T_s}} a_{V_kV_l}
\end{align*}
where the weighing factor $\lambda (T_s)$ is either 1 or $\ex[XU]$ depending on the trek $T_s$.
To conclude the argument, notice the following. The entries of the $\Sigma$ matrix are polynomial functions of the structural coefficients $a$. The same also therefore holds for $\Sigma^{-1}$. Furthermore, the coefficient $c$ is also a polynomial function of coefficients in $a$. Therefore, the condition $c_1^T \widehat{a}_{VY} = 0$ can be written as 
\begin{equation} \label{eq:surface}
    c_1^T (a_{VY} - \frac{c^Ta_{VY}\Sigma^{-1}c}{c^T\Sigma^{-1}c}) = 0,
\end{equation}
where the left hand side is a polynomial expression in the coefficients of $a$. Therefore, the above expression defines an algebraic hypersurface. Any such hypersurface has measure 0 in the space $[-1, 1]^{|E|}$, proving that the set of $0$-TV-compliant SCMs is in fact of measure $0$. Intuitively, the result is saying that the meeting point of an ellipsoid centered at $a_{VY}$ with the characteristic matrix $\Sigma$ and the hyperplane through the origin with the normal vector $c$ with measure 0 also lies on a random hyperplane defined by the normal vector $c_1$ and passing through the origin. 

To extend the result for an $\epsilon > 0$, we proceed as follows. Let $\mathcal{H}(\epsilon)$ be the set of $\epsilon$-TV-compliant SCMs. Let $\mathcal{H}^{DE}(\epsilon)$ be the set of SCMs for which the direct effect is bounded by $\epsilon$ for the $\widehat{f}$. Let $\mathcal{H}^{IE}(\epsilon)$, $\mathcal{H}^{SE}(\epsilon)$ be defined analogously for the indirect and spurious effects.
We then analyze the degrees of the terms appearing in Eq.~\ref{eq:surface}, which defines the surface $\mathcal{H}^{DE}(0)$. In particular, notice that
\begin{align}
    deg (c_1^T (a_{VY} - \frac{c^Ta_{VY}\Sigma^{-1}c}{c^T\Sigma^{-1}c})) &\leq deg (c_1) + deg(a_{VY}) + deg(\frac{c^Ta_{VY}\Sigma^{-1}c}{c^T\Sigma^{-1}c})
\end{align}
and also that
\begin{align}
    deg(\frac{c^Ta_{VY}\Sigma^{-1}c}{c^T\Sigma^{-1}c}) &\leq  deg(c^Ta_{VY}\Sigma^{-1}c) + deg({c^T\Sigma^{-1}c}) \\
                                                       &\leq 2deg(c) + deg(a_{VY}) + deg(\Sigma^{-1}) + 2deg(c) + deg(\Sigma^{-1}).
\end{align}
Now, one can observe the following bounds, where $p = |V|$:
\begin{align}
    deg(c) &\leq p \text{ from Eq.~\ref{eq:coef-c},} \\
    deg(a_{VY}) &= 1 \text{ by definition,} \\
    deg(\Sigma^{-1}) &\leq p^2 * \max deg(\Sigma_{ij}) = p^4 \text{ from Eq.~\ref{eq:coef-sig}}.
\end{align}
from which it follows that the degree of the surface of $0$-TV-compliant SCMs, labeled $\mathcal{H}(0)$, is bounded by $2+4p+2p^2$. Therefore, by an application of the Lojasiewicz's inequality \citep{ji1992global}, there exist constants $k_1, k_2$ such that:
\begin{align}
    \text{vol}(\mathcal{H}^{DE}(\epsilon)) &= \text{vol}\lbrace a \in [-1, 1]^{|E|} \mid | c_1^T (a_{VY} - \frac{c^Ta_{VY}\Sigma^{-1}c}{c^T\Sigma^{-1}c}) | \leq \epsilon \rbrace \\
    &= \text{vol}\lbrace  a \in [-1, 1]^{|E|} \mid d(a, \mathcal{H}^{DE}(0)) \leq k_1\epsilon^{k_2}  \rbrace \label{eq:hepsvol}, 
\end{align}
where the volume in Eq.~\ref{eq:hepsvol} can be bounded above by an application of the Crofton's inequality \citep[p. 45]{adler2007random}, to obtain that
\begin{align}
    \text{vol}(\mathcal{H}^{DE}(\epsilon)) \leq k_1 \epsilon^{k_2} 2^{|E|/2} C(|E|, deg(\mathcal{H}^{DE}(0))) deg(\mathcal{H}^{DE}(0)),
\end{align}
where $C(|E|)$ is a constant coming from the Crofton's inequality. Finally, we can write that for a random $M$ sampled from $\mathcal{S}^{linear}_{n_Z, n_W}$ we have that
\begin{align}
    \pr(M \in \mathcal{H}^{DE}(\epsilon)) = \frac{\text{vol}(\mathcal{H}^{DE}(\epsilon))}{2^{|E|}}. 
\end{align}
By noting that $|E| = p(p+1)$ and setting
\begin{align} \label{eq:eps}
    \epsilon = \Big( \frac{2^{p^2/4}}{4C(|E|)[2+4p+2p^4]k_1} \Big) ^ {1/k_2}
\end{align}
we obtain that $\pr(M \in \mathcal{H}^{DE}(\epsilon)) \leq \frac{1}{4}$. Since we know that

\begin{align}
    \mathcal{H}(\epsilon) = \mathcal{H}^{DE}(\epsilon) \cap \mathcal{H}^{IE}(\epsilon) \cap \mathcal{H}^{SE}(\epsilon)
    &\implies \pr(M \in \mathcal{H}^{}(\epsilon)) \leq \pr(M \in \mathcal{H}^{DE}(\epsilon)). \\
    &\implies \pr(M \in \mathcal{H}^{}(\epsilon)) \leq \frac{1}{4},
\end{align}
for such an $\epsilon$. Intuitively, any SCM in $\mathcal{H}(\epsilon)$ must also be in $\mathcal{H}^{DE}(\epsilon)$. Any SCM in $\mathcal{H}^{DE}(\epsilon)$ must be close to $\mathcal{H}^{DE}(0)$. The maximal deviation of an SCM in $\mathcal{H}^{DE}(\epsilon)$ from $\mathcal{H}^{DE}(0)$ can be bounded by the Lojasiewicz's inequality, whereas the surface area of $\mathcal{H}^{DE}(0)$ can be bounded above by Crofton's inequality. Putting together, we get a bound on the measure of $\epsilon$-TV-compliant SCMs. 
\end{proof}

The behaviour of the $\epsilon$ term given in Eq.~\ref{eq:eps} cannot be theoretically analyzed further, since the constants arising from the Lojasiewicz's inequality are dimension dependent. To this end, for $n_Z = n_W = 5$ we empirically estimate 
\begin{equation}
    \pr(M \in \mathcal{H}^{DE}(\epsilon))
\end{equation}
for a range of $\epsilon$ values, and obtain the plot in Fig.~\ref{fig:fpt-empirical}.
\begin{figure}
	\centering
	\includegraphics[height=60mm,angle=0]{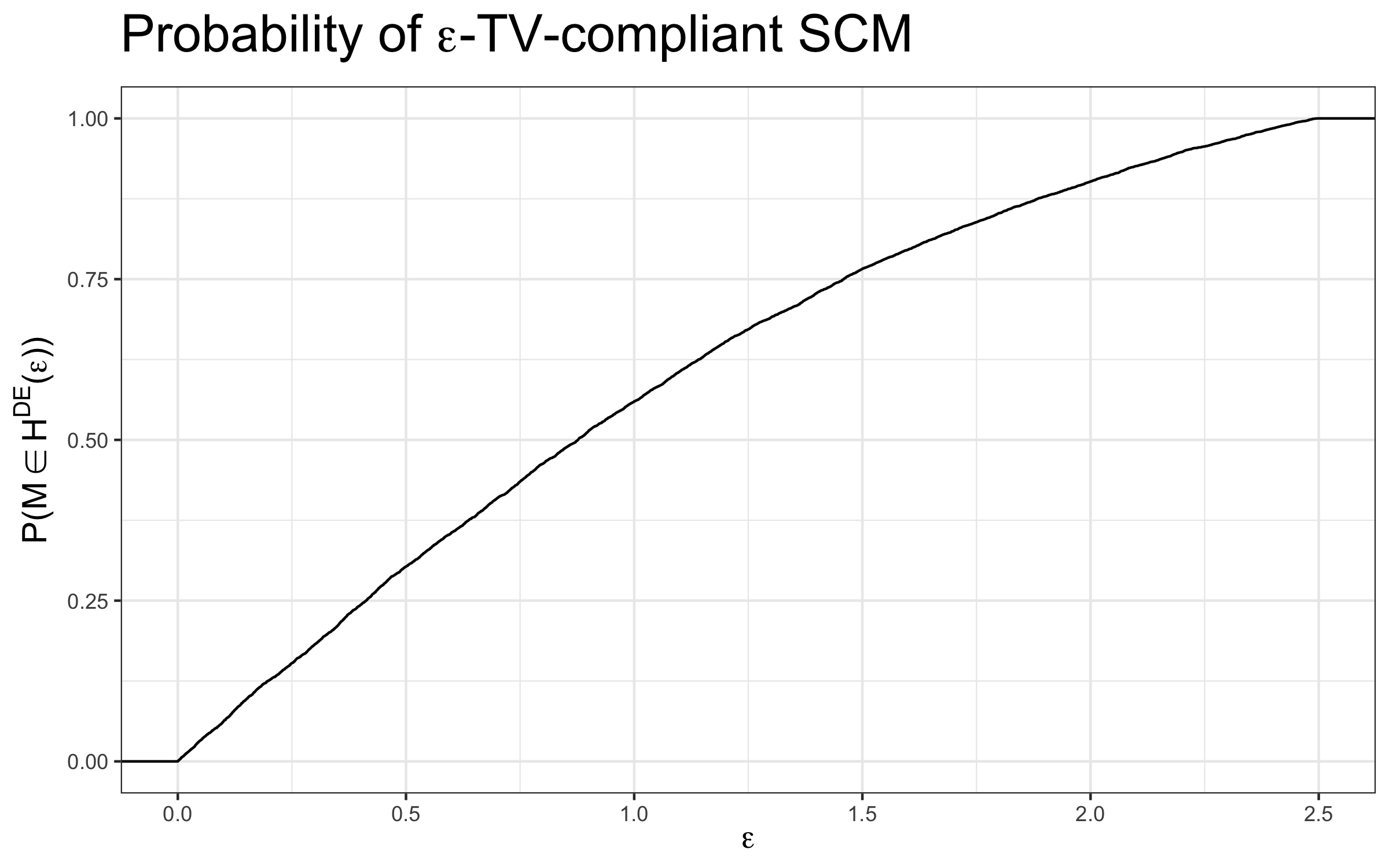}
	\caption{Estimating empirically the probability that a random SCM in $\mathcal{S}^{linear}_{n_Z, n_W}$, for $n_Z = n_W = 5$, has a direct effect smaller than $\epsilon$ after ensuring that TV equals $0$.}
	\label{fig:fpt-empirical}
\end{figure}
\subsection{Proof of Thm.~\ref{thm:causalIF}} \label{appendix:causalIF}

\begin{proof}
We prove the result for the case BN-set$= \emptyset$ (the other cases of BN-sets follow analogously), in the population level case. 
Based on the standard fairness model, we are starting with an SCM $\mathcal{M}$ given by:
\begin{align}
    X &\gets f_X(u_x, u_z) \\
    Z &\gets f_Z(u_x, u_z) \\
    W &\gets f_W(X, Z, u_w) \\
    Y &\gets f_Y(X, Z, W, u_y).
\end{align}
The noise variables $u_x, u_z$ are not independent, but the variables $u_w, u_y$ are mutually independent, and also independent from $u_x, u_z$. 

We now explain how the sequential optimal transport steps extend the original SCM $\mathcal{M}$ (to which we do not have access). Firstly, the conditional distribution $Z \mid X = x_1$ is transported onto $Z \mid X = x_0$. Write $\tau^Z$ for the transport map. On the level of the SCM, this corresponds to extending the equations by an additional mechanism
\begin{align}
    \widetilde{Z} \gets \begin{cases}
                            f_Z(u_x, u_z) &\text{ if } f_X(u_x, u_z) = x_0 \\
                            f_Z(\pi^Z(u_x, u_z)) &\text{ if } f_X(u_x, u_z) = x_1
                        \end{cases}.
\end{align}
Here, there is an implicit (possibly stochastic) mapping $\pi^Z$ that we cannot observe. For simplicity, we assume that the variable $Z$ is continuous and that $\pi^Z$ is deterministic. We can give an optimization problem to which $\pi^Z$ is the solution, namely:
\begin{equation}
    \begin{aligned}
    \pi^Z := \argmin_{\pi} \quad & \int_{\mathcal{U}_X \times \mathcal{U_Z}} \| f_Z(\pi(u_z, u_x)) - f_Z(u_z, u_x) \|^2 du_{xz}^{X = x_1} \\
\textrm{s.t.  \;\;} \quad & \underset{u_x,u_z \sim U_X,U_Z \mid X = x_1}{f_Z(\pi(u_z, u_x))} \overset{d}{=} \underset{u_x,u_z \sim U_X,U_Z \mid X = x_0}{f_Z(u_z, u_x)}.
\end{aligned}
\end{equation}
The measure $du_{xz}^{X = x_1}$ in the objective is the probability measure associated with the distribution $P(u_x, u_z \mid X = x_1)$. The constraint ensures that after the transport, $\widetilde{Z} \mid X = x_1$ is equal in distribution to  $\widetilde{Z} \mid X = x_0$. 
In the second step of the procedure, we are transporting the distribution of $W$. This results in adding the mechanism:
\begin{align}
    \widetilde{W} \gets \begin{cases}
                            f_W(x_0, \widetilde{Z}, u_w) &\text{ if } X = x_0 \\
                            f_W(x_0, \widetilde{Z}, \pi^W(u_w)) &\text{ if } X = x_1
                        \end{cases}.
\end{align}
Similarly as before, $\pi^W$ is a possibly stochastic mapping solving the following optimization problem:

\begin{equation}
    \begin{aligned}
    \pi^W := \argmin_{\pi} \quad & \int_{\mathcal{U}_W} \| f_W(x_0, \widetilde{z}, \pi(u_w)) -  f_W(x_1, \widetilde{z}, u_w)\|^2 du_{w} \\
\textrm{s.t.  \;\;} \quad & f_W(x_0, \widetilde{z}, \pi(u_w)) \overset{d}{=} f_W(x_0, \widetilde{z}, u_w).
\end{aligned}
\end{equation}
The above optimization problem is thought of being solved separately for each value of $\widetilde{Z} = \widetilde{z}$. Finally, in the last step, we are constructing the additional mechanism:
\begin{align}
    \widetilde{Y} \gets \begin{cases}
                            f_Y(x_0, \widetilde{Z}, \widetilde{W}, u_y) &\text{ if } X = x_0 \\
                            f_Y(x_0, \widetilde{Z}, \widetilde{W}, \pi^Y(u_y)) &\text{ if } X = x_1
                        \end{cases}
\end{align}
Again, the implicit mapping $\pi^Y$ is constructed so that it is the solution to
\begin{equation}
    \begin{aligned}
    \pi^Y := \argmin_{\pi} \quad & \int_{\mathcal{U}_Y} \| f_Y(x_0, \widetilde{z}, \widetilde{w}, \pi(u_y)) -  f_y(x_1, \widetilde{z},  \widetilde{w}, u_y)\|^2 du_{y} \\
\textrm{s.t.  \;\;} \quad & f_Y(x_0, \widetilde{z}, \widetilde{w}, \pi(u_y)) \overset{d}{=} f_Y(x_0, \widetilde{z}, \widetilde{w}, u_y).
\end{aligned}
\end{equation}
where the problem is solved separately for each fixed choice of parents $\widetilde{Z} = \widetilde{z}$, $\widetilde{W} = \widetilde{w}$.

After constructing the additional mechanisms $\widetilde{Z}, \widetilde{W},$ and $\widetilde{Y}$, we draw the explicit causal diagram corresponding to the new variables, which includes the unobservables $U_X, U_Z, U_W,$ and $U_Y$ (marked in red), given as follows:
	\begin{center}
		\begin{tikzpicture}
	 [>=stealth, rv/.style={thick}, rvc/.style={triangle, draw, thick, minimum size=9mm}, node distance=18mm]
	 \pgfsetarrows{latex-latex};
	 \begin{scope}
		\node[rv] (x) at (-1.5,0) {$X$};
		\node[rv] (z) at (0,1.5) {$\widetilde{Z}$};
		\node[rv] (w) at (0,-1.5) {$\widetilde{W}$};
		\node[rv] (y) at (1.5,0) {$\widetilde{Y}$};
		
		\node[rv, red] (ux) at (-1.5,1) {$U_X$};
		\node[rv, red] (uz) at (0,2.5) {$U_Z$};
		\node[rv, red] (uw) at (-1,-1.5) {$U_W$};
		\node[rv, red] (uy) at (1.5,1) {$U_Y$};
		
		\draw[->] (ux) -- (x);
		\draw[->] (ux) -- (z);
		\draw[->] (uz) -- (z);
		\draw[->] (uz) -- (x);
		\draw[->] (uw) -- (w);
		\draw[->] (uy) -- (y);
		\draw[->] (x) -- (w);
		\draw[->] (x) -- (y);
		\draw[->] (z) -- (y);
		\draw[->] (z) -- (w);
		\draw[->] (w) -- (y);
		
	 \end{scope}
	 \end{tikzpicture}.
	\end{center}
Note that by marginalizing out the unobserved variables $U_X, U_Z, U_W, U_Y$, we obtain the new causal diagram, which is given by the standard fairness model over the variables $X, \widetilde{Z}, \widetilde{W}, \widetilde{Y}$. Therefore, it follows that the identification expressions for the spurious, indirect, and direct effects are known, and given by:
\begin{align}
x\text{-DE}_{x_0, x_1}(\widetilde{y} \mid x_0) &= \sum_{\widetilde{z}, \widetilde{w}} [P(\widetilde{y} \mid x_1, \widetilde{z}, \widetilde{w}) - P(\widetilde{y}\mid x_0, \widetilde{z}, \widetilde{w})]P(\widetilde{w} \mid x_0, \widetilde{z})P(\widetilde{z} \mid x) \label{eq:xde-tilde} \\
x\text{-IE}_{x_0, x_1}(\widetilde{y} \mid x_0) &= \sum_{\widetilde{z}, \widetilde{w}} P(\widetilde{y}\mid x_0, \widetilde{z}, \widetilde{w})[P(\widetilde{w} \mid x_1, \widetilde{z}) - P(\widetilde{w} \mid x_0, \widetilde{z})]P(\widetilde{z} \mid x) \\
x\text{-SE}_{x_0, x_1}(\widetilde{y}) &= \sum_{\widetilde{z}} P(\widetilde{y} \mid x_0, \widetilde{z})[P(\widetilde{z} \mid x_0) - P(\widetilde{z} \mid x_1)]. \label{eq:xse-tilde}
\end{align}
To finish the proof, notice that by construction (the matching of distributions via optimal transport), we have that
\begin{align}
    P(\widetilde{y} \mid x_1, \widetilde{z}, \widetilde{w}) &= P(\widetilde{y}\mid x_0, \widetilde{z}, \widetilde{w}) \\
    P(\widetilde{w} \mid x_1, \widetilde{z}) &= P(\widetilde{w} \mid x_0, \widetilde{z}) \\
    P(\widetilde{z} \mid x_0) &= P(\widetilde{z} \mid x_1),
\end{align}
implying that all three effects in Eq.~\ref{eq:xde-tilde}-\ref{eq:xse-tilde} are equal to 0. 
\end{proof}

\section{Practical aspects of fairness measures} \label{appendix:IDEst}
\subsection{Identification of measures} \label{Identification}
The structure of the measures used in Causal Fairness Analysis was given by the Fairness Map from Thm.~\ref{thm:map} (see also Fig.~\ref{fig:map}). Moreover, in Theorem \ref{thm:sfm} in Appendix \ref{appendix:sfm} we have shown that many of the measures in the map are identifiable from observational data in the standard fairness model (SFM) and we provided explicit expressions for their identification.

	\begin{table}\centering
	    \setstretch{1.2}
		\begin{tabular}{|M{0.3cm}|M{3.2cm}|M{10cm}|}
            \hline
			& Measure & ID expression \\[1mm] 	\hline
			\multirow{4}{*}{\rotatebox{90}{general}} & TE$_{x_0, x_1}(y)$ & $\sum_z [P(y \mid x_1, z) - P(y \mid x_0, z)]P(z)$ \\\cline{2-3}
			& Exp-SE$_{x}(y)$ & $\sum_z P(y \mid x, z)[P(z) - P(z \mid x)]$ \\\cline{2-3}
			& NDE$_{x_0, x_1}(y)$ & $\sum_{z, w} [P(y \mid x_1, z, w) - P(y\mid x_0, z, w)]P(w \mid x_0, z)P(z)$  \\\cline{2-3}
			& NIE$_{x_0, x_1}(y)$ &  $\sum_{z, w} P(y\mid x_0, z, w)[P(w \mid x_1, z) - P(w \mid x_0, z)]P(z)$  \\\cline{2-3}
            \hline
            \multirow{4}{*}{\rotatebox{90}{$x$-specific}} & ETT$_{x_0, x_1}(y \mid x)$ & $\sum_z [P(y \mid x_1, z) - P(y \mid x_0, z)]P(z \mid x)$ \\\cline{2-3}
			& Ctf-SE$_{x_0, x_1}(y)$ & $\sum_z P(y \mid x_0, z)[P(z \mid x_0) - P(z \mid x_1)]$ \\\cline{2-3}
			& Ctf-DE$_{x_0, x_1}(y\mid x)$ & $\sum_{z, w} [P(y \mid x_1, z, w) - P(y\mid x_0, z, w)]P(w \mid x_0, z)P(z \mid x)$  \\\cline{2-3}
			& Ctf-IE$_{x_0, x_1}(y\mid x)$ & $\sum_{z, w} P(y\mid x_0, z, w)[P(w \mid x_1, z) - P(w \mid x_0, z)]P(z \mid x)$  \\\hline
			\multirow{3}{*}{\rotatebox{90}{$z$-specific}}& $z$-TE$_{x_0, x_1}(y\mid x)$ & $P(y \mid x_1, z) - P(y \mid x_0, z)$  \\\cline{2-3}
			&$z$-DE$_{x_0, x_1}(y\mid x)$ & $\sum_{w} [P(y \mid x_1, z, w) - P(y\mid x_0, z, w)]P(w \mid x_0, z)$  \\\cline{2-3}
			& $z$-IE$_{x_0, x_1}(y\mid x)$ & $\sum_{w} P(y\mid x_0, z, w)[P(w \mid x_1, z) - P(w \mid x_0, z)]$  \\\cline{2-3}
			\hline
		\end{tabular}
		\caption{Population level and $x$-specific causal measures of fairness in the TV-family, and their identification expressions under the standard fairness model $\mathcal{G}_{SFM}$.}
		\label{table:IDexpressions}
	\end{table}

The natural question is whether these measures remain identifiable when some assumptions of the SFM are relaxed. To answer this question, we consider what happens to identifiability of different measures when we add bidirected edges to the $\mathcal{G}_{\text{SFM}}$.

\subsubsection{Identification under Extended Fairness Model} \label{extended}
There are five possible bidirected edges that could be added to the $\mathcal{G}_{\text{SFM}}$ (since the bidirected edge $X \bidir Z$ is assumed to be present already). The other five possibilities include the $W \bidir Y$ (mediator-outcome), $Z \bidir Y$ (confounder-outcome), $X \bidir W$ (attribute-mediator), $Z \bidir W$ (confounder-mediator) and $X \bidir Y$ (attribute-outcome). We analyze these cases in the respective order.

\paragraph{Bidirected edge $Z \bidir Y$.} Consider the case of confounder-outcome confounding, represented by the $Z \bidir Y$ edge. An example of such a model is given in the LHS of Table \ref{table:zybidir}. In this case, without expanding the $Z$ set, none of the fairness measures are identifiable (due to the set $Z$ not satisfying the back-door criterion with respect to variables $X$ and $Y$). However, this does not necessarily mean there is no hope for identifying our fairness measures. What we do next is refine the $Z$ set, in the hope that the additional assumptions obtained in this process might help us identify our quantities of interest. In some sense, the assumptions encoded in the clustered diagram are not sufficient for identification. It might turn out, however, that by spelling out all the variables in the cluster, some additional assumptions might help with identification.  Consider the example on the RHS of Table \ref{table:zybidir}, where the full causal graph is given, after refining the previously clustered $Z$ set. Interestingly, in this case the set $\lbrace Z_1, Z_2 \rbrace$ can be shown as back-door admissible for the effect of $X$ on $Y$. Furthermore, the identification expression for all the quantities remains the same as in the standard fairness model, given by the expressions in Table \ref{table:IDexpressions}.
\begin{table}\centering
	\caption{An example of the extended fairness model with a bidirected $Z \bidir Y$ edge (left side), in which refining the set of variables $Z$ yields a graph (right side) in which all fairness measures are identifiable.}\medskip\small
	\begin{tabular}{M{5.5cm}|M{4.1cm}}
		cluster model & refined model \\[1mm]\hline
     	\begin{tikzpicture}
 [>=stealth, rv/.style={thick}, rvc/.style={triangle, draw, thick, minimum size=7mm}, node distance=18mm]
 \pgfsetarrows{latex-latex};
 \begin{scope}
	\node[rv] (0) at (0,1.5) {$Z$};
	\node[rv] (1) at (-1.5,0) {$X$};
	\node[rv] (3) at (1.5,0) {$Y$};
	\draw[->] (0) -- (3);
	\path[->] (1) edge[bend left = 0] (3);
	\path[<->] (1) edge[bend left = 30, dashed] (0);
	\path[<->] (0) edge[bend left = 30, dashed] (3);
 \end{scope}
 \end{tikzpicture}
     & 
    	\begin{tikzpicture}
 [>=stealth, rv/.style={thick}, rvc/.style={triangle, draw, thick, minimum size=7mm}, node distance=18mm]
 \pgfsetarrows{latex-latex};
 \begin{scope}
	\node[rv] (0) at (-1,1.5) {$Z_1$};
	\node[rv] (1) at (1,1.5) {$Z_2$};
    \node[rv] (2) at (-1.5,0) {$X$};
	\node[rv] (3) at (1.5,0) {$Y$};
	\draw[->] (0) -- (3);
	\draw[->] (1) -- (3);
	\path[<->] (1) edge[bend left = 30, dashed] (3);
	\path[<->] (0) edge[bend right = 30, dashed] (2);
	\draw[->] (2) -- (3);
 \end{scope}
 \end{tikzpicture}
 \end{tabular}
	\label{table:zybidir}
\end{table}

\paragraph{Bidirected edge $W \bidir Y$.} Next consider the case where there is a bidirected edge between the group of variables $W$ and the outcome $Y$. Firstly, we note that the identification of causal (TE/ETT) and spurious measures (Exp-SE/Ctf-SE) is unaffected by the $W \bidir Y$ edge, and that these quantities are identified by the same expressions as in Table \ref{table:IDexpressions}. The quantities measuring direct and indirect effects are not identifiable, at least not without further refining the $W$ set. Consider the example given in Table \ref{table:wybidir}. 
\begin{table}\centering
	\caption{An example of the extended fairness model with a bidirected $W \bidir Y$ edge (left side), in which refining the set of variables $W$ yields a graph (right side) in which all fairness measures are identifiable.}\medskip\small
	\begin{tabular}{M{5.5cm}|M{4.1cm}}
		cluster model & refined model \\[1mm]\hline
			\begin{tikzpicture}
 [>=stealth, rv/.style={thick}, rvc/.style={triangle, draw, thick, minimum size=7mm}, node distance=18mm]
 \pgfsetarrows{latex-latex};
 \begin{scope}
	\node[rv] (0) at (0,-1.5) {$W$};
	\node[rv] (1) at (-1.5,0) {$X$};
	\node[rv] (3) at (1.5,0) {$Y$};
	\draw[->] (0) -- (3);
	\draw[->] (1) -- (0);
	\path[->] (1) edge[bend left = 0] (3);
	\path[<->] (0) edge[bend left = -30, dashed] (3);
 \end{scope}
 \end{tikzpicture} & 
    	\begin{tikzpicture}
 [>=stealth, rv/.style={thick}, rvc/.style={triangle, draw, thick, minimum size=7mm}, node distance=18mm]
 \pgfsetarrows{latex-latex};
 \begin{scope}
    \node[rv] (x) at (-1.5,0) {$X$};
	\node[rv] (y) at (1.5,0) {$Y$};
	\node[rv] (w1) at (-1,-1.5) {$W_1$};
	\node[rv] (w2) at (1,-1.5) {$W_2$};
	\draw[->] (x) -- (w1);
	\draw[->] (x) -- (y);
	\draw[->] (w1) -- (w2);
	\draw[->] (w2) -- (y);
	\path[<->] (y) edge[bend right = 0, dashed] (w1);
 \end{scope}
 \end{tikzpicture}
 \end{tabular}
	\label{table:wybidir}
\end{table}
In the LHS of the table we have a model in which $W$ is clustered and NDE or NIE quantities are not identifiable. On the RHS, after expanding the previously clustered $W$ set, the natural direct (and indirect) effects can be identified, by the virtue of the \textit{front-door criterion} \citep{pearl:2k}. However, note that in this case, the identification expression for the natural direct effect \textit{is different from the identification expression for the natural direct effect in the standard fairness model.} Whenever front-door identification is used, we expect the expression to change, compared to the baseline SFM case.

\paragraph{Bidirected edge $X \bidir W$.} The case of the $X \bidir W$ edge is similar to that of $W \bidir Y$, yet slightly different. None of the measures discussed are identifiable in this case, before refining the $W$ set. However, similarly as in the $W \bidir Y$ example in Table \ref{table:wybidir}, when refining the $W$ set, we might find that in fact the effect of $X$ on $Y$ is identifiable via the front-door. Again, the identification expression in this case will change. For the sake of brevity we skip an explicit example. 

\paragraph{Bidirected edge $Z \bidir W$.} In the case of the $Z \bidir W$ edge, none of the measures are identifiable. However, refining the $Z$ and $W$ sets might help. To see an example, consider the following graph
\begin{center}
        	\begin{tikzpicture}
 [>=stealth, rv/.style={thick}, rvc/.style={triangle, draw, thick, minimum size=7mm}, node distance=18mm]
 \pgfsetarrows{latex-latex};
 \begin{scope}
    \node[rv] (z1) at (0,1) {$Z_1$};
	\node[rv] (w1) at (-1,-1) {$W_1$};
	\node[rv] (w2) at (1,-1) {$W_2$};
    \node[rv] (x) at (-2,0) {$X$};
	\node[rv] (y) at (2,0) {$Y$};
	\draw[->] (x) -- (w1);
	\draw[->] (x) -- (y);
	\draw[->] (w1) -- (w2);
	\draw[->] (w2) -- (y);
	\draw[->] (z1) -- (y);
	\draw[->] (z1) -- (w1);
	\draw[->] (z1) -- (w2);
	\path[<->] (x) edge[bend left = 20, dashed] (z1);
	\path[<->] (z1) edge[bend right = 30, dashed] (w2);
 \end{scope}
 \end{tikzpicture}.
\end{center}
In this case, all of the measures of fairness in Table \ref{table:IDexpressions} are identifiable, but again with different expressions that those presented in the table.

\paragraph{Bidirected edge $X \bidir Y$.} 
The attribute-outcome confounding represented by the $X \bidir Y$ edge is the most difficult case. When this edge is present, none of the fairness quantities can be identified. The reason why this case is hard is that the $X \bidir Y$ introduces a bidirected edge between $X$ and its child $Y$. This causes the effect of $X$ on $Y$ to be non-identifiable \citep{tian2002general}. For more general identification strategies for when a combination of observational and experimental data is available, refer to  \citep{lee2019general} and \citep{correa2021counterfactual}, and for partial identification ones, see  \citep{zhang2022partialctf}. 

The summary of the discussion of the five cases of bidirected edges in the extended fairness model, and what can be done under their presence, is given in Table \ref{table:id}. 
\begin{table}\centering
	\caption{Identification of causal fairness measures under latent confounding.}\medskip\small
	\begin{tabular}{|M{3cm}|M{2.4cm}|M{2.4cm}|M{2.4cm}|M{2.4cm}|}
		\hline
		 	 \begin{tikzpicture}
		 [>=stealth, rv/.style={circle, draw, thick, minimum size=6mm}, rvc/.style={triangle, draw, thick, minimum size=7mm}, node distance=18mm]
		 \pgfsetarrows{latex-latex};
		 \begin{scope}
			 	\node[rv] (0) at (1,0) {$Y$};
				\node[rv] (1) at (-1,0) {$W$};
				\path[<->] (0) edge[bend right = 30, dashed] (1);
			\end{scope}
	\end{tikzpicture} & \ding{52} & \ding{52} & Refine $W$ & Refine $W$\\\hline
	\begin{tikzpicture}
	 	[>=stealth, rv/.style={circle, draw, thick, minimum size=6mm}, rvc/.style={triangle, draw, thick, minimum size=7mm}, node distance=18mm]
		\pgfsetarrows{latex-latex};
		\begin{scope}
			\node[rv] (0) at (1,0) {$Y$};
			\node[rv] (1) at (-1,0) {$Z$};
			\path[<->] (0) edge[bend right = 30, dashed] (1);
		\end{scope}
	\end{tikzpicture} & Refine $Z$ & Refine $Z$ & Refine $Z$ & Refine $Z$\\\hline
		\begin{tikzpicture}
	 		[>=stealth, rv/.style={circle, draw, thick, minimum size=6mm}, rvc/.style={triangle, draw, thick, minimum size=7mm}, node distance=18mm]
	 		\pgfsetarrows{latex-latex};
	 		\begin{scope}
	 			\node[rv] (0) at (1,0) {$W$};
	 			\node[rv] (1) at (-1,0) {$X$};
	 			\path[<->] (0) edge[bend right = 30, dashed] (1);
	 		\end{scope}
	 \end{tikzpicture} & Refine $W$ & Refine $W$ & Refine $W$ & Refine $W$\\\hline
	\begin{tikzpicture}
 		[>=stealth, rv/.style={circle, draw, thick, minimum size=6mm}, rvc/.style={triangle, draw, thick, minimum size=7mm}, node distance=18mm]
 		\pgfsetarrows{latex-latex};
 		\begin{scope}
 			\node[rv] (0) at (1,0) {$W$};
 			\node[rv] (1) at (-1,0) {$Z$};
 			\path[<->] (0) edge[bend right = 30, dashed] (1);
 \end{scope}
 \end{tikzpicture} & Refine $Z, W$ & Refine $Z, W$ & Refine $Z, W$ & Refine $Z, W$\\\hline
		 \begin{tikzpicture}
	  [>=stealth, rv/.style={circle, draw, thick, minimum size=6mm}, rvc/.style={triangle, draw, thick, minimum size=7mm}, node distance=18mm]
	  \pgfsetarrows{latex-latex};
	  \begin{scope}
	 	\node[rv] (0) at (1,0) {$Y$};
	 	\node[rv] (1) at (-1,0) {$X$};
	 	\path[<->] (0) edge[bend right = 30, dashed] (1);
	  \end{scope}
	  \end{tikzpicture} & \ding{56} & \ding{56} & \ding{56} & \ding{56} \\\hline
	\end{tabular}
	\label{table:id}
\end{table}
Identification checks and suggestions about when to refine the $Z$- or $W$-sets are included in the \texttt{faircause} R-package. We end with an example that fits the extended fairness model with all bidirected edges apart from the $X \bidir Y$, but in which case all the fairness measures in Table \ref{table:IDexpressions} (albeit not with the same expression as in the table), showing that refining $Z$ and $W$ sets sometimes might help:
\begin{center}
	\begin{tikzpicture}
 [>=stealth, rv/.style={thick}, rvc/.style={triangle, draw, thick, minimum size=7mm}, node distance=18mm]
 \pgfsetarrows{latex-latex};
 \begin{scope}
 	\node[rv] (z1) at (-1,1.5) {$Z_1$};
	\node[rv] (z2) at (1,1.5) {$Z_2$};
    \node[rv] (x) at (-3,0) {$X$};
	\node[rv] (y) at (3,0) {$Y$};
	\node[rv] (w1) at (-1,-1.5) {$W_1$};
	\node[rv] (w2) at (1,-1.5) {$W_2$};
	
	\draw[->] (z1) -- (y);
	\draw[->] (z1) -- (w1);
	\draw[->] (z1) -- (w2);
	\draw[->] (z2) -- (y);
	\draw[->] (z2) -- (w1);
	\draw[->] (z2) -- (w2);
	\draw[->] (x) -- (w1);
	\draw[->] (x) -- (y);
	\draw[->] (w1) -- (w2);
	\draw[->] (w2) -- (y);
	
	\path[<->] (x) edge[bend left = -60, dashed] (w2);
	\path[<->] (y) edge[bend right = -60, dashed] (w1);
	\path[<->] (z1) edge[bend right = 20, dashed] (x);
	\path[<->] (z2) edge[bend left = 20, dashed] (y);
	\path[<->] (w1) edge[bend right = 15, dashed] (z2);
	\path[<->] (w2) edge[bend right = -15, dashed] (z1);
	
 \end{scope}
 \end{tikzpicture}.
\end{center}
\subsection{Estimation of measures} \label{Estimation}
Suppose we found that a target causal measure of fairness is identifiable from observational data (after possibly refining the SFM). The next question is then how to estimate the causal measure in practice. There is a large body of literature on the estimation of causal quantities, based on which our own implementation is built. We focus on describing how to estimate $\ex(y_x)$ and $\ex(y_{x_1, W_{x_0}})$. Most fairness measures can then be derived from taking (conditional) differences of these two estimands.

\subsubsection{Doubly Robust Estimation}
In the SFM, a standard way of computing the quantity $\ex(y_x)$ would be using inverse propensity weighting. The mediator $W$ can be marginalized out and the estimator
\begin{equation}\label{eq:ipw}
	\frac{1}{n} \sum_{i=1}^{n} \frac{\mathbb{1}(X_i = x)Y_i}{\widehat{p}(X_i \mid Z_i)},
\end{equation}
where $\widehat{p}(X_i \mid Z_i)$ is the estimate of the conditional probability $\pr(X_i = x \mid Z_i)$, can be used. The additional assumption necessary for such an approach is the positivity assumption:
\begin{definition}[Positivity assumption]
	The positivity assumption holds if $\;\forall \; x,z$, $\pr(X = x \mid Z = z)$ is bounded away from $0$, that is $$\delta < \pr(X = x \mid Z = z) < 1-\delta,$$ for some $\delta > 0$.
\end{definition}
Such an assumption is needed for the estimation of causal quantities we discuss (together with the assumptions encoded in the SFM that are used for identification).

\noindent However, more powerful estimation techniques have been developed and applied very broadly. In particular, \textit{doubly robust} estimators have been proposed for the estimation of causal quantities \citep{heckman1998matching, bang2005doubly}. In context of the estimator in Equation \eqref{eq:ipw}, a doubly robust estimator would be
\begin{equation}\label{eq:dripw}
	\frac{1}{n} \sum_{i=1}^{n} \frac{\mathbb{1}(X_i = x)(Y_i - \widehat{\mu}(Y_i \mid Z_i, X_i))}{\widehat{p}(X_i \mid Z_i)} + \widehat{\mu}(Y_i \mid Z_i, X_i),
\end{equation}
where $\widehat{\mu}$ denotes the estimator of the conditional mean $\ex[Y \mid Z = z, X = x]$. In fact, only one of the two estimators $\widehat{\mu}(Y_i \mid Z_i, X_i)$ and $\widehat{p}(X_i \mid Z_i)$ needs to be consistent, for the entire estimator in Equation \eqref{eq:dripw} to be consistent. Such robustness to model misspecification is a rather desirable property.



Estimating $\ex(y_{x_1, W_{x_0}})$ in a doubly robust fashion is somewhat more involved. This problem has been studied under the rubric of causal mediation analysis \citep{robins1992identifiability, pearl:01, robins2003semantics}. \cite{tchetgen2012semiparametric} proposed a doubly robust estimator of the expected potential outcome $\ex[Y_{x_1, W_{x_0}}]$ defined via: 
\begin{align} \label{eq:mediated}
	\phi_{x_0, x_1}(X, W, Z) = &\frac{\mathbb{1}(X = x_1) f(W \mid x_0, Z)}{p_{x_1}(Z)f(W \mid x_1, Z)}[Y - \mu(x_1, W, Z)]\nonumber\\
	&+\frac{\mathbb{1}(X = x_0)}{p_{x_0}(Z)}\big[ \mu(x_1, W, Z) - \int_{\mathcal{W}} \mu(x_1, w, Z)f(w \mid x_0, Z) \;dw\big]\\
	&+ \int_{\mathcal{W}} \mu(x_1, w, Z)f(w \mid x_0, Z)\;dw. \nonumber
\end{align}
The estimator is given by $\frac{1}{n} \sum_{i=1}^n \widehat{\phi}_{x_0, x_1}(X_i, W_i, Z_i)$, where in $\widehat{\phi}$ the quantities $p_x(Z)$, $\mu(X, W, Z)$ and $f(W \mid X, Z)$ are replaced by respective estimates. Such an estimator is multiply robust (one of the three models can be misspecified). However, the estimator also requires the estimation of the conditional density $f(W \mid X, Z)$. In case of continuous or high-dimensional $W$, estimating the conditional density could be very hard and the estimator could therefore suffer in performance. We revisit the estimation of $\ex[y_{x_1, W_{x_0}}]$ shortly.

\subsubsection{Double Machine Learning}
Doubly (and multiply) robust estimation allows for model misspecification of one of the models, while retaining consistency of the estimator. However, we have not discussed the convergence rates of these estimators yet.
In some cases fast, $O(n^{-\frac{1}{2}})$ rates are attainable for doubly robust estimators, under certain conditions. For example, one such condition is that $p_{x}(Z), \mu(X, W, Z)$ and their estimates belong to the Donsker class of functions \citep{benkeser2017doubly}. For a review, refer to \citep{kennedy2016semiparametric}. However, modern ML methods do not belong to the Donsker class.

In a recent advance, \cite{chernozhukov2018double} showed that the Donsker class condition can, in many cases (including modern ML methods), be relaxed by using a cross-fitting approach. This method was named \textit{double machine learning} (DML). For estimating $\ex[Y_{x}]$ we make use of the estimator in Equation \eqref{eq:dripw} and proceed as follows:

\begin{enumerate}
	\item Split the data $\mathcal{D}$ into $K$ disjoint folds $\mathcal{D}_1$, $\mathcal{D}_2$, ..., $\mathcal{D}_K$,
	\item \label{step:trestimators} Using the complement of fold $\mathcal{D}_k$ (labeled $\mathcal{D}^C_k$) compute the estimates $\widehat{p}^{-(k)}_x(Z)$, $\widehat{\mu}^{-(k)}(X, Z)$ of $P(X = x \mid Z = z)$ and $\ex[Y \mid Z = z]$,
	\item Compute  
	    \begin{equation} \label{eq:dripwcrossfit}
	        \frac{\mathbb{1}(X_i = x)(Y_i - \widehat{\mu}(Y_i \mid Z_i, X_i))}{\widehat{p}(X_i \mid Z_i)} + \widehat{\mu}(Y_i \mid Z_i, X_i),
	    \end{equation}
	for each observation $(X_i, Z_i, Y_i)$ in $\mathcal{D}_k$ by plugging in estimators $\widehat{p}^{-(k)}_x(Z)$,  $\widehat{\mu}^{-(k)}(X, Z)$ obtained on the complement $\mathcal{D}^C_k$,
	\item Taking the mean of the terms in Equation \eqref{eq:dripwcrossfit} across all observations.
\end{enumerate}

\noindent For estimating $\ex[y_{x_1, W_{x_0}}]$ we follow the approach of \cite{farbmacher2020causal}. The authors propose a slightly different estimator than that based on Equation \eqref{eq:mediated}, where they replace $\phi_{x_0, x_1}(X, W, Z)$ by
\begin{align} \label{eq:mediateddml}
	\psi_{x_0, x_1}(X, W, Z) = &\frac{\mathbb{1}(X = x_1) p_{x_0}(Z, W)}{p_{x_1}(Z, W)p_{x_0}(Z)}[Y - \mu(x_1, W, Z)]\nonumber\\
	&+\frac{\mathbb{1}(X = x_0)}{p_{x_0}(Z)}\big[ \mu(x_1, W, Z) - \ex[\mu(x_1, W, Z) \mid X = x_0, Z]\big]\\
	&+ \ex[\mu(x_1, W, Z) \mid X = x_0, Z], \nonumber
\end{align}
which avoids the computation of densities in a possibly high-dimensional case. The terms $\psi_{x_0, x_1}(X, W, Z)$ are estimated in a cross-fitting procedure as described above, with the slight extension that in Step \ref{step:trestimators} we need to further split the complement $\mathcal{D}^C_k$ into two parts, to estimate the conditional mean $\mu(X, W, Z)$ and the nested conditional mean $\ex[\mu(x_1, W, Z) \mid X = x_0, Z]\big]$ on disjoint subsets of the data. This approach is used in the \texttt{faircause} R-package.
\vskip 0.2in
\bibliography{CFA}

\end{document}